\pdfoutput=1
\def\coltversion{0}
\def\standardversion{1}
\def\siamversion{2}
\ifdefined\version
\else
  \def\version{\siamversion}
\fi
\ifdefined\final
\else
  \def\final{1}
\fi

\ifnum\version=\coltversion
  \documentclass[12pt]{colt2016}
\fi
\ifnum\version=\standardversion
  \documentclass[letterpaper,11pt]{article}
\fi
\ifnum\version=\siamversion
  \documentclass[final,hidelinks]{siamart0216}
\fi

\ifnum\version=\standardversion
  \usepackage{amsmath,amsfonts,amsthm,amssymb}
  \usepackage[margin=1in]{geometry}
  \usepackage[numbers]{natbib}
  \usepackage{algorithm,algorithmicx,algpseudocode}
  \newcommand{\opdot}{.}
  \usepackage{ltxtable}
  \newcommand{\newsiamthm}[1]{}
  \newcommand{\coltauthor}[1]{}
\fi
\ifnum\version=\coltversion
  \usepackage{amssymb, amsfonts}
  \newcommand{\qedhere}{}   % Does nothing.
  \usepackage{times}
  \newcommand{\opdot}{.}
  \newcommand{\newsiamthm}[1]{}
\fi
\ifnum\version=\siamversion
  \usepackage{amsmath,amsfonts,amssymb}
  \usepackage{algorithm,algorithmicx,algpseudocode}
  \newcommand{\coltauthor}[1]{}
  \def\citep{\cite}
  \newcommand{\opdot}{}
  \usepackage{ltxtable}
\fi
\usepackage{booktabs}
\usepackage{enumitem}
\usepackage{bbm}
\usepackage{etoolbox}
\usepackage{needspace}
\usepackage{color}
\usepackage{graphicx}
\usepackage{tabto}
\usepackage{versions}
\usepackage{microtype}
\usepackage{fixfoot}

\ifnum\version=\siamversion
\usepackage{hyperref}
\else
\usepackage[bookmarks=false]{hyperref}
\usepackage{cleveref}

\fi
\usepackage{ifthen}

\ifnum\version=\coltversion
  \usepackage{algorithm}
  \usepackage{algorithmic}
  \newcommand{\State}{\STATE}
  \newcommand{\Function}[2]{\State \textbf{function} \textsc{#1}(#2)}
  \renewcommand{\For}{\FOR}
  \newcommand{\EndFor}{\ENDFOR}
  \newcommand{\EndFunction}{}
\fi

\ifnum\version=\siamversion
\else
  \newcommand{\email}[1]{\href{mailto:#1}{#1}}
\fi

\newcommand{\ifrac}[2]{{#1}/{#2}}
\newcommand{\mycitet}[3][]{\processifversion{usenatbib}{%
    \ifthenelse{\equal{#1}{}}{\citet{#3}}{\citet[#1]{#3}}}%
  \processifversion{nonatbib}{%
    \ifthenelse{\equal{#1}{}}{#2\@ \cite{#3}}{#2\@ \cite[#1]{#3}}}%
}

\ifnum\version=\siamversion
  \newsiamthm{thm}{Theorem}
  \newsiamthm{prop}{Proposition}
  \newsiamthm{cor}{Corollary}
  \newsiamthm{lem}{Lemma}
  \newsiamthm{obs}{Observation}
  \newsiamthm{fact}{Fact}
  \newsiamthm{construction}{Construction}
  \newsiamremark{claim}{Claim}
  \newsiamremark{rmk}{Remark}
  \newsiamremark{assump}{Assumption}

  \newtheorem{case}{Case}

  \theoremsymbol{\ensuremath{\vartriangle}}
  \theoremstyle{nonumberplain}
  \newtheorem{subproof}{Proof of Claim}
  \newcommand{\proofword}{Proof }
  \newcommand{\qedhere}{}
  \newenvironment{defn}{\begin{definition}}{\end{definition}}
  \newcommand{\resetcasecount}{\setcounter{case}{0}}
\else
  \ifnum\version=\coltversion
  \def\theoremstyle{}
  \fi

  \ifnum\version=\standardversion
  \theoremstyle{plain}
  \fi
  \newtheorem{thm}{Theorem}[section]
  \newtheorem*{thm*}{Theorem}
  \newtheorem{prop}[thm]{Proposition}
  \newtheorem{cor}[thm]{Corollary}
  \newtheorem{lem}[thm]{Lemma}
  \newtheorem{obs}[thm]{Observation}
  \newtheorem{fact}[thm]{Fact}
  \newtheorem{claim}{Claim}[thm]
  \newtheorem{rmk}[thm]{Remark}

  \AtBeginEnvironment{thm}{\Needspace{5\baselineskip}}
  \ifnum\version=\standardversion
  \theoremstyle{remark}
  \fi
  \newtheorem{case}{Case}
  \newtheorem{assump}[thm]{Assumption}

  \ifnum\version=\standardversion
  
  \newenvironment{subproof}[1][Proof of claim]{\begin{proof}[#1] }{\end{proof}}
  \newcommand{\proofword}{Proof }
  \fi
  \ifnum\version=\coltversion
  \newenvironment{subproof}%
  {%
    \par\noindent{\bfseries\upshape Proof of Claim\ }%
  }%
  {\hfill$\blacktriangle$\\[2mm]}
  \newcommand{\proofword}{}
  \fi

  \ifnum\version=\standardversion
  \theoremstyle{definition}
  \fi
  \newtheorem{defn}[thm]{Definition}
  \newcommand{\resetcasecount}{\setcounter{case}{0}}
\fi

\numberwithin{theorem}{section}

  \crefname{thm}{Theorem}{Theorems}
  \crefname{prop}{Proposition}{Propositions}
  \crefname{cor}{Corollary}{Corollaries}
  \crefname{lem}{Lemma}{Lemmas}
  \crefname{obs}{Observation}{Observations}
  \crefname{fact}{Fact}{Facts}
  \crefname{construction}{Construction}{Constructions}
  \crefname{claim}{Claim}{Claims}
  \crefname{rmk}{Remark}{Remarks}
  \crefname{assump}{Assumption}{Assumptions}
  \crefname{case}{Case}{Cases}

\newlength{\parsave}
\newlength{\caseindent}
\newenvironment{caseblock}{\setlength{\parsave}{\parindent}%
  \begin{list}{}{\setlength{\leftmargin}{\caseindent}}\item\relax%
    \setlength{\parindent}{\parsave}}{\end{list}}

\newcommand{\fg}[1][]{F}

\newcommand{\ind}[1]{{(#1)}}
\newcommand{\disunion}{\sqcup}

\newcommand{\Pmap}{\ensuremath{P}}
\newcommand{\sphere}{S}
\newcommand{\suchthat}{\mid}

\newcommand{\dm}{m}
\newcommand{\ldim}{\dm}
\newcommand{\dn}{d}

\newcommand{\myZ}{e}
\newcommand{\myZv}{\vec \myZ}
\newcommand{\hbev}{\myZv}
\newcommand{\elpow}[1]{\ensuremath{^{\langle #1 \rangle}}}

\newcommand{\tsim}{\ensuremath{{\sim}}}
\newcommand{\indicator}[1]{\mathbbm 1_{[#1]}}
\newcommand{\orthsymb}{\ensuremath{Q}}
\newcommand{\porth}{\ensuremath{\orthsymb_+}}

\newcommand{\sef}{odeco function}
\newcommand{\Sef}{Odeco function}
\newcommand{\SEF}{Odeco Function}
\newcommand{\ansef}{an {\sef}}
\newcommand{\Ansef}{An {\sef}}
\newcommand{\psef}{positive {\sef}}

\newcommand{\apsef}{a {\psef}}

\newcommand{\range}{\mathcal R}
\newcommand{\Jacob}{\mathrm{D}}
\newcommand{\rcomp}[1]{^{(#1)}}
\newcommand{\pertgF}{\ensuremath{\widehat{\nabla F}}}

\def\tcompress{\ensuremath{\tau_{\raisebox{-1.5pt}{$\scriptstyle{\ref{prop:GI-Loop-Small-Coords-Error}}$}}}}

\newcommand{\fcondexpand}{\ensuremath{\big(\frac{\beta \gamma}{\alpha \delta}\big)}}

\newcommand{\expectation}{\operatorname{\mathbb{E}}}
\newcommand{\e}{\expectation}

\newcommand{\compl}[1]{\bar{#1}}

\newcommand{\Id}{\mathcal{I}}

\def\pinterval{\ensuremath{[0, 1]}}
\def\ninterval{\ensuremath{[-1, 0]}}

\newcommand{\tvec}[1]{\hat{\vec{#1}}}

\newlist{compactitem}{itemize}{3}\setlist[compactitem]{topsep=3pt,partopsep=0pt,itemsep=0pt,parsep=3pt}\setlist[compactitem,1]{label=\textbullet}\setlist[compactitem,2]{label=---}\setlist[compactitem,3]{label=*}
\newlist{compactdesc}{description}{3}\setlist[compactdesc]{topsep=0pt,partopsep=0pt,itemsep=0pt,parsep=0pt}
\newlist{compactenum}{enumerate}{3}\setlist[compactenum]{topsep=0pt,partopsep=0pt,itemsep=0pt,parsep=0pt}
\setlist[compactenum,1]{label=\arabic*.,ref=\arabic*}\setlist[compactenum,2]{label=\alph*., ref=\alph*}\setlist[compactenum,3]{label=\roman*.,ref=\roman*}

\allowdisplaybreaks[3]
\DeclareMathOperator{\vol}{vol}

\ifnum\final=0  %namely if we allow comments in the output
\newcommand{\lnote}[1]{[{\small Luis: \textbf{#1}}]\marginpar{*}}
\newcommand{\mnote}[1]{[{\small Misha: \textbf{#1}}]\marginpar{*}}
\newcommand{\vnote}[1]{[{\small Voss: \textbf{#1}}]\marginpar{*}}
\newcommand{\vnotep}[1]{[xxx {\small Voss: \textbf{#1}} xxx]}
\newcommand{\anonnote}[1]{[{\small anon: \textbf{#1}}]\marginpar{*}}
\newcommand{\sidecomment}[1]{\marginpar{\tiny #1}}
\newcommand{\details}[1]{[[#1]]}
\else % in this case [final=1] we don't want any comments to show
\newcommand{\lnote}[1]{}
\newcommand{\mnote}[1]{}
\newcommand{\vnote}[1]{}
\newcommand{\vnotep}[1]{}
\newcommand{\anonnote}[1]{}
\newcommand{\sidecomment}[1]{}
\newcommand{\details}[1]{}
\fi  %ok, we're done with defining comment macros
\ifnum\version=\coltversion
  \includeversion{vlong}
  \excludeversion{vshort}
  \includeversion{vcolt}
  \excludeversion{vstandard}
  \includeversion{usenatbib}
  \excludeversion{nonatbib}
\fi

\ifnum\version=\standardversion
  \includeversion{vlong}
  \excludeversion{vshort}
  \excludeversion{vcolt}
  \includeversion{vstandard}
  \includeversion{usenatbib}
  \excludeversion{nonatbib}
\fi

\ifnum\version=\siamversion
  \includeversion{vlong}
  \excludeversion{vshort}
  \excludeversion{vcolt}
  \includeversion{vstandard}
  \excludeversion{usenatbib}
  \includeversion{nonatbib}
\fi

\usepackage{amsmath}

\newcommand{\abs}[1]{{\ensuremath | #1 |}}
\newcommand{\Abs}[1]{\ensuremath \left| #1 \right|}
\newcommand{\restr}[1]{\ensuremath \vert_{#1}}

\newcommand{\ip}[2]{\ensuremath{\langle #1,\, #2 \rangle}}

\newcommand{\norm}[2][]{\ensuremath{\lVert #2 \rVert_{#1}}}
\newcommand{\Norm}[2][]{\ensuremath{\left\lVert #2 \right\rVert_{#1}}}
\newcommand{\norms}[1]{{\lVert#1\rVert}^2}

\newcommand{\E}{\mathbb{E}}
\newcommand{\N}{\mathbb{N}}

\newcommand{\R}{\mathbb{R}}

\renewcommand{\Pr}{\mathbb{P}}

\newcommand{\CC}{\mathcal{C}}

\newcommand{\GG}{G}
\newcommand{\RR}{\mathcal{R}}
\newcommand{\HH}{\mathcal{H}}
\newcommand{\LL}{\mathcal{L}}

\newcommand{\NN}{\mathcal{N}}
\renewcommand{\SS}{\mathcal{S}}
\newcommand{\XX}{\mathcal{X}}

\newcommand{\gvec}[1]{\ensuremath{\boldsymbol{#1}}}

\DeclareMathOperator{\argmax}{arg\ max}
\DeclareMathOperator{\argmin}{arg\ min}
\DeclareMathOperator{\cov}{cov}

\DeclareMathOperator{\diag}{diag}
\DeclareMathOperator{\poly}{poly}
\DeclareMathOperator{\sign}{sign}
\DeclareMathOperator{\spn}{span}

\DeclareMathOperator{\var}{var}

\newcommand{\grad}{\nabla}

\renewcommand{\vec}[1]{\ensuremath{\mathbf{#1}}}

\newcommand*{\Cdot}[1][1.25]{%
  \mathpalette{\CdotAux{#1}}\bullet%
}
\newdimen\CdotAxis
\newcommand*{\CdotAux}[3]{%
  {%
    \settoheight\CdotAxis{$#2\vcenter{}$}%
    \sbox0{%
      \raisebox\CdotAxis{%
        \scalebox{#1}{%
          \raisebox{-\CdotAxis}{%
            $\mathsurround=0pt #2#3$%
          }%
        }%
      }%
    }%
    \dp0=0pt %
    \sbox2{$#2\bullet$}%
    \ifdim\ht2<\ht0 %
      \ht0=\ht2 %
    \fi
    \sbox2{$\mathsurround=0pt #2#3$}%
    \hbox to \wd2{\hss\usebox{0}\hss}%
  }%
}

\newcommand{\argdot}{\Cdot[.8]}

\newcommand{\ipCanonical}[2]{\ip{#1}{#2}}
\newcommand{\ipCanonicalp}[2]{\ip{#1}{#2}}

\title{Eigenvectors of Orthogonally Decomposable Functions%
  \thanks{An earlier version of this paper was presented at the 2016 Conference on Learning Theory (COLT) under the title ``Basis Learning as an Algorithmic Primitive''.}}
\ifnum\version=\standardversion
  \author{
    Mikhail Belkin \\
    Ohio State University \\
    \email{mbelkin@cse.ohio-state.edu}
    \and
    Luis Rademacher \\
    University of California, Davis \\
    \email{lrademac@ucdavis.edu}
    \and
    James Voss\thanks{Corresponding author.} \\
    Ohio State University \\
    \email{vossj@cse.ohio-state.edu}
  }
\fi

\ifnum\version=\siamversion
  \author{
    Mikhail Belkin%
    \thanks{Ohio State University, Columbus, OH (\email{mbelkin@cse.ohio-state.edu}).}
    \and
    Luis Rademacher%
    \thanks{University of California, Davis, CA (\email{lrademac@ucdavis.edu}).}
    \and
    James Voss%
    \thanks{Ohio State University, Columbus, OH (\email{vossj@cse.ohio-state.edu})}
  }
\fi

\ifnum\version=\coltversion
  \coltauthor{%
    \Name{Mikhail Belkin}     \Email{mbelkin@cse.ohio-state.edu}
    \\
    \Name{Luis Rademacher}    \Email{lrademac@cse.ohio-state.edu}
    \\
    \Name{James Voss}         \Email{vossj@cse.ohio-state.edu} \\
    \addr The Ohio State University
  }
\fi

\DeclareFixedFootnote{\footfixedpt}{These fixed points are fixed possibly up to a sign flip.
  Alternatively stated, these are fixed points in projective space.}

\begin{document}
\maketitle

\begin{abstract}
The Eigendecomposition of quadratic forms (symmetric matrices) guaranteed by the spectral theorem is a foundational result in applied mathematics.
Motivated by a shared structure found in inferential problems of recent interest---namely orthogonal tensor decompositions, Independent Component Analysis (ICA),  topic models, spectral clustering, and Gaussian mixture learning---we generalize the eigendecomposition from quadratic forms to a broad class of ``orthogonally decomposable" functions.
We identify a key role of convexity in our extension, and we generalize two traditional characterizations of eigenvectors:
First, the eigenvectors of a quadratic form arise from the optima structure of the quadratic form on the sphere.
Second, the eigenvectors are the fixed points of the power iteration.

In our setting, we consider a simple first order generalization of the power method which we call gradient iteration.
It leads to efficient and easily implementable methods for basis recovery.
It includes influential Machine Learning methods such as cumulant-based FastICA and the tensor power iteration for orthogonally decomposable tensors as special cases.

We provide a complete theoretical analysis of gradient iteration using the structure theory of discrete dynamical systems to show almost sure convergence
and fast (super-linear) convergence rates.
The analysis also extends to the case when the observed function is only approximately orthogonally decomposable, with bounds that are polynomial in dimension and other relevant parameters, such as perturbation size.
Our perturbation results can be considered as a non-linear version of the classical Davis-Kahan theorem for perturbations of eigenvectors of symmetric matrices.

\end{abstract}
% Keywords:  generalized eigenvectors, nonconvex optimization, gradient iteration, discrete dynamical systems
%
% AMS subject classifications:
% 68Q32 -- Computational learning theory
% 

\ifnum\version=\standardversion
\newpage
\fi
%\newpage \setcounter{page}{1}
%\begin{vshort}
%  \begin{center}
%   % \textbf{[Extended Abstract]}
%  \end{center}
%\end{vshort}

\section{Introduction}

%A good algorithmic primitive is a procedure which is simple, allows for theoretical analysis, and ideally for efficient implementation.  It should also be applicable to a range of interesting problems.  An example of an extremely successful and widely used primitive, both in theory and practice, is the diagonalization/eigendecomposition of symmetric matrices.

The spectral theorem for symmetric matrices is no doubt among  the most fundamental mathematical results used ubiquitously throughout mathematics and applications.
The spectral theorem states that a symmetric matrix $A$ can be diagonalized in some orthonormal ``eigenvector" basis $\myZv_i$ or, equivalently, that any quadratic form $\langle \vec u, A \vec u\rangle $  can be written as $\langle \vec u, A \vec u\rangle = \sum_i \lambda_i \ip{\vec u}{\myZv_i}^2 $.  Recovering  the basis $\myZv_i$'s  accurately and efficiently is one of the key problems in numerical analysis and a subject of an extensive literature.

More recently it has been realized that a number of problems in data analysis and signal processing can be recast as recovering an orthogonal  basis from a more general non-quadratic function.

In this paper we introduce ``orthogonally decomposable" functions,  a generalization of quadratic forms and orthogonally decomposable tensors, allowing for a basis decomposition similar to that given by the spectral theorem.
We identify a key role of convexity in the extension of traditional characterizations of eigenvectors of quadratic forms to our framework.
Moreover, we will show that  a number of problems and techniques of recent theoretical and practical interest can be viewed within our setting.

Let   $\{\myZv_1, \dotsc, \myZv_\dm\}$ be a full or partial, unknown orthonormal basis in $\R^\dn$.
Choosing  a set of one-dimensional {\it contrast functions\footnote{We call the $g_i$s contrast functions following the Independent Component Analysis (ICA) terminology. Note, however, that in the ICA setting our ``contrast functions" correspond to different scalings of the ICA contrast function.}}  $g_i : \R \rightarrow \R$, we define the orthogonally decomposable (odeco) function $F : \R^\dn \rightarrow \R$ as
\begin{equation}%\label{eq:fg-def}
%  \fg(\vec u) := \sum_{i=1}^\dm \alpha_i [g(\beta_i \ip {\vec u} {\myZv_i}) - g(0)]
  \fg(\vec u) := \sum_{i=1}^\dm g_i(\ip {\vec u} {\myZv_i}) \ .
\end{equation}
Our goal  will be to recover the set $\{\myZv_1, \dotsc, \myZv_\dm\}$ (fully or partially) through access to  $\nabla \fg(\vec u)$ (the exact setting), or to provide a provable approximation to these vectors given an estimate of $\nabla\fg(\vec u)$ (the noisy/perturbation setting).
% We will see that several important problems can be reduced to recovering the basis from an odeco function  (see~\cref{sec:example-befs}).

To see how this basis recovery problem for orthogonally decomposable functions generalizes the eigenvector recovery problem for a symmetric matrix $A$, first consider the quadratic form $F_A(\vec u) := \ip{\vec u}{A \vec u}$.
From the eigendecomposition of $A$, we obtain $F_A(\vec u) = \sum_i \lambda_i \ip{\vec u}{\myZv_i}^2$, and we see that $F_A$ is an odeco function with contrast functions $g_i(t) = \lambda_i t^2$.
In addition, there are two characterizations of matrix eigenvectors which lead to algorithms for eigenvector recovery which we wish to generalize:
%There are also two characterizations of matrix eigenvectors to generalize:
\begin{enumerate}
\item (dynamical system) the eigenvectors are the fixed points\footnote{%
Up  to sign  or in projective space.
} of the map ${\vec u} \mapsto \ifrac{A \vec u}{\norm{A\vec u}}$ and 
\item (maximization) the eigenvector with the largest eigenvalue corresponds to the global maximum %\footnote{Assuming multiplicity one here.}
of the quadratic form on the sphere, the second largest is the maximum in the orthogonal direction to the largest, and so on.
\end{enumerate}
%Interestingly, both of these characterizations lead to algorithms for eigenvector recovery.
Note that the discrete dynamical system point of view leads to the classical power method for matrix eigenvector recovery while the maximization view suggests various optimization procedures.

In what follows we identify conditions which allow these characterizations to be extended to a broad class of general orthogonally decomposable functions.
It turns out that the key is a specific kind of convexity, namely that  the functions $\abs{g_i(\sqrt x)}$ need to be convex.

Taking the dynamical systems point of view, we propose a fixed point method for recovering the hidden basis.
The basic algorithm consists simply of replacing the point with the normalized gradient at each step using the ``gradient iteration'' map ${\vec u} \mapsto \ifrac{\grad \fg(\vec u)}{\norm{\grad \fg(\vec u)}}$.
%
% This gradient iteration map also arises in generalizing the eigenvector problem for symmetric matrices and tensors: When $F$ is a quadratic form or a higher order polynomial arising from a symmetric tensor, the fixed points of this map may be taken as a definition of the matrix/tensor eigenvectors~\citep{qi2005eigenvalues,lim2006singular}.
%
%
%For symmetric matrices, the classical power method for eigenvector recovery may be written as gradient iteration. % (see \cref{sec:example-befs}).
%We extend this idea to our \sef\@ setting.
%%In our generic function-based setting, we show
%In our setting,
We show that when $\abs{g_i(\sqrt x)}$ is {\it strictly} convex, the desired basis directions  %(or alternatively, the maxima of $\abs{\fg(\vec u)}$ on the sphere)
are the only stable fixed points of the gradient iteration, and moreover that
% Moreover, it turns out that
the gradient iteration converges to one of the basis vectors given almost any starting point. % on a set of full measure.
Further, we link this gradient iteration algorithm to optimization of $F$ over the unit sphere by demonstrating that the hidden basis directions (that is, the stable fixed point of the gradient iteration) are also a complete enumeration of the local maxima of $\abs{F(\vec u)}$.
%
% This gradient iteration map also arises in generalizing the eigenvector problem for symmetric matrices and tensors: When $F$ is a quadratic form or a higher order polynomial arising from a symmetric tensor, the fixed points of this map may be taken as a definition of the matrix/tensor eigenvectors~\citep{qi2005eigenvalues,lim2006singular}.
%

In this paper, we analyze the odeco function framework in the setting
where each $\abs{g_i(\sqrt x)}$ is \textit{strictly} convex.
Since in the matrix setting each $g_i(x) = \lambda_i x^2$ satisfies
only that $g_i(\sqrt{x})$ is convex but not \textit{strictly} so, the
matrix setting is precluded from the analysis.\footnote{ While our
  analysis does not capture the matrix case, some of the ideas
  underlying our analysis apply to the matrix setting. In the
  introduction of \cref{sec:guar-conv-hidd}, we briefly sketch how the
  proof techniques in our analysis relate to standard proofs that the
  matrix power iteration converges to the top eigenvector.
} The matrix setting ends up being a limit case to the analysis with
related but slightly different properties.

The proposed gradient iteration as analyzed directly generalizes several influential fixed point methods for performing hidden basis recovery in the machine learning and signal processing contexts including
% the classic power iteration for finding the eigenvectors of symmetric matrices,
cumulant-based FastICA \citep{Hyvarinen1999} and
 the tensor power method \citep{de1995higher,AnandkumarTensorDecomp} for orthogonally decomposable symmetric tensors. % (see \cref{sec:example-befs}).
One of our  main conceptual contributions is to demonstrate that the success of such power iterations need not be viewed as a consequence of a linear or multi-linear algebraic structure, but instead relies on an orthogonal decomposition of the function $F$ combined with a more fundamental \emph{convexity} structure.
Compared to the matrix and tensor cases, the dynamics of the general gradient iteration is significantly more complex. To show convergence, we use general results on stable/unstable manifolds for discrete dynamical systems.
%In particular, our most important assumption
%%\footnote{%
%%  We use mildly stronger assumptions in our write-up which lead to
%%  stronger guarantees found in many of our motivating examples but not
%%  present in matrix eigenvector recovery.
%%  In particular, stronger assumptions are required to demonstrate that
%%  all of the hidden basis directions $\myZv_i$ are attractors of the
%%  gradient iteration and that the rate of convergence to these
%%  attractors is super-linear.
%%}
%is that  the contrast functions $g_i$
%%be smooth, (anti-)symmetric, and
%satisfy that $x \mapsto \abs{g_i(\sqrt x)}$ be convex\footnote{For technical reasons our analysis requires strict convexity while convexity is sufficient for the classical matrix power method. The analysis of matrix power iteration is  a limit case of our setting with slightly different properties.}.

% [{\bf move forward}]

%There are also connections to the eigenvector problem for general symmetric tensors.
%When $F$ is a quadratic form or a higher order polynomial arising from a symmetric tensor, the fixed points of the gradient iteration map may be taken as a definition of the matrix/tensor eigenvectors~\citep{qi2005eigenvalues,lim2006singular}.

Under our assumptions, we demonstrate that the gradient iteration exhibits superlinear convergence as opposed to the linear convergence of the standard power iteration for matrices but in line with some known results for ICA and tensor power methods~\citep{Hyvarinen1999,Nguyen2009,AnandkumarTensorDecomp}.
We provide conditions on the contrast functions $g_i$ to obtain specific higher orders of convergence.
%  and show corresponding recovery time bounds in the perturbed case.

% In the noisy setting, the situation is more complicated but we are able to show that a minor modification of the algorithms by adding a small random step, guarantees recovery of one of the hidden basis vectors.
It turns out that a  similar analysis still holds when we only have access to an approximation of $\nabla F$ (the noisy setting). In order to give polynomial run-time bounds  we analyze gradient iteration with
occasional random jumps\footnote{%
In a related work, \mycitet{Ge et al.}{ge2015escaping} use the standard gradient descent with random jumps to escape from saddle points in the context of online tensor decompositions.
}. The resulting algorithm still provably recovers an approximation to a hidden basis element.
By repeating the algorithm we can  recover the  full basis  $\{\myZv_1, \dotsc, \myZv_\dm\}$.
We provide an analysis of the resulting algorithm's accuracy and running time under a general perturbation model. % in terms of the perturbation size and the computational complexity of the algorithm.
Our bounds involve low degree polynomials in all relevant parameters---e.g., the ambient dimension, the number of basis elements to be recovered, and the perturbation size---and capture the superlinear convergence speeds of the gradient iteration.
Our accuracy bounds can be considered as a non-linear version of the classical perturbation theorem of {Davis and Kahan}~\citep{davis1970rotation} for eigenvectors of symmetric matrices.
Interestingly, to obtain these bounds we only require approximate access to $\nabla F$  and do not need to assume anything about the perturbations of the second derivatives of $F$ or even $F$ itself.
We note that our perturbation results allow for substantially more general perturbations than those used in the matrix and tensor settings, where the perturbation of a matrix/tensor is still a matrix/tensor.
In many realistic settings the perturbed model does not have the same structure as the original.
For example, in computer computations, $A \vec x$ is not  actually a linear function of $\vec x$ due to finite precision of floating point arithmetic.
Our perturbation model for $\nabla F$ still applies in these cases.

To highlight the parallels and differences with  the classical matrix case we provide a brief summary in the table below:
\vspace{1em}
\begin{small}
\begin{center}
\begin{tabular}{ccc}
\toprule
  & Symmetrix matrix $A$ & Odeco function \\
\midrule
  Functional form  & $\fg(\vec u)  =\langle \vec u, A \vec u\rangle = \sum_i \lambda_i \ip{\vec u}{\myZv_i}^2 $ & $\fg(\vec u) = \sum_{i=1}^\dm g_i(\ip {\vec u} {\myZv_i})$\\
\midrule
  Fixed point iteration & ${\vec u} \mapsto \frac{A \vec u}{\norm{\vec u}} $ & ${\vec u} \mapsto \frac{\grad \fg(\vec u)}{\norm{\grad \fg(\vec u)}}$ \\
\midrule
 ``Eigenvalues"  &  Constants $\lambda_i$ & monotone functions, Eq.~\ref{equ:monotone}  \\
\midrule
 Maxima on sphere & Top eigenvector&All ``eigenvectors"\\
\midrule
Attractors of iteration& Top eigenvector & All ``eigenvectors"\\
\midrule
Convergence rate & Linear & Superlinear\\
\midrule
Analysis & Based on homogeneity & Stable/unstable manifolds\\
&&discrete dynamical systems\\
\midrule
Perturbation stability& Linear& Linear\\
\bottomrule
\end{tabular}
\end{center}
\end{small}
\vspace{1em}

Below in \cref{sec:example-befs} we will show how  a number of  problems can be viewed in terms of hidden basis recovery.
Specifically, we briefly discuss how our primitive
can be used to recover clusters in spectral clustering, independent components in Independent Component Analysis (ICA), parameters of Gaussian mixtures and certain tensor decompositions.
Finally, in \cref{sec:robust-ICA} we apply our framework to obtain the first provable ICA recovery algorithm for arbitrary model perturbations.
%, cryptoanalysis and LDA ???.
%Additionally,  as a direct consequence of our results we obtain the first   analysis of ICA for arbitrary perturbations of the product distribution model.

\iffalse
In related work~\citep{AnandkumarTensorDecomp}, a form of orthogonal tensor decomposition was proposed  for solving a variety of problems by generalizing previous works on learning mixtures of spherical Gaussians~\citep{hsu2013learning}, latent Dirichlet allocation~\citep{AnandkumarLDA2012}, and learning hidden Markov models~\citep{anandkumar2012method}. The authors also introduced a tensor power method. We will see (\cref{sec:example-befs}) %and~\cref{sec:interpr-grad-iter})
that it can be considered a special case of our framework  by choosing the contrast functions to be $g_i(x) = \lambda_i x^r$ with $r\geq 3$ an integer. Perhaps counter-intuitively, our results imply that the success of these methods  relies not  as much on their tensorial structure but on a more general ``hidden convexity" inherent in the problem.
In another related work~\citep{ge2015escaping}, the idea of using a random jump to escape from a saddle point was recently used in the context of online tensor decompositions.
\fi

\paragraph{Organization of the paper\opdot}
In \cref{sec:framework} we introduce the problem of basis recovery and sketch the main theoretical results of the paper.
We also show how our framework relates to spectral clustering, ICA, matrix and tensor decompositions and Gaussian Mixture Learning.
In \cref{sec:extrema-structure} we analyze the structure of the extrema of {\sef}s.
In \cref{sec:gi-stability-structure} we show that the fixed points of gradient iteration are in one-to-one correspondence with the {\sef}'s maxima and analyze convergence of gradient iteration in the exact case.
In \cref{sec:interpr-grad-iter} we give an interpretation of the gradient iteration algorithm as a form of adaptive gradient ascent.
In \cref{sec:gi-error-analysis} we describe a robust version of our algorithm and give a complete theoretical analysis for arbitrary perturbations.
\begin{vstandard}
Then, in \cref{sec:robust-ICA}, we show how to apply our framework to obtain a perturbation analysis of ICA under arbitrary model perturbations.
\end{vstandard}
\begin{vcolt}
In \cref{sec:robust-ICA} \begin{vshort}of the full version of this extended abstract \end{vshort}we apply our framework to obtain the first ICA algorithm which is provably robust for arbitrary model perturbations.
\end{vcolt}

%\Cref{sec:gi-stability-structure,sec:gi-error-analysis} contain many technical proof details.
%For this reason, these sections start with an introduction which sketch the main ideas and results of the section.
%The technical proof details follow as subsections which flesh out these sketches.

\section{Problem description and the main results}
\label{sec:framework}

% This paper investigates a form of hidden basis recovery as a new algorithmic primitive for learning a variety of latent variable models.
% In this section, we define our optimization framework, state our main results, and briefly describe several example problems that fit within the proposed framework.

% \subsection{Basis recovery: the problem statement}

We consider a function optimization framework for hidden basis recovery.
More formally, let $\{\myZv_1, \dotsc, \myZv_\dm\}$ be a non-empty set of orthogonal unit vectors in $\R^\dn$.
These unit vectors form the unseen basis.
% We may consider these points to be the vertices of the simplex $\Delta^{\dm - 1} := \conv(\myZv_1, \dotsc, \myZv_\dm)$.
A function on a closed unit ball $F : \overline{B(0, 1)} %[-1, 1]^\dn
\rightarrow \R$
%\lnote{I think you mean unit sphere or ball, not $[-1, 1]^\dn$, otherwise could take $z_1 = (1,...,1)/\sqrt{n}$ and $u=(1,...,1)$ and evaluate $g_1$ at $\sqrt{n}$}
is defined from ``contrast functions'' $g_i : [-1, 1] \rightarrow \R$ as:
\begin{equation}\label{eq:fg-def}
%  \fg(\vec u) := \sum_{i=1}^\dm \alpha_i [g(\beta_i \ip {\vec u} {\myZv_i}) - g(0)]
  \fg(\vec u) := \sum_{i=1}^\dm g_i(\ip {\vec u} {\myZv_i}) \ .
\end{equation}
We call $\fg$ an \textit{orthogonally decomposable function (\sef)}\label{defn:sef} with the associated tuples $\{(g_i, \myZv_i) \suchthat i \in [\dm]\}$.
The goal is to recover the hidden basis vectors $\myZv_i$ for $i \in [\dm]$ up to sign given evaluation access to $\fg$ and its gradient.
We will assume that $\dn \geq 2$ since otherwise the problem is trivial.
We consider contrast functions $g_i \in \CC^{(2)}([-1, 1])$
%\lnote{$\R$ above but $[-1,1]$ here?}
which satisfy the following assumptions:
%\vnote{Does $[-1, 1]$ work?  Or would this make it so that the derivatives at the end points are not actually defined from both directions and somehow mess things up?}
%\begin{enumerate}[label=A\arabic{enumi}., ref=A\arabic{enumi}, topsep=3pt,partopsep=0pt,itemsep=0pt,parsep=3pt]
% \item The constants $\alpha_1, \dotsc, \alpha_\dm$ and $\beta_1, \dotsc, \beta_\dm$ are non-zero.
% \item The function $g$ is twice continuously differentiable.  %$g \in \CC^{(2)}(\R)$)
\begin{assump}
  \label{assumpt:first}\label{assumpt:gsymmetries} $g_i$ is either an even or
  odd function.
\end{assump}
\begin{assump}
  \label{assumpt:convex}
  % For each $i$,
  Strict convexity of $\abs{g_i(\sqrt x)}$: Either $\frac {d^2} {dx^2}
  g_i(\sqrt x) > 0$ on $(0, 1]$ or $- \frac {d^2} {dx^2} g_i(\sqrt x) > 0$ on
  $(0,
  1]$. %\vnote{Can this be done on the domain $[0, 1]$ instead?  Or are there some boundary issues?}
\end{assump}
\begin{assump}
  \label{assumpt:deriv0}\label{assumpt:last-non-essential}
  % For each $i\in[\dm]$, $\frac d {dx} g_i( \sqrt{ x }) |_{x = 0} = 0$.
  The right derivative at the origin $\frac d {d x} g_i( \sqrt{ x }) |_{x =
    0^+}=0$.
\end{assump}
\begin{assump}
  \label{assumpt:origin-val} % For each $i \in [\dm]$,
  $g_i(0) = 0$.
  \label{assumpt:last}
\end{assump}
% \end{enumerate}
\Cref{assumpt:convex} is slightly stronger than stating that one of $\pm g_i(\sqrt x)$ is strictly convex on $(0, 1]$.
From now on  $F$ and the term {\sef} will refer to {\ansef} with associated $\myZv_i$s and $g_i$s satisfying \crefrange{assumpt:first}{assumpt:last} unless otherwise stated.
%It will be seen that basis recovery is possible.\\
% We define $\sigma_{i-} = \sigma_{i+}$ if $g_i$ is an even function and $\sigma_{i-} = -\sigma_{i+}$ if $g_i$ is an odd function, and we note that $x \mapsto \sigma_{i-}g_i(-\sqrt{\abs x})$ is a strictly convex function on $\ninterval$.
% when Assumptions~\ref{assumpt:first}--\ref{assumpt:last} are satisfied.\\
%\noindent{\bf Remark:}
\begin{rmk}
  \Cref{assumpt:origin-val} is non-essential.
  If each $g_i$
  satisfies \crefrange{assumpt:first}{assumpt:last-non-essential}, then $x
  \mapsto [g_i(x) - g_i(0)]$ satisfies \crefrange{assumpt:first}{assumpt:last}
  making $[\fg(\vec u) - \fg(\vec 0)] = \sum_{i=1}^\dm [g_i(\ip{\vec
    u}{\myZv_i}) - g_i(0)]$ {\ansef} of the desired form.
\end{rmk}

We shall see in section \cref{sec:example-befs} that {\sef}s arise naturally in a number of problems of interest within machine learning.
However, we will first summarize our main results showing that given {\ansef}, the directions $\myZv_1, \dotsc, \myZv_{\dm}$ can be efficiently recovered up to sign.
% using simple optimization techniques.

\subsection{Summary of the main results}
\label{sec:results-summary}
% For the remainder of this paper, $F$ will be assumed to be {\ansef} unless otherwise stated.
% Similarly, any constants and functions introduced as being associated with the {\sef} $F$ (e.g., the $g_i$s, and $\myZv_i$s) are assumed to be defined by their association with $F$ unless otherwise stated.
% Where multiple {\sef}s are being considered or where {\ansef} has additional restrictions than \crefrange{assumpt:first}{assumpt:last}, this will be made clear by context.

In what follows it will be convenient to append arbitrary orthonormal directions $\myZv_{\dm+1}, \dotsc, \myZv_{\dn}$ to  our hidden ``basis"
to obtain a full basis.
%$\myZv_1, \dotsc, \myZv_\dm$ such that $\myZv_1, \dotsc, \myZv_\dn$ are orthonormal.
For the remainder of this paper, we simplify our notation by indexing vectors in $\R^{\dn}$ with respect to this hidden basis $\myZv_1, \dotsc, \myZv_\dn$.
That allows us to introduce the notation $u_i:=\ip {\vec u}{\myZv_i}$ for $\vec u \in \R^{\dn}$. Thus, $\fg(\vec u) = \sum_{i=1}^\dm g_i(u_i)$.

We now state the first result indicating that {\ansef} encodes   the basis $\myZv_1, \dotsc, \myZv_\dm$.
We use $S^{\dn - 1} := \{\vec u \suchthat \norm{\vec u} = 1\}$ to denote the unit sphere in $\R^{\dn}$.

% We note that \cref{assumpt:origin-val} is included strictly for notational convenience.
% In particular, given access to a function $\fg$ with associated ``contrast'' $g$ such that $x \mapsto (g(x) - g(0))$ satisfies the other assumptions, then the function $\fg(\vec u) - \fg(\vec 0) = \sum_{i = 1}^\dm \alpha_i[g(\beta_i \ip{\vec u}{\myZv_i}) - g(\vec 0)]$ is a simplex encoding function.

\begin{thm}\label{thm:gen-simplex-optima}
  The set $\{\pm \myZv_i \suchthat i \in [\dm]\}$ is a complete enumeration of the local %\lnote{local?} -- added
  maxima of $\abs \fg $ with respect to the domain $S^{\dn - 1}$.
\end{thm}
\processifversion{vshort}{\vskip-7pt}

\Cref{thm:gen-simplex-optima} implies that a form of gradient ascent can be used to recover maxima of $\abs \fg$ and hence the hidden basis\footnote{We note that \cref{assumpt:gsymmetries} is stronger than what is actually required in \cref{thm:gen-simplex-optima}.
  In particular, we could replace \cref{assumpt:gsymmetries} with the assumption that $x \mapsto g_i(-\sqrt{\abs x})$ is either strictly convex or strictly concave on $\ninterval$ for each $i \in [\dm]$.}. % elements.
However, the performance of gradient ascent is dependent on the choice of a learning rate parameter.
We propose  a simple and practical parameter-free fixed point  method, {\it gradient iteration},  for finding the hidden basis elements $\myZv_i$ in this setting.

% While gradient ascent is a very important practical algorithm, its speed of convergence to a local maximum depends upon a chosen learning rate.
% Choosing a good learning rate for fast convergence is an unsolved problem in general and often requires the use of heuristics.
% In order to overcome these issues, we propose a fast converging, fixed point algorithm based on a ``gradient iteration'' for recovering the {\sef}'s hidden basis $\myZv_1, \dotsc, \myZv_\dm$.

The proposed method is  based on  the \textit{gradient iteration function} $\GG: S^{\dn - 1} \rightarrow S^{\dn - 1}$ defined by
\begin{displaymath}%\label{eq:grad-it-def}
\GG(\vec u) :=
%\begin{cases}
\frac{\grad \fg(\vec u)}{\norm{\grad \fg(\vec u)}}  %\text{if %  $\norm{\grad \fg(\vec u)} \neq 0$} \\
%  $\grad \fg(\vec u) \neq \vec 0$} \\
%  ~~~~~\vec u & \text{otherwise.}
%\end{cases}
\end{displaymath}
with the convention that $\GG(\vec u) = \vec u$ if  $\grad \fg(\vec u) = \vec 0$.
We use the map $\GG$ as a fixed point iteration for recovering the hidden basis elements\footnote{A special case of this iteration was introduced in the context of ICA~\citep{voss2013fast}.}.

However, there is a difficulty:
at any given step, the derivative $\partial_i F(\vec u)$ can be of a different sign than $u_i$ causing $\sign(u_i) \ne \sign(\GG_i(\vec u))$. Note that we do not know which coordinates flip their signs as the coordinates are hidden. As it turns out, this does not affect the algorithm, but
the analysis is more transparent in a space of equivalence classes%
%
%
%Given a sequence $\{\vec u \ind n\}_{n=0}^\infty$ defined recursively by $\vec u(n) = \GG(\vec u(n-1))$, it may happen that for oscillating sign values $s \ind i \in \{-1, +1\}$ the sequence $s\ind n \vec u \ind n \rightarrow\myZv_i$ as $n \rightarrow \infty$.
%Since we do not distinguish between recovery of $\myZv_i$ and $-\myZv_i$, the sequence $\{\vec u(n)\}_{n=0}^\infty$ should be viewed as recovering $\myZv_i$ even though it is  oscillating.
\footnote{Alternative approaches to fixing the sign issue include analyzing the fixed points of the double iteration $\vec u \to \GG(\GG(\vec u))$ or working in projective space.}. We divide $S^{\dn - 1}$ into equivalence classes using the equivalence relation $\vec v \sim \vec u$ if $\abs{v_i} = \abs{u_i}$ for each $i \in [\dn]$.
Given $\vec v \in S^{\dn-1}$, we denote by $[\vec v]$ its corresponding equivalence class.
The resulting quotient space $S^{\dn - 1}/{\sim}$ may  be identified with the positive orthant of the sphere  $\porth^{\dn - 1} := \{\vec u \in S^{\dn - 1} \suchthat u_i \geq 0 \text{ for all } i \in [\dn]\}$.
There is a bijection $\phi: S^{\dn - 1}/{\sim} \rightarrow \porth^{\dn - 1}$ given by $\phi([\vec u]) = \sum_{i=1}^\dn \abs{u_i} \myZv_i$. We treat $S^{\dn - 1} /\tsim$ as a metric space with the metric $\mu([\vec u], [\vec v]) = \norm{\phi([\vec v]) - \phi([\vec u])}$.
Under \cref{assumpt:gsymmetries}, if $\vec u \sim \vec v$ then $\GG(\vec u) \sim \GG(\vec v)$.
As such, sequences are consistently defined modulo this equivalence class, and we consider the fixed points of $\GG/{\sim}$.

We will use the following terminology.
A class $[\vec v]$ is a \textit{fixed point} of $\GG/{\sim}$ if $\GG(\vec v) \sim \vec v$.
We will consider sequences of the form $\{\vec u(n)\}_{n=0}^\infty$ defined recursively by $\vec u(n) = \GG(\vec u(n-1))$.
%A class $[\vec v]$ is \textit{Lyapunov stable} if for every neighborhood $N$ of $[\vec v]$ there exists a neighborhood $N' \subset N$ of $[\vec v]$ such that if $[\vec u(0)] \in N'$ then $[\vec u(n)] \in N$ for every $n\in\N$.
%% A class $[\vec v]$ is \textit{unstable} if it is not {Lyapunov stable}.
%Finally, a Lyapunov stable class $[\vec v]$ is an \textit{attractor} of $\GG/{\sim}$ if there exists a neighborhood $N$ of $[\vec v]$ such that for any $[\vec u(0)] \in N$, the sequence $[\vec u(n)] \rightarrow [\vec v]$ as $n\rightarrow \infty$.
%\vnote{Possibly remove equivalence class notation after this point? The following paragraph shows how this could be done, though we would also have to make it apply to series in order to completely get rid of the notation here.}
In addition, by abuse of notation, we will sometimes refer to a vector $\vec v \in S^{\dn-1}$ as a fixed point %, stable point, unstable point, or attractor
of $\GG/{\sim}$.
%This is an abuse of terminology which should be understood to mean that $[\vec v]$ is a fixed point %, stable point, unstable point, or attractor
%of $\GG / {\sim}$.

We demonstrate that the attractors of $\GG/{\sim}$ are precisely the hidden basis elements, and that all other fixed points of $\GG/{\sim}$ are non-attractive (unstable hyperbolic). % given almost any starting $\vec u(0)$.
Further, convergence to a hidden basis element is guaranteed given almost any starting point $\vec u(0) \in S^{\dn - 1}$.
%, and the rate of convergence to these fixed points is fast (super-linear).
%More formally, we have the following main results.
%The following results demonstrate that the attractors of $\GG/{\sim}$ are precisely the hidden basis elements, and that convergence to these fixed points is fast (super-linear).

\begin{thm}[Gradient iteration stability]\label{thm:gi-stability}
  % Let $\fg$ be {\ansef}, and let $\GG$ be the associated gradient iteration function.
  % The dynamical system $\GG/{\sim}$ has the following fixed point properties:
  %  The classes $\{[\myZv_i] \suchthat i \in [\dm]\}$ are attractors of the dynamical system $\GG/{\sim}$.
  %  Further, the following hold:
  The hidden basis elements $\{[\myZv_i] \suchthat i \in [\dm]\}$ are attractors of the dynamical system $\GG/{\sim}$.
  Further,
  there is a full measure set $\mathcal X \subset S^{\dn - 1}$ such that for all $\vec u(0)  \in \mathcal X$, $[\vec u(n)] \rightarrow [\myZv_i]$ for some $\myZv_i$ as $n \rightarrow \infty$.
\end{thm}
\processifversion{vshort}{\vskip -7pt}

One implication of \cref{thm:gi-stability} is that given a $\vec u(0) \in S^{\dn-1}$ drawn uniformly at random, then with probability 1, $\vec u(n)$ converges (up to $\sim$) to one of the hidden basis elements.

From \cref{thm:gi-stability,thm:gen-simplex-optima} we see how
odeco function basis recovery closely resembles the problem of
recovering the top eigenvector of a symmetric matrix.
A symmetric matrix $A$ may be represented by the quadratic form
$f(\vec u) = \vec u^T A \vec u$.
From $f$, the top eigenvector of $A$ may be characterized in two ways:
(1) as the attractive fixed point of the map $ \vec u \mapsto \nabla
f(\vec u) / \norm{\nabla f(\vec u)}$ in projective space or (2) as the
maximum of $f$ restricted to the unit sphere.
From \cref{thm:gen-simplex-optima,thm:gi-stability}, we see that for
the odeco function $\fg$ each basis element $\vec e_1, \dotsc, \vec
e_\dm$ satisfies both characterizations of being a top eigenvector.
See \cref{sec:example-befs} for more discussion on ways in which odeco function basis recovery and the gradient iteration relate to the symmetric matrix eigenvector problem and some tensorial generalizations.

When recovering a hidden basis element via repeated application of the gradient iteration function $\GG$, the rate of convergence is superlinear. to the hidden basis elements is fast (superlinear).
For reference, we include definitions of convergence rates in 
%\Cref{thm:gi-convergence} in turn guarantees that the resulting rate of convergence is superlinear.

\begin{thm}[Gradient iteration convergence rate]\label{thm:gi-convergence}
  % Let $\fg$ be {\ansef}, and let $\GG$ be the associated gradient iteration function.
  % Let $\{\vec u\ind n\}_{n = 0}^\infty$ be a sequence defined recursively from a starting $\vec u \ind 0 \in S^{\dn - 1}$ and $\vec u\ind{n} = \GG(\vec u \ind{n-1})$.
  %The following hold:
  %\begin{compactenum}
  If $[\vec u(n)] \rightarrow [\myZv_i]$ as $n \rightarrow \infty$, then the convergence is superlinear.
  Specifically, if  $x \mapsto g_i(x^{1 / r})$ is convex on $\pinterval$ for some  $r > 2$ , then the rate of convergence is at least of order $r-1$.
  % If $\vec u(0) \perp \myZv_i$, then $[\vec u(n)] \not \rightarrow [\myZv_i]$ as $n \rightarrow \infty$.
\end{thm}
\processifversion{vshort}{\vskip -7pt}
The above Theorems suggest the following practical algorithm for recovering the hidden basis elements:
\begin{algorithm}[H]\caption{\label{alg:gi-generic}The gradient iteration algorithm.}
  \begin{enumerate}
    \item \label{alg:gi-general:init}
    Choose an initial $\vec u \in \sphere^{\dn - 1}$
    \item \label{alg:gi-general:iterate}
    Repeat the iteration $\vec u \leftarrow \GG(\vec u)$ until convergence is
    achieved to recover a single hidden basis direction.
    \item Repeat steps~\ref{alg:gi-general:init}
    and~\ref{alg:gi-general:iterate} with the starting $\vec u$ chosen in the
    orthogonal complement to previously found $\vec u$ in order to recover
    additional hidden basis directions.
  \end{enumerate}
\end{algorithm}
\noindent In practice, one may threshold $\min(\norm{\GG(\vec u) - \vec u},\ \norm{-\GG(\vec u) - \vec u})$ to determine if convergence is achieved.

From a practical standpoint, the fast and guaranteed convergence properties of the gradient iteration make it an attractive algorithm for hidden basis recovery.
We also demonstrate that the gradient iteration is robust to a perturbation. Specifically, we modify the gradient iteration algorithm by occasionally performing a small random jump of size $\sigma$ on the sphere. We call this algorithm
\textsc{RobustGI-Recovery} and show that it approximately recovers all hidden basis elements.
More precisely, we consider the following notion of a perturbation of $\nabla F$:  If for every $\vec u \in \overline{B(0, 1)}$, $\norm{\nabla F(\vec u) - \pertgF(\vec u)} \leq \epsilon$, then we say that $\pertgF$ is an $\epsilon$-approximation of $F$.
Further, if $F$ satisfies a strong version of \cref{assumpt:convex}, namely that there exists positive constants $\alpha \geq \beta$ and $\gamma \leq \delta$ such that for each $i \in [\dm]$, $\beta x^{\delta - 1} \leq \abs{\frac {d^2}{d x^2} g_i(\sqrt x)} \leq \alpha x^{\gamma - 1}$ for all $x \in (0, 1]$, then our perturbation result can be summarized as follows.

\begin{thm}[simplified]
  %Treating $\gamma$ and $\delta$ as constants, there exists an algorithm \textsc{RobustGI-Recovery} which consists of scalar operations, vector operations, and computations of $\pertgF(\vec u)$ on the unit sphere with the following property.
  Treating $\gamma$ and $\delta$ as constants, if $\sigma \leq \poly^{-1}(\frac \alpha \beta, \dn, \dm)$ and if $\epsilon \leq \sigma\beta \poly^{-1}(\frac \alpha \beta, \dm, \dn)$, then with probability $1 - p$, \textsc{RobustGI-Recovery} takes
  \begin{displaymath}
  \poly(\frac 1 \sigma, \frac \alpha \beta, \dm, \dn)\log(\frac 1 p) + \poly(\dn, \dm)\log_{1+2\gamma}(\log(\frac{\beta}{\epsilon}))
  \end{displaymath}
  time to recover $O(\epsilon/\beta)$ approximations of each $\myZv_i$ up to a sign.
  Specifically, \textsc{RobustGI-Recovery} returns vectors $\gvec \mu_1, \dotsc, \gvec \mu_\dm$ such that there exists
  %$s_1, \dotsc, s_\dm \in \{ \pm 1 \}$
  a permutation $\pi$ of $[\dm]$ such that $\norm{\pm \gvec \mu_i - \myZv_{\pi(i)}} \leq O(\epsilon / \beta)$ for all $i \in [\dm]$.
\end{thm}
Several observation are now in order:
\begin{compactenum}
  \item We note that we only need a zero-order error bound for $\nabla
  F(\vec u)$ for the perturbation analysis and do not need to assume
  anything about the perturbations of the second derivatives of $F$ or
  even $F$ itself.
  This perhaps surprising fact is due to the convexity conditions.
  \item Our perturbation results allow for substantially more general
  perturbations than those used in the matrix and tensor settings,
  where the perturbation of a tensor is still a tensor. In our setting
  the perturbation of {\ansef} corresponding to a tensor does not have to
  be tensorial in structure. This situation is very common whenever an
  observation of an object is not exact.
  For example, $A \vec x$ is not a linear function of $\vec x$ on a finite precision machine. The same phenomenon occurs in the tensor case.
  \item $\log_{1 + 2\gamma}(\log(\frac{\beta}{\epsilon}))$ above
  corresponds to the superlinear convergence from
  \cref{thm:gi-convergence} in the unperturbed setting.
\end{compactenum}
%\begin{vlong}
The full algorithm and analysis for \textsc{RobustGI-Recovery}, complete with more precise bounds, can be found in \cref{sec:gi-error-analysis}\begin{vshort} of the long version of this extended abstract\end{vshort}.

Finally, in \cref{sec:robust-ICA}\begin{vshort} of the long version of this extended abstract\end{vshort}, we show how to apply \textsc{RobustGI-Recovery} to cumulant-based ICA under an arbitrary perturbation from the ICA model.
%\end{vlong}
%\begin{vshort}
%The full algorithm and analysis for \textsc{RobustGI-Recovery}, complete with more precise bounds, can be found in the long version of this extended abstract.
%
%Finally, in the long version of this extended abstract, we show how to apply \textsc{RobustGI-Recovery} to cumulant-based ICA under an arbitrary perturbation from the ICA model.
%\end{vshort}
In this setting, \textsc{RobustGI-Recovery} provides an algorithm for robustly recovering the approximate ICA model.

\subsection{Motivations for and Examples of \SEF\@ Recovery}
\label{sec:example-befs}
Before proving our main results on \sef\@ recovery, we first motivate why the \sef\@ recovery problem is of interest through a series of examples.
We first show how \sef\@ recovery and the gradient iteration relate to ideas from the eigenvector analysis of matrices and tensors.
Then, we discuss several settings where the problem of \sef\@ recovery arises naturally in machine learning.

\paragraph{Connections to matrix eigenvector recovery\opdot}
Our algorithm can be viewed as a generalization of the classical  power iteration method for eigendecomposition of symmetric matrices.
Let $A$ be a symmetric matrix. Put $\fg(\vec u) = \vec u^T A \vec u$. From the spectral theorem for matrices, we have
% what about SVD???
$\fg(\vec u) = \sum_i \lambda_i \ip{\vec u}{\myZv_i}^2$ where each $\lambda_i$ is an eigenvalue of $A$ with corresponding eigenvector $\myZv_i$. We see that  $\fg(\vec u)$ is {\ansef}\footnote{Note that \cref{assumpt:convex} is not satisfied as in this case $g_i(\sqrt x)$ is convex but not strictly convex.} with the contrast functions $g_i(x) := \lambda_i x^2$.
%, with $\lambda_i$ being the eigenvalues of $A$.
It is easy to see that our gradient iteration is an equivalent update to the power method update $\vec u \mapsto \ifrac{A \vec u}{\norm{A \vec u}}$.
As such, the fixed points\footfixedpt\@ of the gradient iteration are eigenvectors of the matrix $A$.
We also note that it is not necessary to know each $g_i(x)$ to have access to the {\sef} $\fg(\vec u)$ or its derivative $\nabla \fg(\vec u)$.

In addition, we note that the gradient iteration for \ansef\@ may be written to look very much like the power iteration for matrices.
Let $F(\vec u) = \sum_{i=1}^\dm g_i(\ip{\vec u}{\myZv_i})$ denote \ansef\@.
In order to better capture the convexity \cref{assumpt:convex}, we may define functions $h_i(t) := g_i(\sign(t) \sqrt{\abs t})$.
To compress notation, we use $\pm$ to denote the $\sign(\ip{\vec u}{\myZv_i})$.
Then, $F(\vec u) = \sum_{i=1}^\dm h_i(\pm \ip{\vec u}{\myZv_i}^2)$.
Taking derivatives, we obtain that
\begin{equation}
\label{equ:monotone}
\nabla F(\vec u) = 2\sum_{i=1}^\dm \pm h_i'(\pm \ip{\vec u}{\myZv_i}^2)\ip{\vec u}{\myZv_i}\myZv_i \ .
\end{equation}
Note that in the matrix example above, the power iteration can be expanded as
\begin{displaymath}
  A \vec u = \sum_i \lambda_i \ip{\vec u}{\myZv_i}\myZv_i \ .
\end{displaymath}
%When renormalizing to the unit sphere,
We see that the formula for $\nabla F(\vec u)$ is the same as the power iteration for matrices with the (constant) eigenvalues $\lambda_i$ being replaced by the functional term $\pm h_i'(\pm \ip{\vec u}{\myZv_i}^2)$.
By \cref{assumpt:convex},  $|h_i(t)|$ is strictly convex, and in particular each $|h_i'(t)|$ is strictly increasing as a function of $|t|$.
The gradient iteration for general \sef\@s may be thought of as a power iteration where matrix eigenvalues are being replaced by functions whose magnitude grows with the magnitude of their respective coordinate values  $\ip{\vec u}{\myZv_i}$. % in their coordinate value.
The change in these ``eigenvalues'' by location allows each of the basis directions $\myZv_1, \dotsc, \myZv_\dm$ to become an attractor locally since there is no single fixed ``top eigenvalue" as in the matrix setting.

\paragraph{Connections to the tensor eigenvector problem\opdot}
While in general not a special case of the \sef\@ framework, there are also connections between the gradient iteration algorithm and the definition of an eigenvector of a symmetric tensor~\citep{qi2005eigenvalues,lim2006singular}.
In particular, given a symmetric tensor $T \in \R^{\dn \times \cdots \times \dn}$ (with $r$ copies of $\dn$), we may treat $T$ as an operator on $\R^\dn$ using the operation
$T \vec u^r := \sum_{i_1, \dotsc, i_r}T_{i_1\dotsc i_r}u_{i_1}\cdots u_{i_r}$.
We note that this formula encapsulates the matrix quadratic form $\vec u^T A \vec u =A \vec u^2$ as a special case.
We also denote by $T \vec u^{r-1}$ the vector such that $[T \vec u^{r-1}]_j = \sum_{i_2,\dotsc,i_r}T_{j i_2\dotsc i_r}u_{i_2}\cdots u_{i_r}$.
If we define the function $f(\vec u) = T \vec u^r$, then
the Z-eigenvectors of $T$ are defined to be vectors $\vec u$ for which there exists $\lambda \in \R$ such that $\nabla f(\vec u) =  r \lambda \vec u$.
% , or alternatively as fixed points\useFN\sft\@ of the map $\vec u \mapsto \frac{\nabla [T\vec u^r]}{\norm{T \vec u^r}}$.
Expanding this formula, we get the slightly more familiar looking form that the Z-eigenvectors of $T$ are the points such that $T \vec u^{r-1} = \lambda \vec u$, or alternatively the fixed points\footfixedpt\@ of the iteration $\vec u \mapsto \frac{T\vec u^{r-1}}{\norm{T \vec u^{r-1}}}$.
Note that this iteration may alternatively be written as $\vec u \mapsto \frac{\nabla f(\vec u)}{\norm{\nabla f(\vec u)}}$.
Replacing the function $f(\vec u) =T \vec u^r$  with \ansef\@ $F$, the fixed points\footfixedpt\@ of the gradient iteration $\vec u \mapsto \frac{\nabla F(\vec u)}{\norm{\nabla F(\vec u)}}$ are like eigenvectors for our function $F$ in this dynamical systems sense.

\paragraph{Orthogonal tensor decompositions\opdot}
In a recent work~\citep{AnandkumarTensorDecomp}, it was shown that the tensor eigenvector recovery problem for tensors with orthogonal decompositions\footnote{%
  Another related work \citep{anandkumar2015learning} investigates properties of the tensor power method in certain settings where the symmetric tensor is not orthogonal decomposable and has symmetric rank exceeding $\dn$.
} can be applied to a variety of problems including ICA and previous works on learning mixtures of spherical Gaussians~\citep{hsu2013learning}, latent Dirichlet allocation~\citep{AnandkumarLDA2012}, and learning hidden Markov models~\citep{anandkumar2012method}.
%In addition, they show how to use their orthogonal tensor decomposition algorithm to perform ICA.

Their framework involves using the moments of the various models to obtain a tensor of the form $T = \sum_{k=1}^\dm w_k \gvec \mu_k^{\otimes r}$ where (1) each $w_k \in \R \setminus \{0\}$, (2) each $\gvec \mu_k \in \R^{\dn}$ is a unit vector, and (3) $\gvec \mu_k^{\otimes r}$ is the tensor power defined by $(\gvec \mu_k^{\otimes r})_{i_1\dotsc i_r} = (\gvec \mu_k)_{i_1}\dotsc(\gvec \mu_k)_{i_r}$.
The $\gvec \mu_k$s may be assumed to have unit norm by rescaling the $w_k$s appropriately.
In the special case where the $\gvec \mu_k$s are orthogonal, then the direction of each $\gvec \mu_k$ can be recovered using tensor power methods~\citep{AnandkumarTensorDecomp}.
% Treating $T$ as an operator using the definition %$(T\vec u^k)_{i_1\dotsc i_{r-k}} := \sum_{i_1,\dotsc,i_k \in [\dn]^k}T_{i_1\dotsc i_r}u_{i_1}\cdots u_{i_k}$, and in particular with
% $T \vec u^r := \sum_{i_1,\dotsc,i_k \in [\dn]^k}T_{i_1\dotsc i_r}u_{i_1}\cdots u_{i_r}$,  it can be seen that $T \vec u^r = \sum_{k=1}^\dm w_k \ip{\vec u}{\gvec \mu_k}^r$.
It can be shown that $T \vec u^r = \sum_{k=1}^\dm w_k \ip{\vec u}{\gvec \mu_k}^r$.
In particular, the function $F(\vec u) = T \vec u^r$ is {\ansef} with the contrasts $g_i(x) := w_i x^r$ and hidden basis elements $\myZv_k := \gvec \mu_k$.
% In \cref{sec:interpr-grad-iter}, we will show that the tensor power method is a special case of our gradient iteration.
Further, the fixed point iteration $\vec u \mapsto \frac{T \vec u^{r-1}}{\norm{T \vec u^{r-1}}}$ proposed by \mycitet{Anandkumar et al.}{AnandkumarTensorDecomp} for eigenvector recovery in this setting can be equivalently written as the gradient iteration update $\vec u \mapsto \frac{\nabla F(\vec u)}{\norm{\nabla F(\vec u)}}$.

\paragraph{Spectral clustering\opdot}

Spectral clustering is a class of methods for multiway cluster analysis. We describe now a prototypical version of the method that works in two phases \citep{DBLP:journals/jmlr/BachJ06,ng2002spectral,ShiMal00,yu2003multiclass}. The first phase, spectral embedding, constructs a similarity graph based on the features of the data and then embeds the data in $\R^\dn$ (where $\dn$ is the number of clusters) using the bottom $\dn$ eigenvectors of the Laplacian matrix of the similarity graph. The second phase clusters the embedded data using a variation of the $k$-means algorithm. A key aspect in the justification of spectral clustering is the following observation: If the graph has $\dn$ connected components, then a pair of data points is either mapped to the same vector if they are in the same connected component or mapped to orthogonal vectors if they are in different connected components \citep{weber2004perron}. If the graph is close to this ideal case, which can be interpreted as a realistic graph with $\dn$ clusters, then the embedding is close to that ideal embedding.

This suggests the following alternate approach \citep{BelkinRV14} to the second phase of spectral clustering by interpreting it as a hidden basis recovery problem: Let $\vec x_1, \dotsc, \vec x_n \in \R^\dn$ be the embedded points. Let $g : \R \to \R$ be a function satisfying \crefrange{assumpt:first}{assumpt:last}. Let
\begin{equation}\label{equ:spectralbef}
F(\vec u) = \sum_{i=1}^n g(\ip{\vec u}{\vec x_i}).
\end{equation}
In the ideal case, there exists an orthonormal basis $\vec Z_1, \dotsc, \vec Z_\dn$ of $\R^\dn$ and positive scalars $b_1, \dotsc, b_\dn$ such that $\vec x_i = b_j \vec Z_j$ for every $i$ in the $j$\textsuperscript{th} connected component of the graph.
% In the ideal case we have $\{\vec x_1, \dotsc, \vec x_n \} = \{b_1 \vec Z_1, \dotsc, b_\dn \vec Z_\dn \}$, where $\{\vec Z_j\}_{j=1}^\dn$ is an orthonormal basis and $\{b_j \}_{j=1}^\dn$  are positive scalars.
Thus, in the ideal case we can write
\begin{displaymath}
  F(\vec u) = \sum_{j=1}^\dn a_j g(b_j \ip{\vec u}{\vec Z_j})
\end{displaymath}
where $a_j$ is the number of points from the $j$\textsuperscript{th} connected component. Thus, $F$ is \ansef\@ in the ideal case with contrasts $g_j(t) := a_j g(b_j t)$.
In the general case, it is a perturbed {\sef} and the hidden basis can be approximately recovered using our %robust algorithm (\cref{sec:gi-error-analysis,sec:spectral-pert-analysis}).
robust algorithm (\cref{sec:gi-error-analysis}).
Note that via \cref{equ:spectralbef}, $F$ and its derivatives can be evaluated at any $\vec u$ just with knowledge of the $\vec x_i$s, and without knowing the hidden basis.

We note that for this spectral clustering application, the choice of $g$ is arbitrary so long as it satisfies \crefrange{assumpt:first}{assumpt:last}.
In particular, this is an example where the generality of the gradient iteration beyond the tensorial setting provides greater flexibility.

\paragraph{Independent component analysis (ICA)\opdot}

In the ICA model, one observes samples of the random vector $\vec X = A \vec S$ where $A\in \R^{\dn \times \dn}$ is a mixing matrix and $\vec S = (S_1, \dotsc, S_\dn)$ is a latent random vector such that the $S_i$s are mutually independent and non-Gaussian.
The goal is to recover the mixing matrix $A=[A_1 | \cdots | A_\dn]$, typically with the goal of using $A^{-1}$ to invert the mixing process and recover the original signals.
This recovery is possible up to natural indeterminacies, namely the ordering of the columns of $A$ and the choice of the sign of each $A_i$~\citep{comon1994independent}.
ICA has a vast literature (see the books \citep{Comon2010,Hyvarinen2001} for a broad overview) with numerous applications including speech separation~\citep{makino2007blind}, denoising of EEG/MEG brain recordings~\citep{vigario2000independent}, and various vision tasks~\citep{bartlett2002face, Bell1997} to name a few.
% Further, ICA has been amenable to theoretical analysis using cumulant statistics (see e.g.,~\cite{AroraGMS12,Belkin2012,Nguyen2009}) and the more general second characteristic function (see e.g.,~\cite{DBLP:journals/corr/GoyalVX13}).

% For a (univariate) random variable $Y$, the cumulants are defined as the Taylor coefficients of a certain cumulant generating function.
% Letting $\kappa_r(X)$ denote the order-$r$ cumulant, then $\log(\E[\exp(tX)]) = \sum_{r=1}^\infty \frac {t^r} {r!} \kappa_r(X)$.
To demonstrate that ICA fits within our {\sef} framework, we rely on the properties of the cumulant statistics.\begin{vlong}\footnote{%
An important class of ICA methods with guaranteed convergence to the columns of $A$ are based on the optimization of $\kappa_4(\ip{\vec u}{\vec X})$ over the unit sphere (see e.g., \citep{AroraGMS12,delfosse1995adaptive,Hyvarinen1999}).
Other contrast functions are also frequently used in the practical implementations of ICA (see e.g., \citep{DBLP:journals/sigpro/HyvarinenO98}).
However, these non-cumulant functions can have spurious maxima \citep{wei2015study}.}
\end{vlong}
Let $\kappa_r(X)$
% $\kappa_r(X) = \frac {d^r} {dt^r} \log(\E[\exp(tX)]) \big|_{t=0}$.
denote the $r$\textsuperscript{th} cumulant of a random variable $X$.
The cumulant $\kappa_r(X)$ satisfies the following:  (1) Homogeneity: $\kappa_r(\alpha X) = \alpha^r\kappa_r(X)$ for any $\alpha \in \R$ and (2) Additivity: if $X$ and $Y$ are independent, then $\kappa_r(X + Y) = \kappa_r(X) + \kappa_r(Y)$.
Given an ICA model $\vec X = A\vec S$, these properties imply that for all $\vec u \in \R^\dn$, $\kappa_r(\ip{\vec u}{\vec X}) = \kappa_r(\sum_{i=1}^\dn\ip{\vec u}{A_i} S_i) = \sum_{i=1}^\dn\ip{\vec u}{A_i}^r\kappa_r(S_i)$.
A preprocessing step called whitening (i.e., linearly transforming the observed data to have identity covariance) makes the columns of $A$ into orthogonal unit vectors.
Under whitening, the columns of $A$ form a hidden basis of the space.
In particular, defining the contrast functions $g_i(x) := x^r \kappa_r(S_i)$ and the basis encoding elements $\myZv_i := A_i$, then the function $F(\vec u) := \kappa_r(\ip{\vec u}{\vec X}) = \sum_{i=1}^\dn g_i(\ip{\vec u}{\myZv_i})$ is {\ansef} so long as each $\kappa_r(S_i) \neq 0$.
Further, these directional cumulants and their derivatives have natural sample estimates (see e.g., \citep{kenney1962mathematics,voss2013fast} for the third and fourth order estimates), and as such this choice of $F$ will be admissible to our algorithmic framework for basis recovery.

% \vnote{To say?:  In this paper, their justification of correctness will be generalized within the {\sef} framework.}
% Such methods may be considered as special cases of algorithms defined within the {\sef} framework.

Interestingly, it has been noted in several places~\citep{Hyvarinen1999,Nguyen2009,zarzoso2010robust} that cubic convergence rates can be achieved using optimization techniques for recovering the directions $A_i$, particularly when performing ICA using the fourth cumulant or the closely related fourth moment.
One explanation as to why this is possible arises from the dual interpretation (discussed in \cref{sec:interpr-grad-iter}) of the gradient iteration algorithm as both an optimization technique and as a power method.
In the ICA setting, the gradient iteration algorithm for cumulants was introduced by~\mycitet{Voss et al.}{voss2013fast}.
This paper provides a significant generalization of those ideas as well as a theoretical analysis.

%generalizes the gradient iteration algorithm to other settings and provides stronger theoretical guarantees for the gradient iteration algorithm.

\paragraph{Parameter estimation in a spherical Gaussian Mixture Model\opdot}

A Gaussian Mixture Model (GMM) is a parametric family of probability distributions.
A spherical GMM is a distribution whose density can be written in the form $f(\vec x) = \sum_{i=1}^k w_i f_i(\vec x)$, where $w_i \geq 0$, $\sum_i w_i = 1$ and $f_i$ is a $d$-dimensional Normal density with mean $\gvec \mu_i$ and covariance matrix $\sigma_i^2 \Id$, for $\sigma_i > 0$.
The parameter estimation problem is to estimate $w_i, \gvec \mu_i, \sigma_i$ given i.i.d. samples of random vector $\vec x$ with density $f$. For clarity of exposition, we only discuss the case $k=d$ and $\sigma_i = \sigma$ for some fixed, unknown $\sigma$.
Our argument is a variation of the moment method of \mycitet{Hsu and Kakade}{hsu2013learning}. As in their work,
similar ideas should work for the case $k<d$ and non-identical $\sigma_i$s.

We explain how to recover the different parameters from observable moments. Firstly, $\sigma^2$ is the smallest eigenvalue of the covariance matrix of $\vec x$. This recovers $\sigma$.
Let $\vec v$ be any unit norm eigenvector corresponding to the eigenvalue $\sigma^2$.
Define
$%\begin{align*}
%M_1 &= \e(x (v^T(x-\e(x))^2) \in \R^d,\\
M_2 := \e(\vec x \vec x^T) - \sigma^2 \Id \in \R^{d \times d}
%M_3 &= \e(x \otimes x \otimes x) - \sigma^2 \sum_{i=1}^d (\e(x) \otimes e_i \otimes e_i + e_i \otimes \e(x) \otimes e_i + e_i \otimes e_i \otimes \e(x) ) \in \R^{d\times d\times d}.
%.\end{align*}
$.
Then we have
%\cite[Theorem 1]{hsu2013learning}
$M_2 = \sum_{i=1}^d w_i \gvec \mu_i \gvec \mu_i^T$.
% and $M_3 = \sum_{i=1}^d w_i \mu_i \otimes \mu_i \otimes \mu_i$.
Denote $D = \diag(w_1, \dotsc, w_d)$, $A = (\gvec\mu_1, \dotsc, \gvec\mu_d) \in \R^{d \times d}$. With this notation we have $M_2 = A D A^T$. Let $M = M_2^{1/2}$ (symmetric). This implies
%\begin{equation}\label{equ:directions}
$M = A D^{1/2} R$,
%\end{equation}
where $R$ is some orthogonal matrix.

We have $\e(\ip{\vec x}{\vec u}^3) = \sum_{i=1}^d w_i \ip{\gvec \mu_i}{\vec u}^3 + 3 \sigma^2 \norms{\vec u} \e(\ip{\vec x}{\vec u})$.\details{The mixture can be written as $x=D+G$ where $D$ is discrete and $G$ is Gaussian. For origin symmetric $Y$, we have $\e(X+Y)^3 = \e(X^3) + 3\e(X) \e(Y^2)$. The claim follows.}
Then,
%\begin{align*}
%F(u) &= M_3(M^{-1}u, M^{-1}u, M^{-1}u) \\
%&= \sum_{i=1}^d w_i (u^T R^T D^{-1/2} e_i)^3 \\
%&= \sum_{i=1}^d w_i^{-1/2} (u^T R_{i \cdot})^3
%\end{align*}
\begin{align*}
F(\vec u) &:= \e(\ip{\vec x}{M^{-1}\vec u}^3) - 3 \sigma^2 \norms{M^{-1} \vec u} \e(\ip{\vec x}{M^{-1} \vec u})
= \sum_{i=1}^d w_i \ip{\gvec\mu_i}{M^{-1} \vec u}^3 \\
&= \sum_{i=1}^d w_i (\vec u^T R^T D^{-1/2} \vec e_i)^3
= \sum_{i=1}^d w_i^{-1/2} \ip{\vec u}{\vec R_{i \cdot}}^3% = \sum_{i=1}^\dm g_i(\ip {\vec u} {\myZv_i})
\end{align*}
is {\ansef} encoding the rows of $R$, with basis vectors $\vec z_i = \vec R_{i\cdot}$ and contrasts $g_i (t) = w_i^{-1/2} t^3$.
The recovery of the rows of $R$ allows the recovery of the directions of the columns of $A$, that is, the directions of $\gvec \mu_i$s. The actual $\gvec \mu_i$s then can be recovered from the identity $\ip{\gvec \mu_i}{\vec v} = \ip{\e(\vec x)}{\vec v}$.
\vnote{The previous sentence uses that $\vec v$ is a 0 eigenvector of $\cov(X)$.}
Finally, denoting $\vec w = (w_1, \dotsc, w_d)$ we have $\e(\vec x) = A\vec w$ and we recover $\vec w =  A^{-1} \e(\vec x)$.

\section{Extrema structure of Basis Encoding Functions}
\label{sec:extrema-structure}

In this section, we investigate the maximum structure of $|F|$ on the unit sphere and prove \cref{thm:gen-simplex-optima}.
% At its most basic level, the proposed {\sef} primitive is of interest because the problem of learning a hidden basis $\myZv_1, \dotsc, \myZv_\dm$ in our space simplifies into a simple optimization problem, namely recovering the local maxima of $\abs \fg$.
% This is due to \cref{thm:gen-simplex-optima} which we prove in this section.

% Throughout this section, we will assume unless otherwise stated that $F$ is {\ansef} with associated tuples $\{(\sigma_{i+}, \sigma_{i-}, g_i, \myZv_i) \suchthat i \in [\dm]\}$.
The optima structure of $\fg$ relies on the hidden convexity implied by \cref{assumpt:convex}.
To capture  this structure, we define $h_i : [-1, 1] \rightarrow \R$ as $h_i(x) := g_i(\sign(x) \sqrt{\abs x})$ for $i \in [\dm]$ and $h_i := 0$ for $i \in [\dn] \setminus [\dm]$.
Thus,
\begin{equation}
  \label{eq:fg-via-his}
  \fg(\vec u) % = \sum_{i=1}^\dn h_i(\sign(u_i) u_i^2) 
    = \sum_{i=1}^\dm h_i(\sign(u_i) u_i^2) \ .
\end{equation}
These $h_i$ functions capture the convexity from \cref{assumpt:convex}.
Indeed, the functions $h_i$ have the following properties:

\begin{lem} \label{lem:h-props}
The following hold for all $i \in [\dm]$:
  \begin{compactenum}
  \item \label{lem:h-props:hi-convex} The magnitude function $\abs{h_i(t)}$ is strictly convex. 
  \item \label{lem:h-props:hi-deriv0} $h_i'(0) = 0$.
  \item \label{lem:h-props:hi-continuous} $h_i$ is continuously differentiable.
  \item \label{lem:h-props:hi-increase} 
    The derivative's magnitude function $\abs{h_i'(t)}$ is strictly increasing as a function of $\abs t$.
    In particular, $\abs{h_i'(t)} > 0$ for all $t \neq 0$.
  \item\label{lem:h-props:sign} Fix $I$ to be one of the intervals $(0, 1]$ or $[-1, 0)$.
    If $h_i$ is strictly convex on $I$, then $\sign(t)h_i'(t) > 0$ for all $t \in I$, and otherwise $\sign(t)h_i'(t) < 0$ for all $t \in I$.
  \end{compactenum}
\end{lem}
\begin{proof}
  We first show parts~\ref{lem:h-props:hi-deriv0} and \ref{lem:h-props:hi-continuous}.
  We compute the derivative of $h_i$ to see
  \begin{displaymath}
    h_i'(x) =
    \begin{cases}
      \frac 1 2 g_i'(\sign(x)\sqrt {\abs x} )/\sqrt {\abs x} & \text{if } x \neq 0 \\
      0 & \text{if } x = 0
    \end{cases}    
  \end{displaymath}
  where the derivative at the origin is due to \cref{assumpt:gsymmetries,assumpt:deriv0}.
  Since the derivative $h_i'(t)$ exists for all $t$, and since one of $\pm h_i$ is convex on either of the intervals $[0, 1]$ and $[-1, 0]$, it follows that $h_i'$ is continuous \citep[Corollary 4.2.3]{hiriart1996convex}.

  To see part~\ref{lem:h-props:hi-increase}, we note that $h_i'(0) = 0$ and apply \cref{assumpt:convex} to see that $h_i'$ is strictly monotonic on $[0, 1]$.
  As such, $\abs{h_i'(t)}$ is strictly increasing on $[0, 1]$.
  The symmetries of \cref{assumpt:gsymmetries} imply that $\abs{h_i'(t)}$ is strictly increasing more generally as a function of $\abs{t}$.

  To see part~\ref{lem:h-props:sign}, we note that $h_i'(t) = h_i'(0) + \int_0^t h_i''(x) \, dx = \int_0^t h_i''(x) \, dx$.
  Then, we use \cref{assumpt:convex} to obtain the stated correspondence between $\sign(h_i''(x))$ (which is +1 on $I$ if $h_i$ is convex and $-1$ otherwise) and $\sign(h_i'(t))$.

  To see that $\abs{h_i}$ is strictly convex, it suffices to use that $\abs{h_i}$ is continuously differentiable and to show that $\frac d {dt}\abs{h_i(t)}$ is strictly increasing.
  Note that $\frac d {dt}\abs{h_i(t)} = \sign(h_i(t))h_i'(t)$, and also that $\sign(h_i(t)) = \sign(\int_0^t h_i'(t)) = \sign(t h_i'(t))$ by part~\ref{lem:h-props:sign}.
  It follows that $\sign(\frac d {d t} \abs{h_i(t)}) = \sign(t)$.
  Taking this sign into account, part~\ref{lem:h-props:hi-increase} implies part~\ref{lem:h-props:hi-convex}.
\end{proof}

In order to avoid dealing with unnecessary sign values, we restrict ourselves to analyzing the optima structure of $\abs{\fg}$ over the domain $\porth^{\dn - 1}$ (the all positive orthant of the sphere).
Due to the symmetries of the of the problem (\cref{assumpt:gsymmetries}), it is actually sufficient to analyze the maxima structure of $\abs F$ on $\porth^{\dn - 1}$ in order to fully characterize the maxima of $\abs F$ on the entire sphere $S^{\dn - 1}$.

\begin{vlong}
To characterize the extrema structure of the restriction of $\abs F$ to $\porth^{\dn - 1}$, we will use its derivative structure expanded in terms of the $h_i$ functions.
It will be useful to establish some relationships between the $g_i$ and $h_i$ functions.
We denote by $\indicator \bullet$ the indicator function, and we use the convention that any summand containing a $\indicator{\textsc{False}}$ coefficient is 0 even if the term is indeterminant (e.g., $\indicator{\textsc{False}}/0 = 0$ and $\infty\cdot\indicator{\textsc{False}} = 0$).

\begin{lem}\label{lem:h-g-relations}
  The following hold for each $i \in [\dm]$:
  \begin{compactenum}
  \item For $x \in [0, 1]$, $g_i'(x) = 2 h_i'(x^2) x$ and $h_i'(x^2) = \frac{g'(x)}{2 x}\indicator{x \neq 0}$.
  \item For $x \in [0, 1]$, $g_i''(x) = \indicator{x\neq 0}[4 h_i''(x^2)x^2 + 2 h_i'(x^2)]$
  \item Fox $x \in (0, 1]$, $h_i''(x^2) = \frac 1 4 [g_i''(x)/x^2 - g_i'(x)/x^3]$.
  \end{compactenum}
\end{lem}
\begin{proof}
  By construction, $h_i(x^2) = g_i(x)$.
  Taking derivatives, we obtain $2 h_i'(x^2)x = g_i'(x)$.
  Since $h_i'(0) = 0$ by \cref{assumpt:deriv0}, $h_i'(x^2) = \frac{ g_i'(x) }{2 x}\indicator{x \neq 0}$.

  Taking a second derivative away from $x = 0$, we see that $g_i''(x) = 4 h_i''(x^2)x^2 + 2h_i'(x^2)$.
  At $x = 0$, the right derivative is given by:
  \begin{align*}\label{eq:g-second-deriv}
    \partial_+ g_i'(0) &= \lim_{c \rightarrow 0^+} \frac{g_i'(c) - g_i'(0)}{c}
    = 2 \lim_{c \rightarrow 0^+} \frac 1 2 \frac{g_i'(\sqrt c)}{\sqrt c} \\
    &= 2 \lim_{c \rightarrow 0^+} \Big(\frac d {dx} g_i(\sqrt x)\Big)\Big|_{x = c} 
    = 2 \Big(\frac d {dx} g_i(\sqrt x)\Big)\Big|_{x = 0}
    = 0 \ .
  \end{align*}
  In the above, the second equality uses that $g_i'(0) = 0$, a fact which is implied by \cref{assumpt:deriv0} (in particular, $\abs{g_i'(0)} \leq \abs{\lim_{h \rightarrow 0^+} \frac{g(\sqrt h) - g(0)}{\sqrt h}} \leq \abs{\lim_{h \rightarrow 0^+} \frac{g(\sqrt h) - g(0)}{h}} = 0$ since $h \leq \sqrt h$ in a neighborhood of the origin).
  The fourth equality uses that $\frac d {dx} g_i(\sqrt x)$ is continuous due to the convexity of $g_i(\sqrt x)$ \citep[Corollary 4.2.3]{hiriart1996convex}.
  The final equality uses \cref{assumpt:deriv0}.
  
  The argument in \eqref{eq:g-second-deriv} also holds if we replace
  $g_i(x)$ with $g_i(-x)$ since $g_i(-x)$ also satisfies
  \cref{assumpt:gsymmetries,assumpt:deriv0,assumpt:convex,assumpt:origin-val}.
  Hence, $\partial_- g_i'(0) = 0$ and $g_i''(0) = 0$.
  We thus obtain the formula on $[0, 1]$ of
  \begin{displaymath}
    g_i''(x) = \indicator{x\neq 0}[4h_i''(x^2)x^2 + 2h_i'(x^2)]
  \end{displaymath}
  as desired.

  When $x \neq 0$, we may rearrange terms to obtain:
  \begin{displaymath}
    h_i''(x^2) 
    = \frac{g_i''(x) - 2h_i'(x^2)}{4x^2}
    = \frac{g_i''(x)}{4x^2} - \frac{g_i'(x)}{4x^3} \ . \qedhere
  \end{displaymath}
\end{proof}

As $F(\vec u) = \sum_{i=1}^\dm g_i(u_i)$ has first and second order derivatives of $\grad F(\vec u) = \sum_{i=1}^\dm g_i'(u_i)\myZv_i$ and $\HH F(\vec u) = \sum_{i = 1}^\dm g_i''(u_i) \myZv_i \myZv_i^T$, we obtain the following derivative formulas for $F(\vec u)$ in terms of the $h_i$ functions for any $\vec u \in \porth^{\dn - 1}$:
\end{vlong}
\begin{vshort}
  We denote by $\indicator \bullet$ the indicator function of a Boolean variable with the convention that any summand containing a $\indicator{\textsc{False}}$ coefficient is 0 even if the term is undefined (e.g., $\indicator{\textsc{False}}/0 = 0$ and $\infty\cdot\indicator{\textsc{False}} = 0$).
  We obtain the following derivative formulas for $F(\vec u)$ in terms of the $h_i$ functions for any $\vec u \in \porth^{\dn - 1}$:
\end{vshort}
\begin{vlong}
\begin{align}
  \label{eq:F-derivs}
  \grad F(\vec u) &= 2\sum_{i=1}^\dm h_i'(u_i^2)u_i \myZv_i &
  \HH F(\vec u) &= \sum_{i=1}^\dm\indicator{u_i\neq 0}[4 h_i''(u_i^2) + 2 h_i'(u_i^2)]\myZv_i \myZv_i^T
\end{align}
\end{vlong}
\begin{vshort}
  $\grad F(\vec u) = 2\sum_{i=1}^\dm h_i'(u_i^2)u_i \myZv_i$ and 
  $\HH F(\vec u) = \sum_{i=1}^\dm\indicator{u_i\neq 0}[4 h_i''(u_i^2) + 2 h_i'(u_i^2)]\myZv_i \myZv_i^T$.
\end{vshort}

The first derivative necessary condition for $\vec u \in S^{\dn - 1}$ to be an extrema of $F$ over $\porth^{\dn - 1}$ can be obtained using the Lagrangian function $\LL : \overline{B(0, 1)}  \times \R$ defined as $\LL(\vec u, \lambda) := \fg(\vec u) - \lambda [\norm{\vec u}^2 - 1]$.
In particular, a point $\vec u \in \porth^{\dn - 1}$ is a critical point of $\fg$ with respect to $\porth^{\dn - 1}$ \processifversion{vlong}{(that is, it satisfies the first order necessary conditions to be a local maximum of $F$ with respect to $\porth^{\dn - 1}$)} if and only if there exists $\lambda \in \R$ such that $(\vec u, \lambda)$ is a critical point of $\LL$.
The following result then enumerates the critical points of $\fg$ with respect to the $\porth^{\dn - 1}$.
\begin{lem}\label{lem:fg-crit-points}
  Let $\vec u \in \porth^{\dn - 1}$ and $\lambda \in \R$.
  The pair $(\vec u, \lambda)$ is a critical point of $\LL$ if and only if $\lambda \indicator{u_i \neq 0} = h_i'(u_i^2)$ for all $i \in [\dn]$.
\end{lem}
\vnote{The remainder of this section can be cut and summarized in the short version, if need be.}
\begin{proof}
  We set the derivative
  \begin{vlong}
  \begin{equation}\label{eq:Lagrangian-deriv}
  \frac{\partial}{\partial u_i} \LL(\vec u, \lambda) = \partial_i F(\vec u) - 2 \lambda u_i = 2h_i'(u_i^2) u_i - 2\lambda u_i  
  \end{equation}
  \end{vlong}
  \begin{vshort}
    $\frac{\partial}{\partial u_i} \LL(\vec u, \lambda) = \partial_i F(\vec u) - 2 \lambda u_i = 2h_i'(u_i^2) u_i - 2\lambda u_i$
  \end{vshort}
   equal to 0 to obtain $h_i'(u_i^2) u_i = \lambda u_i$.
   If $u_i = 0$, then $h_i'(u_i^2) = h_i'(0) = 0$ by \cref{assumpt:deriv0}.
   Otherwise, $h_i'(u_i^2) = \lambda$.
\end{proof}

While there are exponentially many (with respect to $\dm$) critical points of $F$ as a function on the sphere, it turns out that only the hidden basis directions correspond to maxima of $F$ on the sphere. The proof of the following statements 
uses the convexity structure from \cref{lem:h-props}.

%All other critical points are either saddle points or minima.
%This is formalized in the following Propositions.

\begin{prop}\label{prop:absfg-maxima-canonical}
  If $j \in [\dm]$, then $\myZv_j$ is a strict local maximum of $\abs \fg$ with respect to $\porth^{\dn - 1}$.
\end{prop}
\begin{vlong}
\begin{proof}
  We will prove the case where $h_j$ is strictly convex on $[0, 1]$ and note that the case $h_j$ is strictly concave is exactly the same when replacing $\fg$ with $-\fg$.

  We first note that $\fg(\myZv_j) = h_j(1) > 0$ since $h_j'$ is strictly increasing (see \cref{lem:h-props}).
  In particular, using continuity of each $g_i$, it follows that $\fg(\vec u) > 0$ on a neighborhood of $\myZv_j$, and it suffices to demonstrate that $\fg$ takes on a maximum with respect to % $\porth^{\dn - 1}$ 
  $S^{\dn-1}$ at $\myZv_j$.
  Letting $D_{\vec u}$ denote the derivative operator with respect to the variable $\vec u$ and continuing from \cref{eq:Lagrangian-deriv}, we obtain
  \begin{equation}
    \label{eq:Lagrangian-hessian}
    D_{\vec u}^2 \LL(\vec u, \lambda)
    = \HH F(\vec u) - 2\lambda D_{\vec u} \vec u
    = \sum_{i=1}^\dm \indicator{u_i \neq 0}[4h_i''(u_i^2) u_i^2 + 2h_i'(u_i^2)]\myZv_i\myZv_i^T - 2 \lambda I \ .
  \end{equation}

  We now use the Lagrangian criteria for constrained extrema (see e.g., \citep[chapter 11]{luenberger2008linear} for a discussion of the first order necessary and second order sufficient conditions for constrained extrema) to show that $\myZv_j$ is a maximum of $\fg \restr{\porth^{\dn - 1}}$.
  From \cref{lem:fg-crit-points}, we see that $(\myZv_j, h_j'(1))$ is a critical point of $\LL$.
  Further, for any non-zero $\vec v$ such that $\vec v \perp \myZv_j$, we obtain
  $\vec v^T (D_{\vec u}^2\LL)(\vec \myZv_j, h_j'(1))\vec v = -2h_j'(1) \norm{\vec v}^2$.
  As $h_j'(1) > 0$, it follows that $\vec v^T (D_{\vec u}^2\LL)(\vec \myZv_j, h_j'(1))\vec v < 0$.
  Thus, $\myZv_j$ is a local maximum of $\fg$.
\end{proof}
\end{vlong}

\begin{prop}\label{prop:absfg-maxima-no-extra}
  If $\vec v \in \porth^{\dn - 1}$ is not contained in the set $\{\myZv_i \suchthat i \in [\dm]\}$, then $\vec v$ is not a local maximum of $\abs \fg$ with respect to $\porth^{\dn - 1}$.
\end{prop}
\begin{vshort}
\cref{thm:gen-simplex-optima} follows from \cref{prop:absfg-maxima-canonical,prop:absfg-maxima-no-extra}.
\end{vshort}
\begin{vlong}
\begin{proof}
  We first consider the case in which $v_i = 0$ for all but at most one $i \in [\dm]$.
  We will call this $i \in [\dm]$ for which $v_i \neq 0$ as $j$ if it exists and otherwise let $j \in [\dm]$ be arbitrary.
  Fix any $\vec w \in \porth^{\dn - 1}$ such that $w_j > v_j$ and $w_i = 0$ for $i \in [\dm] \setminus \{j\}$.
  Such a choice is possible since $\vec v \neq \myZv_j$ implies $v_j < 1$.
  Then, $\abs{\fg(\vec v)} = \abs{h_j(v_j^2)}$ and $\abs{\fg(\vec w)} = \abs{h_j(w_j^2)}$.
  Since $\abs{h_j(t)}$ is a strictly increasing function on $[0, 1]$ from $\abs{h_j(0)} = 0$ (see \cref{lem:h-props}), it follows that $\abs{\fg(\vec w)} > \abs{\fg(\vec v)}$.
  Since $\vec w$ can be constructed in any open neighborhood of $\vec v$, $\vec v$ is not a local maximum of $\abs \fg$ on $\porth^{\dn - 1}$.

  Now, we consider the case where $\vec v$ is an extremum (either a maximum or a minimum) of $\abs \fg$ with respect to $\porth^{\dn - 1}$ such that there exists $j, k \in [\dm]$ distinct such that $v_j > 0$ and $v_k > 0$.
  We will demonstrate that this implies that $\vec v$ is a minimum of $\abs \fg$.
  % Construct the vector $\vec w = \vec e_j - \vec e_k$.  Then, $\vec w \in \R^\dn \cap \vec v^\perp$, and $\vec w^T D^2_{\vec u}\LL(\vec v, \lambda^*)\vec w = 4[h_j''(\psi_j(\vec v))v_j^2 + h_k''(\psi_k(\vec v))v_k^2]$.
  % Note that for any $\delta > 0$, strict convexity of $h$ implies that 

  We use the notation for a vector $\vec u$, $\vec u \elpow k := \sum_i u_i^k \vec e_i$ is the coordinate-wise power.
  Fix $\eta > 0$ sufficiently small that for all $\delta \in (-\eta, \eta)$ we have that $\vec w(\delta) := (\vec v\elpow 2 + \delta \myZv_j - \delta \myZv_k)\elpow{1/2} \in \porth^{\dn - 1}$.
  We now consider the difference $\fg(\vec w(\delta)) - \fg(\vec v)$ for a non-zero choice of $\delta \in (-\eta, \eta)$:
  \begin{align*}
    \fg(\vec w(\delta))-\fg(\vec v) 
    &= h_j(w_j(\delta)^2) - h_j(v_j^2) + h_k(w_k(\delta)^2) - h_k(v_k^2) \\
    &= h_j'(x_j(\delta)^2)[w_j(\delta)^2 - v_j^2] + h_k'(x_k(\delta)^2)[w_k(\delta)^2 - v_k^2] \\
    &= \delta[h'_j(x_j(\delta)^2) - h'_k(x_k(\delta)^2)] \ ,
  \end{align*}
  where $x_i(\delta) \in (v_j, w_j(\delta))$ and $x_i(\delta) \in (w_k(\delta), v_k)$ under the mean value theorem.

  As $\vec v$ must be an extremum of $\fg$ in order to be an extremum of $\abs \fg$, there exists $\lambda$ such that the pair $(\vec v, \lambda)$ is a critical point of $\LL$.
  Let $\SS = \{ i \suchthat v_i \neq 0\}$.
  \Cref{lem:fg-crit-points} implies that $\lambda = h_i'(v_i^2)$ for all $i \in \SS$.
  In particular, $\sign(h_i'(v_i^2))$ is the same for each $i \in \SS$, and we will call this sign value $s$.
  Under \cref{eq:fg-via-his}, we have $\fg(\vec v) = \sum_{i \in \SS} h_i(v_i^2)$.
  By \cref{lem:h-props}, $sh_i$ is strictly increasing from $sh_i(0)=0$ on $\pinterval$ for each $i \in \SS$.
  As such, $\fg(\vec v)$ is separated from 0 and $\sign(\fg(\vec v)) = s$.
  Further,
  \begin{displaymath}
    s[\fg(\vec w(\delta))-\fg(\vec v)] 
    = s\delta[h'_j(x_j(\delta)^2) - h'_k(x_k(\delta)^2)]
    < s\delta[\lambda - \lambda] = 0
  \end{displaymath}
  holds by noting that each $sh_i'$ is strictly increasing on $\pinterval$ (by \cref{lem:h-props}).
  Thus, $\vec v$ is a minimum of $\abs \fg$.
\end{proof}

\Cref{thm:gen-simplex-optima} follows by combining \cref{prop:absfg-maxima-canonical,prop:absfg-maxima-no-extra} and using the symmetries of $F$ from \cref{assumpt:gsymmetries}. % \cref{lem:orth-transform}.
\end{vlong}
%%% Local Variables: 
%%% mode: latex
%%% TeX-master: "main"
%%% End: 

\section{Stability and convergence of gradient iteration}
\label{sec:gi-stability-structure}

In this section we will sketch the analysis for the stability and convergence of gradient iteration (\cref{thm:gi-stability,thm:gi-convergence}). It turns out that a special form of {\sef} is sufficient for our analysis.
\begin{defn}
  {\Ansef} $\fg(\vec u) = \sum_{i=1}^\dm g_i(u_i)$ is called a \emph{positive {\sef}} if $x \mapsto g_i(\sign(x)\sqrt{|x|})$ is strictly convex for each $i \in [\dm]$.
\end{defn}

The {\psef} has several especially nice properties.
Its name is justified by the fact that for {\apsef} $\fg$, $\fg(\vec u) \geq 0$ for all $\vec u \in S^{\dn - 1}$.
Further, when we expand $\fg(\vec u) = \sum_{i=1}^\dm h_i(\sign(u_i)u_i^2) = \sum_{i=1}^\dm h_i(u_i^2)$ under \cref{eq:fg-via-his}, we see that each $h_i$ is strictly convex over its entire domain.
Finally,  given {\ansef} $\fg$, we construct a {\psef}  $\bar F(\vec u) := \sum_{i=1}^{\dm} \bar g_i(u_i)$ where $\bar g_i(x) = \abs{g_i(x)}$.
We call $\bar F$ the \emph{{\psef} associated with $F$}.

We first establish that for
{\psef}s, the gradient iteration $\GG$ is a true fixed point method on $S^{\dn - 1}$ without the need to consider equivalence classes (as in \cref{sec:results-summary}).
Let $\phi$ and $\mu$ be defined as in \cref{sec:results-summary}.
%We have the following result.
We identify each orthant of $S^{\dn - 1}$ by a sign vector $\vec v$ where each $v_i \in \{+1, -1\}$ by defining $\orthsymb_{\vec v}^{\dn - 1} := \{ \vec u \in S^{\dn - 1} \suchthat v_iu_i \ge 0 \text{ for each } i \in [\dn]\}$ as the orthant of $S^{\dn - 1}$ containing $\vec v$.
\begin{lem}\label{lem:orthant-dist}
  Let $\vec v \in \R^\dn$ be a sign vector (that is, $v_i \in \{\pm 1\}$ for each $i \in [\dn]$). % defining an orthant $\orthsymb^{\dn - 1}_{\vec v}$ of $S^{\dn - 1}$.
  If $\vec u, \vec w \in \orthsymb^{\dn - 1}_{\vec v}$, then $\mu([\vec u], [\vec w]) = \norm{\vec u - \vec w}$.
\end{lem}
\begin{vlong}
\begin{proof}
  By direct calculation we see:
  \ifnum\version=\siamversion
  \begin{align*}
    \mu([\vec u], [\vec w])^2
    &= \Norm{\sum_{i=1}^\dn \abs {u_i}\myZv_i - \sum_{i=1}^\dn \abs {w_i}\myZv_i}^2
    = \sum_{i=1}^\dn (\abs{u_i} - \abs{w_i})^2 \\ &
    = \sum_{i=1}^\dn (u_i - w_i)^2
    = \norm{\vec u - \vec w}^2 \ .
  \end{align*}
  \else
  \begin{align*}
    \mu([\vec u], [\vec w])^2
    &= \Norm{\sum_{i=1}^\dn \abs {u_i}\myZv_i - \sum_{i=1}^\dn \abs {w_i}\myZv_i}^2
    = \sum_{i=1}^\dn (\abs{u_i} - \abs{w_i})^2
    = \sum_{i=1}^\dn (u_i - w_i)^2
    = \norm{\vec u - \vec w}^2 \ .
  \end{align*}
  \fi
  The first equality uses the definition of $\mu$, and the third equality uses that $\vec u, \vec w \in \orthsymb^{\dn - 1}_{\vec v}$, i.e., $u_i$ and $w_i$ share the same sign (up to the possibility of being 0) for each $i \in [\dn]$.
\end{proof}
\end{vlong}
In \cref{prop:gi-sym-classes} below, we see that $\bar \GG$ is orthant preserving, and that the iterations
$\GG/{\sim}$ and $\bar \GG \restr{\porth^{\dn-1}}$ are equivalent under the isometry $\phi$.
These iterations thus have equivalent fixed point properties.
It will suffice to analyze $\bar \GG \restr{\porth^{\dn-1}}$ in place of $\GG / {\sim}$.
\begin{prop} \label{prop:gi-sym-classes}
  % Suppose that $\fg(\vec u)=\sum_{i=1}^\dm g_i(u_i)$ is {\ansef} with associated gradient iteration function $\GG$.
  % Suppose that $\bar F$ is the {\psef} associated with $\fg$, and suppose that $\bar \GG$ is the gradient iteration function associated with $\bar F$.
  Let $\vec v$ be a sign vector in $\R^\dn$.
  Then, $\bar \GG$ has the following properties:
  \begin{compactenum}
  \item \label{property:PSEF-same-orth} If $\vec u \in \orthsymb_{\vec v}^{\dn - 1}$, then $\bar \GG(\vec u) \in \orthsymb_{\vec v}^{\dn-1}$.
  % \item If $\vec u \in \porth^{\dn - 1}$, then $\bar \GG(\vec u) \in \porth^{\dn-1}$.
  \item \label{property:PSEF-SEF-equiv} If $\vec u, \vec w \in S^{\dn - 1}$ are such that $\vec u \sim \vec w$, then $\GG(\vec u) \sim \bar \GG(\vec w)$.
  \end{compactenum}
\end{prop}
\begin{proof}
  \vnote{This proof can be removed from the short version if we are looking to save space.}
  \processifversion{vlong}{We first demonstrate property \ref{property:PSEF-same-orth} holds.}
  Let $\bar h_1, \dotsc, \bar h_\dn$ be defined for $\bar \fg$ in the same way that $h_1, \dotsc, h_\dn$ are defined for $F$ in \cref{sec:extrema-structure}.
  Then, $\partial_i \bar \fg(\vec u) = 2 \bar h_i'(u_i^2) u_i$ for all $i \in [\dn]$.
  Under \cref{lem:h-props}, $\sign(x)\bar h_i'(x) \geq 0$ on for all $x \in \R$ and all $i \in [\dm]$.
  As $\bar h_i := 0$ for all $i \in [\dn] \setminus [\dm]$, it follows that $\sign(u_i)\partial_i \bar \fg(\vec u) \geq 0$ for all $i \in [\dn]$.
  Thus, $\bar \GG(\vec u) \in \orthsymb_{\vec v}^{\dn - 1}$.

  We now demonstrate that property~\ref{property:PSEF-SEF-equiv} holds.
  Since $\vec u \sim \vec w$, there exist sign values $s_i \in \{+1, -1\}$ such that $u_i = s_i w_i$.
  By \cref{assumpt:gsymmetries} (i.e., $g_i$ and hence its derivative is either an even or odd function), we see that
  %\begin{displaymath}
  $  \abs{\partial_i \fg(\vec u)}
    = \abs{g'_i(u_i)}
    % = \abs{g'_i(s_iw_i)}
    = \abs{g_i'(w_i)}
    = \abs{\partial_i \bar \fg(\vec w)} $.
  %\end{displaymath}
  In particular, it follows that $\norm {\grad \bar \fg(\vec w)} = \norm {\grad  \fg(\vec u)}$, and that $\abs{\bar \GG_i(\vec w)} = \abs{\GG_i(\vec u)}$ for each $i \in [\dn]$.
  Thus, $\bar \GG(\vec w) \sim \GG(\vec u)$.
\end{proof}
% \begin{vlong}
% \vnote{The following Corollary does not add much.
%   It possibly should be removed from the long version as well. . .}
% \begin{cor}\label{cor:psef-equivclass-repr}
%   % We continue with $\fg$, $\GG$, $\bar F$, and $\bar \GG$ defined as in \cref{prop:gi-sym-classes}.
%   Given a sequence $\{\vec u \ind n\}_{n=0}^\infty$ in $S^{\dn - 1}$ defined recursively % from a base element $\vec u \ind 0$ and the formula
% by $\vec u \ind n := \GG(\vec u \ind {n-1})$, then we may consider a parallel sequence $\{\vec v \ind n\}_{i=0}^\infty$ in $\porth^{\dn-1}$ defined by $\vec v \ind 0 := \phi([\vec u\ind 0])$ and $\vec v \ind n := \bar \GG(\vec v \ind{n-1})$.
%   Then, for any $\vec w \in \porth^{\dn -1}$ and any fixed $n$, $\mu([\vec u \ind n], [\vec w]) = \norm{\vec v \ind n - \vec w}$.
% \end{cor}
% \end{vlong}

Throughout this section, we will assume that $\fg(\vec u) = \sum_{i=1}^\dm g_i(u_i)$ is {\apsef}. The functions $h_i$ are defined  as in \cref{sec:extrema-structure}.
We will analyze the associated gradient iteration function $\GG$ on the domain $\porth^{\dn -1}$.
%\processifversion{vlong}
{It suffices to analyze {\psef}s on $\porth^{\dn - 1}$, and the results can be easily extended to general {\sef}s on $S^{\dn - 1}$ due to \cref{prop:gi-sym-classes}.}
Unless otherwise stated, we will also assume in this section that $\{\vec u(n)\}_{n=0}^\infty$ is a sequence in $\porth^{\dn - 1}$ satisfying $\vec u(n) = \GG(\vec u(n-1))$ for all $n \geq 1$.

\begin{vshort}
In the long version of this paper we proceed with the formal analysis of the global stability structure and the rate of convergence of our dynamical system $G/{\sim}$.
\end{vshort}
\begin{vlong}
We now proceed with the formal analysis of the global stability structure and the rate of convergence of our dynamical system $G/{\sim}$.
It will be seen in \cref{sec:gi-convergence} that the fast convergence properties of the gradient iteration are due to the strict convexity in \cref{assumpt:convex}.
\end{vlong}
\begin{vshort}
We show in \cref{sec:gi-stability-structure} of the long version of this extended abstract that the fast convergence properties of the gradient iteration are due to the strict convexity in \cref{assumpt:convex}.
\end{vshort}
However, we will spend most of our time characterizing the stability of fixed points of $\GG/{\sim}$, in particular demonstrating that the hidden basis elements $\myZv_1, \dotsc, \myZv_\dm$ are attractors, and that for almost any starting point $\vec u(0)$, $\vec u(n)$ converges to one of the hidden basis elements as $n\rightarrow \infty$.

We now give a brief outline of the argument for the global attraction of the hidden basis elements.
For simplicity, we provide this sketch for the case where $\dn = \dm$.
However, we will later provide all statements and proofs necessary to obtain the global stability in full generality.
This argument has four main elements.\lnote{changed $\dm$ to $\dn$ in sketch}

\paragraph{1. Enumeration of the fixed points of the gradient iteration\processifversion{vlong}{ (\cref{sec:enum-fixed-points})}\opdot}
We enumerate the fixed points of $\GG$ and see that, including the hidden basis elements $\myZv_1, \dotsc, \myZv_\dn$, the dynamical system $\GG$ actually has $2^\dn - 1$ fixed points in $\porth^{\dn - 1}$.
In particular, we will see that for any subset $\SS \subset [\dn]$, there exists exactly one fixed point $\vec v$ of $\GG$ in $\porth^{\dn - 1}$ such that $v_i \neq 0$ iff $i \in \SS$.
The proof of this enumeration of fixed points is based on the expansion $\GG(\vec u) = \frac{\nabla F(\vec u)}{\norm{\nabla F(\vec u)}}$ where $\nabla F(\vec u) = \sum_{i=1}^\dm h_i'(u_i^2) \myZv_i$ and the monotonicity of the $h_i'$ functions from \cref{lem:h-props}.
The proof also uses an observation that the fixed points of $\GG$ are exactly the critical points of $F$ on $S^{\dn - 1}$ arising in the optimization view.

\paragraph{2. Hyperbolic fixed point structure and stability/instability implications\processifversion{vlong}{ (\cref{sec:ae-attraction})}\opdot}
We show that all fixed points of $\GG$
%other than the hidden basis elements $\myZv_1, \dotsc, \myZv_\dn$
are hyperbolic, i.e. the eigenvalues of the Jacobian matrix are different from $1$ in absolute value%
\processifversion{vlong}{ (\cref{lem:crit-point-eigendecomp})}.
As such, the stability properties of the fixed points of $\GG$ can be inferred from the eigenvalues of its Jacobian.

We denote by $\Jacob \GG_{\vec u}$ the Jacobian of $\GG$ evaluated at $\vec u$, and we let $\vec p$ be a fixed point of $\GG$ outside of the set $\{\myZv_1, \dotsc, \myZv_\dn\}$.
Then, we show that as a linear operator $\Jacob \GG_{\vec p} : \vec p^{\perp} \rightarrow \vec p^\perp$,
%$\GG_{\vec p}$,
% we show that the eigenvalues of $\Jacob \GG_{\vec p}$ are contained in $\R \setminus \{-1, 1\}$
$\Jacob \GG_{\vec p}$ has at least one eigenvalue with magnitude strictly greater than 1.
This implies that $\vec p$ is locally repulsive for the discrete dynamical system $\GG$ except potentially on a low dimensional manifold called the local stable manifold of $\vec p$\processifversion{vlong}{
  (\cref{lem:GG-LSM})}. %\citep[see][]{luo2012regularity}.
As the local stable manifold of $\vec p$ is low dimensional, it is also of measure zero.
By analyzing the measure of repeated compositions of $\GG^{-1}$ applied to the local stable manifold of $\vec p$, % under the gradient iteration, % $\GG^{-1}$,
we are able to demonstrate that globally on the sphere, the set of starting points $\vec u(0)$ such that $\vec u(n) \rightarrow \vec p$ is measure zero%
\processifversion{vlong}{ (\cref{thm:vol0-unstable-pt-global})}.

We will also see that at a hidden basis element $\myZv_i$, $\Jacob \GG_{\myZv_i} : \myZv_i^\perp \rightarrow \myZv_i^\perp$ is the zero  map.
In particular, $\myZv_i$ is an attractor of the dynamical system $\GG$.
Taken together, these results show that the hidden basis directions $\myZv_i$ are the attractors of the gradient iteration, and that all other fixed points are unstable.

\paragraph{3. The big become bigger, and the small become smaller\processifversion{vlong}{ (\cref{sec:gi-divergence-criteria})}\opdot}
We show that coordinates of $\vec u(n)$ go to zero as $n \rightarrow \infty$ under certain conditions.
In particular, let $\SS \subset [\dn]$ and let $\vec v$  be the fixed point of $\GG$ such that $v_i \neq 0$ if and only if $i \in \SS$.
An implication of the convexity \cref{assumpt:convex} is that if $u_i > v_i$, then $\partial_i F(\vec u) / u_i > \partial_i F(\vec v) / v_i$, and similarly if $u_i < v_i$, then $\partial_i F(\vec u) / u_i < \partial_i F(\vec v) / v_i$.
To see that these orderings hold, we use the expansion $F(\vec u) = \sum_{i=1}^\dm h_i(u_i^2)$ to see $\partial_i F(\vec u) / u_i = 2h_i'(u_i^2)$ and we recall (from \cref{lem:h-props}) that each $h_i'$ is an increasing function.
Using this monotonicity, we show that each gradient iteration update has the effect of increasing the gap (as a ratio) between $\max_{i \in \SS} \partial_i F(\vec u) / u_i$ and $\min_{i \in \SS} \partial_i F(\vec u) / u_i$.
This implies a divergence between the coordinates of $\vec u(n)$ under the gradient iteration.

In particular, we show that if there exists an $i \in \SS$ and $k \in \N$ such that $u_i(k) > v_i$, then the ratio between maximum magnitude and minimum magnitude coordinate values of $\vec u(n)$ within $\SS$ %of $\vec u(n)$
 goes to infinity as $n \rightarrow \infty$.
In particular, there will exist an $i\in \SS$ such that $u_i(n) \rightarrow 0$ as $n \rightarrow \infty$.

\paragraph{4. Global attraction of the hidden basis\processifversion{vlong}{ (\cref{thm:canonical-global-attraction})}\opdot}
We alternate between applying parts 2 and 3 of this sketch in order to demonstrate that for almost any $\vec u(0)$, all but one of the coordinates of $\vec u(n)$ go to zero as $n$ goes to infinity.
Part 3 of the sketch allows us to force coordinates of $\vec u(n)$ to approach 0. %, creating global properties for the trajectory of $\vec u(n)$.
By part 2, the trajectory never converges to one of the unstable fixed points of $\GG$.  This  guarantees for any particular unstable fixed point $\vec v$ that a coordinate of $\vec u(n)$ eventually exceeds the corresponding non-zero coordinate of $\vec v$ due to the interplay with part 3.
% Then, part 2 allows us to see that for one of the coordinates that has not yet been forced towards 0,
As all but one of the hidden coordinates of $\vec u(n)$ must eventually go to 0, it follows that $\vec u(n) \rightarrow \myZv_i$ for some $i \in [\dm]$ as $n \rightarrow \infty$.

\begin{vlong}
% The proofs in this section largely use the power method interpretation of $\GG$.

%\subsection{Fixed point stability}
%\label{sec:gi-stability}
\subsection{Enumeration of fixed points}
\label{sec:enum-fixed-points}

We now begin the process of enumerating the fixed points of $\GG$.
First, we observe that the fixed points of $\GG$ are very closely related to the maxima structure of $F$.
\begin{obs}
  \label{obs:stationary-criterion}
  A vector $\vec v \in \porth^{\dn - 1}$ is a stationary point of $\GG$ if and only if there exists $\lambda^*$ such that $(\vec v, \lambda^*)$ is a critical point of the Lagrangian\footnote{This is the same Lagrangian function which arose in \cref{sec:extrema-structure}.
Its critical points $(\vec u, \lambda)$ give the locations $\vec u$ where $\fg$ satisfies the first order conditions for a constrained extrema on the sphere.} function $\LL(\vec u,\ \lambda) = \fg(\vec u) - \lambda[\norm u ^2 - 1]$.
  In particular, if $\vec v$ is a stationary point of $\GG$, then $\lambda^* \indicator{v_i \neq 0} = h_i'(v_i^2)$ for each $i \in [\dn]$.
\end{obs}
\begin{proof}
  This is a result of \cref{lem:gi-is-grad-ascent,lem:fg-crit-points}.
\end{proof}

With this characterization, we are actually able to enumerate the fixed points $\GG$.
Note that if $v_i = 0$ for each $i \in [\dm]$, then by the definition of $\GG$, $\vec v$ is a stationary point.
The remaining stationary points are enumerated by the following lemma.
\vnote{Note:  The following Lemma is used to slightly simplify the proof of \cref{prop:unstable-stationary-pts} in this section, but is not really necessary in this section.  However, this Lemma is required in \cref{sec:robust-grad-iter}, for instance in \cref{claim:robust-GI-progress-single-step}.}
\begin{lem}\label{lem:stationary-pt-enum}
  Let $\SS \subset [\dm]$ be non-empty.
  Then there exists exactly one stationary point $\vec v$ of $\GG\restr{\porth^{\dn-1}}$ such that $v_i \neq 0$ for each $i \in \SS$ and $v_i = 0$ for each $i \in [\dm] \setminus \SS$.
  Further, $v_i = 0$ for each $i \in [\dn] \setminus \SS$.
\end{lem}
\begin{proof}
  We prove this in two parts.
  First, we show that a $\vec v$ exists with all of the desired properties.
  Then, we show uniqueness.
  \begin{claim}\label{claim:stationary-pt-existence}
    There exists $\vec v$ a stationary point of $\GG\restr{\porth^{\dn-1}}$ such that $v_i \neq 0$ if and only if $i \in \SS$.
  \end{claim}
  \begin{subproof} %[Proof of Claim]
    % \vnote{This construction can probably be simplified by constructing a sequence via binary search on the ranges of the $h_i'$s instead with convergence to the unit sphere. This would at least remove the need for interpreting pseudocode.}
  We will construct $\vec v$ as the limit of a sequence.
  Consider the following construction of an approximation to $\vec v$ whose precision depends on the magnitude of $\frac 1 N$ where $N \in \N$.

        \begin{algorithmic}[1]
        \Function{ApproxFixPt}{$N$}
          \State $\vec u \leftarrow \vec 0$
          % \For{$i \in \SS$}
          % \State $u_i \leftarrow \sqrt{u_i^2 + \frac 1 N}$
          % \EndFor
          \For{$i \leftarrow 1$ to $N$}
            \State\label{alg:appr-fix-pt:ln-argmin} $j \leftarrow \argmin_{k\in \SS} h_k'(u_k^2)$
            \State $u_j \leftarrow \sqrt{u_j^2 + \frac 1 N}$
          \EndFor
          \State \Return $\vec u$
        \EndFunction
        \end{algorithmic}

  Let $\epsilon_0 > 0$ be fixed.
  Let $\epsilon_k = \frac 1 k \epsilon_0$ for each $k \in \N$.
  Since $[0, 1]$ is a compact space and there are a finite number of $h_i'$ functions, the $h_i'$s are uniformly equicontinuous on this domain.
  Thus for each $k \in \N \cup \{0\}$ and all $i \in \SS$, there exists $\delta_k > 0$ such that for $x, y \in [0, 1]$, $\abs{x-y} \leq \delta_k$ implies that $\abs{h_i'(x) - h_i'(y)} \leq \epsilon_k$.
  We fix constants $N_k \in \N \cup \{0\}$ such that (1) $\frac 1 {N_k} \leq \delta_k$ for each $k$, (2) for each $k \geq 1$, $N_k$ is an integer multiple of $N_0$, and (3) $N_0 \geq \abs \SS$.
  Then we construct a sequence $\{\vec u(k)\}_{k=0}^\infty$ by setting $\vec u(k) = \textsc{ApproxFixPt}(N_k)$ for each $k \in \N \cup \{0\}$.
  It follows by construction that $\abs{h_i'(u_i^2(k)) - h_j'(u_j^2(k))} \leq \epsilon_k $ for each $i, j \in \SS$.

  It can be seen that $\min_{i\in\SS} h_i'(u_i^2(k)) \geq \min_{i\in\SS} h_i'(u_i^2(0)) > 0$ for each $k \in \N$.
  To see the second inequality $\min_{i\in \SS} h_i'(u_i^2(0)) > 0$, we note that the $h_i'$s are strictly increasing from 0 by \cref{lem:h-props}, and in particular during the first $\abs \SS$ iterations of the loop in \textsc{ApproxFixPt}, a new coordinate of $\vec u$ will be incremented.
  To see the first inequality $\min_{i\in\SS} h_i'(u_i^2(k)) \geq \min_{i\in\SS} h_i'(u_i^2(0))$ for each $k \in \N$, we argue by contradiction.
  Let $j = \argmin_{i \in \SS} h_i'(u_i^2(k))$.
  If $h_j'(u_j^2(k)) < \min_{\min i\in \SS} h_i'(u_i^2(0))$, then $u_j^2(k) < \min_{i \in \SS} u_i^2(0)$, and thus there exists $\ell \in \SS$ with $\ell \neq j$ such that $u_\ell^2(k) > u_\ell^2(0)$.
  However, for this to be true, then during course of the execution of \textsc{ApproxFixPt}($N_k$) the decision must be made at line~\ref{alg:appr-fix-pt:ln-argmin} that $\ell = \argmin_{k \in \SS} h_k'(u_k^2)$ when $u_\ell^2 = u_\ell^2(0)$ (since $N_k$ is an integer multiple of $N_0$).
  During this update, strict monotonicity of $h_i'$ implies that $h_j'(u_j^2) \leq h_j'(u_j^2(k)) < \min_{\min i\in \SS} h_i'(u_i^2(0)) \leq h_\ell'(u_\ell^2)$.
  But this contradicts that $\ell = \argmin_{k \in \SS} h_k'(u_k^2)$ at line~\ref{alg:appr-fix-pt:ln-argmin}.
  It follows that there exists a $\Delta > 0$ such that for each $i \in \SS$ and each $k \in \{0, 1, 2, \dotsc\}$ we have $h_i'(u_i^2(k)) > \Delta$, and in particular that $u_i^2(k) \geq \min_{j \in \SS}(h_j')^{-1}(\Delta) > 0$.

  Since $S^{\dn - 1}$ is compact, there exists a subsequence $i_1, i_2, i_3, \dotsc$ of $0, 1, 2, \dotsc$ such that the sequence $\{\vec u(i_k)\}_{k=1}^\infty$ converges to a vector $\vec v \in S^{\dn - 1}$.
  Since each $\vec u(i_k) \in \porth^{\dn - 1}$, $\vec v \in \porth^{\dn - 1}$.
  Further, since the $u_j^2(i_k)$s are bounded from below by a constant $\Delta' = \min_{j \in \SS}(h_j')^{-1}(\Delta) > 0$ for each $j \in \SS$, we see that $v_j^2 \geq \Delta' > 0$ for each $j \in \SS$.
  That is, $v_i = 0$ if and only if $i \in \SS$.
  Further, for any $j, \ell \in \SS$, $h_\ell'(v_\ell^2) - h_j'(v_j^2) = \lim_{k\rightarrow \infty } [h_\ell'(u_\ell^2(i_k)) - h_j'(u_j^2(i_k))] = 0$, and in particular $h_\ell'(v_\ell^2) = h_j'(v_j^2)$.
  % Thus $v_i = 0$ if and only if $i \in \SS$.
  By \cref{obs:stationary-criterion}, $\vec v$ is a stationary point of $\GG$.
\end{subproof}

  \begin{claim}\label{claim:stationary-pt-uniqueness}
    There exists only one stationary point $\vec v$ of $\GG\restr{\porth^{\dn-1}}$ such that the following hold: (1) $v_i \neq 0$ if $i \in \SS$ and (2) $v_i = 0$ if $i \in [\dm] \setminus \SS$.
  \end{claim}
  \noindent \textit{Proof of Claim.}
    We first show that if $\vec v$ is a stationary point of $\GG\restr{\porth^{\dn-1}}$ meeting the conditions of the claim, then $v_i = 0$ for each $i \in [\dn] \setminus [\dm]$.
    To see this, we use \cref{obs:stationary-criterion}, and we note that for each $i, j \in [\dn]$ such that $u_i \neq 0$ and $u_j \neq 0$, then $h_i'(u_i^2) = h_j'(u_j^2)$.
    In particular, choosing $i \in \SS$, we see that $h_i'(u_i^2) > 0$.
    But for each $i \in [\dn] \setminus [\dm]$, $h_i := 0$ implies that $h_i'(u_i^2) = 0$.
    In particular, for $i \in [\dn] \setminus [\dm]$, $u_i = 0$.

    Now suppose that there are two stationary points $\vec v$ and $\vec w$ meeting the requirements of this Claim.
    By \cref{obs:stationary-criterion}, there exists $\lambda_{\vec v}$ and $\lambda_{\vec w}$ such that $h_i'(v_i^2) = \lambda_{\vec v}$ and $h_i'(w_i^2) = \lambda_{\vec w}$ for each $i \in \SS$.
    If $\lambda_{\vec v} < \lambda_{\vec w}$, then strict monotonicity of each $h_i'$ implies that $v_i^2 < w_i^2$ for each $i \in \SS$.
    But this contradicts that $\sum_{i\in \SS} v_i^2 = 1 = \sum_{i \in \SS}w_i^2$.
    By similar reasoning, it cannot be that $\lambda_{\vec w} < \lambda_{\vec v}$.
    As such, $\lambda_{\vec v} = \lambda_{\vec w}$, and further for each $i \in \SS$ it follows that $h_i'(v_i^2) = h_i'(w_i^2)$.
    Using strict monotonicity of the $h_i'$s, we see that $\vec v = \vec w$.

    Note that the $\vec v$ constructed in \cref{claim:stationary-pt-existence} gives the unique solution to this claim.
\end{proof}

%\end{document}

\subsection{Convergence to the hidden basis directions}
\label{sec:guar-conv-hidd}

So far, we have enumerated the fixed points of the dynamical $\GG$ on $\porth^{\dn - 1}$.
We now analyze the stability properties of these fixes points.
In \cref{sec:gi-divergence-criteria}, we create a divergence criteria from the fixed points of $\GG$ excluding the hidden basis elements $\myZv_1, \dotsc, \myZv_\dm$.
This divergence criterion sets up a natural manner under which the large coordinates of $\vec u(0)$ can increase in magnitude while other coordinates are driven rapidly towards 0.
Then, in \cref{sec:ae-attraction}, we demonstrate that the set of hidden basis elements of $\GG$ are essentially global attractors of the dynamical system.
In particular, it is seen that each $\myZv_i$ is locally an attractor, and that for $\vec u(0)$ drawn from a set of full measure on $S^{\dn - 1}$, the sequence $\vec u(n)$ converges to one of the hidden basis elements.

The main intuition for why the gradient iteration converges to a hidden basis direction comes from two key concepts (working in $\porth^{\dn - 1}$ with a {\psef}):
\begin{enumerate}
\item\label{item:expand-all-explanation-ordering} For each iteration there is an implicit ordering $i \prec_n j$ if $h_i'(u_i(n)) < h_j'(u_j(n))$ such that the ratio
\begin{equation}\label{eq:power-iteration-expand-demo}
  \frac{u_{i}(n+1)}{u_{j}(n+1)} 
  = \frac{\GG_{i}(\vec u(n))}{\GG_{j}(\vec u(n))}
  = \frac{h_{i}'(u_{i}(n)^2)u_{i}(n)}{h_{j}'(u_{j}(n)^2)u_{j}(n)}
\end{equation}
expands from $u_i(n)/u_j(n)$ if and only if $i \succ_n j$.
\item The function $n \mapsto \max_i h_i'(u_i(n)^2)$ is non-decreasing and $n \mapsto \min_i h_i'(u_i(n))$ is non-increasing\footnote{See \cref{lem:increasing-maxhi'}.  We don't actually prove that $n \mapsto \min_i \abs{h_i'(u_i(n))}$ is non-decreasing as we don't end up needing this fact, but the proof closely resembles that of \cref{lem:increasing-maxhi'} with the roles of $\max$ and $\min$ swapped and the also ordering role of $<$ reversed to $>$.}.  As such, the maximal expansionary effect seen between coordinates in part \ref{item:expand-all-explanation-ordering} can only increase with each iteration.
\end{enumerate}
These two ideas provide intuition for why the gradient iteration should drive coordinates to zero until a hidden basis direction $\myZv_i$ is recovered.
By repeated use of the gradient iteration, we expect $u_j(n) \rightarrow 0$ for some $j$ in order to support the expansion of the fraction $u_i(n) / u_j(n)$.

Interestingly, these two observations hold even if \cref{assumpt:convex} is relaxed to allow functions $g_i$ where $g_i(\sqrt{x})$ is convex but not \textit{strictly} convex, and if \cref{assumpt:deriv0} is omitted\footnote{In particular, the proof of \cref{lem:increasing-maxhi'} holds under this relaxation.}.
We conjecture that guarantees similar to those of matrix eigenvector recovery via the power iteration are achievable for more generic odeco functions satisfying such relaxed assumptions.
However, formulating precise statements for when odeco basis recovery works in such relaxed conditions is beyond the scope of this paper.

Where the analysis of odeco functions is more challenging in our
setting than in the matrix and tensor settings is that the ordering
$\prec_n$ changes with $n$.
In the matrix setting, each $h_i'(u_i(n)^2) = \lambda_i$ is simply the
$i$\textsuperscript{th} eigenvalue making the ordering $\prec_n$ fixed
for all $n$.
If $\lambda_1$ corresponds to a maximal eigenvalue with multiplicity 1
and $u_i(0) \neq 0$, then it is not difficult to see that
$\abs{u_1(n)}/ \abs{u_i(n)} \rightarrow \infty$ as $n \rightarrow
\infty$ under the power iteration by repeated application of
\eqref{eq:power-iteration-expand-demo} for all $i \neq 1$; indeed,
this is a standard proof of convergence to the top eigenvector in the
matrix setting.
In the odeco function setting $\prec_n$ is not fixed.
Because we cannot directly chain together applications of
\eqref{eq:power-iteration-expand-demo}, we do the following.
\begin{enumerate}
\item In \cref{sec:gi-divergence-criteria} we rely on an interplay between the coordinates of $\vec u(n)$ in the $\prec_n$ ordering and the coordinate values of a fixed point $\vec v$ of $\GG$ to create a globally expanding ratio which forces a coordinate of $\vec u(n)$ to zero.
\item In \cref{sec:ae-attraction} we use stable-unstable manifold theory to show $\{\vec u(n)\}_{n=0}^\infty$ converges to a hidden basis direction $\myZv_i$ given almost any starting point. 
\end{enumerate}

%%% Local Variables: 
%%% mode: latex
%%% TeX-master: "main"
%%% End: 

\paragraph{Notation\opdot}
Throughout this subsection, we will make use of the following notations.
Given a $\SS \subset [\dn]$, we define the projection matrix $\Pmap_\SS := \sum_{i \in \SS} \myZv_i\myZv_i^T$.
In particular, this implies $\Pmap_\SS \vec u := \sum_{i \in \SS} u_i \myZv_i$.
We will denote the set complement by $\bar \SS := [\dn] \setminus \SS$.
Two projections will be of particular interest:  the projection onto the distinguished basis elements $\Pmap_{[\dm]} \vec u := \sum_{i=1}^\dm u_i \myZv_i$ and its complement projection which we will denote by $\Pmap_0 \vec u := \sum_{i=\dm + 1}^\dn u_i \myZv_i$.
In addition, if $\mathcal X$ is a subspace of $\R^\dn$, we will denote by $\Pmap_{\mathcal X}$ the orthogonal projection operator onto the subspace $\mathcal X$.

We denote by $\vol_{k-1}$ the volume measure on the unit sphere $S^{k-1}$.
When the value of $k$ is clear, we suppress it from the notation and simply write $\vol$ for the volume measure on the unit sphere (``surface area measure'').
Finally, if $f : M \rightarrow N$ (with $M$ and $N$ manifolds), we denote by $\Jacob f_{\vec x}$ the Jacobian (or transposed derivative) of $f$ evaluated at $\vec x$.
We also treat $\Jacob f_{\vec x}$ as linear operator between tangent spaces:  $\Jacob f_{\vec x} : T_{\vec x} M \rightarrow T_{f(\vec x)} N$, where $T_{\vec x}M$ denotes the tangent space of $M$ at $\vec x$. See the book of \mycitet{do Carmo Valero}{do1992riemannian} for an overview of Riemannian manifolds and the definition of volume on manifold surfaces.
%
%Given $\SS \subset [\dn]$, we define $\Pmap_\SS \vec u := \sum_{i \in \SS} u_i \myZv_i$.
%We will denote the set complement by $\bar \SS := [\dn] \setminus \SS$.
%Two projections will be of particular interest:  the projection onto the distinguished basis elements $\Pmap_{[\dm]} \vec u := \sum_{i=1}^\dm u_i \myZv_i$ and its complement projection which we will denote by $\Pmap_0 \vec u := \sum_{i=\dm + 1}^\dn u_i \myZv_i$.
%In addition, if $\mathcal X$ is a subspace of $\R^\dn$, we will denote by $\Pmap_{\mathcal X}$ the orthogonal projection operator onto the subspace $\mathcal X$.
%In particular, if $\SS \subset [\dn]$, then the operators $\Pmap_\SS$ and $\Pmap_{\spn(\{\myZv_i \suchthat i \in \SS\})}$ are identical.

\subsubsection{Divergence criteria for unstable fixed points}\label{sec:gi-divergence-criteria}
\newcounter{thmsavetmp}
\newcounter{thmsavegiexpand}
\begin{prop}\label{prop:gi-expand-all}
  There exists $\epsilon > 0$ such that the following holds.
  Let $\vec v \in \porth^{\dn - 1}$ be a stationary point of $\GG$, denote by $\SS_{\vec v} := \{ i \suchthat v_i \neq 0\}$, and suppose $\SS_{\vec v} \subset [\dm]$.
  Suppose $\norm{\Pmap_{\compl \SS_{\vec v}} \vec u(0)} \leq \epsilon$ and there exists $i \in \SS_{\vec v}$ such that $u_i(0) > v_i$, then there exists $j \in \SS_{\vec v}$ such that $u_j(n) \rightarrow 0$ as $n \rightarrow \infty$.
\end{prop}
\ifnum\version=\siamversion
\else \setcounter{thmsavegiexpand}{\value{thm}} \fi

We now proceed with the proof of \cref{prop:gi-expand-all}.
We will need a couple of facts about the behavior of small coordinates of $\{\vec u(n)\}_{n=0}^\infty$ under the gradient iteration.
In particular, we need to show that $\GG(\vec u)$ is generally well behaved (i.e., $\norm{\nabla F(\vec u)}$ is typically separated from 0), and that the small coordinates of $\vec u(0)$ are attracted to 0.
% In particular, we demonstrate that if $u_i(0)$ is sufficiently close to 0, then $u_i(n) \rightarrow 0$ as $n \rightarrow \infty$.

\begin{lem}\label{lem:normF-lower-bound}
  Let $F$ be a fixed {\sef}.
  Given $\Delta \in [0, 1)$, there exists $L > 0$ such that the following holds:
  For all $\vec u \in \porth^{\dn - 1}$ such that $\norm{\Pmap_0 \vec u} \leq \Delta$, $\norm{\nabla F(\vec u)} > L$.
\end{lem}
\begin{proof}
  Since $\sum_{i\in [\dm]} u_i^2 = 1 - \norm{\Pmap_0 \vec u}^2 \geq 1 - \Delta^2$, there exists $j \in [\dm]$ such that $u_j \geq \sqrt{ \frac {1 - \Delta^2}{\dm}}$.
  It follows that
  \begin{align*}
    \norm{\nabla F(\vec u)}^2
    &= \sum_{ i = 1 }^\dm (2h_i'(u_i^2)u_i)^2
    \geq \max_{i \in [\dm]} 4 h_i'(u_i^2)^2 u_i^2 \\
    &\geq 4h_j'(u_j^2)^2u_j^2
    \geq \min_{i \in [\dm]} 4h_i'\Big(\frac{1-\Delta^2}{\dm}\Big) \cdot \frac{1-\Delta^2}{\dm} > 0 \ .
  \end{align*}
  For the last inequality, we use that each $h_i'$ is strictly increasing on $[0, 1]$ from 0.
\end{proof}

\begin{lem}\label{lem:zero-coord-attraction}
  Let $F$ be {\apsef}, let $C > 0$, and let $\Delta \in [0, 1)$.
  There exists $\epsilon > 0$ such that the following holds:
  Let $\vec u \in \porth^{\dn - 1}$ be such that $\norm{\Pmap_0 \vec u} \leq \Delta$.
  Define $A_\epsilon := \{ i \suchthat u_i \leq \epsilon\}$.
  For all $i \in A_\epsilon$, $\GG_i(\vec u) < Cu_i$.
\end{lem}
\begin{proof}
  For all $i \in [\dm]$, $h_i'$ is continuously increasing from $h_i'(0) = 0$.
  Given any $L > 0$, there exists $\epsilon > 0$ such that for all $i \in [\dm]$, $u_i \leq \epsilon$ implies that $2 h_i'(u_i^2) < C L$.
  With the choice of $L$ from \cref{lem:normF-lower-bound} and the above construction of $\epsilon$, we obtain the following:
  For all $i \in A_\epsilon$, $\GG_i(\vec u) = \frac{2 h_i'(u_i^2) u_i}{\norm{\nabla F(\vec u)}} < \frac{C L u_i}{L} = C u_i$.
\end{proof}

\begin{cor}\label{cor:zero-coord-attraction}
  Let $F$ be {\ansef}.
  There exists $\epsilon > 0$ such that the following holds:
  Let $\{\vec u(n)\}_{n = 0}^\infty$ be a sequence in $\porth^{\dn - 1}$ defined recursively by $\vec u(n) = \GG(\vec u(n-1))$ such that $\norm{\Pmap_0\vec u(0)} \neq 1$.
  Let $A_{\epsilon}(n) := \{ i \suchthat u_i(n) \leq \frac 1 {2^n} \epsilon \}$.
  Then, $A_\epsilon(0) \subset A_\epsilon(1) \subset A_\epsilon(2) \subset \cdots$.
\end{cor}
\begin{proof}
  % We first note that for all $i \in [\dn] \setminus [\dm]$, $\partial_i F(\vec u) = 0$ implies that $u_i(n) = 0$ for all $n \in \N$.
  We then apply \cref{lem:zero-coord-attraction} with the choice of $C = \frac 1 2$ in order to choose $\epsilon$.
  With this choice of $\epsilon$, we see that $A_{\epsilon}(n) \supset A_{\epsilon}(n-1)$ for all $n \in \N$ by \cref{lem:zero-coord-attraction}.
\end{proof}

In the following Lemma, we identify a useful notion of progress for the gradient iteration.
\begin{lem}\label{lem:increasing-maxhi'}
  The function $n \mapsto \max_{i \in [\dm]} \abs{h_i'(u_i(n)^2)}$ is a non-decreasing function of $n$.
\end{lem}
Note that when given a stationary point $\vec v \in \porth^{\dn - 1}$, \cref{obs:stationary-criterion} implies the existence of $\lambda > 0$ such that $h_i'(v_i^2) = \lambda$ for all $i \in \SS_{\vec v}$.
As the $h_i'$s are strictly increasing functions, we note that for an $i \in \SS_{\vec v}$ $u_i(k) > v_i$ if and only if each $h_i'(u_i(k)^2) > h_i'(v_i^2)$.
This criterion will be useful in demonstrating that once there exists
\begin{proof}[\proofword of \cref{lem:increasing-maxhi'}]
  Let $A := \{ i \suchthat u_i(0) \neq 0\} \cap [\dm]$.
  We may assume that $A \neq \emptyset$ as otherwise $\{\vec u(n)\}_{n=0}^\infty$ is a constant sequence, leaving nothing to prove.
  We only need consider the indices in $A$ since for all $i \in \compl A$, $\GG_i(\vec u(n+1)) \propto h_i'(u_i(n)^2)u_i(n) = 0$.
  We note that for $i, j \in A$,
  \begin{displaymath}
    \frac{\GG_i(\vec u(n+1))}{\GG_j(\vec u(n+1))} = \frac{h_i'(u_i(n)^2)}{h_j'(u_j(n)^2)} \cdot \frac {u_i(n)}{u_j(n)} \ .
  \end{displaymath}
  Fixing $i^* = \argmax_{i \in A} \abs{h_i'(u_i(n)^2)}$, we see that the ratio $\frac{\abs{\GG_{j}(\vec u(n+1))}}{\abs{\GG_{i^*}(\vec u(n+1))}} \leq \frac {\abs{u_j(n)}}{\abs{u_{i^*}(n)}}$ for all $j \in A$.
  In particular,
  \begin{displaymath}
    \frac 1 {\GG_{i^*}(\vec u(n+1))^2}
    = \sum_{j \in A} \frac{\GG_{j}(\vec u(n+1))^2}{\GG_{i^*}(\vec u(n+1))^2}
    \leq \sum_{j \in A} \frac{u_j(n)^2}{u_{i^*}(n)^2}
    = \frac 1 {u_{i^*}(n)^2}
  \end{displaymath}
  implies that $\abs{\GG_{i^*}(\vec u(n+1))} \geq \abs{u_{i^*}}$.
  As each $h_i'$ is a monotone function on $[0, 1]$, it follows:
  \begin{displaymath}
    \max_{i \in [\dm]} \abs{h_i'(u_i(n+1)^2)} \geq \abs{h_{i^*}'(u_{i^*}(n+1)^2)} \geq \abs{h_{i^*}'(u_{i^*}(n)^2)} = \max_{i \in [\dm]} \abs{h_i'(u_i(n)^2)} \ . \qedhere
  \end{displaymath}
\end{proof}

We now proceed with the proof of \cref{prop:gi-expand-all}.
\ifnum\version=\siamversion
\else
\setcounter{thmsavetmp}{\value{thm}}
\setcounter{thm}{\value{thmsavegiexpand}}
\fi
\begin{proof}[\proofword of \cref{prop:gi-expand-all}]
  We set $\lambda = h_i'(v_i^2)$ for any $i \in \SS_{\vec v}$.
  Using \cref{obs:stationary-criterion}, we see that $\lambda = h_i'(v_i^2)$ for all $i \in \SS_{\vec v}$.
  We choose $\epsilon > 0$ sufficiently small such that $u_i(0) \leq \epsilon$ implies that $h_i'(u_i(0)) < \lambda$, and also such that $\epsilon$ satisfies the conditions of \cref{cor:zero-coord-attraction}.

  We will assume that $u_i(0) \neq 0$ for each $i \in \SS_{\vec v}$, since otherwise $u_i(n) = 0$ for all $n \in \N$ (for this choice of $i$), leaving nothing to prove.
  We will make use of the following claims.
  \begin{claim}\label{claim:lt-fixed-coord}
    For any $\vec w \in \porth^{\dn - 1}$, there exists $j \in \SS_{\vec v}$ such that $w_j \leq v_j$.
  \end{claim}
  \begin{subproof}
    As $\norm{\Pmap_{\SS_{\vec v}}\vec w}^2 = \sum_{i\in \SS_{\vec v}} w_i^2 \leq 1 = \sum_{i \in \SS_{\vec v}} v_i^2$, it must hold that for some $j \in \SS_{\vec v}$, $w_j \leq v_j$.
    Otherwise, we would reverse the inequality, i.e.\@ $\sum_{i \in \SS_{\vec v}} w_i(n)^2 > \sum_{i \in \SS_{\vec v}} v_i^2 = 1$, which yields a contradiction.
  \end{subproof}

  \begin{claim}\label{claim:expand}
    Given a fixed $\eta > 0$, there exists a choice of $\Delta > 0$ such that the following holds:
    If $\vec w \in \porth^{\dn -1}$ satisfies that $w_i \neq 0$ for all $i \in \SS_{\vec v}$, that there exists $i \in \SS_{\vec v}$ such that $w_i > v_i$, and that $\max_{i, j \in \SS_{\vec v}} \frac{w_i/v_i}{w_j / v_j} \geq 1 + \eta$, then $\max_{i, j} \frac{\GG_i(\vec w)/v_i}{\GG_j(\vec w)/v_j} \geq (1+\Delta)\max_{i, j \in \SS_{\vec v}} \frac{w_i/v_i}{w_j / v_j}$.
  \end{claim}
  \begin{subproof}
    Using \cref{obs:stationary-criterion}, there exists $\lambda$ such that $\lambda = h'_i(v_i^2)$ for each $i \in \SS_{\vec v}$.
    Since $h'_i$ is strictly increasing on $\pinterval$ for each $i \in [\dm]$, there exists a $\Delta > 0$ satisfying the following for each $i \in \SS_{\vec v}$:
    \begin{compactenum}
    \item Whenever $x > v_i + \eta/4$, then $\frac{h_i'(x^2)}{\lambda} > 1+\Delta$ for each $i \in \SS_{\vec v}$.
    \item Whenever $x < v_i - \eta/4$, then $\frac{h'_i(x^2)}{\lambda} < \frac 1 {1+\Delta}$ for each $i \in \SS_{\vec v}$.
    \end{compactenum}
    Further, whenever $\frac{w_i/v_i}{w_j/v_j} \geq 1+\eta$, either $w_i > v_i + \eta/4$ or $w_j < v_j - \eta / 4$ holds.
    This can be seen by arguing via the contrapositive:  If neither condition holds, then
  \begin{displaymath}
    \frac {w_k/v_k}{w_\ell/v_\ell} \leq \frac{1 + \eta/4}{1 - \eta/4}
    = 1 + \frac{\eta/2}{1 - \eta/4}
    < 1 + \eta \ ,
  \end{displaymath}
  where the last inequality uses that $1 - \eta/4 < \frac 1 2$.

  Choosing $(i, j) = \argmax_{i, j \in \SS_{\vec v}} \frac{ w_i/v_i}{w_j / v_j}$, we write:
  \begin{displaymath}
    \frac {\GG_i(\vec w)/v_i}{\GG_j(\vec w)/v_k}
    = \frac{h'_i(w_i^2)  w_i/v_i}{h'_k(w_j^2)  w_j/v_k}
    = \frac{h_i'(w_i^2) / \lambda}{h_j'(w_j^2)/\lambda} \cdot \frac{w_i/v_i}{w_j/v_j} \ .
  \end{displaymath}
  But by the construction of $\Delta$, we see that one of $h_i'(w_i^2)/\lambda > 1+\Delta$ or $[h_j'(w_j^2)/\lambda]^{-1} > 1+\Delta$.
  Using that $h_i'$ is strictly increasing we obtain that $h_i'(w_i^2)/\lambda \geq 1+\Delta$ and $[h_j'(w_j^2)/\lambda]^{-1} \geq 1$.
  Combining these results yields
  \begin{displaymath}
    \frac {\GG_i(\vec w)/v_i}{\GG_j(\vec w)/v_k} > (1+\Delta) \frac{w_i/v_i}{w_j/v_j} \ . \qedhere
  \end{displaymath}
  \end{subproof}

  \begin{claim}\label{claim:expand-recursed}
    Suppose there exists $i_0 \in \SS_{\vec v}$ such that $u_{i_0}(0) > v_{i_0}$.
    Then there exists $\Delta > 0$ such that the following holds:
    Defining $M_n := \max_{i, j \in \SS_{\vec v}}\frac{u_i(n)/v_i}{u_j(n)/v_j}$,
    then $M_n \geq (1+\Delta)^n$.
  \end{claim}
  \begin{subproof}
    Setting $\eta = u_{i_0}/v_{i_0} - 1$, we construct $\Delta$ as in \cref{claim:expand}.
    We define $i_n := \argmax_{i\in \SS_{\vec v}} u_i(n) / v_i$ and $j_n := \argmin_{j \in \SS_{\vec v}} u_j(n) / v_j$.

    We proceed by induction on $n$ with the following inductive hypothesis:
    For all $n \in \N$, $M_n \geq (1+\eta)(1+\Delta)^n$ and $u_{i_n}(n) \geq v_{i_n}$.

    \noindent
    \textbf{Base case $n=0$.}
    By \cref{claim:lt-fixed-coord}, there exists $j \in \SS_{\vec v}$ such that $u_j(0) \leq v_j$.
    Thus, $u_{j_0} \leq v_j$.
    It follows that $\frac{u_{i_0}(0)/v_{i_0}}{u_{j_0}(0)/v_{j_0}} \geq 1+\eta$.
    Note that $u_{i_0}(0)/v_{i_0} \geq 1 + \eta$ by the construction of $\eta$.

    \noindent
    \textbf{Inductive case.}
    We assume the inductive hypothesis for $n$
    We apply \cref{claim:expand} to see the final inequality in:
    \begin{displaymath}
      \frac{u_{i_{n+1}}(n+1)/v_i}{u_{j_{n+1}}(n+1)/v_j}
      \geq \frac{\GG_{i_n}(\vec u(n))/v_{i_n}}{\GG_{j_n}(\vec u(n))/v_{i_n}}
      > (1+\Delta) \frac{u_{i_{n}}(n)/v_i}{u_{j_{n}}(n)/v_j} \ .
    \end{displaymath}
    To see that $u_{i_{n+1}}(n+1) > v_i$, we use that $h_i'(v_i^2) = \lambda$, strict monotonicity of the $h_i'$s, and \cref{lem:increasing-maxhi'} to see that $\max_{i \in \SS_{\vec v}} h_i'(u_i(n+1)^2) \geq h_{i_n}'(u_{i_n}(n)^2) > h_{i_n}'(v_{i_n}^2) = \lambda$.
    It follows that there exists $i \in \SS_{\vec v}$ such that $u_i(n+1) > v_i$, and in particular $u_{i_{n+1}}(n+1) > v_{i_{n+1}}$.
  \end{subproof}
  Note that as a consequence of \cref{claim:expand-recursed}, $\min_{i \in \SS_{\vec v}} u_i(n) \rightarrow 0$ as $n \rightarrow \infty$.
  Choose $\epsilon > 0$ according to \cref{cor:zero-coord-attraction}.
  There exists $j \in \SS_{\vec v}$ and $N > 0$ such that $u_j(N) < \epsilon$.
  Applying \cref{cor:zero-coord-attraction} on the sequence $\{ \vec u(n+N) \}_{n=0}^\infty$, we obtain that $u_j(n+N) \leq \frac 1 {2^n} \epsilon$ for all $n \in \N$, and in particular $u_j(n) \rightarrow 0$ as $n \rightarrow \infty$.
\end{proof}
\ifnum\version=\siamversion
\else \setcounter{thm}{\value{thmsavetmp}} \fi

\subsubsection{Almost everywhere attraction of the hidden basis}\label{sec:ae-attraction}
In this subsection, we demonstrate that given a generic starting point $\vec u(0)$, then $\vec u(n) \rightarrow \myZv_i$ as $n \rightarrow \infty$ for some $i$.
Our proof relies on the theory of stable-unstable manifolds of dynamical systems.
Before proceeding, we first review the results from stable-unstable manifold theory that we require.

Given a linear operator $T : \R^k \rightarrow \R^k$, its eigenspace may be decomposed into several subspaces.
Denote by $(\lambda_1, \vec v_1), \dotsc, (\lambda_k, \vec v_k)$ the eigenvalue-eigenvector pairs for $T$, and define the subspaces:
\begin{align*}
  \mathcal E^S(T) &:= \spn\{\vec v_i \suchthat \abs{\lambda_i}< 1 \} \\
  \mathcal E^U(T) &:= \spn\{\vec v_i \suchthat \abs{\lambda_i}> 1 \} \\
  \mathcal E^C(T) &:= \spn\{\vec v_i \suchthat \abs{\lambda_i}= 1 \} \ .
\end{align*}
$\mathcal E^S(T)$, $\mathcal E^U(T)$, and $\mathcal E^C(T)$ are called
the stable, unstable, and center subspaces of the linear operator $T$.
Each subspace name captures a property of the fixed point $\vec 0$:
it is an attractor on the stable subspace, a repellor on the unstable
subspace, and neither on the center subspace.

Let $\vec x^*$ be a fixed point of a general (non-linear) discrete dynamical system $f$.
When $\Jacob f(\vec x^*)$ has no center subspace (or alternatively when $\Jacob f(\vec x^*)$ can be decomposed as $\mathcal E^S(\Jacob f(\vec x^*)) \oplus \mathcal E^U(\Jacob f(\vec x^*))$), then $\vec x^*$ is said to be a \textit{hyperbolic fixed point} of $f$.
For a hyperbolic fixed point of a discrete dynamical system, the dimensionality of the space on which the dynamical system converges to $\vec x^*$ is locally determined by the dimensionality of $\mathcal E^S(\Jacob f(\vec x^*))$.

More precisely, letting choices of $\vec x(0)$ implicitly define sequences $\{\vec x(n)\}_{n=0}^\infty$ recursively by $\vec x(n) = f(\vec x(n-1))$, the locally stable manifold of a fixed point $\vec x^*$ is defined as follows.
\begin{defn}\label{defn:LSM}
  Within a neighborhood $U$ of $\vec x^*$, the manifold 
  \begin{displaymath}
    \mathcal L_{loc}(\vec x^*) := \{ \vec x(0) \in U \suchthat \lim_{k \rightarrow \infty} \vec x(k) = \vec x^*, \vec x(k) \in U \ \forall k \in \mathbb N \}
  \end{displaymath}
  is called the \emph{local stable manifold}.
\end{defn}

The following result is a special case of Theorem 2.2 of~\mycitet{Luo}{luo2012regularity}.
\begin{thm}\label{thm:LSM-dimensionality}
  Let $f : \mathcal X \rightarrow \mathcal X$ be a discrete dynamical system with a hyperbolic fixed point $\vec x^*$ such that $f$ is continuously differentiable on a neighborhood of $\vec x^*$.
  Then, $\dim(\mathcal L_{loc}(\vec x^*)) = \dim( \mathcal E^S(\Jacob f(\vec x^*)) )$.
  Further, there exists $\delta > 0$ such that for all $\vec x(0) \not \in L_{loc}$, there exists $N \in \mathbb N$ such that $\norm{\vec x(N) - \vec x^*} > \delta$.
\end{thm}

We now proceed in arguing that for a sequence $\{\vec u(n)\}_0^{\infty}$ in $\porth^{\dn - 1}$ the gradient iteration for a {\psef} converges to one of the odeco basis elements $\myZv_i$ given almost any starting point.
We first demonstrate (in \cref{lem:crit-point-eigendecomp} below) that the fixed points of $\GG$ are hyperbolic except for the odeco basis directions.
As a direct implication, locally to any fixed point $\vec v$ of $\GG$ besides the hidden basis elements, the local stable manifold $M$ of $\vec v$ is not of full dimension, making it so that locally the gradient iteration is repulsive except on the 0-measure set $M$ (\cref{lem:GG-LSM}).
We later build up global convergence properties from these local results.

In what follows, we will make use of $T_{\vec v} S^{\dn - 1}$ the tangent space (or tangent plane) of the sphere $S^{\dn - 1}$ at $\vec v$ with $\vec v$ treated as the origin.
This may alternatively be defined as $T_{\vec v} S^{\dn - 1} := \vec v^\perp = \{ \vec u \in \R^{\dn} \suchthat \vec u \perp \vec v \}$.

\begin{lem}[Hyperbolicity of fixed points]\label{lem:crit-point-eigendecomp}
  Let $\vec v \in \porth^{\dn - 1}$ be a fixed point of $\GG$ and suppose that $\SS_{\vec v} := \{ i \suchthat v_i \neq 0\}$ is contained in $[\dm]$.
  Let $\phi : T_{\vec v} S^{\dn - 1} \rightarrow S^{\dn - 1}$ be the exponential%
  \footnote{The exponential map for a point on the sphere $\exp_{\vec v}:T_{\vec v}S^{\dn - 1} \rightarrow S^{\dn - 1}$ is defined by $\exp_{\vec v}(\vec x) = \vec v \cos(\norm{\vec x}) + \frac{\vec x}{\norm{\vec x}}\sin(\norm{\vec x})$.
      For our purposes, we only use that $\exp_{\vec v}$ is a coordinate system $\phi$ of $S^{\dn - 1}$ containing $\vec v$ such that $\Jacob \phi_{\vec v} = \Jacob \phi^{-1}_{\vec v} = \Pmap_{\vec v^\perp}$.}
  map.
  We let $R = \range(\Pmap_{\SS_{\vec v}})$ and $K = \range(\Pmap_{\compl \SS_{\vec v}})$.
  Then, $\Jacob[\phi \circ \GG \circ \phi^{-1}]_{\phi(\vec v)}$ is a symmetric matrix which satisfies:
  \begin{compactenum}
  \item $[\Jacob[\phi \circ \GG \circ \phi^{-1}]_{\phi(\vec v)}]\restr K$ is the 0 map.
  \item $[\Jacob[\phi \circ \GG \circ \phi^{-1}]_{\phi(\vec v)} - \Id]\restr{R \cap \vec v^\perp}$ is strictly positive definite.
    In particular, there exists $\lambda > 0$ such that for any $\vec w \in R \cap \vec v^\perp$, $\vec w^T[\Jacob[\phi \circ \GG \circ \phi^{-1}]_{\vec v} - \Pmap_{\SS}]\vec w \geq \lambda$.
  \end{compactenum}
\end{lem}
\begin{proof}
  We expand the formula $\Jacob[\phi \circ \GG \circ \phi^{-1}]_{\vec v}$ to obtain:
  \begin{equation}\label{eq:Jacob-GG-v-local-expand}
    \Jacob[\phi \circ \GG \circ \phi^{-1}]_{\vec v}
    = \Jacob \phi_{\GG\circ \phi^{-1}(\vec v)} \Jacob \GG_{\phi^{-1}(\vec v)} \Jacob \phi^{-1}_{\vec v}
    = \Pmap_{\vec v^\perp} \Jacob \GG_{\vec v} \Pmap_{\vec v^\perp}
  \end{equation}
  Since $\GG(\vec u) = \frac{\nabla F(\vec u)}{\norm{\nabla F(\vec u)}}$, the Jacobian of $\GG$ is
  \begin{equation}
    \label{eq:Jacob-G-in-F}
    \Jacob \GG_{\vec u} =
    \frac{\HH F(\vec u)}{\norm{\nabla F(\vec u)}} - \frac{\nabla F(\vec u)\nabla F(\vec u)^T\HH F(\vec u)}{\norm{\nabla F(\vec u)}^3}
    = \frac{\Pmap_{\GG(\vec u)^{\perp}}\HH F(\vec u)}{\norm{\nabla F(\vec u)}} \ .
  \end{equation}
  As $\vec v$ is a fixed point of $\GG$, \cref{eq:Jacob-G-in-F} implies that $\Jacob \GG_{\vec v} = \frac {\Pmap_{\vec v^\perp} \HH F(\vec v)}{\norm{\nabla F(\vec v)}}$.
  As such, \cref{eq:Jacob-GG-v-local-expand} becomes
  \begin{displaymath}
    \Jacob[\phi \circ \GG \circ \phi^{-1}]_{\vec v}
    = \frac 1 {\norm{\nabla F(\vec v)}} \Pmap_{\vec v^\perp} \HH F(\vec v) \Pmap_{\vec v^\perp}
  \end{displaymath}
  which is a symmetric map.

  Since $\vec v = \GG(\vec v) = \frac{\nabla F(\vec v)}{\norm{\nabla F(\vec v)}} = \frac{\sum_{i \in \SS}2 h_i'(v_i^2) v_i \myZv_i}{\norm{\nabla F(\vec v)}}$, we see that $2 h_i'(v_i^2) = \norm{\nabla F(\vec v)}$ for each $i \in \SS$.
  Expanding $\HH F(\vec v)$, we thus obtain:
  \begin{displaymath}
    \frac{\HH F(\vec v)}{\norm{\nabla F(\vec v)}}
    = \sum_{i \in \SS}\frac{4h_i''(v_i^2)v_i^2 + 2h_i'(v_i^2)}{\norm{\nabla F(\vec v)}}\myZv_i \myZv_i^T
    = \sum_{i \in \SS}\frac{4h_i''(v_i^2)v_i^2}{\norm{\nabla F(\vec v)}}\myZv_i \myZv_i^T + \Pmap_{\SS} \ .
  \end{displaymath}
  Notice that the first summand is strictly positive definite on $\range(\Pmap_\SS)$, and that the second term is the identity map on $\range(\Pmap_\SS)$.
  Careful inspection of the resulting equation
  \begin{displaymath}
    \Jacob[\phi \circ \GG \circ \phi^{-1}]_{\vec v}
    = \Pmap_{\vec v^\perp}\Big[ \sum_{i \in \SS}\frac{4h_i''(v_i^2)v_i^2}{\norm{\nabla F(\vec v)}}\myZv_i \myZv_i^T + \Pmap_{\SS} \Big]\Pmap_{\vec v^\perp}
  \end{displaymath}
  gives all of the claimed results.
  In particular if $\vec x \in K$, we note that $\vec x \in \vec v^\perp$ and $\vec x \perp \vec e_i$ for each $i \in \SS$; thus, $\Big[ \sum_{i \in \SS}\frac{4h_i''(v_i^2)v_i^2}{\norm{\nabla F(\vec v)}}\myZv_i \myZv_i^T + \Pmap_{\SS} \Big]\Pmap_{\vec v^\perp} \vec x = 0$.
  Further, for a non-zero $\vec x \in R \cap \vec v^\perp$, we have that the non-zero coordinates of $\vec x$ are contained in $\SS$.
  Thus, using that the coefficients $4 h_i''(v_i^2)v_i^2$ are strictly positive for $i \in \SS$, we obtain:
  \begin{displaymath}
    \vec x^T [\Jacob[\phi \circ \GG \circ \phi^{-1}]_{\vec v} - \Id ]\vec x
    = \vec x^T \Big[ \sum_{i \in \SS}\frac{4h_i''(v_i^2)v_i^2}{\norm{\nabla F(\vec v)}}\myZv_i \myZv_i^T\Big] \vec x  > 0 \ . \qedhere
  \end{displaymath}
\end{proof}

\begin{lem}[Local stable manifold]\label{lem:GG-LSM}
  Suppose that $\vec v \in\porth^{\dn - 1}$ is a stationary point of $\GG$.
  Let $\SS_{\vec v} = \{ i \suchthat v_i \neq 0\}$.
  Suppose that $\SS_{\vec v} \subset [\dm]$.
  In a neighborhood $U$ of $\vec v$ on $S^{\dn - 1}$, there is a manifold $M_K \subset U$ such that
  \begin{compactenum}
    \item\label{lem:GG-LSM-fp} $\vec v \in M_K$.
    \item\label{lem:GG-LSM-dim} $\dim(M_K) = \dim(\range(\Pmap_{\compl S})) = \dn - \abs {\SS_{\vec v}}$.
    \item\label{lem:GG-LSM-divergence} There exists a $\delta > 0$ such that if $\vec u(0) \in U \setminus M_K$, then for some $N \in \N$, $\norm{\vec u(N) - \vec v} \geq \delta$.
    \item\label{lem:GG-LSM-attraction} If $\vec u(0) \in M_K$, then $\vec u(N) \rightarrow \vec v$ as $n \rightarrow \infty$.
  \end{compactenum}
\end{lem}
In \cref{lem:GG-LSM}, $M_K$ is called the local stable manifold of $\vec v$.
\begin{proof}[\proofword of \cref{lem:GG-LSM}]
Notice in \cref{lem:crit-point-eigendecomp}, $K:= \RR(\Pmap_{\compl \SS})$ is the 0-eigenspace of $[\Jacob[\phi \circ \GG \circ \phi^{-1}]_{\vec v}]\restr {v^\perp}$, and $R := \RR(\Pmap_{\SS})$ is the span of non-zero eigenvectors of $[\Jacob[\phi \circ \GG \circ \phi^{-1}]_{\vec v}]\restr {v^\perp}$, with each eigenvalue of $R$ being strictly greater than 1.
Further, $\dim(K) = d - \abs{\SS_{\vec v}}$.

Applying \cref{thm:LSM-dimensionality}, we obtain the existence of a locally stable manifold $M_K$ for the discrete dynamical system $\GG$ with $\dim(M_K) = \dim(K) = (d-1)$ (that is property~\ref{lem:GG-LSM-dim}).
The construction from \cref{thm:LSM-dimensionality} also implies that $M_K$ satisfies the properties~\ref{lem:GG-LSM-fp}, \ref{lem:GG-LSM-divergence}, and~\ref{lem:GG-LSM-attraction}.
%That $\vec v \in M_K$ follows from the definition of a locally stable manifold (\cref{defn:LSM}) and that $\vec v$ is a fixed point of $\GG$.
  % We apply Theorem 2.2 of the book ``Regularity and Complexity in Dynamical Systems'' by Albert C. J. Luo on the function $\phi \circ \GG \circ \phi^{-1}$.
  % This can be done due to \cref{lem:crit-point-eigendecomp}.
  % To set $\delta$, we use that the open sets in $\range(\phi)$ and use that $\phi^{-1}$ is a homeomorphism onto the metric space defined by $d(\vec x, \vec y) = \norm{\vec x - \vec y}$ for the embedded sphere $S^{\dn - 1}$.
\end{proof}

In \cref{lem:GG-LSM}, $\dim(M_K) = d-\abs{\SS_{\vec v}}$ implies a number of things.
If $\vec v$ is one of the hidden basis elements $\myZv_i$, then $\abs {\SS_{\myZv_i}} = 1$ implies that $\dim(M_K) = \dim(S^{\dn-1})$.
In this case, $M_K$ is an open neighborhood of $\myZv_i$.
Thus, the hidden basis elements are stable attractors.
\begin{prop}\label{prop:ei-attractors}
    % Let $\fg$ be a simplex encoding function with positive associated constants $\alpha_i, \beta_i$.
    % Let $\GG$ be constructed from $\fg$ according to \cref{eq:grad-it-def}.
    % Then,
    The directions $\myZv_1, \dotsc, \myZv_\dm$ are attractors of
    $\GG \restr{\porth^{\dn - 1}}$ .
\end{prop}

Also under \cref{lem:GG-LSM}, if $\vec v \not \in \{\myZv_1, \dotsc, \myZv_\dm\}$, then $\abs {\SS_{\vec v}} \geq 2$ and $\dim(M_K) \leq \dn - 2$.
In this case, $M_K$ has volume measure 0 on the sphere's surface, and in particular $\vec v$ is an unstable fixed point of $\GG$.

We now wish to demonstrate that the set $\mathcal X := \{ \vec u(0) \in S^{\dn - 1} \suchthat \vec u(n) \rightarrow \vec v \text{ as }n\rightarrow \infty\}$ has measure 0 globally on $S^{\dn - 1}$.
We will proceed first in the setting in which $\dn = \dm$.
In this setting, we will see that $\GG^{-1}$ is a well defined function which maps measure 0 sets to measure 0 sets (\cref{lem:G-bijective} and \cref{lem:Ginv-measure-0-maps}).
Using that $\mathcal X$ can alternatively be viewed as the set of preimages of $M_K$ under repeated application of $\GG^{-1}$ we will obtain that $\vol(\mathcal X) = 0$ as desired (\cref{thm:vol0-unstable-pt-global}).

\begin{lem}\label{lem:G-bijective}
  Suppose that $\dn = \dm$ and that $F$ is a \psef.  Then $\GG : S^{\dn - 1} \rightarrow S^{\dn - 1}$ is a continuous bijection.
\end{lem}
\begin{proof}
%  By \cref{prop:gi-sym-classes}, we may simplify notation by assuming without loss of generality that $F$ is a {\psef}, thus making
 Since $F(\vec u) = \sum_{i=1}^\dn h_i(u_i^2)$, we obtain
  \begin{equation}\label{eq:G-IFT-procondition:1}
    \GG(\vec u) = \frac{\sum_{i=1}^\dn 2h_i'(u_i^2)u_i \vec e_i}{\norm{\nabla F(\vec u)}} \ .
  \end{equation}
  To see that $\GG$ is continuous, we note that \cref{lem:normF-lower-bound} implies that $\norm{\nabla F(\vec u)} \neq 0$ on its entire domain (since $\dn = \dm$).
  As both the numerator and denominator of \cref{eq:G-IFT-procondition:1} are continuous, $\GG$ is continuous.

  To see that $\GG$ is one-to-one, we fix $\vec x, \vec y \in S^{\dn - 1}$ and suppose that $\GG(\vec x) = \GG(\vec y)$.
  Then, $\GG(\vec x) = \GG(\vec y)$ implies that $2h_i'(x_i^2)x_i \propto 2h_i'(y_i^2)y_i$, and in particular there exists $\lambda > 0$ (positive since $\GG$ is orthant preserving by \cref{prop:gi-sym-classes}) such that $2h_i'(x_i^2)x_i = \lambda 2h_i'(y_i^2)y_i$ for all $i \in S^{\dn - 1}$.
  If $\lambda < 1$, then $\abs{h_i'(x_i^2)x_i} < \abs{h_i'(y_i^2)y_i}$ implies (by monotonicity of each $h_i'$) that $x_i^2 < y_i^2$ for all $i$, which contradicts that $\norm {\vec x}^2 = \norm{\vec y}^2 = 1$.
  Similarly, it cannot happen that $\lambda > 1$.
  Thus, $\lambda = 1$, and that $\vec x = \vec y$.

  We now argue that $\GG$ is onto.
  Fix $\vec u \in S^{\dn - 1}$.  We will show that there exists $\vec w$ such that $\GG(\vec w) = \vec u$.
  By the symmetries of the problem, we may assume without loss of generality that $u_i \geq 0$ for each $i \in [\dn]$.

  We let $\alpha_1, \dotsc, \alpha_\ell$ be an enumeration of $\SS := \{ i \suchthat u_i \neq 0\}$.
  For each $k \in [\ell]$, we define $\Gamma^{(k)} : (0, 1] \rightarrow \R^k$ by $\Gamma^{(k)}(C) = (x_1, \dotsc, x_k)$ such that $h_{\alpha_i}'(x_i^2)x_i / (h_{\alpha_j}'(x_j^2)x_j) = u_{\alpha_i} / u_{\alpha_j}$ for each $i, j \in [k]$, $\norm{\Gamma^{(k)}(C)} = C$, and $x_i > 0$ for all $i \in [k]$.
  We proceed by induction on $k$ in proving that $\Gamma^{(k)}$ is well defined.
  In the base case, $\Gamma^{(1)}(C) = (C)$.
  We now consider the inductive step.

  Suppose the inductive hypothesis holds for $k$.
  Define $\beta(C, t) := (\Gamma^{(k)}(\sqrt{C-t^2}), t)$.
  As the functions $x \mapsto h_i'(x^2)x$ are continuous and strictly increasing from 0 when $x_i = 0$, it follows that
  \begin{displaymath}
    \rho_C(t) := \frac{h_{\alpha_{k+1}}'(\beta_{k+1}(C, t)^2)\beta_{k+1}(C, t)}{h_{\alpha_{k}}'(\beta_{k}(C, t)^2)\beta_{k}(C, t)}
  \end{displaymath}
  satisfies $\lim_{t \rightarrow 0^+} \rho_C(t) = 0$ and $\lim_{t\rightarrow C^+} \rho_C(t) = +\infty$.
  Since $\rho_C$ is a continuous function on $(0, C)$, there exists $t_0 \in (0, C)$ such that $\rho_C(t_0) = \frac{u_{\alpha_{k+1}}}{u_{\alpha_k}}$.
  In particular, defining $\Gamma^{(k+1)}(C) = (\Gamma^{(k)}(\sqrt{C - t_0^2}), t_0)$ according to this construction, it can be verified that $\norm{\Gamma^{(k+1)}(C)} = C$ and that
  \begin{displaymath}
  \frac{h_{\alpha_i}'(\Gamma^{(k+1)}_i(C)^2)\Gamma^{(k+1)}_i(C)}
       {h_{\alpha_j}'(\Gamma^{(k+1)}_j(C)^2)\Gamma^{(k+1)}_j(C)}
  = \frac{u_{\alpha_i}}{u_{\alpha_j}}
  \end{displaymath}
%  $\Gamma^{(k+1)}_i(C) / \Gamma^{(k+1)}_j(C) = u_{\alpha_i} / u_{\alpha_j}$
  for all $i, j \in [k]$ as desired.

  By construction, $\GG\big(\sum_{i = 1}^\ell \Gamma_i^{(\ell)}(1) \myZv_{\alpha_i} \big) = \vec u$.
%  In particular, $\GG$ is onto.
\end{proof}

\begin{lem}\label{lem:G-IFT-precondition}
  Let $A := \{ \vec u \in S^{\dn - 1} \suchthat u_i \neq 0 \text{ for all } i \in [\dn] \}$.
  If $\dn = \dm$ and if $F$ is a {\psef}, then $\GG$ has the following properties:
  \begin{compactenum}
    \item\label{it:G-IFT-precondition:A-range} $\GG(A) = A$.
    \item\label{it:G-IFT-precondition:A-Jacob} For all $\vec p \in A$, $\Jacob \GG_{\vec p} : T_{\vec p} S^{\dn - 1} \rightarrow T_{\GG(\vec p)} S^{\dn - 1}$ is full rank (invertible).
    \item\label{it:G-IFT-precondition:Acompl-range} $\GG(\compl A) = \compl A$.
  \end{compactenum}
\end{lem}
\begin{proof}
  We first prove parts~\ref{it:G-IFT-precondition:A-range} and~\ref{it:G-IFT-precondition:Acompl-range}.
  Since each $h_i'(u_i^2)u_i = 0$ if and only if $u_i = 0$ (by \cref{lem:h-props} and by anti-symmetry of $h_i'$), it follows from \cref{eq:G-IFT-procondition:1} both that $\vec u \in A$ implies $\GG(\vec u)\in A$ and that $\vec u \in \compl A$ implies $\GG(\vec u) \in \compl A$.
  Thus, $\GG(A) \subset A$ and $\GG(\compl A) \subset \compl A$.
  Since $\GG(S^{\dn - 1}) = S^{\dn - 1}$ (by \cref{lem:G-bijective}), it follows that $\GG(A) = A$ (since otherwise, $A \not \subset \GG(A)$ and $A \cap \GG(\compl A) = \emptyset$ implies that $A \not \subset \GG(A \cup \compl A) = \GG(S^{\dn-1})$).
  By similar reasoning, $\GG(\compl A) = \GG(\compl A)$.

  We now prove part~\ref{it:G-IFT-precondition:A-Jacob}.
  Fix $\vec p \in A$. % and demonstrate that $\Jacob \GG_{\vec p}$ is full rank.
  Without loss of generality, we assume that $p_i > 0$ for all $i \in [\dn]$.
  Fix a non-zero $\vec x \in T_{\vec p} S^{\dn - 1}$.
  Since $\ip{\vec p}{\vec x} = \sum_{i\in \dn} p_ix_i = 0$ and $\vec x \neq 0$, there exists $j, k \in [\dn]$ such that $x_j < 0$ and $x_k > 0$.
  Note that
  \begin{displaymath}
    \Jacob \GG(\vec p) = \frac{\Pmap_{\GG(\vec p)^{\perp}}\HH F(\vec p)}{\norm{\nabla F(\vec p)}} = \frac{1}{\norm{\nabla F(\vec p)}} \Pmap_{\GG(\vec p)^{\perp}} \sum_{i=1}^\dn[4 h_i''(p_i^2)p_i^2 + 2 h_i'(p_i^2)] \vec e_i \vec e_i^T
  \end{displaymath}
  satisfies (by \cref{lem:h-props,assumpt:convex}) that each $[\HH F(\vec u)]_{ii} > 0$; it follows that $[\HH F(\vec u) \vec x]_j < 0$ and $[\HH F(\vec u) \vec x]_k > 0$.
  Since $\GG(\vec p) \in A$ satisfies $\GG_i(\vec p) > 0$ for all $i \in [\dn]$, we see that $\HH F(\vec u) \vec x \nparallel \GG(\vec p)$, and thus $\Jacob \GG(\vec p)\vec x \neq \vec 0$.
\end{proof}

\begin{lem}\label{lem:Ginv-measure-0-maps}
  Suppose $\dn = \dm$ and that $B \subset S^{\dn - 1}$ has volume measure 0.
  Then, $\vol(\GG^{-1}(B)) = 0$.
\end{lem}
\begin{proof}
  We let the set $A$ be as in \cref{lem:G-IFT-precondition}. %, and we let $\vol$ denote the volume measure on $S^{\dn - 1}$.
  % We divide $B$ into three sets:  $B_1 := \GG(A) \cap B$, $B_2 := (A \setminus \GG(A)) \cap B$, and $B_3 := \compl A \cap B$.
  % Since $\GG \restr A$ maps into $A$, it follows that $B = B_1 \sqcup B_2 \sqcup B_3$.
  Since $\GG(\compl A) = \compl A$ (by \cref{lem:G-IFT-precondition}), then $\GG^{-1}( \compl A ) = \compl A$.
  In particular, $\GG^{-1}(B\cap \compl A) \subset \compl A$ implies that $\vol(\GG^{-1}(B \cap \compl A) \leq \vol(\compl A) = 0$.
  % \cref{lem:G-IFT-precondition} further implies that $\GG^{-1}(\compl A) \subset \compl A$.
  % As $B_3 \subset \compl A$, it follows that $\vol(\GG^{-1}(B_3)) \leq \vol(\compl A) = 0$.

  On the open set $A = \GG(A)$, \cref{lem:G-IFT-precondition} combined with the inverse function theorem implies that $\GG^{-1}$ exists and is continuously differentiable function.
  As $B \cap A \subset A$ is a measure 0 set, \cref{thm:continuous-f-measure0} implies that $\GG^{-1}(B \cap A)$ is a measure 0 set by using an appropriate choice of coordinate atlas for $S^{\dn - 1}$.
  For instance, we fix $\vec p \in \compl A$ and let
  $\phi : \R^{\dn - 1}\rightarrow S^{\dn - 1} \setminus \{\vec p\}$ denote the coordinates arising from the stereographic projection through $\vec p$.
  Then, consider the map $\phi^{-1} \circ \GG^{-1} \circ \phi : \phi^{-1}(A) \rightarrow \phi^{-1}(A)$.
  As the canonical Riemannian metric on the sphere has everywhere positive determinant, $\vol(B\cap A) = 0$ implies that $\phi^{-1}(B\cap A)$ has Lebesgue measure 0.
  By \cref{thm:continuous-f-measure0}, it follows that $\phi^{-1}(\GG^{-1}(B\cap A)) = \phi^{-1} \circ \GG^{-1} \circ \phi( \phi^{-1}(B \cap A) )$ has Lebesgue measure 0, and hence $\vol(\GG^{-1}(B\cap A)) = 0$.

  Combining these results, we see $\vol(\GG^{-1}(B)) = \vol(\GG^{-1}(B\cap A)) + \vol(\GG^{-1}(B\cap \compl A)) = 0$.
\end{proof}

% We now demonstrate in \cref{thm:vol0-unstable-pt-global} that convergence to the unstable fixed points of $S^{\dn - 1}$ only occurs on from a starting  set of measure 0 in the special case that $\dn = \dm$.

\begin{thm}\label{thm:vol0-unstable-pt-global}
  Suppose that $\dn = \dm$.
  Let $\SS \subset [\dm]$ be such that $\abs \SS \geq 2$, and let $\vec v$ be the stationary point of $\GG$ such that $v_i \neq 0$ if and only if $i \in \SS$.
  Define $\mathcal X_{\vec v} := \{ \vec u(0) \in S^{\dn - 1} \suchthat \vec u(n) \rightarrow \vec v \text{ as }n\rightarrow \infty\}$.
  The set $\mathcal X_{\vec v}$ has volume measure 0 on $S^{\dn - 1}$.
\end{thm}
\begin{proof}
    In this proof, we denote repeated applications of the gradient iteration and its inverse by
  \begin{align*}
    \GG \rcomp k &= \underbrace{\GG \circ \cdots \circ \GG}_{\text{$k$ times}} &
    & \text{and} &
    \GG \rcomp {-k} &= \underbrace{\GG^{-1} \circ \cdots \circ \GG^{-1}}_{\text{$k$ times}}
  \end{align*}
  with $\GG \rcomp 0$ being the identity map.

  Let $U$, $M_K$, and $\delta > 0$ be as in \cref{lem:GG-LSM}.
  For each $\vec u(0) \in \mathcal X$, there exists $N > 0$ such that for all $n \geq N$, $\vec u(n) \in U \cap B(\vec v, \delta)$.
  \Cref{lem:GG-LSM} implies that $\vec u(n) \in M_K$ for all $n \geq N$.
  In particular, it follows that $\vec u(0) \in \GG \rcomp{-n}(M_K)$ for all $n \geq N$.
  As such, $\mathcal X_{\vec v} \subset \bigcup_{n = 0}^{\infty} \GG \rcomp{-n}(M_K)$.

%  Let $\vol$ be the volume measure on $S^{\dn - 1}$.
  Since $\vol(M_K) = 0$, \cref{lem:Ginv-measure-0-maps} implies $\vol(\GG \rcomp{-n}(M_K)) = 0$ for all $n \in \N$.
  As such, $\vol( \mathcal X ) \leq \vol(\bigcup_{n=0}^\infty \GG \rcomp{-n}(M_K)) \leq \sum_{n = 0}^\infty \vol(\GG \rcomp{-n}(M_K)) = 0$.
\end{proof}

We now proceed in showing (in the case where $\dn = \dm$) that for almost any starting point $\vec u(0) \in \porth^{\dn - 1}$, there exists $i \in [\dn]$ such that $\vec u(n) \rightarrow \myZv_i$ as $n \rightarrow \infty$.
The essential ingredients are the preceding measure 0 argument from \cref{thm:vol0-unstable-pt-global} combined with \cref{prop:gi-expand-all,lem:unstable-sep} below.
In particular, \cref{thm:vol0-unstable-pt-global} implies non-convergence to the unstable fixed points of the dynamical system $\GG$ from almost any starting point,
\cref{prop:gi-expand-all} provides criteria under which coordinates of $\vec u(n)$ can be driven towards 0,
and \cref{lem:unstable-sep} below will serve as a bridge between the non-convergence to unstable fixed points of $\GG$ and the preconditions of \cref{prop:gi-expand-all} for demonstrating that all coordinates are driven to 0.
%Our result is stated more formally in \cref{thm:canonical-global-attraction}.

\begin{lem}\label{lem:unstable-sep}
  Let $\vec v \in \porth^{\dn - 1}$ be a fixed point of $\GG$, and let $\SS :=
  \{ i \suchthat v_i \neq 0 \}$.
  Let $\vec u \in \porth^{\dn - 1}$ be such that $\norm{\Pmap_{\compl \SS}
    \vec u} < \frac 1 2 \eta$ and such that $\norm{\vec u - \vec v} >
  \eta$.
  Then there exists $i \in \SS$ such that $u_i > (1+ \frac 1 4 \eta^2)v_i$ and $j \in \SS$.
\end{lem}

\begin{proof}
  Expanding $\norm{\vec u - \vec v}^2 > \eta^2$ yields $\norm{\vec u}^2 - 2
  \ip {\vec u}{\vec v} + \norm{\vec v}^2 > \eta^2$.
  Hence, $\sum_{i \in \SS} u_i v_i < 1 - \frac 1 2 \eta^2$ (since
  $\norm{\vec u}^2 = \norm{\vec v}^2 = 1$).
  Assume for the sake of contradiction that $u_i \leq (1+\epsilon) v_i$ for all
  $i \in \SS$ where $\epsilon \geq 0$ is arbitrary to be chosen later.
  Then, $\sum_{i \in \SS} u_i v_i \geq \frac{1}{1+\epsilon} \sum_{i \in \SS}
  u_i^2 = \frac 1 {1+\epsilon} (1-\norm{\Pmap_{\compl \SS} \vec u}^2)$.
  In particular, we obtain:
  \begin{align*}
    \frac 1 {1+\epsilon}(1 - \norm{\Pmap_{\compl \SS} \vec u}^2)
    & < 1 - \frac 1 2 \eta^2 \\
    1 &< (1+\epsilon)(1 - \frac 1 2 \eta^2) + \norm{\Pmap_{\compl \SS} \vec
      u}^2 \ .
  \end{align*}

  In particular, with the choices of $\epsilon < \frac 1 4 \eta^2$ and
  $\norm{\Pmap_{\compl \SS} \vec u}^2 < \frac 1 4 \eta^2$, we obtain that
  \begin{align*}
    1 &< (1+\epsilon)(1 - \frac 1 2 \eta^2) + \norm{\Pmap_{\compl \SS} \vec u}^2 \\
    &< (1 + \frac 1 4 \eta^2)(1-\frac 1 2 \eta^2) + \frac 1 4 \eta^2 \\
    &= 1 - \frac 1 4 \eta^2 - \frac 1 8 \eta^4 + \frac 1 4 \eta^2 < 1 \ ,
  \end{align*}
  which is a contradiction.
\end{proof}

\begin{thm}[Global attraction of the hidden basis]\label{thm:canonical-global-attraction}
  Suppose that $d = m$.
  There exists a set $\mathcal \XX \subset \porth^{\dn - 1}$ with $\vol(\mathcal X) = 0$ and the following property:
  If $\vec u(0) \in \porth^{\dn - 1} \setminus \mathcal \XX$, then there exists $i \in [\dm]$ such that $\vec u(n) \rightarrow \myZv_i$ as $n \rightarrow \infty$.
\end{thm}
\begin{proof}
  Let $\gvec \mu : 2^{[\dm]} \rightarrow \porth^{\dn - 1}$ (denoting by $2^{[\dm]}$ the power set of $[\dm]$) be the map which takes $\SS \subset [\dm]$ to $\gvec \mu(\SS)$ the stationary point of $\GG$ in $\porth^{\dn - 1}$ such that $\mu_i(\SS) \neq 0$ if and only if $i \in \SS$.
  We define $\mathcal \XX_{\gvec \mu(\SS)}$ as in \cref{thm:vol0-unstable-pt-global}.
  Let $\mathcal \XX := \bigcup \{\XX_{\gvec \mu(\SS)} \suchthat \SS \subset [\dm], \abs \SS \geq 2\}$.
  Using \cref{thm:vol0-unstable-pt-global}, we see that $\vol(\mathcal \XX) \leq \sum_{\SS \subset [\dm], \abs{\SS}\geq 2} \vol(\gvec \mu(\SS)) = 0$.
  It remains to be seen that $\vec u(0) \not \in \mathcal \XX$ implies the existence of $i \in [\dm]$ such that $\vec u(n) \rightarrow \myZv_i$ as $n \rightarrow \infty$.
  The main idea behind the proof is to demonstrate various coordinates of $\vec u(n)$ approach 0 until only one coordinate remains separated from $0$.
  We will recurse on the following Claim.
  \begin{claim}\label{claim:NCZ}
    Let $\SS \subset [\dm]$ be such that $\abs \SS \geq 2$.
    If $u_i(n) \rightarrow 0$ for all $i \in \compl \SS$ as $n \rightarrow \infty$,
    then there exists $j \in \SS$ such that $u_j(n) \rightarrow 0$ as $n \rightarrow \infty$.
  \end{claim}
  \begin{subproof}
    Fix $\vec v = \gvec \mu(\SS)$.
    Since $\vec u(0) \not \in \XX_{\vec v}$, there exists $\eta > 0$ and an infinite subsequence $n_0, n_1, n_2, n_3, \dotsc$ of $\N$ such that $\norm{\vec u(n_i) -\vec v} \geq \eta$ for each $i \in \N$.
    Further, since $\norm{\Pmap_{\compl \SS} \vec u(n)} \rightarrow 0$ as $n \rightarrow \infty$, there exists $N \in \N$ such that $\norm{\Pmap_{\compl \SS} \vec u(n)} \leq \frac 1 2 \eta$ for all $n \geq N$.
    Choose $i \in \N$ such that $n_i \geq N$.
    By \cref{lem:unstable-sep}, there exists $j \in \SS$ such that $u_j(n_i) > v_j$.
    Thus, \cref{prop:gi-expand-all} implies the existence of $k \in \SS$ such that $u_k(n) \rightarrow 0$ as $n \rightarrow \infty$.
  \end{subproof}
  We set $\SS_0 = [\dm]$.
  Using \cref{claim:NCZ}, we see that there exists $i \in [\dm]$ such that $u_i(n) \rightarrow 0$ as $n \rightarrow \infty$.
  We construct $\SS_1 = \SS_0 \setminus \{ i \}$.

  By repeating this application of \cref{claim:NCZ}, we can construct a strictly decreasing sequence $\SS_0 \supset \SS_1 \supset \cdots \supset \SS_{\dm - 1}$ such that for each $k$, $\abs{\SS_k} = \dm - k$ and for all $i \in \SS_k$ $u_i(n) \rightarrow 0$ as $n \rightarrow \infty$.
  As $\norm{ \Pmap_{\SS_{\dm - 1}} \vec u(n)}^2 + \norm{ \Pmap_{\compl \SS_{\dm - 1}} \vec u(n)}^2 = 1$ with $\Pmap_{\compl \SS_{\dm - 1}} \vec u(n) \rightarrow \vec 0$ as $n \rightarrow \infty$, it follows that $\norm{\Pmap_{\SS_{\dm - 1}} \vec u(n)}^2 \rightarrow 1$ as $n \rightarrow \infty$.
  Letting $j$ be the lone element in $\SS_{\dm -1}$, we see that $\vec u(n) \rightarrow \myZv_j$ as $n \rightarrow \infty$.
\end{proof}

We now extend our result from \cref{thm:canonical-global-attraction} to the general setting in which $\dn \geq \dm$.

\begin{thm}\label{thm:canonical-global-attraction2}
  Suppose that $1 \leq \dm \leq \dn$.
  There exists a set $\mathcal X \subset \porth^{\dn - 1}$ with $\vol_{\dn-1}(\mathcal X) = 0$ which has the following property: If $\vec u(0) \not \in \mathcal \XX$, then there exists $i \in [\dm]$ such that $\vec u(0) \rightarrow \myZv_i$ for some $i \in [\dm]$.
\end{thm}
\begin{proof}
  We note that the case that $\dm = 1$ is trivial as $\GG(\vec u) = \myZv_1$ for all $\vec u \in \porth^{\dn - 1} \setminus \myZv_1^\perp$, and since $\myZv_1^\perp$ has volume 0.
  We assume without loss of generality that $\dm \geq 2$, thus making $S^{\dm - 1}$ a smooth manifold.

  Throughout this proof, we will treat $\R^\dm$ as a subset of $\R^\dn$ within $\spn\{\myZv_i \suchthat i \in [\dm]\}$ by mapping $(x_1, \dotsc, x_{\dm}) \mapsto (x_1, \dotsc, x_\dm, 0, \dotsc, 0)$ so that we can abuse notation and have $\vec x \in \R^{\dm}$ also part of the domain $\R^\dn$.
  In particular, we also will view $S^{\dm - 1} \subset S^{\dn - 1}$ in this fashion.

  We first construct a new family of {\sef}s.
  In particular, we let $A := B(\vec 0, 1) \cap \spn\{ \myZv_i \suchthat i \not \in [\dm] \}$ (with $B(\vec 0, 1)$ the open ball of radius 1 in $\R^\dn$).
  We define the functions $\mathfrak g_i : A \times \R$ by $\mathfrak g_i(\vec p, t) := g_i(t\sqrt{1 - \norm{\vec p}^2})$, $\mathfrak F : A \times \R^\dm$ by $\mathfrak F(\vec p, \vec u) = \sum_{i = 1}^\dm \mathfrak g_i(\vec p, u_i)$, and $\mathfrak G : A \times \porth^{\dm - 1} \rightarrow \porth^{\dm - 1}$ such that $\mathfrak G(\vec p, \bullet)$ is the gradient iteration function associated with $\mathfrak F(\vec p, \bullet)$.
  Notice that the functions $\mathfrak F(\vec p, \bullet)$ are {\sef}s.
  Further, it can be verified that $\mathfrak G(\vec p, \vec u) = \GG( \vec p + \vec u \sqrt{1 - \norm{\vec p}^2})$.
  It will sometimes be more convenient to use a more pure function notation, and we thus define $\mathfrak G_{\vec p} := \mathfrak G(\vec p, \bullet)$.

  Define $\mathcal X_\dm$ as $\mathcal X$ from \cref{thm:canonical-global-attraction} for the function $\mathfrak G(\vec 0, \bullet) = \GG \restr{\spn\{\myZv_i \suchthat i \in [\dm] \}}$.
%  We will denote by $\vol_k$ the volume measure on the sphere $S^{k - 1}$.
  We note that $\vol_{\dm-1}(\mathcal X_\dm) = 0$.
  By \cref{lem:Ginv-measure-0-maps}, we see that $\vol_{\dm-1}(\mathfrak G_{\vec p}^{-1}(\mathcal X_{\dm})) = 0$ for any $\vec p \in A$.
  As such,
  \begin{align*}
    \vol_{\dn-1}( \GG^{-1}(\mathcal X_{\dm}) )
    &= \int_{\vec p \in A} (1 - \norm{\vec p}^2)^{\dm/2}\vol_{\dm-1}( \mathfrak G_{\vec p}^{-1}(\mathcal X_{\dm}) ) d \vec p = 0 \ .
  \end{align*}

  We define $\mathcal X := \GG^{-1}(X_{\dm}) \cup \{\vec u \suchthat u_i = 0 \text{ for all } i \in [\dm] \}$.
  Note that
  \begin{displaymath}
    \vol_{\dn-1}(\mathcal X) \leq \vol_{\dn-1}(\GG^{-1}(\mathcal X_{\dm})) + \vol_{\dn-1}(\{\vec u \suchthat u_i = 0 \text{ for all } i \in [\dm] \}) = 0 \ .
  \end{displaymath}
  % Define the set $B := \porth^{\dn - 1} \cap \spn\{\myZv_i \suchthat i \in [\dm] \}$.
  Also note that for any $\vec u(0) \not \in \mathcal X$, $\vec u(1) \in \porth^{\dm - 1}$ and $\vec u(1) \not \in \mathcal X_{\dm}$.
  Applying \cref{thm:canonical-global-attraction} to the sequence $\{ \vec u(n) \}_{n=1}^\infty$ with gradient iteration function $\GG\restr{\porth^{\dm - 1}}$, we obtain that $\vec u(n) \rightarrow \myZv_i$ for some $i \in [\dm]$.
\end{proof}

Using the symmetries of the gradient iteration (\cref{prop:gi-sym-classes}), \cref{thm:gi-stability} is implied by \cref{prop:ei-attractors,thm:canonical-global-attraction2}.

\subsection{Fast convergence of the gradient iteration}
\label{sec:gi-convergence}

We now proceed with the proof of \cref{thm:gi-convergence}.
The stability analysis relied on the change of variable $ \vec u \mapsto (u_i^2) $ (which gave rise to the definitions of $h_i$ for $i \in [\dn]$) due the fact that for each $i \in [\dm]$, $g_i( x ^ {1/2} )$ is convex on $\pinterval$.
The fast convergence of the gradient iteration algorithm relies on a more general change of variable $\vec u \mapsto (u_i^r)$ where $r\geq 2$, and in particular it is assumed that $g_i(x ^ {1/r})$ is convex on $\pinterval$ for each $i \in [\dm]$.
We encode this potentially stronger convexity constraint within our {\psef} by extending the definition of the $h_i$s from \cref{sec:extrema-structure} to the more general family of maps $\gamma_{ir} : \pinterval \rightarrow \R$ defined by $\gamma_{ir}(x) := g_i(x^{\frac 1 r})$ for $i \in [\dm]$ and $\gamma_{ir} = 0$ for $i \not \in [\dm]$.
We note that $h_i = \gamma_{i2}$ on $\pinterval$ for each $i \in [\dn]$.
We then write
\begin{equation}
  \label{eq:fg-gamma-vers}
  \fg(\vec u) = \sum_{i = 1}^\dm g_i(u_i) = \sum_{i=1}^\dm \gamma_{ir}(u_i^r) \ ,
\end{equation}
where each $\gamma_{ir}$ is a convex function.
% Note that this is a proper extension of the previously defined maps seeing as the function $h_i$ and $\gamma_{i2}$ are identical.

\begin{lem} \label{lem:convexity-deriv-equivs}
  For all $i \in [\dm]$, %the derivative functions $\gamma'_{ir}$ and $\gamma'_{i2}$ are related by
  $\gamma_{ir}'(x) = \frac 2 r \gamma_{i2}'(x^{\frac 2 r})x^{\frac {2-r} r}$ on the domain $(0, 1]$.
\end{lem}
\begin{proof}
  This is by direct computation.  We have the formulas:
  \begin{align*}
    \gamma'_{i2}(x) &= \frac 1 2 g_i'(x ^{\frac 1 2})x^{-\frac 1 2} &
    \gamma'_{ir}(x) &= \frac 1 r g_i'(x ^{\frac 1 r})x^{\frac {1-r} r}
  \end{align*}
  We may rewrite $\gamma'_{ir}(x)$ as follows:
  \begin{displaymath}
    \gamma'_{ir}(x) = \frac 2 r \cdot \frac 1 2 g_i'((x^{\frac 2 r})^{\frac 1 2})(x^{\frac 2 r})^{- \frac 1 2} x^{\frac {2-r} r}
    = \frac 2 r \gamma'_{i2}(x^{\frac 2 r}) x^{\frac {2-r} r} \ . \qedhere
  \end{displaymath}

\end{proof}

\begin{prop}\label{prop:gi-convergence}
  % Let $\fg$ be {\apsef} with associated gradient iteration function $\GG$.
  Suppose that $\{\vec u(n)\}_{n=0}^\infty$ is a sequence in $\porth^{\dn-1}$ defined recursively by $\vec u(n) = \GG(\vec u(n-1))$ which converges to a $\myZv_j$ for some $j \in [\dm]$.  Then, the following hold:
  \begin{compactenum}
    \item The sequence $\{\vec u(n)\}_{n=0}^\infty$ converges to $\myZv_j$ at a super-linear rate.
    \item Fix $r \geq 2$.
      If $x \mapsto g_i(x^{\frac 1 r})$ is convex for every $i \in [\dm]$,
      then $\{\vec u(n)\}_{n=0}^\infty$ converges to $\myZv_j$ with order of convergence at least $r - 1$.
  \end{compactenum}
\end{prop}
\begin{proof}
  It is sufficient to consider a sequence converging to $\myZv_1$.  If
  there exists $n_0$ such that $\vec u(n_0) = \myZv_1$, then there is nothing
  to prove as $\myZv_1$ is a stationary point of $\GG$.  So, we assume that
  $\vec u(n) \neq \myZv_1$ for all $n \in \N$.

  Taking derivatives of $\fg$ from \cref{eq:fg-gamma-vers}, we get:
  % Defining $H : \RR(\psi) \rightarrow \R$ by $F \circ \psi^{-1}$.  Then,
  % restricted to the domain $\porth^{d-1}$, we have
  $\partial_i F(\vec v) = r \gamma'_{ir}(v_i^r) v_i^{r-1}$.
  We will make use of the following ratios in analyzing the rate of convergence of $\vec u(n)$:
  \begin{displaymath}
    \rho(i, j; n) := \frac{u_i(n)}{u_j(n)} = \frac{\gamma_{ir}'(u_i(n-1)^r)u_i(n-1)^{r-1}}{\gamma_{jr}'(u_j(n-1)^r)u_j(n-1)^{r-1}} \ .
  \end{displaymath}

  Define $U = \gamma_{1r}'(1)$ and $L = \max_{j \neq 1} \{ \lim_{x \rightarrow 0^+} \gamma_{jr}'(x)\}$.
  We note that the strict convexity of $x \mapsto g_i(\sqrt x)$ (for $i \in [\dm]$) implies
  that $\gamma_{i2}'(1) > 0$, and since \cref{lem:convexity-deriv-equivs} implies $\gamma_{ir}'(1) = \frac 2 r \gamma_{i2}'(1) > 0$, it follows that $U > 0$.
  Since $\gamma_{ir}$ is convex, $\gamma_{jr}'$ is a non-decreasing function.
  It follows that $L$ is well defined and is also equal to $\max_{j\neq 1}\{\inf_{x > 0} \gamma_{jr}'(x) \}$.
  Finally, noting that $\gamma_{i2}'$ is non-negative on $\pinterval$ (indeed, $\gamma_{i2}'$ is increasing from $\gamma_{i2}'(0) = 0$ by \cref{lem:h-props}), it follows from \cref{lem:convexity-deriv-equivs} that $\gamma'_{ir}(x) \geq 0$ for all $x > 0$, and in particular $L \geq 0$.

 Fix $\epsilon \in (0, \frac 1 2 U)$.  There exists $\delta > 0$ such that:
 \begin{enumerate}
 \item If $\vec v \in \porth^{\dn - 1}$ is such that $1 - v_1 < \delta$, then $\gamma'_{1r}(u_1) > U - \epsilon$.
   The existence of such a choice for $\delta$ is implied by the continuity of $g_1'$ and hence $\gamma'_{1r}$ near 1.
 \item If $\vec v \in \porth^{\dn - 1}$ is such that $v_j < \delta$ for some $j \neq 1$, then $\gamma'_{jr}(u_j) < L + \epsilon$.
   The existence of such a $\delta$ follows from the characterization of $L$ as $\max_{j\neq 1}\{\inf_{x > 0} \gamma_{jr}'(x) \}$ and $\gamma_{jr}'$ being monotonic on $\pinterval$.
 \end{enumerate}
 % As $\vec u(n) \neq \vec e_1$, there exists $j \neq 1$ such that $u_j(n) \neq 0$.
  Fix $N$ sufficiently large that for each $n \geq N$, $\norm[1]{\myZv_1 - \vec u(n)} < \delta$.
  With any fixed $j \neq 1$ and $n \geq N+1$, it follows that
  \begin{equation}\label{eq:ratio-bound}
    \rho(j, 1; n) = \frac{\gamma_{jr}'(u_j(n-1)^r)u_j(n-1)^{r-1}}{\gamma_{1r}'(u_i(n-1)^r)u_1(n-1)^{r-1}}
    < \frac {L + \epsilon}{U - \epsilon} \cdot \frac{u_j(n-1)^{r-1}}{u_1(n-1)^{r-1}} \ .
  \end{equation}
  Denote by $\vec u'$ the vector $\sum_{i=2}^\dn u_i\myZv_i$. Then,
  \begin{align*}
    \norm{\myZv_1 - \vec u(n)} &= \norm{\myZv_1(1-u_1(n)) - (\vec u(n) - u_1(n)\myZv_1)} \\
    &\leq \norm{\myZv_1(1 - u_1(n))} + \norm{\vec u'(n)}
    = 1-u_1(n) + \norm{\vec u'(n)} \ .
  \end{align*}
  Since $\vec u$ is a unit vector, we see that $u_1(n) + \norm{\vec u'(n)} \geq u_1(n)^2 + \norm{\vec u'(n)}^2 = 1$.
  It follows that $1 - u_1(n) \leq \norm{\vec u'(n)}$.  Thus,
  \begin{align*}
    \norm{\myZv_1 - \vec u(n)}
    &\leq 2\norm{\vec u'(n)}
    \leq 2 \norm[1]{\vec u'(n)}
    = 2\sum_{i=2}^\dn u_i(n)  \\
    &\leq 2 \sum_{i=2}^\dn \rho(i, 1; n)
    < 2\cdot \frac{L+\epsilon}{U-\epsilon} \cdot \frac {\sum_{i=2}^\dn u_i(n-1)^{r-1}}{u_1(n-1)^{r-1}}
  \end{align*}
  where the second to last inequality uses that $\vec u(n)$ is a unit vector making $u_1(n) \leq 1$, and the last inequality uses \cref{eq:ratio-bound}.
  Continuing (with $n \geq N+1$), we see $u_1(n-1) \geq 1-\norm[1]{\myZv_1 - \vec u(n-1)} \geq 1-\delta$. Hence,
  \begin{displaymath}
    \norm{\myZv_1 - \vec u(n)}
    < 2\cdot \frac{L+\epsilon}{(U-\epsilon)(1-\delta)^{r-1}} \cdot \sum_{i=2}^\dn {u_i(n-1)^{r-1}} \ .
  \end{displaymath}
  Since for each $i \geq 2$ we have $u_i(n-1) \leq \norm{\myZv_1 - \vec u(n-1)}$
  \begin{displaymath}
    \frac{\norm{\myZv_1 - \vec u(n)}}{\norm{\myZv_1 - \vec u(n-1)}^{r-1}}
    < 2 \dn \cdot \frac{L + \epsilon}{(U-\epsilon)(1-\delta)^{r-1}} \ .
  \end{displaymath}
  As the right hand side is a finite constant, the sequence has order of
  convergence at least $r - 1$.
  In the case where $r = 2$, \cref{lem:h-props} combined with the fact that $\gamma_{i2}' = 0$ for each $i \in [\dn]\setminus [\dm]$ implies that $\lim_{x \rightarrow 0^+} \gamma_{i2}'(x) = 0$ for each $i \in [\dn]$; and in particular, $L = 0$.
  Since $\epsilon$ can be chosen arbitrarily small, the sequence $\{\vec u\ind n\}_{n=0}^\infty$ has super-linear convergence even when $r = 2$.
\end{proof}

Under \cref{prop:gi-sym-classes}, %~\cref{cor:psef-equivclass-repr},
part 1 of \cref{thm:gi-convergence} is implied by \cref{prop:gi-convergence}.
Part 2 of \cref{thm:gi-convergence} follows from the fact that for any $i$ such that $\vec u \perp \myZv_i$, then $\partial_i F(\vec u) = 0$ implies that $\GG(\vec u) \perp \myZv_i$.
In particular, it can be seen by induction on $n$ that for a sequence defined recursively by $\vec u(n) = \GG(\vec u(n-1))$ and $\vec u(0) \perp \myZv_i$, then $\vec u(n) \perp \myZv_i$ for all $n \in \N$ and in particular $\vec u(n) \not \rightarrow \myZv_i$.
\end{vlong}

% \section{The gradient iteration and power methods}
\section{Connections of gradient iteration to gradient ascent and power methods}
\label{sec:interpr-grad-iter}

In this section, we briefly interpret the gradient iteration as a form of adaptive, projected gradient ascent. %  and also as a generalized power method.
As the gradient iteration is also a generalized power iteration,
these dual interpretations closely link the gradient iteration and other power methods with hill climbing techniques for finding the maxima of a function%
\begin{vlong}%
\footnote{We note that in a special setting of recovering a parallelopiped a closely related observation was made by \mycitet{Nguyen and Regev}{Nguyen2009}.}%
\end{vlong}.
In particular, this connection gives a conceptual explanation of the relationship between the fixed points of the gradient iteration and the maxima structure of \ansef\@ $F$ on the unit sphere.
For the remainder of this section, we take $\fg$ to be {\apsef}.

%\paragraph{Gradient iteration as adaptive gradient ascent\opdot}
%\label{sec:gi-adaptive-grad-ascent}
%The function $\fg$ can be maximized on the unit sphere using a variation on gradient ascent.
The projected gradient ascent update (with learning rate $\eta$) is given in the function \textsc{GradAscentUpdate} below.

\begin{algorithm}[H]
\caption[Algorithm:  Projected gradient ascent update.]{
  \label{alg:GradAscentUpdate}
  A single projected gradient ascent step for function maximization over $S^{\dn - 1}$.
}
\begin{algorithmic}[1]
  \Function{GradAscentUpdate}{$\vec u$, $\eta$} \label{alg:GradAscentStep1}
  \State $\vec u' \leftarrow \vec u + \eta \Pmap_{\vec u^\perp} \grad \fg(\vec u)$ \label{alg:grad-ascent:step-update}
  \State \Return $\frac {\vec u'}{\norm{\vec u'}}$ \label{alg:grad-ascent:step-project}
  \EndFunction
\end{algorithmic}
\end{algorithm}

\begin{vlong}
The update in \textsc{GradAscentUpdate} differs from the standard gradient ascent in two ways.
First, the update occurs in the direction $\Pmap_{\vec u^\perp} \grad \fg(\vec u)$ rather than $\grad \fg(\vec u)$.
This takes into account the geometry structure of $S^{\dn-1}$ by updating within the plane tangent to $S^{\dn-1}$ at $\vec u$.
This arises naturally when treating $S^{\dn-1}$ as a manifold with the local coordinate system defined by the projective space centered at $\vec u$.
Second, $\vec u'$ is projected back onto the sphere in order to stay within $S^{\dn-1}$.
\end{vlong}

\vnote{This paragraph is essentially a proof of the following Lemma and is currently in the short version.  This should be late on the chopping block, but this paragraph can be removed and replaced with a 1 sentence summary if need be.}
We now compare the update rules $\vec u \leftarrow \textsc{GradAscentUpdate}(\vec u,\ \eta)$ and $\vec u \leftarrow \GG(\vec u)$.
If $\Pmap_{\vec u^\perp} \grad \fg(\vec u) = \vec 0$, then both updates are the identity map and are thus identical.
If $\Pmap_{\vec u^\perp} \grad \fg(\vec u) \neq \vec 0$, then %we may write $\GG(\vec u)$ as
\begin{equation}\label{eq:G-grad-ascend-decomp}
  \GG(\vec u)
    = \frac{\grad \fg(\vec u)}{\norm{\grad \fg(\vec u)}}
    = \frac{\ip{\grad \fg(\vec u)}{\vec u}\vec u + \Pmap_{\vec u^\perp} \grad \fg(\vec u)}{\norm{\grad \fg(\vec u)}}
    = \frac{\vec u + \Pmap_{\vec u^\perp} \grad \fg(\vec u)/\ip{\grad \fg(\vec u)}{\vec u}}{\norm{\grad \fg(\vec u)}/\ip{\grad \fg(\vec u)}{\vec u}} \ .
\end{equation}
The numerator of the rightmost fraction can be interpreted as line~\ref{alg:grad-ascent:step-update} of \textsc{GradAscentUpdate}($\vec u,\ \eta$) using the choice $\eta = \ip{\vec u}{\grad \fg(\vec u)}^{-1}$.
\Cref{lem:h-props} implies that $u_i > 0$ if and only if $\partial_i \fg(\vec u) = 2h_i'(u_i^2)u_i > 0$.
More generally, the symmetries from \cref{assumpt:gsymmetries} imply that $\sign(u_i) = \sign(\partial_i F(\vec u))$ for all $i \in [\dm]$.
As such, $\eta = \ip{\vec u}{\grad \fg(\vec u)}^{-1} > 0$ is a valid learning rate generically (whenever $\nabla F(\vec u) \neq \vec 0$).
The denominator of the rightmost fraction in \cref{eq:G-grad-ascend-decomp} gives the normalization to project back onto the unit sphere (line~\ref{alg:grad-ascent:step-project} of \textsc{GradAscentUpdate}).
We obtain the following relationship between gradient ascent and gradient iteration.

\begin{lem} \label{lem:gi-is-grad-ascent}
  % Let $\fg$ be a {\psef} with associated gradient iteration $\GG$, and let $\vec u \in S^{\dn - 1}$.
  The update $\vec u \leftarrow \GG(\vec u)$ is an adaptive form of projected gradient ascent.
  Specifically,
  \begin{compactenum}
  \item If $\grad \fg(\vec u) \neq \vec 0$, then $\GG(\vec u) =$ \textsc{GradAscentUpdate}$(\vec u, \ip{\vec u}{\grad \fg(\vec u)}^{-1})$.
  \item If $\grad \fg(\vec u) = \vec 0$ and $\eta \in \R$, then $\GG(\vec u) =$ \textsc{GradAscentUpdate}$(\vec u, \eta)$.
  \end{compactenum}
\end{lem}

\begin{vlong}
\vnote{This paragraph is not necessary in the short version.}
The step size chosen by the gradient iteration function is in several ways very good.
By \cref{prop:gi-sym-classes}, $\GG(\vec u)$ and hence $\grad \fg(\vec u)$ belong to the same orthant as $\vec u$.
As such we never overshoot a basis direction $\myZv_i$ during the ascent procedure.
Further, the gradient iteration has the fast convergence properties stated in \cref{thm:gi-convergence}.
\end{vlong}

% Then, we may write $\fg(\vec u) = \sum_{i=1}^\dm \alpha_i g(\beta_i u_i) = \sum_{i = 1}^\dn h_i(u_i)$, and for any $\vec u \in \porth^\dn$ we get that $\grad \fg(\vec u) = 2 \sum_{i = 1}^\dm h_i'(u_i^2) u_i \myZv_i$.
% From \cref{lem:h-props}, we have that each $h$ is strictly convex and that $h'$ is a strictly increasing function with $h_i'(0) = 0$.

\begin{vshort}
\refstepcounter{section}\label{sec:gi-error-analysis}
\end{vshort}
\begin{vlong}
\section{Gradient Iteration Under a Perturbation}
\label{sec:gi-error-analysis}

\begin{algorithm}[tb]
  \caption[GI-Loop]{\label{alg:GI-LOOP} Perform the gradient iteration for a predetermined number of iterations.
    The inputs are %$\hat \GG$ (the Gradient iteration function handle),
    $\vec u(0)$ (an initialization vector) and $N$ (the number of iterations).
    The output is $\vec u(N)$ (the $N$\textsuperscript{th} element of the resulting gradient iteration sequence).
  }
  \begin{algorithmic}
  \Function{GI-Loop}{$\vec u(0), N$}
  \For{$n\leftarrow 1$ to $N$}
  \State $\vec u(n) \leftarrow \hat \GG(\vec u(n-1))$
  \EndFor
  \State \Return $\vec u(N)$
  \EndFunction
  \end{algorithmic}
\end{algorithm}

\begin{algorithm}[tb]
  \caption[\textsc{FindBasisElement}]{\label{alg:SingleRecovery}
    A robust extension to the gradient iteration algorithm for guaranteed recovery of a single hidden basis element.}
    Inputs:\\
    \begin{tabularx}{\linewidth}{lX}
    $\{\gvec \mu_1, \dotsc, \gvec \mu_k\}$ & A (possibly empty) set of approximate hidden basis directions.  \\
    $\sigma$ & Positive parameter determining jump size to break stagnation of $\hat \GG$. \\
    $\pertgF$ & Function pointer to our estimate of $\nabla F$.  $\hat \GG$ is also being implicitly defined from this in our pseudo-code. \\
    $N_1$, $N_2$, $I$ & Parameters which determine total loop iterations.
    \end{tabularx}\\
    Outputs: An approximate basis element not estimated by any of $\gvec \mu_1, \dotsc, \gvec \mu_k$.
  %   \begin{tabular}{ll}
  %   $\gvec \nu \qquad\qquad\quad$ & An approximate basis element not estimated by any of $\gvec \mu_1, \dotsc, \gvec \mu_k$.
  % \end{tabular}
    \hrule
  \NumTabs{2}
  \begin{algorithmic}[1]
    \Function{FindBasisElement}{$\{\gvec \mu_1, \dotsc, \gvec \mu_k\}$, $\sigma$, $\pertgF$, $N_1$, $N_2$, $I$}
      \State // Find a starting vector sufficiently outside the subspace $\spn(\myZv_{\dm + 1}, \dotsc, \myZv_\dn)$.
      \State \label{alg:step-xi-spread} Let $\vec x_1, \dotsc, \vec x_{\dn - k}$ be orthonormal vectors in $\spn(\gvec \mu_1, \dotsc, \gvec \mu_k)^\perp$.
      \State\label{alg:step:best-xi} $j \leftarrow \argmax_{i \in [\dn-k]} \norm{\pertgF(\vec x_i)}$
      \State\label{alg:step:zero-nonbasis-coords}
        $\vec u \leftarrow \hat \GG(\vec x_j)$ \tab // ``Zero'' the values of $u_{m+1}, \dotsc, u_\dn$.
        % \State\label{alg:step:zero-nonbasis-coords2}
        % $\vec u \leftarrow \hat \GG(\vec u)$
        \State\label{alg:step:zero-nonbasis-coords2}
        $\vec u \leftarrow \Call{GI-Loop}{\vec u, N_1}$
        % \hfill // ``Zero'' coordinates corresponding to directions $\gvec \mu_1, \dotsc, \gvec \mu_k$.
      %\For{$i\leftarrow 1$ to $\hat \dm - k - 1$}
      \For{$i\leftarrow 1$ to $I$}
        \tab // Start of the main loop
        \State\label{alg:step:draw-jump}
        Draw $\vec x$ uniformly at random from $\sigma \sphere^{\dn - 1} \cap \vec u^\perp$ %$\mathcal N(\vec 0, \sigma^2 \Pmap_{\vec u^\perp})$
        \State\label{alg:step:jump1}
        $\vec w \leftarrow \vec u \cos(\norm{\vec x}) + \frac{\norm{\vec x}}{\vec x}\sin(\norm{\vec x})$
        \tab // A random jump from $\vec u$
        \State\label{alg:step:gi-in-main-loop}
            $\vec u \leftarrow \textsc{GI-Loop}(\vec w, N_2)$
            %$\vec u \leftarrow \textsc{GI-Loop}(\vec u \cos(\norm{\vec x}) + \frac{\norm{\vec x}}{\vec x}\sin(\norm{\vec x}), N_2)$
            \EndFor
      \State\label{alg:step:exit}
          \Return $\vec u$
    \EndFunction
  \end{algorithmic}
\end{algorithm}

\begin{algorithm}[tbh]
  \caption[\textsc{RobustGI-Recovery}]{\label{alg:FullRecovery}
    A robust algorithm to recover approximations to all of the hidden basis elements.
  }
    Inputs:\\
    \begin{tabularx}{\linewidth}{lX}
    $\hat \dm\qquad\qquad$
    & The desired number of basis elements to recover.  It is required that $\hat \dm \geq \dm$. \\
    $\sigma$ & Parameter determining perturbation noise added to escape near ``stationary points'' of $\hat \GG$. \\
    $\pertgF$ & Function pointer to our estimate of $\nabla F$.  $\hat \GG$ is also being implicitly defined from this in our pseudo-code. \\
    $N_1$, $N_2$, $I$ & Parameters which determine total loop iterations.
    \end{tabularx}\\
    Outputs:\\
    \begin{tabular}{ll}
    $\gvec \mu_1, \dotsc, \gvec \mu_{\hat \dm}$
    & The first $\dm$ of these are approximate hidden basis elements.
    \end{tabular}\hrule
  \begin{algorithmic}[1]
    \Function{RobustGI-Recovery}{$\hat \dm$, $\sigma$, $\pertgF$, $N_1$, $N_2$, $I$}
    \For{$i \leftarrow 1$ to $\hat \dm$}
      \State $\gvec \mu_i \leftarrow$ \Call{FindBasisElement}{$\{\gvec \mu_1, \dotsc, \gvec \mu_{i-1}\}$, $\sigma$, $\pertgF$, $N_1$, $N_2$ $I$}
    \EndFor
    \State \Return $\gvec \mu_1, \dotsc, \gvec \mu_{\hat \dm}$
    \EndFunction
  \end{algorithmic}
\end{algorithm}

In \cref{sec:gi-stability-structure}, we saw that the hidden basis
elements $\myZv_i$ are attractors, that convergence to this set of attractors
is guaranteed except on a set of measure 0, and that the rate of convergence is super-linear.
In this section, we provide a robust extension to the gradient iteration
algorithm for recovering all of the hidden basis elements.
We demonstrate that for a wide class of contrasts, the recovery process is
robust to a perturbation, and that the hidden basis elements $\myZv_1, \dotsc,
\myZv_\dm$ can be efficiently recovered given approximate access to $\nabla F$.

To provide quantifiable algorithmic bounds, we require quantifiable
assumptions upon the hidden convexity (or concavity) of the $h_i$ functions
associated with $F$.
For smooth functions, convexity is characterized by the second derivative of
the function.
%In order to place bounds upon the hidden convexity of our transformed
% contrasts $h_k$,
In particular, we use the following notion of robustness.
\begin{defn}\label{defn:BEF-robust}
  Let $\alpha$, $\beta$, $\gamma$, and $\delta$ be strictly positive constants, and let $D \subset \R$.
  A contrast function $g : D \rightarrow \R$ satisfying \crefrange{assumpt:first}{assumpt:last} is said to be \emph{$(\alpha, \beta, \gamma, \delta)$-robust}\index{contrast function!robust} if for all $x > 0$, $\beta \abs x^{\delta - 1} \leq \Abs{\frac {d^2}{dt^2}[g(\sqrt t)]\big|_{t = x}} \leq \alpha \abs x^{\gamma - 1}$.
  Further, \ansef\@ $F(\vec u) = \sum_{i = 1}^\ldim g_i(\ipCanonical{\vec u}{\hbev_i})$ is said to be \emph{$(\alpha, \beta, \gamma, \delta)$-robust}\index{Basis Encoding Function!robust orthogonal} if each of its contrast functions $g_i$ are $(\alpha, \beta, \gamma, \delta)$-robust on the domain $[-1, 1]$.
\end{defn}

This definition is designed to capture a broad class of functions of interest.
For instance, we capture monomials of the form $p_{a, r}(x) = \frac{a}{(r+1)r}x^{2r+2}$ on $[0, 1]$ where $r > 0$ and $a > 0$ are real (with either positive or negative reflections of this on $[-1, 0]$.
Indeed, the robustness criterion in \cref{defn:BEF-robust} may alternatively be stated as
$\frac{d^2}{dt^2}(p_{\beta, \delta}(\sqrt t))\bigr|_{t = x}
\leq \Abs{\frac {d^2}{dt^2}[g(\sqrt t)]\big|_{t = x}}
\leq \frac{d^2}{dt^2}(p_{\alpha, \gamma}(\sqrt t))\bigr|_{t = x}$ for all $x > 0$ in the domain of $g$.
In particular, the monomial functions $ax^r$ with $r \geq 3$ an integer which arise in the setting of orthogonal tensor decompositions are captured as a special case.

\Cref{defn:BEF-robust} provides several natural condition numbers which arise in our analysis.

\begin{rmk}\label{rmk:condition-nums-robust-BEF}
  If $F$ is $(\alpha, \beta, \gamma, \delta)$-robust, then $\alpha \geq \beta$ and $\gamma \leq \delta$.
\end{rmk}
\begin{proof}
  To see that $\alpha \geq \beta$, we note that $\alpha x^{\gamma-1} \geq \beta x^{\delta - 1}$ holds at $x = 1$.
  To see that $\gamma \leq \delta$, we note that asymptotically as $x \rightarrow 0$ from the right, $\beta x^{\delta - 1} = O(x^{\gamma - 1})$.
\end{proof}
Under \cref{rmk:condition-nums-robust-BEF}, we see that $\frac \alpha \beta$ and $\frac \delta \gamma$ are both lower bounded by $1$.
These ratios will act as condition numbers in our time and error bounds.

For the remainder of this section, we will assume that $F$ is $(\alpha, \beta, \gamma, \delta)$-robust unless otherwise specified.
Hatted objects such as $\pertgF$ and $\hat \GG$ will represent the natural estimates of un-hatted objects, and in particular $$\hat \GG(\vec u) := \begin{cases}
\pertgF(\vec u)/\norm{\pertgF(\vec u)} & \text{if } \pertgF(\vec u) \neq \vec 0 \\
\vec u & \text{otherwise}
\end{cases} \ .$$
For $\epsilon > 0$, we say that $\pertgF$ is an \emph{$\epsilon$-approximation} of $\nabla F$ if $\norm{\pertgF(\vec u) - \grad F(\vec u)} \leq \epsilon$ for all $\vec u \in \overline{B(0, 1)}$.
We assume (unless otherwise stated) throughout this section that $\pertgF$ is an $\epsilon$-approximation of $\nabla F$ with any bounds on $\epsilon$ being made clear by context.

We will see that under these assumptions, we are able to recover approximations of hidden basis elements using \textsc{FindBasisElement} (\cref{alg:SingleRecovery}).
We use the following notion of recovery since we do not care about the ordering or the sign associated with the original hidden basis.
\begin{defn}
  Consider the distance $d(\vec u, \vec v) := \min(\norm{\vec u - \vec v},
  \norm{-\vec u - \vec v})$.
  We say that $\tilde \myZv_1, \dotsc, \tilde \myZv_k$ is an
  \emph{$\epsilon$-recovery} of the basis $\myZv_1, \dotsc, \myZv_k$ if
  there exists a permutation $\pi$ of $[k]$ such that $d(\tilde \myZv_i,
  \myZv_{\pi(i)}) \leq \epsilon$ for all $i \in [k]$.
\end{defn}
If we have approximately recovered several hidden basis elements using \textsc{FindBasisElement}, then we may use \textsc{FindBasisElement} to approximately recover a new hidden basis element.
% The recovery takes into account previously found basis elements.
% In particular, if $\myZv_{i_1}, \dotsc, \myZv_{i_k}$ are approximated by $\gvec \mu_1, \dotsc, \gvec \mu_k$, then \textsc{FindBasisElement} approximately recovers a hidden basis element $\myZv_{i_{k+1}}$ such that $i_{k+1} \in [\dm] \setminus \{i_1, \dotsc, i_k\}$.
In particular, \textsc{FindBasisElement} may be run repeatedly to recover all hidden basis elements.
Formally, we have the following result.

For clarity, we will denote by $C_0, C_1, C_2, \dotsc$ positive universal constants in the main theorem statements.
These can represent different constant values in different theorem statements.
% If a $C_i$ appears in two different theorems or lemmas, it can represent distinct universal constants.
\begin{thm}\label{thm:single-recovery-main}
  Suppose that
  \begin{compactitem}
  \item $\epsilon \leq C_1 4^{-\frac{4+2\delta}{\gamma}}\frac{\sigma \beta}{\delta}\big[\frac {\beta\gamma}{\alpha \delta} \big]^{\frac{4\delta + 7}{2\gamma}}\dm^{-\frac \delta \gamma(2\delta - \gamma + \frac 7 2)}\dn^{-\frac 1 2 - \delta}$,
  \item $\sigma \leq \frac{C_0}{\sqrt{\dn(1+\delta)}} \big[ \frac {\beta\gamma}{16\alpha\delta}\big]^{\frac 1 \gamma}\dm^{-\frac \delta \gamma}$,
  \item $N_1 \geq C_2 \lceil \log_{1+2\gamma}(\log_2(\frac{\beta}{\delta \epsilon}))\rceil$, and
  \item $N_2 \geq C_3\bigr\lceil 4^{\frac 2 \gamma}\frac{\sqrt \dn}{\sigma} (\frac{\alpha\delta}{\beta\gamma})^{\frac {\delta+2} \gamma}\dm^{\frac \delta \gamma(\delta - \gamma + 2)}[\frac 1 \gamma \log(\frac{\alpha \delta}{\beta \gamma}) + \frac \delta \gamma \log(\dm)]\bigr\rceil \newline \phantom{.}\qquad \ + C_2 \lceil \log_{1+2\gamma}(\log_2(\frac{\beta}{\delta \epsilon}))\rceil$.
  \end{compactitem}
  Let $p \in (0, 1)$.
  Suppose that $I \geq C_3 \dm \lceil \log(\dm / p) \rceil$,
  that $\gvec \mu_1, \dotsc, \gvec \mu_k$ is a $C_4 \delta \epsilon / \beta$-recovery of $\myZv_1, \dotsc, \myZv_k$
  % $\norm{\gvec \mu_i - \myZv_i} \leq C_4 \delta \epsilon / \beta$ for all $i \in [k]$,
  and that $k < \dm$.
  After executing
  \begin{displaymath}
    \gvec \mu_{k+1} \leftarrow
    \textsc{FindBasisElement}(\{\gvec \mu_1, \dotsc, \gvec \mu_k\}, \sigma, \pertgF, N_1, N_2, I) \ ,
\end{displaymath}
  then with probability at least $1 - p$, there exists $j \in \{k+1, k+2, \dotsc, \dm\}$ such that $\gvec \mu_{k+1}$ is a $C_4 \delta \epsilon / \beta$-recovery of $\myZv_j$.
\end{thm}

\textsc{FindBasisElement} operates as follows.
We first find a warm start $\vec u$ which is approximately contained in $\spn(\gvec \mu_1, \dotsc, \gvec \mu_k)$ for which $\norm{\Pmap_0 \vec u}$ is small.
Then, we enter the main loop.
There are three main ideas underlying the main loop and its analysis:
\paragraph{Small Coordinates Decay Rapidly}
There exists a threshold $\tau > 0$ such that if $i \in [\dm]$ satisfies that $\abs{u_i} \leq \tau$, then when applying the gradient iteration $\abs{\hat \GG_i(\vec u)} \leq C \abs{u_i}$ (with $C < 1$) unless $u_i$ is already on the order of $\epsilon$.
We call coordinates of $\vec u$ small if they are below such a threshold and large if they are above it.
This constant $C$ actually gets smaller as the $u_i$ gets smaller, and we are able to see that the small coordinates of $\vec u$, and we see super-exponential decay in the small coordinates of $\vec u$.
This super-exponential decay is seen in the lower bound on $N_1$, which interestingly includes the only dependency on $\epsilon$ seen in the running time of \textsc{FindBasisElement}.
This phenomenon is analyzed in \cref{sec:small-coord-analysis}.

\paragraph{The Big Become Bigger}
During the execution of step~\ref{alg:step:gi-in-main-loop}, we may consider a fixed point $\vec v$ of $\GG / {\sim}$ such that $v_i \neq 0$ if and only if $i$ corresponds to a large coordinate of $\vec w$.
Similarly to what was seen before in \cref{prop:gi-expand-all} in the exact case, if there is an $i$ such that $w_i > v_i$ with a sufficient gap, then the gradient iteration drives one of the large coordinates to become small.
The remaining large coordinates become bigger to compensate.
When finally only one hidden coordinate of $\vec u$ remains big, we have recovered an approximate hidden basis element.
This phenomenon is analyzed in \cref{sec:large-coord-analysis-1}.

\paragraph{Jumping Out of Stagnation}
It is possible for the gradient iteration to stagnate.
In particular, this can occur as follows.
If $\SS \subset [\dm]$ is the set of large coordinates, $\vec v$ is the fixed point of $\GG / {\sim}$ such that $v_i \neq 0$ if and only if $i \in \SS$, and if $\abs{u_i} \leq \abs{v_i}$ (or under the perturbed setting, $\abs{u_i}$ is not sufficiently larger than $\abs{v_i}$ from the unperturbed setting), then the large coordinate progress from the preceding paragraph is not guaranteed.
However, by taking a small random jump from $\vec u$ as is done in steps~\ref{alg:step:draw-jump} and~\ref{alg:step:jump1} of \textsc{FindBasisElement}, then with at least constant probability, we can make one of the large coordinates of $\vec u$ sufficiently greater than the corresponding coordinate of $\vec v$.
Then, the large coordinate analysis from the preceding paragraph applies.
It is from this interplay between the big becoming bigger and the jumping out of stagnation that we are able guarantee with probability $1-\Delta$ that $O(\dm \log(\dm / \Delta))$ iterations of the main loop suffice to drive all but one of the hidden coordinates of $\vec u$ to 0, and hence producing an approximation to one of the hidden basis elements.
This jumping phenomenon is analyzed in \cref{sec:perturb-step-analysis}.

Finally, in \textsc{RobustGI-Recovery} (\cref{alg:FullRecovery}), we run \textsc{FindBasisElement} until all hidden basis elements are well approximated.
More formally, we have the following result.
\begin{thm}\label{thm:full-recovery-main}
  Suppose that
  \begin{compactitem}
  \item $\sigma \leq \frac{C_0}{\sqrt{\dn(1+\delta)}} \big[ \frac {\beta\gamma}{16\alpha\delta}\big]^{\frac 1 \gamma}\dm^{-\frac \delta \gamma}$,
  \item $\epsilon \leq C_1 4^{-\frac{4+2\delta}{\gamma}}\frac{\sigma \beta}{\delta}\big[\frac {\beta\gamma}{\alpha \delta} \big]^{\frac{4\delta + 7}{2\gamma}}\dm^{-\frac \delta \gamma(2\delta - \gamma + \frac 7 2)}\dn^{-\frac 1 2 - \delta}$,
  \item $N_1 \geq C_2 \lceil \log_{1+2\gamma}(\log_2(\frac{\beta}{\delta \epsilon}))\rceil$, and
  \item $N_2 \geq C_3\bigr\lceil 4^{\frac 2 \gamma}\frac{\sqrt \dn}{\sigma} (\frac{\alpha\delta}{\beta\gamma})^{\frac {\delta+2} \gamma}\dm^{\frac \delta \gamma(\delta - \gamma + 2)}[\frac 1 \gamma \log(\frac{\alpha \delta}{\beta \gamma}) + \frac \delta \gamma \log(\dm)]\bigr\rceil \newline \phantom{.}\qquad \ + C_2 \lceil \log_{1+2\gamma}(\log_2(\frac{\beta}{\delta \epsilon}))\rceil$.
  \end{compactitem}
  Let $p \in (0, 1)$,
  and suppose that $I \geq C_3 \dm \lceil \log(\dm / p) \rceil$.
  If $\hat \dm \geq \dm$ and we execute $\gvec \mu_1, \dotsc, \gvec \mu_{\hat \dm} \leftarrow \textsc{RobustGI-Recovery}(\hat \dm, \sigma, \pertgF, N_1, N_2, I)$, then with probability at least $1-p$, $\gvec \mu_1, \dotsc, \gvec \mu_\dm$ is a $C_4\delta \epsilon / \beta$-recovery of $\myZv_1, \dotsc, \myZv_\dm$.
\end{thm}

To simplify the exposition, the proofs of \cref{thm:single-recovery-main,thm:full-recovery-main} are deffered to \cref{sec:proof-details-pert}.

We now consider the running time of \textsc{RobustGI-Recovery}.
First, $I$, $N_1$, and $N_2$ can be viewed as parameters controlling the running time of the algorithm.
More formally, we have the following result.
\begin{thm}\label{thm:running-time}
  Suppose that we are working in a computation model supporting the following operations:
  Basic arithmetic operations, square roots, and trigonometric functions on scalars, branches on conditional; inner products in $\R^\dn$; and computations of $\pertgF(\vec u)$.
  Then,
  \textsc{RobustGI-Recovery} runs in $O(\hat \dm( N_1 + I N_2 ) + \hat \dm \dn^2)$ time.
\end{thm}

To see the $O(\hat \dm\dn^2)$ portion of the upper bound on scalar and vector operations in \cref{thm:running-time}, we note that step~\ref{alg:step-xi-spread} of \textsc{FindBasisElement} can be implemented using Gram-Schmidt orthogonalization involving the $\gvec \mu_i$s and the canonical vectors in the ambient space.
When the desired number of basis elements $\dm$ is known, then $\hat \dm$ can be chosen as $\dm$.
When the number of basis elements is unknown, then $\hat \dm$ may be chosen as $\dn$, and in a more practical setting the values of $\norm{\pertgF(\gvec \mu_\ell)}$ may be thresholded to determine which returned vectors correspond to hidden basis elements.
% Our general bound on the running time for \textsc{RobustGI-Recovery} is stated as follows:
% To obtain the time bound seen in \cref{thm:main-informal}, we choose $N_1 = N_2$, we note that the main loop runs at most $\hat \dm^2$ times, and we use a somewhat larger, simplified lower bound for $\hat \dm^2 N_2 + C\hat \dm \dn^2$ as our lower bound for $N$ (with $C$ a constant).
% The $C\hat \dm \dn^2$ portion of this bound comes from step~\ref{alg:step-xi-spread}, which can be implemented using Gram-Schmidt orthogonalization involving the $\gvec \mu_i$s and canonical vectors in the ambient space.

In addition, we note that $\nabla F$ is an $\epsilon$-approximation to itself for any $\epsilon > 0$.
As such, \cref{thm:full-recovery-main} also implies a polynomial time algorithm for recovering each hidden basis element within a preset but arbitrary precision $\eta$.
In the following Corollary of \cref{thm:full-recovery-main}, we characterize the running time of \textsc{RobustGI-Recovery} as a function of the precision of the hidden basis estimate.
\begin{cor}\label{cor:robust-all-recovery-error-version}
  Suppose that
  \begin{compactitem}
  \item $\sigma \leq \frac{C_0}{\sqrt{\dn(1+\delta)}} \big[ \frac {\beta\gamma}{16\alpha\delta}\big]^{\frac 1 \gamma}\dm^{-\frac \delta \gamma}$,
  \item $\eta \leq C_1 4^{-\frac{4+2\delta}{\gamma}}\sigma \big[\frac {\beta\gamma}{\alpha \delta} \big]^{\frac{4\delta + 7}{2\gamma}}\dm^{-\frac \delta \gamma(2\delta - \gamma + \frac 7 2)}\dn^{-\frac 1 2 - \delta}$,
  \item $N_1 \geq C_2 \lceil \log_{1+2\gamma}(\log_2(\frac{1}{\eta}))\rceil$, and
  \item $N_2 \geq C_3\bigr\lceil 4^{\frac 2 \gamma}\frac{\sqrt \dn}{\sigma} (\frac{\alpha\delta}{\beta\gamma})^{\frac {\delta+2} \gamma}\dm^{\frac \delta \gamma(\delta - \gamma + 2)}[\frac 1 \gamma \log(\frac{\alpha \delta}{\beta \gamma}) + \frac \delta \gamma \log(\dm)]\bigr\rceil \\phantom{.}\qquad\ + C_2 \lceil \log_{1+2\gamma}(\log_2(\frac{1}{\eta}))\rceil$.
  \end{compactitem}
  Let $p \in (0, 1)$,
  and suppose $I \geq C_4 \dm \lceil \log(\dm / p) \rceil$.
  Suppose further that $\pertgF$ is a $C_5 \frac \beta \delta \eta$-approximation to $\nabla F$.
  If $\hat \dm \geq \dm$ and we execute
  \begin{displaymath}
    \gvec \mu_1, \dotsc, \gvec \mu_{\hat \dm} \leftarrow \textsc{RobustGI-Recovery}(\hat \dm, \sigma, \pertgF, N_1, N_2, I) \ ,
  \end{displaymath}
  then with probability at least $1-p$, $\gvec \mu_1, \dotsc, \gvec \mu_\dm$ is an $\eta$-recovery of $\myZv_1, \dotsc, \myZv_\dm$.
\end{cor}

%%% Local Variables:
%%% mode: latex
%%% TeX-master: "main"
%%% End:

\end{vlong}
\begin{vshort}
\refstepcounter{section}\label{sec:robust-ICA}
\end{vshort}
\begin{vlong}
\section{A Provably Robust Algorithm for Independent Component Analysis}
\label{sec:robust-ICA}

In addition to being a very popular technique for blind source separation, Independent Component Analysis (ICA) has been of recent interest in the computer science theory community.
Frieze, et al.~\citep{FriezeJK96} gave an early analysis of ICA in the setting where the underlying source distributions are continuous uniform distributions. The analysis of this setting was simplified in a cryptographic context by~\citep{Nguyen2009}.
More recently, there have been a number of works which discuss provable ICA in the presence of additive Gaussian noise~\citep{Vempala2011,AroraGMS12,Belkin2012,DBLP:conf/stoc/GoyalVX14}. 

In this section, we
show how our {\sef} framework can be used to analyze ICA.
In so doing, we provide the first analysis of a general perturbed ICA model.
We assume throughout this section that $\vec X = A \vec S$ is an ICA model where realizations of $\vec X$ and $\vec S$ are both in $\R^\dn$ (i.e., we consider the fully determined setting in which the number of latent sources equals the ambient dimension of the space).
For a random variable $Y$, we denote its $r$\textsuperscript{th} moment $m_r(Y) := \E[Y^r]$ and its order $r$ cumulant by $\kappa_r(Y)$.
We make the following assumptions:
%\begin{enumerate}[label=B\arabic{enumi}., ref=B\arabic{enumi}, topsep=3pt,partopsep=0pt,itemsep=0pt,parsep=3pt]
\begin{assump}
  \label{ICA-A:first}
  \label{ICA-A:ambiguities}
  $\vec S$ has identity covariance.
\end{assump}
\begin{assump}
  \label{ICA-A:min-kappa4}
  For all $i \in [\dn]$, $\abs{\kappa_4(S_i)} > 0$
\end{assump}
\begin{assump}
  \label{ICA-A:max-m8}
  For all $i \in [\dn]$, $m_8(S_i) < \infty$.
\end{assump}
\begin{assump}
  \label{ICA-A:last} \label{ICA-A:A-orth} $A$ is an orthogonal matrix
  and $\vec S$ has $\vec 0$ mean.
\end{assump}
%\end{enumerate}

\Cref{ICA-A:ambiguities} is commonly used within the ICA literature in order to minimize the ambiguities of the ICA model.
\Cref{ICA-A:min-kappa4} is commonly made for cumulant-based ICA algorithms which are used in practice.
\Cref{ICA-A:max-m8} will play an important role in our error analysis for cumulant estimation.
We include \cref{ICA-A:A-orth} in order to simplify the exposition and more quickly highlight how our framework applies to ICA.
It is common in many ICA algorithms to preprocess the data by placing the data in isotropic position (this is typically referred to as whitening) so that it has $\vec 0$ mean and identity covariance.
After this preprocessing step, $A$ is of the desired form.
By including the final assumption, we remove the necessity of analyzing the whitening step and propagating the resulting error.
Our approach can be generalized to include an error analysis of the whitening step.

We first recall from the discussion on ICA in \cref{sec:example-befs} that the function $F: \sphere^{\dn - 1} \rightarrow \R$ defined by $F(\vec u) := \kappa_4( \ipCanonical{\vec u}{\vec X})$ is a basis encoding function with associated contrasts $g_i(x) := x^4 \kappa_4(S_i)$ (for $i \in [\dn]$) and hidden basis elements $\myZv_i := A_i$ (for $i \in [\dn]$).
We now see that this choice of $F$ is actually a robust {\sef}.

\begin{lem}\label{lem:cum4-BEF-robustness}
  Define $\kappa_{\min} := \min_{i \in [\dn]} \abs{\kappa_4(S_i)}$ and $\kappa_{\max} := \max_{i \in [\dn]}\abs{\kappa_4(S_i)}$.
  Let $F : \sphere^{\dn - 1} \rightarrow \R$ be defined by $F(\vec u) := \kappa_4(\ipCanonical{\vec u}{\vec X})$.
  Then, $F$ is a $(2\kappa_{\max}, 2\kappa_{\min}, 1, 1)$-robust {\sef}.
  % with associated contrasts $g_i(x) := x^4 \kappa_4(S_i)$ and hidden basis elements $\myZv_i := A_i$ (for $i \in [\dn]$).
\end{lem}
\begin{proof}
  Using the definition of $h_i$ from \cref{sec:extrema-structure}, we obtain for all $i \in [\dn]$ that $h_i(x) = g_i(\sign(x)\sqrt{\abs x}) = x^2 \kappa_4(S_i)$.
  Taking derivatives, we see that $h_i''(x) = 2 \kappa_4(S_i)$, and hence that $2 \kappa_{\min} \leq \abs {h_i''(x)} \leq 2 \kappa_{\max}$.
  Recalling \cref{defn:BEF-robust} with $(\alpha, \beta, \gamma, \delta) = (2\kappa_{\max}, 2\kappa_{\min}, 1, 1)$ completes the proof.
\end{proof}

We do not have direct access to $F$.
Instead, we will estimate $F$ from samples.
We note that for any $\vec u \in \sphere^{\dn-1}$, $\var(\ipCanonical{\vec u}{\vec X}) = 1$.
For a 0-mean random variable $Y$ with unit variance, the fourth cumulant is known to take on a very simple form:  $\kappa_4(Y) = m_4(Y) - 3$.
This provides a natural sample estimate for the fourth cumulant in our setting.
Given samples $y \ind 1, y \ind 2, \dotsc, y \ind N$ of a random variable $Y$, we will estimate $\kappa_4(Y)$ by $\hat \kappa_4(y \ind i) := \frac 1 N \sum_{i=1}^N (y \ind i) ^4 - 3$.

Let $p_{\vec Y}$ denote the probability density function of a random vector $\vec Y$.
In order to handle a perturbation away from the ICA model, we will consider metrics
of the form $\mu_k(\vec X, \vec Y) := \int_{t\in \R^{\dn}} \norm{\vec t}^8 \abs{p_{\vec X}(\vec t) - p_{\vec Y}(\vec t)} \mathrm d \vec t$ on the space of probability densities. \vnote{Can use $(1+\norm{t})^8$ rather than $\norm{t}^8$ in order to make it so that $\mu_r(\vec X, \vec Y) \leq \mu_s(\vec X, \vec Y)$ whenever $r \leq s$.}%
We will assume sample access to a random variable $\tvec X$ such that $\mu_8(\vec X, \tvec X)$ is sufficiently small (to be quantified later).
Given samples $\tvec x \ind 1, \tvec x \ind 2, \dotsc, \tvec x \ind N$ i.i.d.\ from $\tvec X$, we estimate $F$ by the function $\hat F(\vec u) := \frac 1 N \sum_{i=1}^N \ipCanonicalp{\vec u}{\tvec x\ind i}^4 - 3$.
The gradient of $\hat F$ is easily computed as $\nabla \hat F(\vec u) = \frac 4 N \sum_{i=1}^N \ipCanonicalp{\vec u}{\tvec x \ind i}^3 \tvec x \ind i $ and acts as an estimate of $\nabla F$.
% and $\HH \hat F(\vec u) = \frac {12} N \sum_{i = 1}^N \ipCanonicalp{\vec u}{\tvec x \ind i}^2 \tvec x \ind i \tvec x \ind i^T$.
As such, we have all of the information required to implement \textsc{RobustGI-Recovery} using $\pertgF:= \nabla \hat F$.

We now provide uniform bounds on the estimate errors for $\nabla F$ under this model.

\begin{lem}\label{lem:ICA-BEF-error-bounds}
  Fix $\delta > 0$ and $\eta > 0$.
  Let $M_8 := \max_{i \in [\dn]} m_8(S_i)$.
  Let $\tvec X$ be a random vector in $\R^\dn$ such that $\mu_8(\vec X, \tvec X)$ is finite.
  Suppose that $\tvec x \ind 1, \tvec x \ind 2, \dotsc, \tvec x \ind N$ are drawn i.i.d.\ from $\hat {\vec X}$ with $N \geq \frac{d^4[M_8 + \mu_8(\vec X, \tvec X)]}{\eta^2 \delta}$.
  If $\hat F(\vec u) := \frac 1 N \sum_{i=1}^\dn \ipCanonicalp{\vec u}{\tvec x\ind i}^4 - 3$ and $F(\vec u) := \kappa_4(\ipCanonical{\vec u}{\vec X})$,
  then with probability $1 - \delta$ the following bounds hold for all $\vec u \in \sphere^{\dn - 1}$: (1) $\abs{F(\vec u) - \hat F(\vec u)} \leq (\eta+\mu_4(\vec X, \tvec X))d^2$ and (2) $\norm{ \nabla F(\vec u) - \nabla \hat F(\vec u)} \leq 4(\eta + \mu_4(\vec X, \tvec X))\dn^2$. % and (3) $\norm{\HH F(\vec u) - \HH \hat F(\vec u)} \leq 12(\eta + \mu_4(\vec X, \tvec X)) \dn^2$.
\end{lem}
\begin{proof}
  In this proof, we proceed with the convention that we are indexing with respect to the hidden basis in which $\myZv_i := A_i$ for all $i \in [\dn]$.
  In particular, this implies $X_i = \ipCanonical{A_i}{\vec X} = S_i$.
  
  We use multi-index notation to compress our discussion as follows:
  $J \in [\dn]^k$ will denote a multi-index $J = (j_1, j_2, \dotsc, j_k)$ such that each $j_\ell \in [\dn]$.
  For a vector $\vec v$, $v_J$ denotes the product $\prod_{\ell = 1}^k v_{j_\ell}$.
  Our objective function $\hat F(\vec u)$ may be expanded as a polynomial of the $u_j$s as follows:
  \begin{align*}
    \hat F(\vec u) 
    &= \frac 1 N \sum_{i=1}^N \ipCanonicalp{\vec u}{\tvec x \ind i}^4 - 3
%    &= \frac 1 N \sum_{i=1}^N \sum_{j_1 = 1}^\dn \cdots \sum_{j_4 = 1}^\dn u_{j_1}u_{j_2}u_{j_3}u_{j_4} \hat x_{j_1} \ind i \hat x_{j_2} \ind i \hat x_{j_3} \ind i \hat x_{j_4} \ind i
    = \frac 1 N \sum_{i=1}^N \sum_{J \in [\dn]^4} u_J \hat x_J(i) - 3 \\
    &= \sum_{J \in [\dn]^4} u_J \biggr[\frac 1 N \sum_{i=1}^N \hat x_J \ind i\biggr] - 3 \ .
  \end{align*}
  By a similar argument, it can be shown that $F(\vec u) = \sum_{J \in [\dn]^4} u_J \E[X_J] - 3$.
  We obtain the error bound $\abs{\hat F(\vec u) - F(\vec u)} \leq \sum_{J \in [\dn]^4} u_J \abs{ \frac 1 N \sum_{i=1}^\dn \hat x_J(i) - \E[X_J] }$.
  Similarly, we can bound the error estimate for $\nabla F(\vec u)$: % and $\HH F(\vec u)$.
  \begin{align*}
    \norm{\nabla \hat F(\vec u) - \nabla F(\vec u)}
    &= 4 \sum_{J \in [\dn]^3} u_J \biggr\lVert\frac 1 N \sum_{i=1}^{N} \hat x_J \ind i \tvec x \ind i - \E[X_J \vec X]\biggr\rVert. % \\
    % \norm{\HH \hat F(\vec u) - \HH F(\vec u)}
    % &= 12 \sum_{J \in [\dn]^2} u_J \biggr\lVert \frac 1 N \sum_{i=1}^{N} \hat x_J \ind i \tvec x \ind i \tvec x \ind i ^T - \E[X_J \vec X \vec X^T] \biggr\rVert \ .
  \end{align*}
  With $J \in [\dn]^4$, we define $\varepsilon_J := \frac 1 N \sum_{i=1}^N \hat x_J(i) - \E[X_J]$ and $\varepsilon_{\max} := \max_{J \in [\dn]^4} \abs {\varepsilon_J}$.
  Using that each $\vec u$ is a unit vector, we see that $\abs{\sum_{J \in [\dn]^k} u_J} \leq \norm[1]{\vec u}^k \leq \dn^{k/2}$.
  Using the norm inequalities that for vector $\vec v \in \R^\dn$ and matrix $A \in \R^{\dn \times \dn}$, $\norm{\vec v} \leq  \max_{i \in [\dn]} \abs {v_i} \sqrt \dn$ and $\norm{A} \leq \max_{(i, j) \in [\dn]^2} \abs{a_{ij}} \dn$, we are able to obtain the following bounds for all $\vec u \in \sphere^{\dn - 1}$:  $\abs{\hat F(\vec u) - F(\vec u)} \leq \dn^2 \varepsilon_{\max}$ and $\norm{\nabla \hat F(\vec u) - \nabla F(\vec u)} \leq 4 \dn^2 \varepsilon_{\max}$. % , and $\norm{ \HH \hat F(\vec u) - \HH F(\vec u) } \leq 12 \dn^2 \varepsilon_{\max}$.
  All that remains is to bound $\varepsilon_{\max}$.
  To do so, we will bound each $\varepsilon_J$ using Chebyshev's inequality.
  
  For each $J \in [\dn]^4$, we obtain under the sampling process that
  \begin{align*}
    \var\biggr(\frac 1 N \sum_{i = 1}^{N}\hat x_J \ind i\biggr)
    &= \frac 1 {N^2} \var\biggr(\sum_{i = 1}^{N}\hat x_J \ind i\biggr)
    = \frac 1 {N} \var( \hat X_J )
    \leq \frac 1 N \E[(\hat X_J)^2] \\
    &\leq \frac 1 N \E[\hat X_{j_1}^4 \hat X_{j_2}^4]^{\frac 1 2}\E[\hat X_{j_3}^4 \hat X_{j_4}^4]^{\frac 1 2}
    \leq \frac 1 N \biggr({\prod_{\ell=1}^{4}}\E[\hat X_{j_\ell}^8]\biggr)^{\frac 1 4}
    \leq \frac 1 N \max_{\ell \in [\dn]} \E[\hat X_\ell^8]
  \end{align*}
  where the first equality uses that variance is order-2 homogeneous\index{homogeneity}, the second equality uses independence, the first inequality follows from the formula $\var(\hat X_J) = \E[(\hat X_J)^2] - \E[\hat X_J]^2$, and
  the second and third inequalities use the Cauchy-Schwartz inequality.
  We bound $\max_{\ell \in [\dn]} \E[\hat X_\ell^8]$ as:
  \begin{align*}
    \E[\hat X_\ell^8]
    &= \int_{\vec t \in \R^\dn} t_{\ell}^8 p_{\tvec X}(\vec t) \mathrm d \vec t \\
    &= \int_{\vec t \in \R^\dn} t_{\ell}^8 p_{\vec X}(\vec t) \mathrm d \vec t +
      \int_{\vec t \in \R^\dn} t_{\ell}^8 (p_{\tvec X}(\vec t) - p_{\vec X}(\vec t)) \mathrm d \vec t
    \leq M_8 + \mu_8(\vec X, \tvec X) \ .
  \end{align*}
  Thus, $\var\bigr(\frac 1 N \sum_{i = 1}^{N}\hat x_J \ind i\bigr) \leq \frac 1 N (M_8 + \mu_8(\vec X, \tvec X))$.

  Chebyshev's inequality states that for any random variable $Y$ and any $k > 0$, $\Pr[\abs{ Y - \E[Y] } \geq k \sqrt{\var(Y)}] \leq \frac 1 {k^2}$.
  We fix any $J \in [\dn]^4$, choose $Y = \frac 1 N \sum_{i = 1}^N \hat x_J(i)$, and choose $k = \frac{\dn^2}{\sqrt \delta}$.
  We obtain that with probability at least $1 - \delta / {\dn^4}$,
  \begin{equation}\label{eq:Chebyshev-bound}
    \Abs{\frac 1 N \sum_{i=1}^N \hat x_J(i) - \E[\hat X_J]} < \frac{\dn^2}{\sqrt \delta} \sqrt{\frac 1 N (M_8 + \mu_8(\vec X, \tvec X))} \leq \eta
  \end{equation}
  by our given bound on $N$.
  Taking a union bound, then with probability at least $1 - \delta$, the bound in \cref{eq:Chebyshev-bound} holds for all $J \in [\dn]^4$.

  We then obtain the following bound on each $\varepsilon_{J}$ for each $J \in [\dn]^4$ (with probability at least $1-\delta$):
  \begin{align*}
    \abs{\varepsilon_{J}}
    &= \Abs{\frac 1 N \sum_{i=1}^N \hat x_J(i) - \E[X_J]}
    \leq \eta + \abs{\E[\hat X_J] - \E[X_J]} \\
    &= \eta + \Abs{\int_{\vec t \in \R^4} t_J[p_{\tvec X}(\vec t) - p_{\vec X}(\vec t)] \mathrm d \vec t }
    \leq \eta + \mu_4(\tvec X, \vec X) \ .
  \end{align*}
  To obtain the result, we use $\varepsilon_{\max} \leq \eta + \mu_4(\tvec X, \vec X)$ in our previously derived uniform bounds over all $\vec u \in \sphere^{\dn-1}$ of $\abs{\hat F(\vec u) - F(\vec u)} \leq \dn^2 \varepsilon_{\max}$ and $\norm{\nabla \hat F(\vec u) - \nabla F(\vec u)} \leq 4 \dn^2 \varepsilon_{\max}$. % and $\norm{ \HH \hat F(\vec u) - \HH F(\vec u) } \leq 12 \dn^2 \varepsilon_{\max}$.
\end{proof}

We now state our result for ICA.
We assume $\vec X = A \vec S$ is an ICA model satisfying \crefrange{ICA-A:ambiguities}{ICA-A:A-orth} with associated constants $\kappa_{\min} := \min_{i \in [\dn]} \abs{\kappa_4(S_i)}$, $\kappa_{\max} := \max_{i \in [\dn]} \abs{\kappa_4(S_i)}$, and $M_8 := \max_{i \in [\dn]} m_8(S_i)$.
We assume $\tvec X$ is a perturbed ICA model, and we approximate the {\sef} $F(\ipCanonical{\vec u}{\vec X}) := \kappa_4(\ipCanonical{\vec u}{\vec X})$ from an i.i.d.\@ sample $\tvec x(1), \dotsc, \tvec x(\NN)$ of $\tvec X$.
That is, we define $\hat F(\vec u) := \frac 1 \NN \sum_{i=1}^\dn \ipCanonicalp{\vec u}{\tvec x \ind i}^4 - 3$ and compute its gradient as $\nabla \hat F(\vec u) := \frac 4 \NN \sum_{i=1}^\dn \ipCanonicalp{\vec u}{\tvec x \ind i}^3\tvec x \ind i$.
% $\HH \hat F(\vec u) := \frac {12} \NN \sum_{i=1}^\dn \ipCanonicalp{\vec u}{\tvec x \ind i}^2\tvec x \ind i \tvec x \ind i ^T$.
We further assume that we are working in a computation model which can perform the following operations in $O(\dn)$ time:  Inner products in $\R^\dn$, scalar operations including basic arithmetic operations, trigonometric functions, square roots, and branches on conditionals.
In the following, $C_1, C_2, \dotsc$ are positive universal constants.

\begin{thm}\label{thm:ICA-result-reprise}
Fix $\delta > 0$ and $\varepsilon > 0$.
Suppose  $\sigma \leq \frac{C_0}{\dn^2}\frac{\kappa_{\min}}{\kappa_{\max}}$,
$\varepsilon \leq C_1 \sigma \big( \frac{\kappa_{\min}}{\kappa_{\max}} \big)^{9/2}\dn^{-6}$, 
$\mu_4(\tvec X, \vec X) \leq C_2 \frac{\kappa_{\min}}{d^2}\varepsilon$, and
$\NN \geq \frac{C_3 \dn^8[M_8 + \mu_8(\tvec X, \vec X)]}{\kappa_{\min}^2 \varepsilon^2 \delta}$.
 Suppose we execute $\hat A_1, \dotsc, \hat A_\dn \leftarrow \textsc{RobustGI-Recovery}(\dn, \sigma, \nabla \hat F, N_1, N_2, I)$, where  
 $N_1 \geq C_4 \big\lceil \log_2(\log_2(\frac{1}{\varepsilon}))\big\rceil$, 
 $N_2 \geq C_5 \big\lceil \frac {\dn^{2.5}}{\sigma} \big(\frac {\kappa_{\max}}{\kappa_{\min}}\big)^3 \log(\dn \cdot \frac{\kappa_{\max} }{\kappa_{\min}}) \big\rceil + \big\lceil \log_2(\log_2(\frac{1}{\varepsilon}))\big\rceil$,
 and $I \geq C_6 \dn \log (\dn / \delta)$.
 Then, with probability at least $1 - \delta$,
 $\hat A_1, \dotsc, \hat A_\dn$ is a $\varepsilon$-recovery of $A_1, \dotsc A_\dn$.
 % there exists a permutation $\pi$ of $[\dn]$ and sign values $s_i \in \{ \pm 1 \}$ such that $\norm{\hat A_i - s_i A_{\pi(i)}} \leq \varepsilon$ for all $i \in [\dn]$.
 \textsc{RobustGI-Recovery} recovers such $\hat A_i$s in $C_7 \NN[\dn^4 + \dn^2 N_1 + \dn^2 I N_2]$ time.
\end{thm}
% \vnote{Note on constant values:  $C = \frac{\sqrt 2}{80 \sqrt 3 \times 16^{5/2}} = \frac{\sqrt 2}{81920 \sqrt 3}$}
\begin{proof}
  % We first wish to show that with probability at least $1-\delta$ over the sampling process of $\hat {\vec X}$, $\hat F$ is a $\frac{\kappa_{\min}}{2\sqrt 2}\varepsilon$-approximation to $F$.
  By \cref{lem:ICA-BEF-error-bounds} with the choice of $\eta = O(\frac{\kappa_{\min}}{\dn^2}\varepsilon)$, we obtain that with probability at least $1 - \frac \delta 2$, 
  \begin{align*}
    \norm{\nabla F(\vec u) - \nabla \hat F(\vec u)}
    &\leq 4(\eta + \mu_4(\vec X, \hat {\vec X}))d^2
    \leq O\Big(\frac{\kappa_{\min}}{\dn^2}\varepsilon + \frac{\kappa_{\min}}{\dn^2}\varepsilon\Big)d^2
    = O(\kappa_{\min} \varepsilon), 
  \end{align*}
  for all $\vec u \in \sphere^{\dn - 1}$.
  In particular, $\hat F$ is an $O(\kappa_{\min}\varepsilon)$-approximation to $F$.

  We recall from \cref{lem:cum4-BEF-robustness} that $F$ is an $(2\kappa_{\max}, 2\kappa_{\min}, 1, 1)$-robust {\sef}.
  As such, we may apply \cref{cor:robust-all-recovery-error-version} to obtain that $\textsc{RobustGI-Recovery}$ returns vectors $\hat A_1, \dotsc, \hat A_\dn$ of the desired form.
  Finally, we note that within our computational model for this theorem, computations of $\nabla \hat F(\vec u)$ take $O(\NN \dn)$ time.
  Thus, applying \cref{thm:running-time} with $\hat \dm = \dn$ yields the claimed time bound.
\end{proof}

%%% Local Variables: 
%%% mode: latex
%%% TeX-master: "main"
%%% End: 

\end{vlong}
\begin{vshort}
  \refstepcounter{section}\label{sec:spectral-pert-analysis}
\end{vshort}
%\begin{vlong}
%  \input{spectral-error-analysis}
%\end{vlong}

\appendix

\begin{vlong}

\begin{vstandard}
\section{Chart of notation}

We use a number of notations throughout this paper, many of which are standard and some of which are not.
For the reader's reference, we list notations used throughout the paper here.

\begin{longtable}{|c|X|}\hline
  $\grad$ & The gradient operator. \\
  $\HH$ & The Hessian operator. \\
  $\partial_i$ & The derivative operator with respect to the $i$\textsuperscript{th} basis element of the space, i.e.\@ $\myZv_i$. \\
  % When working in $[\dn]$, this is the derivative with respect to the direction $\myZv_i$. \\
  $\Jacob f_{\vec x}$ & The Jacobian of $f$ evaluated at $\vec x$. \\
  $A \disunion B$ & The union operation between two disjoint sets $A$ and $B$. \\
  $f \restr D$ & The restriction of $f$ to the domain $D$. \\
  $[\vec u]$ & The equivalence class $\{ \vec v \suchthat \vec v \sim \vec u\}$. \\
  % $\phi$ & The map from $S^{\dn - 1} / {\sim}$ to $\porth^{\dn - 1}$ given by $\phi_i([\vec u]) = \abs{u_i}$.\\
  % $\mu$ & The distance metric on $S^{\dn-1}/\sim$ defined by $\mu([\vec u], [\vec v]) := \norm{\phi(\vec u) - \phi(\vec v)}$. \\
  $[k]$ & The set $\{1, 2, \dotsc, k\}$. \\
  $\lceil \bullet \rceil$ & This is the ceiling operator, i.e.\@ $\lceil x \rceil$ is the least integer which is greater than or equal to $x$. \\
  $\abs \bullet$ & The modulus or absolute value operation. \\
  $\norm \bullet$ & The standard Euclidean 2-norm. \\
  $\ip \bullet \bullet$ & The standard Euclidean inner product, i.e., the dot product. \\
  % $A$ & The matrix $A$. \\
  % $a_{ij}$ & The entry contained in the $i$\textsuperscript{th} row and $j$\textsuperscript{th} column of $A$. \\
  % $B(\vec x, r)$ & The open ball centered at $\vec x$ with radius $r$. \\
  $\overline{B(\vec x, r)}$ & The closed ball centered at $\vec x$ with radius $r$. \\
  $\indicator{E}$ & The indicator function of the event $E$. \\
  $\dn$ & Dimensionality of the ambient space. \\
  $\myZv_i$ & The vectors $\myZv_1, \dotsc, \myZv_\dm$ are the hidden basis elements encoded within {\ansef}.  The vectors $\myZv_{\dm+1},\dotsc, \myZv_{\dn}$ are chosen arbitrarily in order to make $\myZv_1, \dotsc, \myZv_\dn$ an orthonormal basis of $\R^\dn$. \\
  $\E$ & The expectation operator for random variables. \\
  $\fg$ & {\Ansef} with expanded form $\fg(\vec u) = \sum_{i=1}^\dm \alpha_i g(\beta_i u_i)$, defined on page~\pageref{defn:sef}. \\
  $\bar F$ & The {\psef} associated with {\sef} $\fg$. \\
  $\GG$ & The gradient iteration functions associated with {\ansef} $\fg$. \\
  $\Id$ & The identity matrix. \\
  $\dm$ & Number of distinguished hidden basis vectors $\myZv_1, \dotsc, \myZv_\dm$.  Note that $\dm \leq \dn$. \\
  $\Pmap_{\SS}$ & The projection matrix $\sum_{i \in \SS} \myZv_i\myZv_i^T$. \\  
  % $\pmap_{\vec v}$ & The local coordinate map onto the projective space of the sphere:  $\pmap_{\vec v} : S^+_{\vec v} \rightarrow T_{\vec v}S^{\dn - 1}$ defined as $\pmap_{\vec v}(\vec u) := \frac{\vec u}{\ip{\vec u}{\vec v}}$.\\
  $\porth^{\dn-1}$ & The all positive orthant of $S^{\dn - 1}$:  $\{\vec u \in S^{\dn - 1} \suchthat u_i \geq 0 \text{ for all } i \in [\dn] \}$. \\
  $\orthsymb_{\vec v}^{\dn-1}$ & It is assumed that $\vec v \in \R^\dn$ is a vector of signs ($v_i \in \{+1, -1\}$ for all $i \in [\dn]$).
  Then, $\orthsymb_{\vec v}^{\dn-1} := \{\vec u \in S^{\dn - 1} \suchthat v_iu_i \geq 0 \}$ is the orthant of $S^{\dn - 1}$ containing $\vec v$. \\
  % $\vec p_i$ & A basis vector of the tangent space $T_{\vec v}S^{\dn -1}$. \\
  $S^{\dn - 1}$ & The unit sphere in $\R^\dn$:  $\{\vec u \in \R^\dn \suchthat \norm {\vec u} = 1\}$. \\
  $\compl \SS$ & The complement of $\SS$, typically $[\dn] \setminus \SS$. \\
  % $S^+_{\vec v}$ & The hemisphere $\{\vec u \suchthat \norm{\vec u} = 1 \text{ and } \ip{\vec u}{\vec v} > 0\}$. \\
  $\sign(\bullet)$ & The sign indicator on $\R$ defined by $\sign(x) :=
  \begin{cases}
    x/\abs{x} & \text{if $x \neq 0$} \\
    0 & \text{if $x = 0$} \\
  \end{cases}
$ . \\
  $\sim$ & The equivalence relation defined on $S^{\dn - 1}$ given by $\vec u \sim \vec v$ if for each $i \in [\dn]$, $\abs{u_i} = \abs{v_i}$. \\
  % $T_{\vec v} M$ & Tangent space of a manifold $M$ at the point $\vec v$. \\
  % $\vec v$ & A vector $\vec v$. \\
  % $v_i$ & The $i$\textsuperscript{th} element of $\vec v$.
  %        When $\vec v \in \R^\dn$, indexing is done with respect to the hidden basis $\myZv_1, \dotsc, \myZv_\dn$.
  %        In this case, the $i$\textsuperscript{th} element is $v_i := \ip{\vec v}{\myZv_i}$. \\
  $\Pr$ & The probability operator for random events. \\\
  $T_{\vec v} S^{\dn - 1}$ & The tangent space of $S^{\dn - 1}$ at $\vec v$, i.e.\@ $T_{\vec v} S^{\dn - 1} = \vec v^\perp$. \\
  $\vec v \elpow r$ & Vector $\vec v$ taken to the element-wise exponent of $r$, i.e., $(\vec v \elpow r)_i = v_i^r$.\\
  $\vol_{k-1}$ & The canonical volume measure on $S^{k-1}$.  When $k = \dn$, the subscript is often suppressed. \\ \hline
\end{longtable}

%%% Local Variables: 
%%% mode: latex
%%% TeX-master: "main"
%%% End: 

\end{vstandard}

\section{Proof Details for Perturbed Gradient Iteration Results}
\label{sec:proof-details-pert}
In this appendix, we prove the theorem statements made in section~\ref{sec:gi-error-analysis} about our robust algorithm for hidden basis recovery.
We continue with the notation introduced in that section.
In addition, we will make use of several new notations in our analysis.

Given a $\SS \subset [\dn]$, we define the projection matrix $\Pmap_\SS := \sum_{i \in \SS} \myZv_i\myZv_i^T$.
In particular, this implies $\Pmap_\SS \vec u := \sum_{i \in \SS} u_i \myZv_i$.
We will denote the set complement by $\bar \SS := [\dn] \setminus \SS$.
Two projections will be of particular interest:  the projection onto the distinguished basis elements $\Pmap_{[\dm]} \vec u := \sum_{i=1}^\dm u_i \myZv_i$ and its complement projection which we will denote by $\Pmap_0 \vec u := \sum_{i=\dm + 1}^\dn u_i \myZv_i$.
In addition, if $\mathcal X$ is a subspace of $\R^\dn$, we will denote by $\Pmap_{\mathcal X}$ the orthogonal projection operator onto the subspace $\mathcal X$.
In particular, if $\SS \subset [\dn]$, then the operators $\Pmap_\SS$ and $\Pmap_{\spn(\{\myZv_i \suchthat i \in \SS\})}$ are identical.

\subsection{Small Coordinates Decay Rapidly}
\label{sec:small-coord-analysis}
We will be particularly interested in sequences under the gradient iteration, that is sequences of the form: $\{\vec u(n)\}_{n=0}^\infty$ defined recursively by $\hat \GG(\vec u(n)) = \vec u(n-1)$ and $\vec u(0) \in \sphere^{\dn -1}$.
In this section, we demonstrate two main things about sequences of this form:
First, $\norm{\Pmap_0 \vec u(n)}$ should rapidly become very small.
Second, for any $i \in [\dm]$ such that $u_i(0)$ has sufficiently small magnitude, then the gradient iteration should make $u_i(n)$ decay rapidly until it is very small.
Using these two ideas, we will be able to guarantee that under applications of gradient iteration, the number of hidden coordinates of $\vec u(n)$ which are near zero out can only increase.

We quantify these effects in the following Lemmas.
\Cref{lem:small-coord-decrease-step} characterizes how a single step of the gradient iteration makes $\norm{\Pmap_0 \vec u}$ and the small coordinates of $\vec u$ contract.
Then, \cref{lem:small-coord-decrease-time} expands upon \cref{lem:small-coord-decrease-step} to provide a bound on the number of steps required to decay the small coordinates of $\vec u(n)$ down to a magnitude of order $\epsilon$.

\begin{lem}\label{lem:small-coord-decrease-step}
  Fix $\vec u \in \sphere^{\dn - 1}$ such that $\epsilon < \frac 1 2 \norm{\nabla F(\vec u)}$.
  The following hold:
  \begin{enumerate}
    \item\label{lem:small-coord-trivial} $\norm{\Pmap_0\hat \GG(\vec u)} \leq \frac{2 \epsilon}{\norm{\nabla F(\vec u)}}$.
    \item\label{lem:small-coord-projection}
      Fix any $C \geq 0$.
      Let $\SS \subset [\dn]$ be such that $\abs{u_i} \leq \bigr[\frac {C\gamma}{8 \alpha} \norm{ \nabla F(\vec u) }\bigr]^{\frac 1 {2\gamma}}$ for all $i \in \SS \cap [\dm]$.
      Then, $\norm{\Pmap_{\SS} \hat \GG(\vec u)} \leq \max( C \norm{\Pmap_{\SS \cap [\dm]} \vec u}, \frac{4 \epsilon}{\norm{\nabla F(\vec u)}} )$.
  \end{enumerate}
\end{lem}
\begin{proof}
  Let $A \subset [\dn]$.
  Expanding $\norm{\Pmap_A \hat \GG(\vec u)}$ we obtain for each $i \in [\dn]$:
  \begin{equation} \label{eq:Gi-upper-bound-noassumpt}
    \norm{\Pmap_A \hat \GG(\vec u)}
    = \frac{\norm{\Pmap_A \pertgF(\vec u)}}{\norm{\pertgF(\vec u)}}
    \leq \frac{\norm{\Pmap_A \nabla F(\vec u)} + \epsilon}{\norm{\nabla F(\vec u)} - \epsilon} \ .
  \end{equation}
  As we assumed $\epsilon \leq \frac 1 2 \norm{\nabla F(\vec u)}$,
  \begin{equation}\label{eq:Gi-upper-1}
    \norm{\Pmap_A \hat \GG(\vec u)}
    \leq \frac{\norm{\Pmap_A \nabla F(\vec u)} + \epsilon}{\norm{\nabla F(\vec u)} - \frac 1 2 \norm{\nabla F(\vec u)}}
    = 2 \cdot \frac{\norm{\Pmap_A \nabla F(\vec u)} + \epsilon}{\norm{\nabla F(\vec u)}} \ .
  \end{equation}
  If $A = [\dn] \setminus [\dm]$, then $\Pmap_A \nabla F(\vec u) = 0$, and \cref{eq:Gi-upper-1} implies that $\norm{\Pmap_0 \hat \GG(\vec u)} = \norm{\Pmap_A \hat \GG(\vec u)} \leq \frac{2 \epsilon}{\norm{\nabla F(\vec u)}}$.

  We now prove part~\ref{lem:small-coord-projection}.
  If $\norm{\Pmap_\SS\nabla F(\vec u)} \leq \epsilon$, then \cref{eq:Gi-upper-1} implies that $\norm{\Pmap_\SS \hat \GG(\vec u)} \leq \frac{4 \epsilon}{\norm{\nabla F(\vec u)}}$.
  If $\epsilon \leq \norm{\Pmap_\SS\nabla F(\vec u)}$, then we obtain from \cref{eq:Gi-upper-1} that
    $\norm{\Pmap_\SS \hat \GG(\vec u)}
    \leq \frac{4 \norm{\Pmap_\SS\nabla F(\vec u)}}{\norm{\nabla F(\vec u)}}$.
  We expand $\norm{\Pmap_\SS \nabla F(\vec u)}$ to obtain:
  \begin{align*}
    \norm{\Pmap_{\SS} \nabla F(\vec u)}^2
    &= \sum_{i \in \SS} \abs{g_i'(u_i)}^2
    \leq \sum_{i \in \SS\cap [\dm]} (2\frac \alpha \beta \abs{u_i}^{1+2\gamma})^2 \\
    &\leq \sum_{i \in \SS\cap [\dm]} (2 \frac \alpha \beta \abs {u_i} \cdot \frac{C \gamma}{8 \alpha} \norm{\nabla F(\vec u)})^2
    \leq [\frac{C}{4}\norm{\nabla F(\vec u)}\norm{\Pmap_{\SS \cap [\dm]}\vec u}]^2 \ ,
  \end{align*}
  by using \cref{lem:F-robust-implications} for the first inequality and we use the upper bound on each $\abs{u_i}$ for the second inequality.
  It follows that $\norm{\Pmap_{\SS} \hat \GG(\vec u)} \leq C \norm{\Pmap_{\SS \cap [\dm]} \vec u}^2$.
  Whether $\epsilon \leq \norm{\Pmap_\SS \nabla F(\vec u)}$ or $\epsilon \geq \norm{\Pmap_\SS \nabla F(\vec u)}$, % we obtain our desired result that
  $\norm{\Pmap_\SS \hat \GG(\vec u)} \leq \max(C \norm {\Pmap_{\SS \cap [\dm]}\vec u},\allowbreak \frac{4 \epsilon}{\norm{\nabla F(\vec u)}})$ holds.
\end{proof}

\begin{lem}\label{lem:small-coord-decrease-time}
  Let $\{ \vec u(n) \}_{n=0}^\infty$ be a sequence in $\SS^{\dn-1}$ defined recursively by $\vec u(n) = \hat \GG(\vec u(n\nobreak-\nobreak 1))$.
  Let $L > 0$ be such that $\norm{\nabla F(\vec u(n))} \geq L$ for all $n \in \N$.
  Fix $\SS \subset [\dn]$.
  Suppose $\epsilon < \min\Big(\frac 1 2 L, \allowbreak \bigr(\frac{\gamma L^{1+2\gamma}}{8 \cdot 4^{2\gamma}\alpha}\bigr)^{\frac 1 {2\gamma}} \Big)$ and suppose there exists $C \in (0, 1)$ such that $\norm{ \Pmap_{\SS \cap [\dm]} \vec u(0)} \leq [\frac {C\gamma}{8 \alpha} L]^{\frac 1 {2\gamma}}$.
  If
  \begin{displaymath}
    N \geq
    \log_{1+2\gamma}\left(
      \log_{ \frac 1 C } \left(\frac{L \gamma}{8\alpha}\right)
      + 2\gamma \log_{ \frac 1 C } \left(\frac L {4\epsilon}\right)
    \right) \ ,
  \end{displaymath}
  is a positive integer,
  then for each $n \geq N$, $\norm{\Pmap_\SS \vec u(n)} \leq \frac{4 \epsilon}{L}$.
\end{lem}
\begin{proof}
  Let $N_0$ denote the least integer such that $\norm{\Pmap_\SS \vec u(N_0)} \leq \frac{ 4 \epsilon }{\norm{ \nabla F(\vec u)} }$ (or $\infty$ if it does not exist).
  Also, for compactness of notation, we define $A := \SS \cap [\dm]$.
  \begin{claim}\label{claim:small-coord-exponentiation}
    For each $n < N_0$, $\norm{\Pmap_{A} \vec u(n)} \leq [\frac \gamma {8 \alpha} C^{(1+2\gamma)^n} L]^{\frac 1 {2\gamma}}$.
    % \vnote{xxx TODO: Make this contain a $\sqrt{\abs{\SS \cap
    % [\dm]}}$ by allowing coordinate-wise leq in the original
    % formulation.  This will make the time barely worse (though still
    % log log) but will allow for smaller $\epsilon$ later. xxx}
  \end{claim}
  \begin{subproof}
    We proceed by induction on $n$.
    The base case of $n=0$ is true from the givens of this Lemma.
    We now choose $k < N_0 - 1$ and suppose that the claim holds for $n = k$.
    We see
    \begin{align*}
      \norm{\Pmap_\SS \vec u(k+1)}
      &\leq C^{(1+2\gamma)^k}\norm{\Pmap_{A} \vec u(k)}
      \leq C^{(1+2\gamma)^k}[\frac \gamma {8\alpha} C^{(1+2\gamma)^k} L]^{\frac 1 {2\gamma}} \\
      & = [\frac \gamma {8\alpha} C^{2\gamma(1+2\gamma)^k+(1+2\gamma)^k} L]^{\frac 1 {2\gamma}}
      = [\frac \gamma {8\alpha} C^{(1+2\gamma)^{k+1}} L]^{\frac 1 {2\gamma}} \ ,
    \end{align*}
    using \cref{lem:small-coord-decrease-step} in the first inequality.
    Noting that $\norm{\Pmap_{A}\vec u(k+1)} \leq \norm{\Pmap_\SS \vec u(k+1)}$ since $A \subset \SS$ gives the desired result.
  \end{subproof}
  By our assumptions, we may write the lower bound on $N$ as
  $\log_{1+2\gamma}\bigr(\log_{\frac 1  C} (\frac{\gamma L^{1+2\gamma}}{8 \cdot 4^{2\gamma}\alpha\epsilon^{2\gamma}})\bigr)$. Thus,
  % Note that
  \begin{align*}
    [\frac \gamma {8 \alpha} C^{(1+2\gamma)^N} L]^{\frac 1 {2\gamma}}
    &\leq
    \Big[\frac \gamma {8 \alpha} C^{\log_{\frac 1  C} \bigr(\frac{\gamma L^{1+2\gamma}}{8\cdot 4^{2\gamma}\alpha\epsilon^{2\gamma}}\bigr)}L\Big]^{\frac 1 {2\gamma}}
    = \Big[\frac \gamma {8 \alpha}  \Big(\frac{8 \cdot 4^{2\gamma}\alpha\epsilon^{2\gamma}}{\gamma L^{1+2\gamma}}\Big) L\Big]^{\frac 1 {2\gamma}}
    \leq \frac {4 \epsilon} L \ .
  \end{align*}
  By \cref{claim:small-coord-exponentiation}, it follows that $N_0 \leq N$.

  We note that for some constant $C' \in [0, 1)$,
  \begin{displaymath}
    \frac{4\epsilon} L \leq
    4C' \Bigr(\frac{\gamma L^{1+2\gamma}}{8\cdot 4^{2\gamma}\alpha}\Bigr)^{\frac 1 {2\gamma}} / L
    = [\frac \gamma {8 \alpha} (C')^{2\gamma} L]^{\frac 1 {2\gamma}} \ .
  \end{displaymath}
  If $\norm{\Pmap_A \vec u(n)} \leq \frac {4\epsilon}{L}$, then \cref{lem:small-coord-decrease-step} implies that
  \begin{displaymath}
    \norm{\Pmap_A \vec u(n+1)} \leq \norm{\Pmap_\SS \vec u(n+1)} \leq \max\left((C')^{2\gamma}\norm{\Pmap_A \vec u(n)},\ \tfrac{4\epsilon} L\right) \leq \tfrac{4\epsilon} L \ .
  \end{displaymath}
  It follows by induction on $n$ that $\norm{\Pmap_\SS \vec u(n)} \leq \frac{4\epsilon} L$ for all $n \geq N_0$.
\end{proof}

In \cref{lem:small-coord-decrease-step,lem:small-coord-decrease-time}, one detail seems to be missing, namely the dependence on $\norm{\nabla F(\vec u)}$.
Since during most steps of \textsc{FindBasisElement} $\norm{\Pmap_0 \vec u}$ will be small, we will typically be able use of the following lemma to lower bound $\norm{\nabla F(\vec u)}$.
\begin{lem}\label{lem:P0small-gradf-lower-bound}
  Let $\vec u \in \sphere^{\dn - 1}$.
  Let $\SS \subset [\dm]$ be non-empty.
  If $\norm{\Pmap_{\compl \SS} \vec u} \leq \frac{1}{\sqrt{2(1+2\delta)}}$, then $\norm{\Pmap_{\SS} \vec u}^{1+2\delta} \geq \frac 1 2$ and $\norm{\nabla F(\vec u)} \geq \frac{\beta}{\delta} \abs{\SS}^{-\delta}$.
\end{lem}
Note that in the worst case where $\SS = [\dm]$, may apply \cref{lem:P0small-gradf-lower-bound} to obtain the lower bound that $\norm{\nabla F(\vec u)} \geq \frac \beta \delta \dm^{-\delta}$.
\begin{proof}[\proofword\@of \cref{lem:P0small-gradf-lower-bound}]
  Once we prove that $\norm{\Pmap_{\SS} \vec u}^{1+2\delta} \geq \frac 1 2$, then the lower bound on $\norm{\nabla F(\vec u)}$ follows from \cref{lem:fg-upper-bounds}.
  We now focus on the proof of the lower bound for $\norm{\Pmap_{\SS} \vec u}^{1+2\delta}$.

  We let $f : \R \rightarrow \R$ be defined by $f(x) = x^{2+4\delta}$.
  As such, $\sqrt{f(\norm{\Pmap_{\SS}\vec u})} = \norm{\Pmap_{\SS}\vec u}^{1+2\delta}$.
  The Taylor expansion of $f$ around 1 for any $x \in [0, 1]$ is
  \begin{displaymath}
    f(x) = f(1) + f'(1)(x-1) + \frac 1 2 f''(y)(x - 1)^2
  \end{displaymath}
  for some $y \in [x, 1]$.
  Notice that $\frac 1 2 f''(y)(x-1)^2 = \frac 1 2 (2+4\delta)(1+4\delta)y^{4\delta}(x-1)^2 \geq 0$.
  As such, $f(x) \geq f(1) + f'(1)(x-1) = 1-2(1+2\delta)(1-x)$.
  To obtain that $f(x)\geq \frac 1 4$, it suffices to show that $1-2(1+2\delta)(1-x) \geq \frac 1 4$.
  Rearranging terms, we see that this occurs if $x \geq 1 - \frac{3}{8(1+2\delta)}$.

  In order for $\norm{\Pmap_{\SS}\vec u}^{1+2\delta} = \sqrt{f(\norm{\Pmap_{\SS}\vec u})} \geq \frac 1 2$, it suffices that $\norm{\Pmap_{\SS}\vec u} \geq 1 - \frac{3}{8(1+2\delta)}$.
  Note that
  \begin{displaymath}
    \sqrt{\frac 3 {4(1+2\delta)} - \frac{9}{64(1+2\delta)^2}}
    % > \sqrt{\frac 3 {4(1+\delta)} - \frac{9}{64(1+\delta)}}
    > \sqrt{\frac 3 {4(1+2\delta)} - \frac{1}{4(1+2\delta)}}
    = \frac 1 {\sqrt {2(1+2\delta)} }
    \geq \norm{\Pmap_{\compl \SS} \vec u} \ .
  \end{displaymath}
  As such,
  \begin{displaymath}
    \norm{\Pmap_{\SS}\vec u}
    = \sqrt{1 - \norm{\Pmap_{\bar \SS} \vec u}^2}
    \geq \sqrt{ 1 - \frac 3 {4(1+2\delta)} + \frac 9 {64(1+2\delta)^2} }
    \geq 1 - \frac 3 {8(1+2\delta)} \ ,
  \end{displaymath}
  as desired.
\end{proof}

Importantly, since $\norm{\Pmap_0 \vec u(n)}$ rapidly goes to 0 under the gradient iteration, the precondition that $\norm{\Pmap_0 \vec u(n)} \leq \frac 1 {\sqrt {2(1+2\delta)} }$ used in \cref{lem:P0small-gradf-lower-bound} when $\SS = [\dm]$ is actually closed under applications of the gradient iteration when $\epsilon$ is sufficiently small.
In particular, we have the following result.
\begin{cor}\label{cor:P0small-closure}
  Suppose that the sequence $\{\vec u(n)\}_{n=0}^\infty$ is recursively defined by $\vec u(n+1) = \hat \GG(\vec u(n))$, and that $\norm{\Pmap_0 \vec u(0)} \leq \frac{1}{\sqrt{2(1+2\delta)}}$.
  If $\epsilon \leq \frac{\beta}{2 \dm^\delta \delta\sqrt{2(1+2\delta)}}$, then $\norm{\Pmap_0 \vec u(n)} \leq \frac 1 {\sqrt{2(1+2\delta)}}$ and $\norm{\nabla F(\vec u)} \geq \frac \beta \delta \dm^{-\delta}$ for all $n$.
\end{cor}
\begin{proof}
  We argue by induction on the hypothesis $\norm{\Pmap_0 \vec u(n)}
  \leq \frac 1 {\sqrt{2(1+2\delta)}}$.
  The base case is given.
  Further, if $\norm{\Pmap_0 \vec u(n)} \leq \frac{1}{\sqrt{2(1+2\delta)}}$,
  then $\norm{\nabla F(\vec u)} \geq \frac \beta \delta \dm^{-\delta}$ by
  \cref{lem:P0small-gradf-lower-bound}.
  Using \cref{lem:small-coord-decrease-step}, we obtain
  \begin{displaymath}
    \norm{\Pmap_0
      \vec u(n+1)} = \norm{\Pmap_0 \hat \GG(\vec u(n))} \leq
    \frac{2\epsilon}{\norm{\nabla F(\vec u(n))}} \leq 2 \epsilon \dm^\delta
    \delta / \beta \leq \frac 1 {\sqrt{2(1+2\delta)}} \ . \qedhere
  \end{displaymath}
\end{proof}

We now combine the \cref{lem:small-coord-decrease-step,lem:P0small-gradf-lower-bound} to provide a time bound for the rapid decay of the small coordinates of $\vec u(n)$.
Here and later in our analysis, we will introduce a number of useful constants and expressions for various lemmas and propositions that we prove, indexing these expressions by the lemma/proposition number of the result for which they were introduced.
We define $\tcompress := [\frac {\beta\gamma} {16 \alpha\delta }\dm^{- \delta}]^{\frac 1 {2\gamma}}$.
This magnitude is treated as a threshold for the cutoff between small and large coordinates of $\vec u$.
Those coordinates of $\vec u(0)$ for which $\abs{u_i(0)} \leq \tcompress$ shrink and then stay small under the gradient iteration unless $\norm{\Pmap_0 \vec u(0)}$ is unusually large.
More formally, we have the following result.

\begin{prop}\label{prop:GI-Loop-Small-Coords-Error}
  Let $\vec u \in \sphere^{\dn - 1}$.
  Suppose that $\epsilon < \frac \beta {8\delta \sqrt{1+2\delta} }\dm^{- \delta-\frac 1 2}\left[\frac{\beta \gamma}{16\alpha \delta}\dm^{- \delta}\right]^{\frac 1 {2\gamma}}$, that
  $\norm{\Pmap_0 \vec u} \leq \frac 1 {\sqrt{2(1+2\delta)}}$, and that
  \begin{displaymath}
  N \geq \log_{1+2\gamma} \biggr(
  \log_2 \biggr( \frac{\beta \gamma}{8\alpha \delta} \biggr)
  + 2\gamma \log_2 \biggr( \frac{\beta}{4\delta\epsilon} \biggr)
  %  \cdot \dm^{-\frac{\delta(1+\gamma)} 2}\bigg(\frac{\beta}{(1+\delta)\epsilon}\bigg)^\gamma
  \biggr) \ ,
\end{displaymath}
  is a positive integer.
  Let $\vec w \leftarrow \textsc{GI-Loop}(\vec u, 2N)$.
  The following hold:
  \begin{enumerate}
    \item \label{prop:GI-small-coord:it:trivial} $\norm{\Pmap_0 \vec w} \leq 2\delta \dm^{\delta} \epsilon / \beta$.
    \item \label{prop:GI-small-coord:it:nontrivial}
    Let $\SS \subset [\dm]$.
    If $\abs{u_j} \leq [\frac {\beta\gamma} {16 \alpha\delta }\dm^{- \delta}]^{\frac 1 {2\gamma}}$ for all $j \in \SS$,
    then $\norm{(\Pmap_0 + \Pmap_\SS) \vec w} \leq 4\delta (\dm-\abs \SS)^{\delta} \epsilon / \beta$.
  \end{enumerate}
\end{prop}
\def\Nsc{\ensuremath{N_{\ref{prop:GI-Loop-Small-Coords-Error}}}}%
Also for later reference, we define \Nsc\@ to be $\bigr\lceil\log_{1+2\gamma} \bigr(
  \log_2 \bigr( \frac{\beta \gamma}{8\alpha \delta} \bigr)
  + 2\gamma \log_2 \bigr( \frac{\beta}{4 \delta \epsilon} \bigr)\bigr)\bigr\rceil$.
  Note that since $\frac{\beta \gamma}{8 \alpha \delta} < 1$, we have that $\log_2(\frac{\beta \gamma}{8 \alpha \delta}) < 0$.
  In particular, it is actually sufficient in \cref{prop:GI-Loop-Small-Coords-Error} that $N \geq \log_{1+2\gamma}(2\gamma \log_2(\frac {\beta}{4\delta \epsilon}))$.
  Further, $\log_{1+2\gamma}(2\gamma) + \log_{1+2\gamma}( \log_2(\frac {\beta}{4\delta \epsilon})) \leq \log_{1+2\gamma}( \log_2(\frac {\beta}{4\delta \epsilon})) + 1$ implies that it is sufficient that $N \geq C \log_{1+2\gamma}( \log_2(\frac {\beta}{4\delta \epsilon}))$ for some universal constant $C$.
  In particular, this time bound represents a super-linear (order $1+2\gamma$) rate of convergence of the small coordinates to $\epsilon$-error which corresponds to the convergence rate guarantees that were seen in \cref{thm:gi-convergence} for the unperturbed setting.
  We also use this simplified version of the bound in the statements of \cref{thm:single-recovery-main,thm:full-recovery-main}.

\begin{proof}[\proofword\@of \cref{prop:GI-Loop-Small-Coords-Error}]
  We first define the sequence $\{\vec u(n)\}_{n=0}^\infty$ recursively by $\vec u(0) = \vec u$ and $\vec u(n+1) = \hat \GG(\vec u(n))$.
  By construction, $\vec w = \vec u(2N)$.
  As such, it suffices to prove the desired properties on this sequence.

  We first show (by induction on $n$) that for every $n \in \N \cup \{0\}$,
  $\norm{\nabla F(\vec u(n))} \geq \frac \beta \delta \dm^{- \delta}$.
  The base case when $n = 0$ follows by \cref{lem:P0small-gradf-lower-bound} (choosing $\SS$ in \cref{lem:P0small-gradf-lower-bound} as $[\dm]$).
  Letting $L = \frac \beta {\delta} \dm^{-\delta}$,
  it is easily verified that $\epsilon \leq \frac 1 2 L$, and in particular, we may apply \cref{lem:small-coord-decrease-step} whenever our inductive hypothesis holds.
  We suppose that our inductive hypothesis holds for $n = k$.
  Using \cref{lem:small-coord-decrease-step} part~\ref{lem:small-coord-trivial}, we see that
  \begin{align}\label{eq:null-proj-bound}
    \norm{\Pmap_0 \vec u(k+1)}
    &\leq \frac {2 \epsilon}{ \norm{\nabla F(\vec u(k))} }
    \leq \frac {2 \epsilon} L \nonumber \\
    &\leq \frac{\beta}{8\delta\sqrt{1+2\delta}} \dm^{-\delta-\frac 1 2}\left[\frac{\beta\gamma}{16\alpha\delta\dm^{\delta}}\right]^{\frac 1 {2\gamma}} / L
    < \frac 1 {2\sqrt{2(1+2\delta)}} \ .
  \end{align}
  As such, we may apply \cref{lem:P0small-gradf-lower-bound} to see that $\norm{ \nabla F(\vec u(k+1)) } \geq L$ as desired.
  By the principle of mathematical induction, $\norm{\nabla F(\vec u(n))} \geq L$ for all $n \in \N \cup \{0\}$.

  To obtain part~\ref{prop:GI-small-coord:it:trivial}, apply \cref{lem:small-coord-decrease-step} to see that
  \begin{displaymath}
  \norm{\Pmap_0 \vec w} = \norm{\Pmap_0 \vec u(2N)} \leq \frac{2 \epsilon}{\norm{\nabla F(\vec u(2N-1))}} \leq \frac{2\delta\dm^{\delta} \epsilon} \beta \ .
  \end{displaymath}

  We now prove part~\ref{prop:GI-small-coord:it:nontrivial}.
  With $L$ as before, we note that by construction,
  \begin{align*}
    \epsilon
    &< \frac{\beta}{8\delta\sqrt{1+2\delta}\dm^{\delta+\frac 1 2}}
    \left[\frac{\beta\gamma}{16\alpha\delta\dm^{\delta}}\right]^{\frac 1 {2\gamma}} \\
    &= \frac{L}{8\sqrt{1+2\delta}\dm^{\frac 1  2}}\left[\frac{L\gamma}{16\alpha}\right]^{\frac 1 {2\gamma}}
    \leq \min\biggr( \frac 1 2 L, \Big(\frac{\gamma L^{1+2\gamma}}{8\cdot 4^{2\gamma} \alpha} \Big)^{\frac 1 {2\gamma}} \biggr) \ .
  \end{align*}
  Notice that $L$ is a lower bound on $\norm{\nabla F(\vec u(n))}$ for all $n$.
  We apply \cref{lem:small-coord-decrease-time} with the choice $\SS_{\ref{lem:small-coord-decrease-time}} =\{j\}$ such that $j \in \SS$, the choice $C = \frac 1 2$, and our choice of $L$.
  We obtain $\abs{u_j(n)} \leq \frac {4\epsilon} L \leq 4\delta\dm^{\delta}\epsilon / \beta$ for all $n \geq N$.

  We fix $k \geq N$ an arbitrary integer.
  We note that
  \begin{align*}
    \norm{\Pmap_{\SS} \vec u(k)}
    &\leq \sqrt{\textstyle{\sum_{j \in \SS}} (4\delta \dm^{\delta} \epsilon / \beta)^2}
    \leq 4\delta \dm^{\delta+\frac 1 2} \epsilon / \beta \\
    &\leq \min\left( \frac 1 {2\sqrt{2(1+2\delta)}},\ [\frac{\beta \gamma}{16 \alpha \delta} \dm^{- \delta}]^{\frac 1 {2\gamma}}\right)
  \end{align*}
  by our choice of $\epsilon$.
  Combining with \cref{eq:null-proj-bound}, we see that $\norm{(\Pmap_0 + \Pmap_\SS)\vec u(k)} \leq \frac 1 {\sqrt{2(1+2\delta)}}$.
  Applying \cref{lem:P0small-gradf-lower-bound} with $\SS_{\ref{lem:P0small-gradf-lower-bound}} = \compl \SS \cap [\dm]$, we see that $\norm{\nabla F(\vec u(k))} \geq \frac \beta \delta \abs{\compl \SS \cap [\dm]}^{- \delta} = \frac \beta \delta (m - \abs{\SS})^{- \delta}$.
  We set a new choice of lower bound $L = \frac \beta \delta (m - \abs{\SS})^{- \delta}$, and we note that $\norm{\Pmap_\SS \vec u(k)} \leq [\frac{\beta \gamma}{16 \alpha \delta} \dm^{- \delta}]^{\frac 1 {2\gamma}} \leq [\frac{\beta \gamma}{16 \alpha \delta} (\dm-\abs \SS)^{- \delta}]^{\frac 1 {2\gamma}} \leq [\frac{\gamma L}{16 \alpha}]^{\frac 1 {2\gamma}}$.
  With our new choice of $L$, we may thus apply \cref{lem:small-coord-decrease-time} on the sequence $\{\vec u(n)\}_{n=N}^\infty$ to obtain that $\norm{\Pmap_\SS \vec w} = \norm{\Pmap_\SS \vec u(2N)} \leq \frac {4 \epsilon} L \leq 4 \delta (\dm-\abs \SS)^\delta \epsilon / \beta$.
\end{proof}

\Cref{prop:GI-Loop-Small-Coords-Error} foreshadows a bound for the final estimation error for the recovery of any hidden basis element using \textsc{FindBasisElement}.
We have not yet demonstrated that the main loop of \textsc{FindBasisElement} drives every coordinate of $\vec u$ (except 1) to become small in the sense of \cref{prop:GI-Loop-Small-Coords-Error}; however, we will eventually do so.
% With this goal in mind, we can predict what our basis recovery error will be.
Combining \cref{lem:hbe-error-bound} below with the error bound from \cref{prop:GI-Loop-Small-Coords-Error} (with $\SS$ chosen such that $\dm - \abs{\SS} = 1$), we will predict that for $\vec u$ returned by \textsc{FindBasisElement}, there exists a sign $s \in \{ \pm 1 \}$ and a hidden basis element $\myZv_i$ such that $i \in [\dm]$ and $\norm{s \myZv_k - \gvec \nu} \leq 4 \sqrt 2 \delta \epsilon / \beta$.
Later, when arguing about the accuracy of \textsc{FindBasisElement}, we will use this predicted bound when making assumptions on how accurately the previously recovered $\gvec \mu_k$s estimate hidden basis elements.

\begin{lem}\label{lem:hbe-error-bound}
  Fix $\vec u \in \sphere^{\dn - 1}$ and $j \in [\dm]$.
  Let $\SS = \{ j \}$.
  Then, there exists $s \in \{ \pm 1 \}$ such that $\norm{s \myZv_j - \vec u} \leq \norm{P_{\compl \SS} \vec u} \sqrt 2$. % 4 \sqrt 2 \epsilon / \beta$.
\end{lem}
\begin{proof}
  We choose $s$ such that $s u_j = \abs {u_j}$.
  We note that
  \begin{displaymath}
    \norm{s \myZv_j - \vec u}^2
    = \norm{P_{\compl \SS} \vec u}^2 + (s - u_j)^2
    \leq \norm{P_{\compl \SS} \vec u}^2 + \abs{(s - u_j)(s + u_j)} \ ,
  \end{displaymath}
  where the inequality uses that $s$ and $u_j$ are of the same sign and hence that $s+u_j$ has at least the same magnitude as $s - u_j$.
  But since $(s - u_j)(s + u_j) = 1 - u_j^2 \geq 0$, we obtain:
  \begin{displaymath}
    \norm{s \myZv_j - \vec u}^2
    \leq \norm{P_{\compl \SS} \vec u}^2 + 1 - u_j^2 \ .
  \end{displaymath}
  Since $\vec u \in \sphere^{\dn - 1}$ is a unit vector, $u_j^2 = 1 - \norm{P_{\compl \SS} \vec u}^2$.
  Thus, $\norm{s \myZv_j - \vec u}^2
    \leq 2\norm{P_{\compl \SS} \vec u}^2 $.
    % \leq 2[4\delta \epsilon / \beta]^2 \ .
  Taking square roots gives the desired result.
\end{proof}

\subsubsection{Setting Up the Main Loop of \textsc{FindBasisElement}}

We now demonstrate that the steps in \textsc{FindBasisElement} preceding the main loop create a warm start for the main loop.
That is, we wish to demonstrate that after line~\ref{alg:step:zero-nonbasis-coords2} of \textsc{FindBasisElement}, $\norm{\Pmap_0 \vec u}$ is small and so are the coordinates of $\vec u$ corresponding to the hidden basis directions approximately recovered in $\gvec \mu_1, \dotsc, \gvec \mu_k$.
This line of argument is carried out in \cref{lem:zeroing-nonbasis-coords,lem:main-loop-precondition} below.

\begin{lem}\label{lem:zeroing-nonbasis-coords}
    Consider an execution of \textsc{FindBasisElement}.
    Fix $\eta \in [0, \frac 1 {4 \dm \sqrt \dn}]$.
    Suppose that $k < \dm$; that there is a permutation $\pi$ on $[\dm]$; there are sign values $s_1, \dotsc, s_k \in \{+1, -1\}$ such that for each $i \in [k]$, $\norm{s_i \gvec \mu_i - \myZv_{\pi(i)}} \leq \eta$; and that $\epsilon \leq \frac \beta {4 \sqrt 2 \delta} \dm^{- \delta} \dn^{- \frac 1 2 - \delta}$.
    At the end of the execution of step~\ref{alg:step:zero-nonbasis-coords} of \textsc{FindBasisElement}, the following hold:
    \begin{enumerate}
        \item\label{item:step5:null-space} $\norm{\Pmap_0 \gvec u} \leq \dm^{\delta}\dn^{\frac 1 2 + \delta}\delta \epsilon / \beta$.
        \item\label{item:step5:controlled-space} If $\eta \leq 4\sqrt 2 \delta \epsilon/ \beta$, and $i \in \{\pi(j) \suchthat j \in [k]\}$, then $ \abs{u_i} \leq \frac{25 \alpha\delta}{2\beta\gamma} \dm^{\delta} \dn^{\frac 1 2 + \delta} \delta \epsilon/\beta$.
    \end{enumerate}
\end{lem}

% \begin{proof}[\proofword\@of Lemma~\ref{lem:zeroing-nonbasis-coords}]
\begin{proof}
    First, we demonstrate that one of the vectors $\vec x_i$ from step~\ref{alg:step-xi-spread} of \textsc{FindBasisElement} has $\norm{\Pmap_{[\dm]} \vec x_i}^2 \geq \frac{\dm-k}{\dn}$.
    We will later use this to demonstrate that $j$ in step~\ref{alg:step:best-xi} satisfies that $\norm{\nabla F(\vec x_j)}$ is sufficiently large for $\hat \GG(\vec x_j)$ to work as intended.

    \begin{claim}\label{claim:xi-in-non-negligeable-region}
        There exists $i \in [\dn - k]$ such that $\norm{\Pmap_{[\dm]} \vec x_i}^2 \geq \frac{\dm-k}\dn$.
    \end{claim}
    \begin{subproof}
        %We let $\SS$ be the set passed into \textsc{FindBasisElement}.
        We let $\SS = \{\pi(k+1), \pi(k+2), \dotsc, \pi(\dm) \}$.
        % Since $\SS \subset [\dm]$, it follows $\norm{\Pmap_{[\dm]} \vec v} \geq \norm{\Pmap_{\SS} \vec v}$ for any $\vec v\in \sphere^{\dn-1}$.
        We extend the list of vectors $\vec x_1, \dotsc, \vec x_{\dn-k}$ to be an orthonormal basis of the space: $\vec x_1, \dotsc, \vec x_\dn$.
        Since each $\myZv_i$ is a unit vector, it follows:
        \begin{equation}
        \label{eq:PtildeS-mean}
        \frac 1 \dn \sum_{i=1}^{\dn} \norm{\Pmap_\SS \vec x_i}^2
        = \frac 1 \dn \sum_{i=1}^{\dn} \sum_{j=k+1}^\dm \ipCanonicalp{\vec x_i}{\myZv_{\pi(j)}}^2
        = \frac 1 \dn \sum_{j={k+1}}^\dm \norm{\myZv_{\pi(j)}}^2
        = \frac{\dm-k}{\dn} \ .
        \end{equation}
        Treating \cref{eq:PtildeS-mean} as a sample average, there exists $i \in [\dn]$ such that $\norm{\Pmap_{\SS} \vec x_i}^2 \geq \frac{\dm - k}\dn$.

        To complete the proof, we need only demonstrate that for any $i > \dn - k$, $\norm{\Pmap_{S} \vec u} < \sqrt{\frac{\dm - k}\dn}$.
        To show this, we first demonstrate that $\gvec \mu_1, \dotsc, \gvec \mu_k$ span a $k$ dimensional space.
        Note that this implies that $\vec x_1, \dotsc, \vec x_{\dn - k}$ span the space $\spn(\gvec \mu_1, \dotsc, \gvec \mu_k)^\perp$.
        Therefore, for any $i > \dn - k$ we have $\vec x_i \in \spn(\gvec \mu_1, \dotsc, \gvec \mu_k)$.
        Then to complete the proof, we demonstrate that for any $\vec v \in \spn(\gvec \mu_1, \dotsc, \gvec \mu_k)$, we have $\norm{\Pmap_{S} \vec u} < \sqrt{\frac{\dm - k}\dn}$.

        Now consider the matrices $A = A_0 = \sum_{i=1}^k \gvec \mu_i \gvec \mu_i^T$ and $\tilde A = \tilde A_0 = \sum_{i=1}^k \myZv_{\pi(i)} \myZv_{\pi(i)}^T$.
        We note:
        \begin{align*}
        \norm{A_0 - \tilde A_0}
        &= \Big \lVert { \sum_{i=1}^k[(\gvec \mu_i - \myZv_{\pi(i)})\gvec \mu_i^T
            + \myZv_{\pi(i)}(\gvec \mu_i - \myZv_{\pi(i)})^T ]} \Big\lVert \\
        &\leq 2 \sum_{i=1}^k\norm{\gvec \mu_i - \myZv_{\pi(i)}}\norm{\gvec \mu_i}
        \leq 2k\eta \ .
        \end{align*}
        In particular, Weyl's inequality (reproduced in \cref{thm:Weyls-inequality}) implies that the $k$\textsuperscript{th} lowest eigenvalue $\lambda_k(A_0) \geq \lambda_k(\tilde A_0) - 2k\eta \geq 1 - 2k\eta > 0$.
        The vectors $\gvec \mu_1, \dotsc, \gvec \mu_k$ are linearly independent.
        As the $k$ eigenvalues of $A_0$ are contained in the interval $[1 - 2k\eta,\ 1 + 2k\eta]$ by Weyl's inequality,
        \cref{thm:sin-theta-thm} (the Davis-Kahan sin $\Theta$ theorem\index{sin $\Theta$ theorem}) with $\tilde A_1 = \sum_{i=k+1}^\dn 0 \myZv_{\pi(i)}\myZv_{\pi(i)}^T$, implies that
        \begin{displaymath}
        (1-2k\eta)\norm{\Pmap_{\spn(\myZv_{\pi(k+1)}, \dotsc, \myZv_{\pi(\dn)})} \Pmap_{\spn(\gvec \mu_1, \dotsc, \gvec \mu_k)}}
        \leq 2k\eta
        \end{displaymath}
        \begin{align*}
        \norm{\Pmap_{\spn(\myZv_{\pi(k+1)}, \dotsc, \myZv_{\pi(\dn)})} \Pmap_{\spn(\gvec \mu_1, \dotsc, \gvec \mu_k)}}
        &\leq \frac{2k\eta}{1-2k\eta}
        < \frac {1/(2\sqrt{\dn})}{1-1/(2\sqrt \dn)}
        \leq \frac {1/(2\sqrt{\dn})}{1/2} \\
        &\leq \frac 1 {\sqrt \dn}
        \leq \sqrt{\frac {\dm-k} \dn} \ .
        \end{align*}
        As such, if $\vec v \in \spn(\gvec \mu_1, \dotsc, \gvec \mu_k)$, then
        \begin{displaymath}
          \norm {\Pmap_\SS \vec v} \leq \norm{\Pmap_{\spn(\myZv_{\pi(k+1)}, \dotsc, \myZv_{\pi(\dn)})} \Pmap_{\spn(\gvec \mu_1, \dotsc, \gvec \mu_k)}} < \sqrt{\frac{\dm - k}{\dn}} \ .
      \end{displaymath}

    \end{subproof}
    We now fix $i\in [\dn - k]$ such that $\norm{\Pmap_{[\dm]} \vec x_i}^2 \geq \frac{\dm - k}{\dn}$ according to \cref{claim:xi-in-non-negligeable-region}, and we fix $j$ according to step~\ref{alg:step:best-xi} from \textsc{FindBasisElement}.
    Note that,
    \begin{displaymath}
    \norm{\grad F(\vec x_i)}
    \geq \frac {2\beta} {\delta} \norm{\Pmap_{[\dm]} \vec x_i}^{1+2\delta}\dm^{-\delta}
    \geq \frac {2 \beta} \delta \left(\frac{m-k}{\dn}\right)^{\frac {1+2\delta} 2} \dm^{- \delta}
    \geq \frac{\beta}{2\delta} \dm^{- \delta} \dn^{- \frac{1+2\delta} 2} \ ,
    \end{displaymath}
    using \cref{lem:fg-upper-bounds} (with projection on the set $[\dn]$) for the first inequality and that $k < m$ implies $m-k \geq 1$ for the final inequality.

    It follows that
    \begin{equation}\label{eq:1}
    \norm{\pertgF(\vec x_j)}
    \geq \norm{\pertgF(\vec x_i)}
    \geq \norm{\nabla F(\vec x_i)} - \epsilon
    \geq \frac{2\beta} \delta \dm^{-\delta} \dn^{-\frac {1+2\delta} 2} - \epsilon
    \end{equation}
    where the first inequality follows from by the choice of $j$ from step~\ref{alg:step:best-xi} of \textsc{FindBasisElement}, and the second inequality uses that $\pertgF$ is an $\epsilon$-approximation to $\nabla F$.

    We now show part~\ref{item:step5:null-space}.
    By the assumption that
    $\epsilon
    \leq \frac \beta {4 \sqrt 2 \delta} \dm^{- \delta} \dn^{- \frac {1+2\delta} 2 }
    \leq \frac \beta \delta \dm^{- \delta} \dn^{- \frac {1+2\delta} 2 }$,
    we have that $\norm{\pertgF(\vec x_j)} \geq \frac{\beta}{\delta}\dm^{-\delta} \dn^{-\frac {1+2\delta} 2}$.
    We see
    \begin{displaymath}
    \norm{\Pmap_0\hat \GG(\vec x_j)}
    \leq \frac{\epsilon}{\norm{\pertgF(\vec x_j)}}
    \leq \dm^{\delta}\dn^{\frac{1+2\delta} 2} \delta \epsilon / \beta \ .
    \end{displaymath}

    We now show part~\ref{item:step5:controlled-space}.
    We let $\vec w = \vec x_j$.
    We let $\ell \in [k]$, and noting that $\vec w \perp \gvec \mu_\ell$ by construction, we obtain the following bound for $\abs {w_{\pi(\ell)}}$:
    \begin{align}\label{eq:2}
    \abs{w_{\pi(\ell)}}
    &= \abs{\ipCanonical{\vec x_j}{(\myZv_{\pi(\ell)} - s_\ell \gvec \mu_\ell + s_\ell \gvec \mu_\ell)}}
    = \abs{\ipCanonical{\vec x_j}{(\myZv_{\pi(\ell)} - s_\ell \gvec \mu_\ell)}} \nonumber \\
    &\leq \norm{\vec x_j} \norm{\myZv_{\pi(\ell)} - s_\ell \gvec \mu_\ell}
    \leq 4 \sqrt 2 \delta \epsilon / \beta \ .
    \end{align}
    We now fix $i \in \{ \pi(\ell) \suchthat \ell \in [k]\}$
    and bound $\abs {u_i} = \abs {\hat \GG_i(\vec w)}$:
    \begin{align*}
    \abs {u_i} = \abs{\hat \GG_i(\vec w)}
    & = \frac{\abs{\pertgF_i(\vec w)}}{\norm{\pertgF(\vec w)}}
    \leq \frac{\abs{\partial_i F(\vec w)}+\epsilon}{\norm{\pertgF(\vec w)}}
    \leq \frac{\frac {2\alpha}{\gamma} \abs{w_i}^{1+2\gamma} + \epsilon }{ \frac{2 \beta} \delta \dm^{-\delta} \dn^{-\frac {1+2\delta} 2} - \epsilon} \\
    &\leq
    \frac{8\sqrt 2 \cdot \frac{\alpha\delta}{\beta\gamma} \left( 4\sqrt 2 \delta \epsilon / \beta \right)^{2\gamma} \epsilon + \epsilon }{ \frac{\beta}{\delta} \dm^{-\delta} \dn^{-\frac {1+2\delta} 2}}
    \end{align*}
    In the above, the second inequality uses \cref{lem:fg-upper-bounds} and \cref{eq:1}; the third inequality uses \cref{eq:2} and that $\epsilon \leq \frac{\beta}{4\sqrt 2 \delta}\dm^{- \delta}\dn^{- \frac {1+2\delta} 2} \leq \frac{\beta}{\delta}\dm^{- \delta}\dn^{- \frac {1+2\delta} 2}$.
    By our given bound on $\epsilon$, we note that $\left( 4\sqrt 2 \delta \epsilon / \beta \right)^{2\gamma} \leq 1$.
    As such, we obtain that
    \begin{align*}
    \abs {u_i} = \abs{\hat \GG_i(\vec w)}
    &\leq  \frac{[8\sqrt 2 \cdot \frac{\alpha\delta}{\beta\gamma} + 1] \epsilon }{ \frac{\beta}{\delta} \dm^{-\delta} \dn^{-\frac {1+2\delta} 2}}
    \leq \frac{(8\sqrt 2 + 1) \cdot \frac{\alpha\delta}{\beta\gamma}\dm^{\frac {1+2\delta} 2} \epsilon }{ \frac{\beta}{\delta} \dm^{-\delta} \dn^{-\frac {1+2\delta} 2}} \\
    &\leq  (8\sqrt 2 + 1) \cdot \frac{\alpha\delta^2}{\beta^2\gamma} \dm^{\delta} \dn^{\frac {1+2\delta} 2} \epsilon
    \leq \frac{25 \alpha \delta^2}{2 \beta^2\gamma} \dm^{\delta} \dn^{\frac {1+2\delta} 2} \epsilon
    \ . \qedhere
    \end{align*}
\end{proof}

\begin{lem}\label{lem:main-loop-precondition}
    Let $k$ be defined as in an execution of \textsc{FindBasisElement}.
    Suppose that $k < \dm$,
    that
    $\epsilon < \frac 1 {16 \sqrt 2} \frac{\beta^2\gamma}{\alpha \delta^2 \sqrt{1+2\delta}}\dm^{-\frac 1 2 -\delta}\dn^{- \frac 1 2 - \delta}\tcompress$,
    that $N_1 \geq 2 N_{\ref{prop:GI-Loop-Small-Coords-Error}}$,
    and that
    there exists a permutation $\pi$ of $[\dm]$ and sign values $s_1, \dotsc, s_k \in \{\pm 1\}$ such that $\norm{s_j \myZv_{\pi(j)} - \gvec \mu_j} \leq 4 \sqrt 2 \delta \epsilon / \beta$ for each $j \in [k]$.
    At the beginning of the execution of the main loop of \textsc{FindBasisElement}, the following hold for $\vec u$:
    \begin{enumerate}
        \item $\norm{\Pmap_0 \vec u} \leq 2 \dm^{\delta} \delta \epsilon / \beta$.
        \item Let $\SS = \{ \pi(j) \suchthat j \in [k] \}$.
        Then, $\norm{(\Pmap_0 + \Pmap_\SS) \vec u} \leq 4(\dm - \abs \SS)^{\delta} \delta \epsilon / \beta$.
    \end{enumerate}
\end{lem}
\begin{proof}
    We first notice that for each $j \in [k]$,
    \begin{displaymath}
    \norm{s_j \myZv_{\pi(j)} - \gvec \mu_j}
    \leq 4 \sqrt 2 \delta \epsilon / \beta
    < \frac{\beta\gamma}{4\alpha\delta\sqrt{1+2\delta}\dm^{\frac 1 2+\delta}}\dn^{-\frac{1+2\delta} 2}\left[\frac{\beta}{16\alpha}\dm^{-\delta}\right]^{\frac 1 {2\gamma}}
    < \frac 1 {4 \dm \sqrt \dn} \ ,
    \end{displaymath}
    where the final inequality uses that $\delta \geq \gamma$ to see that $\dm^{- \frac{\delta}{2\gamma}} \leq \dm^{-\frac 1 2}$.
    As such, we may apply \cref{lem:zeroing-nonbasis-coords} to see that at the end of step~\ref{alg:step:zero-nonbasis-coords} of \textsc{FindBasisElement} that $\norm{\Pmap_0 \vec u} \leq \dm^{\delta} \dn^{\frac {1+2\delta} 2}\delta \epsilon / \beta$ and for each $j \in [k]$ that $\abs{u_{\pi(j)}} \leq \frac{25 \alpha \delta^2}{2\beta^2\gamma} \dm^{\delta} \dn^{\frac{1+2\delta} 2} \epsilon$.
    By our choice of $\epsilon$, it may be verified that
    \begin{displaymath}
    \norm{\Pmap_0 \vec u} \leq \dm^{\delta} \dn^{\frac {1+2\delta} 2} \delta \epsilon / \beta
    < \frac 1 {\sqrt{2(1+2\delta)}}
    \end{displaymath}
    and that for each $j \in [k]$,
    \begin{displaymath}
    \abs{u_{\pi(j)}}
    \leq \frac{25 \alpha\delta^2}{2\beta^2\gamma} \dm^{\delta} \dn^{\frac{1+2\delta} 2} \epsilon
    < \Bigr[ \frac{\beta\gamma}{16 \alpha\delta}\dm^{-\delta}\Bigr]^{\frac 1 {2\gamma}} \ .
    \end{displaymath}
    As the step $\vec u \leftarrow \textsc{GI-Loop}(\vec u, N_1)$ in line~\ref{alg:step:zero-nonbasis-coords2} sets up the main loop of \textsc{FindBasisElement}, applying \cref{prop:GI-Loop-Small-Coords-Error} gives the desired result.
\end{proof}

\subsection{The Big Become Bigger}
\label{sec:large-coord-analysis}
\label{sec:large-coord-analysis-1}

We saw in \cref{sec:small-coord-analysis} that the small coordinates of $\vec u$ rapidly decay under the gradient iteration until they are on the order of $\epsilon$.
In this section, we demonstrate how the large coordinates of $\vec u$ diverge under the gradient iteration, causing some to become bigger and other large coordinates to become small.
In particular, we create a robust version of \cref{prop:gi-expand-all} from \cref{sec:gi-divergence-criteria}.
We recall that in \cref{prop:gi-expand-all}, it was seen that if $\vec v$ is a fixed point of $\GG / {\sim}$ and $\vec u$ satisfied that there exists $i \in [\dm]$ with $v_i \neq 0$ and $\abs{u_i} > v_i$, then for some $u_j$ with $j$ among the non-zero coordinates of $\vec v$ is driven towards 0 by the gradient iteration.
This proposition provided a very useful characterization of the instability of all fixed points of $\GG/{\sim}$ other than the hidden basis $\pm \myZv_1, \dotsc, \pm \myZv_\dm$.

Before proceeding, we will need the following technical result.
\begin{lem}\label{lem:vi-spread}
 % Suppose that $F$ is $(\alpha, \beta, \gamma, \delta)$-robust.
 Let $\SS \subset [\dm]$ be non-empty, and let $\vec v \in \porth^{\dn - 1}$ be a fixed point of $\GG/{\sim}$ such that $v_i \neq 0$ if and only if $i \in \SS$.
 If $k \in \SS$, then $\abs {v_k} \geq  \big(\frac{\beta \gamma }{\alpha \delta \abs \SS^{\delta}} \big)^{\frac 1 {2\gamma}}$.
\end{lem}
\begin{proof}
  Since $\sum_{i \in \SS} v_i^2 = 1$, there exists $j \in \SS$ with $v_j^2 \geq \frac 1 {\abs \SS}$.
  By \cref{lem:F-robust-implications}, $\abs{h_j'(v_j^2)} \geq \frac \beta \alpha \abs {v_j}^{2\delta} \geq \frac{\beta}{\alpha \abs \SS^{\delta}}$.

  Fix any $k \in \SS$.  Then, by \cref{lem:F-robust-implications}, $\abs{h_k'(v_k^2)} \leq \frac{\alpha}{\gamma} \abs {v_k}^{2\gamma}$.
  Since $h_k'(v_k^2) = h_j'(v_j^2)$ by \cref{obs:stationary-criterion}, it follows that $\frac{\alpha}{\gamma}\abs {v_k}^{2\gamma} \geq \frac{\beta}{\delta \abs \SS^{\delta}}$.
  In particular, $\abs {v_k} \geq \big(\frac{\beta\gamma}{\alpha\delta \abs{\SS}^{\delta} }\big)^{\frac 1 {2\gamma}}$.
  % = \big(\frac{\beta\gamma}{\alpha\delta}\big)^{\frac 1 {2\gamma}} \abs{\SS}^{-\frac{\delta}{2\gamma}}$.
\end{proof}

We now demonstrate that when there exists a large coordinate $k$ such that $\abs{u_k}$ is sufficiently greater than the corresponding coordinate $\abs{v_k}$ of a fixed point $\vec v$ of $\GG/{\sim}$, then the separation between large coordinates of $\vec u$ expands under the Gradient iteration.
This expansion was the main idea underlying the proof of \cref{prop:gi-expand-all} (see \cref{claim:expand}).
\begin{lem}\label{lem:robust-GI-progress}
  Let $\vec u \in \sphere^{\dn-1}$ be such that the set $\SS := \{ i \suchthat \abs{u_i} \geq \tcompress\}$ is a subset of $[\dm]$ containing at least 2 elements.
  Suppose $\norm{\Pmap_0 \vec u} \leq \frac 1 {2\sqrt{1+2\delta}}$.
  Let $\vec v \in \porth^{\dn - 1}$ be the fixed point of $\GG/{\sim}$ such that $v_i \neq 0$ if and only if $i \in \SS$.
  Let ${\ell} = \argmax_{i\in\SS} \frac{\abs{u_i}}{v_i}$, and ${k} = \argmin_{i \in \SS} \frac{\abs{u_i}} {v_i}$.
  Fix $\eta \in (0, 1]$, suppose that $\frac{\abs{u_{\ell}}/v_{\ell}}{\abs{u_{k}}/v_{k}} \geq (1 + \eta)^2$, and that $\abs {u_\ell} \geq v_\ell$.
  The following hold:
  \begin{enumerate}
  \item\label{lem:robust-GI-progress:it:1}
    If $\epsilon \leq \frac 1 8 \frac \beta \delta \big(\frac{\beta \gamma}{\alpha \delta}\big)^{\frac \delta \gamma}\abs{\SS}^{-\frac \delta \gamma(\delta - \gamma)} \tcompress^{1 + 2\delta}\eta$, then
    \begin{displaymath}
      \max_{i, j \in \SS} \frac{\abs{\hat \GG_i(\vec u)}/v_i}{\abs{\hat \GG_j(\vec u)}/v_j}
      \geq \left(1 + \frac 3 4 \fcondexpand^{\frac{\delta}{\gamma}}\abs{\SS}^{-\frac \delta \gamma (\delta - \gamma)}\eta \right)\frac{\abs{u_{\ell}}/v_{\ell}}{\abs{u_{k}}/v_{k}} \ .
    \end{displaymath}
  \item\label{lem:robust-GI-progress:it:3}
    If $\epsilon \leq \frac 9 {256} \frac \beta \delta \tcompress^2 \Big(\frac{\beta \gamma}{\alpha \delta}\Big)^{\frac {2\delta + 1}{2\gamma}}\abs{\SS}^{-\frac {\delta}{\gamma}(\frac 1 2 + \delta - \gamma) - \delta}\eta$, then there exists $i \in \SS$ such that $\abs{\GG_i(\vec u)} \geq v_i$.
  \end{enumerate}
\end{lem}

\newcommand{\epsdiverge}[1]{\ensuremath{E_{\ref{lem:robust-GI-progress}}\left(#1\right)}}
For later use, we define the expression
\begin{displaymath}
  \epsdiverge{\eta, \SS}
  := \frac 9 {256}\frac \beta {\delta} \big(\frac{\beta \gamma}{\alpha \delta}\big)^{\frac {2\delta+1}{2\gamma}}\abs{\SS}^{-\frac {\delta}{\gamma}(\frac 1 2 + \delta - \gamma) - \delta} \tcompress^{2 + 2\delta}\eta \ ,
\end{displaymath}
 which serves as a sufficient upper bound for $\epsilon$ in both parts~\ref{lem:robust-GI-progress:it:1} and~\ref{lem:robust-GI-progress:it:3} of \cref{lem:robust-GI-progress}.

\begin{proof}[\proofword\@of \cref{lem:robust-GI-progress}]
  We first prove part~\ref{lem:robust-GI-progress:it:1}.
  In doing so, we will make use of the following claims.
    \begin{claim}\label{claim:robust-GI-progress-single-step}
      Suppose there exists $\Delta > 0$ such that one of the following holds: (1) $h_{\ell}'(u_{\ell}^2) \geq (1+\Delta)h_{\ell}'(v_{\ell}^2)$ or (2) $h_{k}'(u_{k}^2) \leq (1+\Delta)^{-1}h_{k}'(v_{k}^2)$.
      Suppose there exists $\zeta \in (0, \min(\frac 1 {16}, \frac 1 8 \Delta)]$ such that $\epsilon \leq \zeta \min_{i \in \SS} \abs{\partial_i F(\vec u)}$.
      Then,
      $\max_{i, j \in \SS} \frac{\abs{\hat \GG_i(\vec u)}/v_i}{\abs{\hat \GG_j(\vec u)}/v_j} \geq (1+\frac 1 4 \Delta) \frac{\abs {u_{\ell}}/v_{\ell}}{\abs {u_{k}}/v_{k}}$.

  \end{claim}
  \begin{subproof}
  We first bound the error on calculating $\GG_i(\vec u)$.
  For each $i \in \SS$, we have:
  \begin{displaymath}
    %\label{eq:G-hat-Ubound}
    \abs{\hat \GG_i(\vec u)}
    = \frac{\abs{\pertgF_i(\vec u)}}{\norm{\pertgF(\vec u)}}
    \leq \frac{\abs{\partial_i F(\vec u)} + \epsilon}{\norm{\grad F(\vec u)}-\epsilon}
    \leq \frac{1+\zeta}{1-\zeta} \cdot \frac{\abs{\partial_i F(\vec u)}}{\norm{\grad F(\vec u)}}
    = \frac{1+\zeta}{1-\zeta} \cdot \abs{\GG_i(\vec u)}
  \end{displaymath}
  \begin{displaymath}
    % \label{eq:G-hat-Lbound}
    \abs{\hat \GG_i(\vec u)}
    = \frac{\abs{\pertgF_i (\vec u)}}{\norm{\pertgF(\vec u)}}
    \geq \frac{\abs{\partial_i F(\vec u)} - \epsilon}{\norm{\grad F(\vec u)}+\epsilon}
    \geq \frac{1-\zeta}{1+\zeta} \cdot \frac{\abs{\partial_i F(\vec u)}}{\norm{\grad F(\vec u)}}
    = \frac{1-\zeta}{1+\zeta} \cdot \abs{\GG_i(\vec u)} \ .
  \end{displaymath}
  Since $\sum_{i\in \SS} u_i^2 \leq \sum_{i\in \SS} v_i^2 = 1$, it follows that $\abs{u_{k}} \leq v_{k}$.
  As such, we have both that $\abs{u_{k}} \leq v_{k}$ and $\abs{u_{\ell}} \geq v_{\ell}$.
  We have that
  \begin{align*}
    \max_{i, j\in \SS} \frac{\abs{\hat \GG_i(\vec u)}/v_i}{\abs{\hat \GG_j(\vec u)}/v_j}
    & \geq  \Bigr(\frac{1-\zeta}{1+\zeta}\Bigr)^2 \max_{i, j\in \SS}  \frac{\abs{\GG_i(\vec u)}/v_i}{\abs{\GG_j(\vec u)}/v_j}
    = \Bigr(\frac{1-\zeta}{1+\zeta}\Bigr)^2 \max_{i, j\in \SS} \frac{\abs{h_i'(u_i^2)}\abs {u_i}/v_i}{\abs{h_j'(u_j^2)}\abs {u_j}/v_j} \\
    &\geq \Bigr(\frac{1-\zeta}{1+\zeta}\Bigr)^2 \frac{\abs{h_{\ell}'(u_{\ell}^2)}\abs {u_{\ell}}/v_{\ell}}{\abs{h_{k}'(u_{k}^2)}\abs {u_{k}}/v_{k}}
    \geq \Bigr(\frac{1-\zeta}{1+\zeta}\Bigr)^2\frac{(1+\Delta)\abs{h_{\ell}'(v_{\ell}^2)}\abs{u_{\ell}}/v_{\ell}}{\abs{h_{k}'(v_{k}^2)}\abs{u_{k}}/v_{k}} \\
    &\geq (1+\Delta)\Bigr(\frac{1-\zeta}{1+\zeta}\Bigr)^2\frac{\abs{u_{\ell}}/v_{\ell}}{\abs{u_{k}}/v_{k}} \ .
  \end{align*}
  In the second to last inequality, we use the monotonicity of $h_i'$ (see \cref{lem:h-props}) along with the assumption that one of the following holds:  either (1) $h_{\ell}'(u_{\ell}^2) \geq (1+\Delta)h_{\ell}'(v_{\ell}^2)$ or (2) $h_{k}'(u_{k}^2) \leq (1+\Delta)^{-1} h_{k}'(v_{k}^2)$.
  In the last inequality, we use \cref{obs:stationary-criterion} to note that $\abs{h_{\ell}'(v_{\ell}^2)} = \abs{h_{k}'(v_{k}^2)}$.

  We now only need bound $(1 + \Delta)\Bigr(\frac{1-\zeta}{1+\zeta}\Bigr)^2$.
  We first note that $\Bigr(\frac{1-\zeta}{1+\zeta}\Bigr)^2=\Bigr(1 - \frac{2\zeta}{1+\zeta}\Bigr)^2 \geq (1-2\zeta)^2 \geq 1 - 4 \zeta$.
  Thus, $(1+\Delta)\Bigr(\frac{1-\zeta}{1+\zeta}\Bigr)^2 \geq 1 + \Delta - 4\zeta - 4\zeta\Delta$.
  Using the upper bounds on $\zeta$, we see that
$1 + \Delta - 4\zeta - 4\zeta\Delta \geq 1 + \Delta - \frac 1 2 \Delta - \frac 1 4 \Delta \geq 1 + \frac 1 4 \Delta$.
  Thus, we obtain
  \begin{displaymath}
    \max_{i, j\in \SS} \frac{\abs{\hat \GG_i(\vec u)}/v_i}{\abs{\hat \GG_j(\vec u)}/v_j} \geq \left(1 + \frac 1 4 \Delta\right) \max_{i, j \in \SS} \frac{\abs{u_i}/v_i}{\abs{u_j}/v_j} \ . \qedhere
  \end{displaymath}
\end{subproof}

    \begin{claim}\label{claim:robust-gi-frac-trapping}
      Suppose $\Delta > 0$,
      $\eta \geq \frac 4 3 \Big(\frac {\alpha \delta}{\beta \gamma}\Big)^{\frac \delta \gamma} \abs{\SS}^{\frac \delta \gamma (\delta - \gamma)}\Delta$,
      and $\frac{\abs{u_{\ell}} / v_{\ell}}{\abs{u_{k}}/v_{k}}\geq (1+\eta)^2$.
      Then one of the following holds: either
      (1) $\abs{h_{\ell}'(u_{\ell}^2)} \geq (1+\Delta)\abs{h_{\ell}'(v_{\ell}^2)}$ or
      (2) $\abs{h_{k}'(u_{k}^2)} \leq (1+\Delta)^{-1} \abs{h_{k}'(v_{k}^2)}$.
    \end{claim}
    \begin{subproof}
      Before proceeding, we note that for $y^2 \geq x^2$ (with $x, y \in \R$), we have that \[\sign(h_i'(x^2)) = \sign(h_i'(y^2)) = \sign(h_i''(t))\] for all $t \in [x^2, y^2]$.
      As such, we may use \cref{lem:F-robust-implications} to see:
      \begin{align}\label{eq:hi-change}
        \abs{h_i'(y^2)} - \abs{h_i'(x^2)}
        &= \int_{x^2}^{y^2}\abs{h_i''(t)} dt
        \geq \beta \int_{x^2}^{y^2} t^{\delta - 1} dt
        \geq \frac \beta \delta [\abs{y}^{2\delta} - \abs{x}^{2\delta}]
      \end{align}

      By the assumption $\frac{\abs{u_{\ell}} / v_{\ell}}{\abs{u_{k}}/v_{k}}\geq (1+\eta)^2$, one of the following must hold:
      (1) $\abs{u_{\ell}} / v_{\ell} \geq (1+\eta)$ or
      (2) $\abs{u_{k}}/v_{k} \leq (1+\eta)^{-1}$.
      We consider these cases separately, and demonstrate that in each case one of our desired results holds.
      \resetcasecount
      \begin{case}
        $\abs{u_{\ell}} / v_{\ell} \geq (1+\eta)$.
      \end{case}
      \begin{caseblock}
        We obtain that
        \begin{align*}
          \abs{h_{\ell}'(u_{\ell}^2)} - \abs{h_{\ell}'(v_{\ell}^2)}
          &\geq \abs{h_{\ell}'((1+\eta)^2v_{\ell}^2) - h_{\ell}'(v_{\ell}^2)}
          %&\geq \beta x^{2\delta - 2}((1+\eta)^2v_{\ell}^2 - v_{\ell}^2)
          = \frac\beta\delta v_{\ell}^{2\delta}[(1+\eta)^2 - 1]
          \geq 2\eta \frac\beta\delta v_{\ell}^{2\delta}
        \end{align*}
        where the first inequality uses the monotonicity of $h_{\ell}'$ (see \cref{lem:h-props}), and the second inequality uses \cref{eq:hi-change}.
        We note that for any $i \in \SS$,
        \begin{align}\label{eq:vi2del-lb}
          v_{i}^{2\delta}
          &= \frac{\gamma}{\alpha} v_{i}^{2(\delta-\gamma)} \cdot \frac \alpha \gamma v_{i}^{2\gamma}
          \geq \frac{\gamma}{\alpha} v_{i}^{2(\delta-\gamma)} \abs{h_i'(v_{i}^2)}
          \geq \frac{\gamma}{\alpha} \Big(\frac {\beta \gamma}{\alpha \delta}\Big)^{\frac \delta \gamma} \abs{\SS}^{-\frac \delta \gamma (\delta - \gamma)} \abs{h_{i}'(v_{i}^2)}
        \end{align}
        using \cref{lem:F-robust-implications} for the first inequality and \cref{lem:vi-spread} for the final inequality.
        As such, we obtain:
        \begin{align*}
          \abs{h_{\ell}'(u_{\ell}^2)} - \abs{h_{\ell}'(v_{\ell}^2)}
          &\geq 2\eta \Big(\frac {\beta \gamma}{\alpha \delta}\Big)^{\frac \delta \gamma} \abs{\SS}^{-\frac \delta \gamma (\delta - \gamma)} \abs{h_{\ell}'(v_{\ell}^2)}
        \end{align*}

        By the lower bound on $\eta$, we obtain $\abs{h'_{\ell}(u_{\ell}^2)} - \abs{h_{\ell}'(v_{\ell}^2)} \geq 2 \Delta \abs{h_{\ell}'(v_{\ell}^2)} \geq \Delta \abs{h_{\ell}'(v_{\ell}^2)}$ as desired.
      \end{caseblock}
      \begin{case}
        $\abs{u_{k}}/v_{k} \leq (1+\eta)^{-1}$.
      \end{case}
      \begin{caseblock}
        Note that $\abs{u_{k}} \leq (1 + \eta)^{-1}v_{k}$.
        We apply \cref{eq:hi-change} to obtain
        \begin{displaymath}
          \abs{h_{k}'(v_{k}^2)} - \abs{h_{k}'(u_{k}^2)}
          \geq \abs{h_{k}'(v_{k}^2) - h_{k}'((1 + \eta)^{-2}v_{k}^2)}
          \geq \frac{\beta}{\delta}v_{k}^{2\delta}[1 - (1+\eta)^{-2}]
        \end{displaymath}
        Using the bound $\eta \geq 1$, we see that
        \begin{displaymath}
          1 - (1+\eta)^{-2}
          = \Big[1 + \frac 1 {1+\eta}\Big]\Big[1 - \frac 1 {1+\eta}\Big]
          = \Big[1 + \frac 1 {1+\eta}\Big]\Big[\frac \eta {1+\eta}\Big]
          \geq \frac 3 4 \eta
        \end{displaymath}
        we obtain
        \begin{displaymath}
          \abs{h_{k}'(v_{k}^2)} - \abs{h_{k}'(u_{k}^2)}
          \geq \abs{h_{k}'(v_{k}^2) - h_{k}'((1 + \eta)^{-2}v_{k}^2)}
          \geq \eta \frac{\beta}{\delta}v_{k}^{2\delta} \ .
        \end{displaymath}
        As such, $\abs{h_k'(v_k^2)} - \abs{h_k'(u_k^2)} \geq \frac {3\beta}{4\delta} v_k^{2\delta} \eta$.
%        where the inequality uses in part the monotonicity of $h_{k}'$ (see \cref{lem:h-props}) and the equality uses the mean value theorem.
        Applying \cref{eq:vi2del-lb}, we obtain
        \begin{displaymath}
          \abs{h_{k}'(v_{k}^2)} - \abs{h_{k}'(u_{k}^2)}
          \geq \frac 3 4 \eta \Big(\frac {\beta \gamma}{\alpha \delta}\Big)^{\frac \delta \gamma} \abs{\SS}^{-\frac \delta \gamma (\delta - \gamma)} \abs{h_{k}'(v_{k}^2)}
          \geq \frac 3 4 \Delta \abs{h_{k}'(v_{k}^2)} \ .
        \end{displaymath}
        Rearranging terms yields $\abs{h_{k}'(u_{k}^2)} \leq \abs{h_{k}'(v_{k}^2)}(1-\Delta)$.
        As $(1+\Delta)^{-1} = \frac{1+\Delta - \Delta}{1+\Delta} = 1 - \frac {\Delta}{1+\Delta} \geq 1 - \Delta$, it follows that $\abs{h_{k}'(u_{k}^2)} \leq (1+\Delta)^{-1} \abs{h_{k}'(v_{k}^2)}$ as desired. \qedhere
      \end{caseblock}
    \end{subproof}

    \noindent  To use these claims, we set parameters $\Delta = \frac 3 4 \Big(\frac{\beta \gamma}{\alpha \delta}\Big)^{\frac \delta \gamma}\abs{\SS}^{-\frac \delta \gamma (\delta - \gamma)} \eta$ and $\zeta = \min(\frac 1 8 \Delta, \frac 1 {16})$.
    By \cref{lem:F-robust-implications}, it follows that $\min_{i \in \SS} \abs{\partial_i F(\vec u)} = \min_{i \in \SS} \abs{g_i'(u_i)} \geq 2\frac \beta \gamma \tcompress^{1 + 2\delta}$.
    We note that $\epsilon \leq \frac{\beta}{8\delta}\tcompress^{1+2\delta} \leq \frac 1 {16} \min_{i \in \SS} \abs{\partial_i F(\vec u)}$, and also that
    \begin{displaymath}
      \epsilon
      \leq \frac \beta {8\delta}\Big(\frac {\beta \gamma}{\alpha \delta}\Big)^{\frac \delta \gamma}\big(\frac 1 {\abs{\SS}}\big)^{\frac \delta \gamma(\delta - \gamma)}\tcompress^{1+2\delta}\eta
      \leq \frac 1 {6} \Delta \frac \beta \delta \tcompress^{1+2\delta}
      \leq \frac 1 {12} \Delta \min_{i \in \SS}\abs{\partial_i F(\vec u)} \ .
    \end{displaymath}
    As such, $\epsilon \leq \zeta \min_{i \in \SS} \abs{\partial_i F(\vec u)}$, and we may apply our claims.
    We apply \cref{claim:robust-gi-frac-trapping} followed by \cref{claim:robust-GI-progress-single-step} to complete the proof of part~\ref{lem:robust-GI-progress:it:1}.

    We now prove part~\ref{lem:robust-GI-progress:it:3}.
    We will make use of the following additional claim.
    \begin{claim}\label{claim:robust-increasing-hi-un}
      Suppose that $\Delta > 0$,
      that \[\eta \geq \frac 4 3 \Big(\frac {\alpha \delta}{\beta \gamma}\Big)^{\frac \delta \gamma} \abs{\SS}^{\frac \delta \gamma(\delta - \gamma)}\Delta,\]
      and that $\frac{\abs{u_{\ell}} / v_{\ell}}{\abs{u_{k}}/v_{k}}\geq (1+\eta)^2$.
      If $\epsilon \leq \frac{3}{32} \tcompress^2\frac{2\Delta + \Delta^2}{(1+\Delta)^2} \cdot \big(\frac{\beta\gamma}{\alpha \delta \abs{\SS}^\delta}\big)^{\frac 1 {2\gamma}} \frac{\beta}{\delta \abs{\SS}^\delta}$, then there exists $i \in \SS$ such that $\abs{\GG_i(\vec u)} \geq \abs{v_i}$.
    \end{claim}
    \begin{subproof}
      We define $i^* = \argmax_{i \in \SS} \abs{h_i'(u_i^2)}$.
      There exists $\lambda = \abs{h_i'(v_i^2)}$ for all $i \in \SS$ (by \cref{obs:stationary-criterion}).
      Since $\abs{u_\ell} \geq \abs{v_\ell}$, the monotonicity of the $h_i'$s with $h_i'(0) = 0$ (see \cref{lem:h-props}) implies that $\abs{h_i'(u_{\ell}^2)} \geq \abs{h_i'(v_\ell^2)} = \lambda$.
      In particular, $\abs{h_{i^*}(u_{i^*}^2)} \geq \lambda = h_{i^*}(v_{i^*}^2)$. %, and hence $\abs {u_{i^*}} \geq \abs{v_{i^*}}$.

      We proceed in the following two cases, which by \cref{claim:robust-gi-frac-trapping} cover all possible cases.
      \resetcasecount
      \begin{case}
        $\abs{h_\ell'(u_\ell^2)} \geq (1+\Delta)\lambda$.
      \end{case}
      \begin{caseblock}
        Since $\abs{h_i'(u_{i^*}^2)} \geq \abs{h_\ell'(u_\ell^2)}$, it follows that $\abs{h_{i^*}'(u_i^2)} \geq (1+\Delta)\lambda$.
        Since $\sum_{i \in \SS} u_i^2 \leq \sum_{i \in \SS} v_i^2 = 1$ implies the existence of $i \in \SS$ such that $u_i^2 \leq v_i^2$, it follows that $u_k^2 < v_k^2$ and hence that $\abs{h_k'(u_k^2)} \leq \lambda$.
        In particular, we obtain $\frac{\abs{h_{i^*}'(u_{i^*}^2)}}{\abs{ h_k'(u_k^2)}} \geq \frac{(1+\Delta)\lambda}{\lambda} \geq (1+\Delta)$.
      \end{caseblock}
      \begin{case}
        $\abs{h_k'(u_k^2)} \leq (1+\Delta)^{-1}\lambda$.
      \end{case}
      \begin{caseblock}
        Using that $\abs{h_{i^*}(u_{i^*}^2)} \geq \lambda$, we obtain $\frac{\abs{h_{i^*}'(u_{i^*}^2)}}{\abs{ h_k'(u_k^2)}} \geq \frac{\lambda}{(1+\Delta)^{-1}\lambda} \geq (1+\Delta)$.
      \end{caseblock}
    \noindent
    In both cases, we obtain
    \begin{equation}\label{eq:96305269}
      \frac{\abs{h_{i^*}'(u_{i^*}^2)}}{\abs{h_k'(u_k^2)}}
      \geq (1+\Delta)
    \end{equation}
    We use this fact to bound $\GG_{i^*}(\vec u)^2$ from below.
    \begin{align*}
      \frac{1}{\GG_{i^*}(\vec u)^2}
      &= \frac{\sum_{i=1}^\dn\GG_i(\vec u)^2}{\GG_{i^*}(\vec u)^2}
      = \frac{\sum_{i=1}^\dn[h_i'(u_i^2)u_i]^2}{[h_{i^*}'(u_{i^*}^2)u_{i^*}]^2} \\
      &\leq \frac{\sum_{i=1}^\dn u_i^2 - [1 - (1+\Delta)^{-2}]u_k^2}{u_{i^*}^2}
      \leq \frac{1 - [1 - (1+\Delta)^{-2}]\tcompress^2}{u_{i^*}^2} \\
      &= \frac{(1+\Delta)^2[1-\tcompress^2] + \tcompress^2}{(1+\Delta)^2 u_{i^*}^2} \ .
    \end{align*}
    In the above, for the first inequality, we use our bound from \cref{eq:96305269} and that $\abs{h_{i^*}'(u_{i^*}^2)} \geq \abs{h_i'(u_i^2)}$ for all $i \in [\dn]$.
    To see that this holds for $i \in \compl \SS$, we note that by \cref{cor:tcompress} treating $\nabla F$ as a 0-approximation of itself that $\abs{\GG_i(\vec u)} \leq \frac 1 2 \abs{u_i}$.
    In particular, since $\norm{\GG(\vec u)} = 1 = \norm{\vec u}$, there exists $j \in [\dn]$ such that $\abs{\GG_j(\vec u)} \geq \abs{u_j}$.
    Hence, $1 < \frac{\abs{\GG_j(\vec u)}/{\abs{u_j}}}{\abs{\GG_i(\vec u)}/\abs{u_i}} =\frac{h_j'(u_j^2)}{h_i'(u_i^2)}$ implies that $\abs{h_i'(u_i^2)}$ is not maximal for any $i \in \compl \SS$.
    In particular, $\ell = \argmax_{i \in [\dn]} \abs{h_i'(u_i^2)}$.

    Continuing, we obtain:
    \begin{align*}
      \frac{\GG_{i^*}(\vec u)^2}{u_{i^*}^2}
      &\geq \frac{(1+\Delta)^2}{(1+\Delta)^2[1-\tcompress^2] + \tcompress^2}
      = 1 + \tcompress^2 \frac{(1+\Delta)^2 - 1}{(1+\Delta)^2[1-\tcompress^2] + \tcompress^2} \\
      &= 1 + \tcompress^2\frac{2\Delta + \Delta^2}{(1+\Delta)^2 - \tcompress[(1+\Delta)^2 - 1]}
      \geq 1 + \tcompress^2 \frac{2\Delta + \Delta^2}{(1+\Delta)^2} \ .
      % =1 + \tcompress^2[1 - (1+\Delta)^{-2}] \ .
    \end{align*}
    Applying \cref{lem:G-est-error}, we see that:
    \begin{displaymath}
      \abs{\hat \GG_{i^*}(\vec u)}
      \geq \abs{\GG_{i^*}(\vec u)} - 4 \delta \abs{\SS}^\delta \epsilon / \beta
      \geq \abs{u_{i^*}}\sqrt{1 + \tcompress^2\frac{2\Delta+\Delta^2}{(1+\Delta)^2}}
      - 4 \delta \abs{\SS}^\delta \epsilon / \beta \ .
    \end{displaymath}
    Noting that $\abs{u_{i^*}} \geq \abs{v_{i^*}}$ and applying the lower bound from \cref{lem:vi-spread}, we obtain
    \begin{align*}
      \frac{\abs{\hat \GG_{i^*}(\vec u)}}{\abs{u_{i^*}}}
      &\geq \sqrt{1 + \tcompress^2\frac{2\Delta+\Delta^2}{(1+\Delta)^2}}
      - 4 \frac{\delta \abs{\SS}^\delta \epsilon}{\beta\abs{u_{i^*}}} \\
      &\geq \sqrt{1 + \tcompress^2\frac{2\Delta+\Delta^2}{(1+\Delta)^2}} - 4\big(\frac{\alpha \delta \abs{\SS}^\delta}{\beta \gamma}\big)^{\frac 1 {2\gamma}} \delta \abs \SS^\delta \epsilon / \beta \ .
    \end{align*}
    From the Taylor expansion of $f(x) = \sqrt{1+x}$, we see that $f(x) \geq 1 + \frac 1 2 x - \frac 1 8 x^2$.
    When $x \in [0, 1]$, we have that $f(x) \geq 1 + \frac 3 8 x$.
    In particular,
    \begin{displaymath}
      \frac{\abs{\hat \GG_{i^*}(\vec u)}}{\abs{u_{i^*}}}
      \geq 1 + \frac 3 8 \tcompress^2\frac{2\Delta+\Delta^2}{(1+\Delta)^2}
      - 4\big(\frac{\alpha \delta \abs{\SS}^\delta}{\beta \gamma}\big)^{\frac 1 {2\gamma}} \delta \abs \SS^\delta \epsilon / \beta \ .
    \end{displaymath}
    By the bound of $\epsilon \leq \frac{3}{32} \tcompress^2\frac{2\Delta + \Delta^2}{(1+\Delta)^2} \cdot \big(\frac{\beta\gamma}{\alpha \delta \abs{\SS}^\delta}\big)^{\frac 1 {2\gamma}} \frac{\beta}{\delta \abs{\SS}^\delta}$, we see that $\frac{\abs{\hat \GG_{i^*}(\vec u)}}{\abs{u_{i^*}}} \geq 1$.
  \end{subproof}
  Recall that to apply \cref{claim:robust-increasing-hi-un}, it suffices
  \begin{displaymath}
    \epsilon \leq \frac{3}{32} \tcompress^2\frac{2\Delta + \Delta^2}{(1+\Delta)^2} \cdot \big(\frac{\beta\gamma}{\alpha \delta \abs{\SS}^\delta}\big)^{\frac 1 {2\gamma}} \frac{\beta}{\delta \abs{\SS}^\delta}.
  \end{displaymath}
  We choose $\Delta = \frac 3 4 \frac {\beta \gamma}{\alpha \delta} \Big( \frac{\beta \gamma}{\alpha \delta\abs{\SS}^\delta} \Big)^{\frac{\delta - \gamma}{\gamma}}\eta$.
  We note that $\Delta \in (0, 1)$ since $\eta \in (0, 1)$.
  Thus, $\frac{2\Delta + \Delta^2}{(1+\Delta)^2} \geq \frac 1 2 \Delta$.
  % , and when $\Delta > 1$, then $\frac{2\Delta + \Delta^2}{(1+\Delta)^2} \geq 1 - \frac{1}{(1+\Delta)^2} \geq \frac 3 4 \geq 1 2$.
  % As such, $\frac{2\Delta + \Delta^2}{(1+\Delta)^2} \geq \frac 1 2\max(\Delta, 1)$, and in particular,
  It follows:
  \begin{align*}
    \frac{3}{32} \tcompress^2\frac{2\Delta + \Delta^2}{(1+\Delta)^2} \cdot \big(\frac{\beta\gamma}{\alpha \delta \abs{\SS}^\delta}\big)^{\frac 1 {2\gamma}} \frac{\beta}{\delta \abs{\SS}^\delta}
    &\geq \frac 9 {256}\tcompress^2 \cdot \frac {\beta \gamma}{\alpha \delta} \Big( \frac{\beta \gamma}{\alpha \delta\abs{\SS}^\delta} \Big)^{\frac{\delta - \gamma}{\gamma}}\eta \cdot \big(\frac{\beta\gamma}{\alpha \delta \abs{\SS}^\delta}\big)^{\frac 1 {2\gamma}} \frac{\beta}{\delta \abs{\SS}^\delta} \\
    &\geq \frac{9}{256} \frac \beta \delta \cdot \tcompress^2 \cdot \Big(\frac{\beta \gamma}{\alpha \delta}\Big)^{\frac {2\delta + 1}{2\gamma}}\abs{\SS}^{-\frac {\delta}{\gamma}(\frac 1 2 + \delta - \gamma) - \delta}\eta .
  \end{align*}
  In particular, it suffices that $\epsilon \leq \frac{9}{256} \frac \beta \delta \cdot \tcompress^2 \Big(\frac{\beta \gamma}{\alpha \delta}\Big)^{\frac {2\delta + 1}{2\gamma}}\abs{\SS}^{-\frac {\delta}{\gamma}(\frac 1 2 + \delta - \gamma) - \delta}\eta$ in order to apply \cref{claim:robust-increasing-hi-un} with the choice of $\Delta = \frac {\beta \gamma}{\alpha \delta} \Big( \frac{\beta \gamma}{\alpha \delta\abs{\SS}^\delta} \Big)^{\frac{\delta - \gamma}{\gamma}}\eta$.
  By actually applying \cref{claim:robust-increasing-hi-un}, we complete the proof.
\end{proof}

We now create a time bound for driving a new coordinate of $\vec u(0)$ to be small under the gradient iteration based on the preceding \cref{lem:robust-GI-progress}.
\begin{prop}\label{prop:ui-spread-time-bound}
  Let $\{ \vec u(n) \}_{n=0}^\infty$ be a sequence defined recursively in $\SS^{\dn - 1}$ by $\vec u(n) = \hat \GG(\vec u(n-1))$.
  Define the sets $\SS_n := \{ i \suchthat \abs{u_i(n)} \geq \tcompress \}$.
  Suppose $\abs {\SS_0} \geq 2$,
  $\SS_0 \subset [\dm]$,
  $\vec v\in \porth^{\dn - 1}$ is the stationary point of $\GG/{\sim}$ such that $v_i \neq 0$ if and only if $i \in \SS_0$,
  $\norm{\Pmap_0 \vec u} \leq \frac 1 {2\sqrt{1+2\delta}}$,
  and fix an $\eta \in (0, 1]$.
  Let
  \begin{displaymath}
    N = \Big\lceil \frac 3 {2\eta} \big(\frac{\alpha \delta}{\beta \gamma}\big)^{\frac \delta \gamma}\dm^{\frac \delta \gamma(\delta - \gamma)} \Big[\frac 1 \gamma \ln(\frac{4 \alpha \delta}{\beta \gamma}) + \frac \delta \gamma \ln \dm\Big]\Big\rceil + 2N_{\ref{prop:GI-Loop-Small-Coords-Error}} \ .
  \end{displaymath}
  If $\max_{i, j \in \SS_0} \frac {\abs{u_i(0)}/v_i}{u_j(0)/v_j} \geq (1+\eta)^2$,
  $\max_{i \in \SS_0} \abs{u_i}/v_i \geq 1$,
  and $\epsilon \leq \epsdiverge{\eta, [\dm]}$, then there exists $j \in \SS_0$ such that $\norm{\Pmap_{\compl \SS_0 \cup \{j\}} \vec u(n)} \leq 4\epsilon \delta (\abs {\SS_0}-1)^\delta / \beta$ for all $n \geq N$.
\end{prop}
\newcommand{\Nspread}[1]{\ensuremath{N_{\ref{prop:ui-spread-time-bound}}\left(#1\right)}}
For future reference, we define the expression $\Nspread{\eta}$ to be the value of $N$ in \cref{prop:ui-spread-time-bound} above.
\begin{proof}[\proofword\@of \cref{prop:ui-spread-time-bound}]
  We first use \cref{cor:P0small-closure} to see that $\norm{\Pmap_0 \vec u(n)} \leq \frac{1}{2\sqrt{1+2\delta}}$ for all $n$.
  Further, by repeated application of \cref{cor:tcompress} (with $A = [\dm]$), we see that $\compl \SS_0 \subset \compl \SS_1 \subset \compl \SS_2 \subset \cdots$, and in particular $\SS_0 \supset \SS_1 \supset \SS_2 \supset \cdots$.
  We let $N_0$ be the least integer such that $\SS_{N_0}$ is a strict subset of $\SS_0$.
  In order to compress notation, we define the constant $\kappa := \frac 3 4 \fcondexpand^{\frac{\delta}{\gamma}}\dm^{-\frac \delta \gamma (\delta - \gamma)}$.

  The proof has two parts.
  First we bound $N_0$.
  Then, we apply the small coordinate analysis from \cref{prop:GI-Loop-Small-Coords-Error} to $\vec u(N_0)$ to see that $\norm{\Pmap_{\compl \SS_{N_0}} \vec u(n)}$ becomes small as desired.
  The first part of the proof will rely on the following claim.

  \begin{claim}\label{claim:gi-expand-recursed}
    If $n < N_0$, then
    (1) $\max_{i, j \in \SS_n}\frac{\abs{u_i(n)}/v_i}{\abs{u_j(n)}/v_j} \geq (1+\kappa\eta)^n(1+\eta)^2$ and (2) $\max_{i \in \SS_n} \abs {u_i}/v_i \geq 1$.
  \end{claim}
  \begin{subproof}
    We proceed by induction on $n$.
    The base case $n=0$ is true by the givens of this Lemma.
    If the inductive hypothesis holds for some $k < N_0 - 1$, then we see that:
    \begin{displaymath}
      \max_{i, j \in \SS} \frac{\abs{u_i(k)}/v_i}{\abs{u_j(k)}/v_j} \geq (1+\kappa\eta)^k(1+\eta)^2 \geq (1+\eta)^2 \ .
    \end{displaymath}
    As such, we may apply \cref{lem:robust-GI-progress} to see that
    \begin{displaymath}
      \max_{i, j \in \SS} \frac{\abs{u_i(k+1)}/v_i}{\abs{u_j(k+1)}/v_j}
      \geq (1+\kappa\eta) \max_{i, j \in \SS} \frac{\abs{u_i(k)}/v_i}{\abs{u_j(k)}/v_j}
      \geq (1+\kappa\eta)^{k+1}(1+\eta)^2 \ .
    \end{displaymath}
    Further, part~\ref{lem:robust-GI-progress:it:3} of \cref{lem:robust-GI-progress} implies that $\max_{i\in \SS_k} \abs{u_i(k+1)/v_i} \geq 1$.
  Since $\SS_{k+1} \supset \SS_k$, it follows that $\max_{i\in \SS_{k+1}} \abs{u_i(k+1)/v_i} \geq 1$.
  \end{subproof}
  \noindent
  We now derive an upper bound for $N_0$.
  By construction, $N_0 \geq 1$.
  We note that
  \begin{displaymath}
    \max_{i, j \in \SS_{N_0-1}}\frac{\abs{u_i(N_0)}/v_i}{\abs{u_j(N_0)}/v_j}
    \leq \max_{i, j \in \SS_{N_0-1}} \frac{1/v_i}{\abs{u_j}}
    \leq \tcompress^{-1}\Big(\frac{\alpha \delta \abs {\SS_0}^\delta}{\beta \gamma}\Big)^{\frac 1 {2\gamma}}
    \leq \Big(\frac{4\alpha \delta}{\beta \gamma}\Big)^{\frac 1 \gamma}\dm^{\frac \delta \gamma} ,
  \end{displaymath}
  where we use that $\vec u$ and $\vec v$ are unit vectors in the first inequality, the bounds from the definition of $\SS_{N_0-1}$ and \cref{lem:vi-spread} for the second inequality, and that $\abs {\SS_0} \leq \dm$ combined with the definition of $\tcompress$ in the final inequality.
  From \cref{claim:gi-expand-recursed}, we obtain $\max_{i, j \in \SS_{N_0-1}}\frac{\abs{u_i(N_0-1)}/v_i}{\abs{u_j(N_0-1)}/v_j} \geq (1+\kappa\eta)^{N_0-1}(1+\eta)^2 \geq (1+\kappa\eta)^{N_0+1} \geq (1+\kappa\eta)^{N_0}$ since $\kappa \leq 1$.
  It follows that $(1+\kappa \eta)^{N_0} \leq \big(\frac{4\alpha \delta}{\beta \gamma}\big)^{\frac 1 \gamma}\dm^{\frac \delta \gamma}$, and in particular,
  \begin{displaymath}
    N_0 \leq \frac{\frac 1 \gamma \ln(\frac{4 \alpha \delta}{\beta \gamma}) + \frac \delta \gamma \ln \dm}{\ln(1+\kappa \eta)} \ .
  \end{displaymath}
  We now simplify our bound on $N_0$ using the Taylor expansion of $\ln(1+x) \geq x - \frac 1 2 x^2$.
  In particular, when $x \in [0, 1]$, $\ln(1+x) \geq \frac 1 2 x$.
  Since $\kappa \leq 1$ and $\eta \leq 1$, it follows that $\ln(1+\kappa \eta) \geq \frac 1 2 \kappa \eta$.
  In particular,
  \begin{displaymath}
    N_0 \leq \frac 2 {\kappa \eta}\Big[\frac 1 \gamma \ln\Big(\frac{4 \alpha \delta}{\beta \gamma}\Big) + \frac \delta \gamma \ln \dm \Big] \ .
  \end{displaymath}

  We now construct the final time bound for our lemma.
  Fix $j \in \SS_0 \setminus \SS_{N_0}$.
  By applying \cref{prop:GI-Loop-Small-Coords-Error} to $\vec u(N_0)$, we obtain that $\norm{\Pmap_{\compl \SS_0 \cup \{j\}} \vec u(n)} \leq \norm{\Pmap_{\compl \SS_{N_0}} \vec u(n)} \leq 4 \delta \abs{\SS_{N_0}}^\delta\epsilon / \beta$ for all $n \geq N_0 + 2N_{\ref{prop:GI-Loop-Small-Coords-Error}}$.
\end{proof}

\subsection{Jumping Out of Stagnation}\label{sec:perturb-step-analysis}

In this section, we analyze the effect of taking a random jump on the sphere from a starting point $\vec p \in \sphere^{\dn - 1}$.
In particular, we analyze steps~\ref{alg:step:draw-jump} and~\ref{alg:step:jump1} of \textsc{FindBasisElement} and demonstrate that under the random jump $\vec w \leftarrow \vec u \cos(\norm{\vec x}) + \frac{\norm{\vec x}}{\vec x}\sin(\norm{\vec x})$ from step~\ref{alg:step:jump1}, we with non-negligible probability obtain a new starting vector $\vec w$ from which running the \textsc{GI-Loop} causes a coordinate of $\vec w$ to be driven towards zero under \cref{prop:ui-spread-time-bound}.

We first recall that for $\vec p \in \sphere^{\dn - 1}$, the tangent space of the sphere $\sphere^{\dn - 1}$ is $T_{\vec p}\sphere^{\dn - 1} = \vec p^{\perp}$.
Geometrically, this can be interpreted as the plane perpendicular to $\vec p$ treating $\vec p$ as the plane's origin.
In this section, we will be particularly interested in the map $\vec x \mapsto \vec u \cos(\norm{\vec x}) + \frac{\vec x}{\norm{\vec x}}\sin(\norm{\vec x})$ from step~\ref{alg:step:jump1} of \textsc{FindBasisElement}.
This map is sometimes referred to as the exponential map on the sphere $\sphere^{\dn - 1}$.
That is, at any $\vec p \in \sphere^{\dn - 1}$ the exponential map $\exp_{\vec p} : T_{\vec p} \sphere^{\dn - 1} \rightarrow \sphere^{\dn - 1}$ is defined as
\begin{displaymath}
  \exp_{\vec p}(\vec x) =
  \vec p \cos(\norm {\vec x}) + \frac{\vec x}{\norm{\vec x}}\sin(\norm{\vec x})  ,
\end{displaymath}
where it is understood that $\exp_{\vec p}(\vec 0) = \vec p$.

We now proceed in showing that the random jump from steps~\ref{alg:step:draw-jump} and~\ref{alg:step:jump1} of \textsc{FindBasisElement} is able to break stagnation of the gradient iteration algorithm without causing any harm.
By breaking stagnation, we mean that if $\SS \subset [\dm]$ is the set of large coordinates---i.e., $\abs{u_i} \geq \tcompress$ if and only if $i \in \SS$---then with constant probability, applying the random jump to $\vec u$ should set up the preconditions of \cref{prop:ui-spread-time-bound}.
By causing no harm, we mean that if coordinates in $\compl \SS$ are very small---i.e., $\norm{\Pmap_{\compl \SS} \vec u} \leq 4\delta \abs \SS^\delta \epsilon / \beta$---then by applying the random jump, the coordinates in $\compl \SS$ remain small---i.e., $\norm{\Pmap_{\compl \SS} \vec u} \leq \tcompress$.
As we apply the gradient iteration for sufficiently many steps between random jumps in step~\ref{alg:step:gi-in-main-loop} of \textsc{FindBasisElement}, any coordinates of $\vec u$ which are small after a random jump return to being very small before the next random jump.
In this manner, the random jump causes no harm.

We first provide conditions under which the random jump causes no harm.

\begin{lem}\label{lem:jump-small-coords}
  Let $\vec u \in \SS^{\dn - 1}$.
  Let $\SS \subset [\dm]$ be such that $\norm{\Pmap_{\compl \SS} \vec u} \leq 4 \delta \abs \SS^\delta \epsilon / \beta$.
  Suppose that
  $\epsilon \leq \frac 1 {8\sqrt{2(1+2\delta)}} \tcompress \frac{\beta}{\delta \abs \SS^\delta}$,
  and suppose that $\sigma \leq \frac 1 {2\sqrt{2(1+2\delta)}} \tcompress$.
  If $\vec w$ is drawn uniformly at random from $\sigma \sphere^{\dn - 1} \cap \vec u^\perp$ and $\vec w = \exp_{\vec u}(\vec x)$, then $\norm{\Pmap_{\compl \SS} \vec u} \leq \frac{\tcompress}{\sqrt{2(1+2\delta)}}$.
\end{lem}
\begin{proof}
    We note:
  \begin{align*}
    \norm{\Pmap_{\compl \SS} \vec w}
    &= \Norm{\Pmap_{\compl \SS} [\vec u \cos (\norm{\vec x})
      + \frac {\vec x}{\norm{\vec x}}\sin(\norm{\vec x})] } \\
    &\leq \norm{\Pmap_{\compl \SS} \vec u} + \abs{\sin(\norm{\vec x})}
    \leq \norm{\Pmap_{\compl \SS} \vec u} + \norm{\vec x}
    \leq \norm{\Pmap_{\compl \SS} \vec u} + \sigma \\
    &\leq \frac \tcompress {2\sqrt{2(1+2\delta)}} +  4\epsilon \delta \abs \SS^\delta / \beta
    \leq \frac{\tcompress}{\sqrt{2(1+2\delta)}}.
  \end{align*}
  In the above, the first inequality uses the triangle inequality, that $\abs {\cos(\norm{\vec x})} \leq 1$ to obtain the first summand, and that $\norm{\Pmap_{\compl \SS}\vec u} \leq 1$ and $ \norm{\frac {\vec x}{\norm{\vec x}}} = 1$ to obtain the second summand.
  The second inequality uses the Taylor series of $\sin$ to bound $\sin(\norm {\vec x}) \leq \norm {\vec x}$.
  The third inequality uses that $\norm {\vec x} = \sigma$ since $\vec x$ is drawn from a sphere with radius $\sigma$.
  The fourth inequality uses the bounds on $\sigma$ and $\norm{\Pmap_{\compl \SS} \vec u}$.
  The final inequality uses the upper bound on $\epsilon$.
\end{proof}

Before proceeding stating conditions under which the random jump breaks stagnation, we first state two minor technical results which will useful in showing that the random jump actually creates a significant change in the coordinates of interest.

\begin{fact}\label{fact:perturb-sphere-0sep}
  Let $\sigma \geq 0$, and suppose that $\vec x$ is drawn uniformly at random from the sphere $\sigma \sphere^{n-1}$.
  If $\XX \ni \vec 0$ is a $k$-dimensional subspace of $\R^n$,
  then with probability at least $(1-\exp(-\tfrac{1}{16} k))(1 - \exp(-\tfrac{2-\sqrt 3}{2} (n-k))$, $\norm{\Pmap_{\XX} \vec x}^2 \geq \frac k {4 n}\sigma^2$.
\end{fact}
\def\cprobi{\ensuremath{C_{\ref{fact:perturb-sphere-0sep}}}}
For later usage, we denote by $\cprobi := \left(1-\exp\left(-\frac 1 {16}\right)\right)\left(1-\exp\left(-\frac {2-\sqrt 3}{2}\right)\right)$, which is a constant lower bound on the probability that is obtained in \cref{fact:perturb-sphere-0sep} with $0 < k < n$.
\begin{proof}[\proofword\@of \cref{fact:perturb-sphere-0sep}]
  We consider the following sampling process for constructing $\vec x$: (1) draw $\vec w \sim \NN(\vec 0, \Id)$, and (2) let $\vec x \leftarrow \sigma\frac{\vec w}{\norm{\vec w}}$.
  We note that this process is equivalent to drawing $\vec x$ from $\sigma \sphere^{n - 1}$ uniformly at random almost surely.
  Further, $\norm{\Pmap_{\XX} \vec w}^2$ and $\norm{\Pmap_{\XX^\perp} \vec w}^2$ are independent random variables which are distributed according to the chi-squared distribution with $k$ and $n-k$ degrees of freedom respectively.
  In particular, applying \cref{cor:chi2-concentration}, we see that with probability at least
  $$\left(1-\exp\left(-\tfrac{1}{16} k\right)\right)\left(1 - \exp\left(-\tfrac{2-\sqrt 3}{2} (n-k)\right)\right),$$
  both of the events $\norm{\Pmap_{\XX} \vec w}^2 \geq \frac 1 2 k$ and $\norm{\Pmap_{\XX^\perp} \vec w}^2 \leq 2(n-k)$ occur.

  We proceed in the setting in which these events occur.
  We see that
  \begin{displaymath}
    \frac 1 {\sigma^2} \norm{\Pmap_{\XX} \vec x}^2
    = \frac{\norm{\Pmap_{\XX} \vec w}^2}{\norm{\Pmap_{\XX} \vec w}^2 + \norm{\Pmap_{\XX^\perp} \vec w}^2}
    \geq \frac{\norm{\Pmap_{\XX} \vec w}^2}{\norm{\Pmap_{\XX} \vec w}^2 + 2(n - k)} .
  \end{displaymath}
  We note that for any $C \geq 0$, the function $f(t) = \frac t {t+C} = 1 - C(t+C)^{-1}$ has derivative $f'(t) = C(t+C)^{-2} \geq 0$.
  In particular, $f$ is an increasing function.
  As such, we obtain:
  \begin{displaymath}
    \frac 1 {\sigma^2} \norm{\Pmap_{\XX} \vec x}^2
    \geq \frac{\frac 1 2 k}{\frac 1 2 k + 2(n - k)}
    \geq \frac{\frac 1 2 k}{2 k + 2(n - k)}
    = \frac k {4n}. \qedhere
  \end{displaymath}
\end{proof}

\begin{lem}\label{lem:pert-good-growth}
  Let $\SS \subset [\dm]$ be such that $\abs \SS \geq 2$.
  Suppose that $\vec u \in \sphere^{\dn - 1}$ and
  that $\tau \leq \min_{i \in \SS} \abs {u_i}$.
  Let $\mathcal X := T_{\vec u} \sphere^{\dn - 1} \cap \spn(\{\myZv_i \suchthat i \in \SS\})$.
  Let $\vec x \in \mathcal X$ and let $s \in \{+, -\}$ be a sign.
  We define $\Lambda_s(\vec x) := \{ i \in \SS \suchthat \sign(u_i) = s \sign(x_i)\}$.
  Then, $\norm{\Pmap_{\Lambda_s} \vec x} \geq \frac \tau {\sqrt 2} \norm{\vec x}$.
\end{lem}
\begin{proof}
  Since $\vec u \perp \vec x$, it follows that $\ipCanonical{\vec u}{\Pmap_{\Lambda_+}\vec x} = - \ipCanonical{\vec u}{\Pmap_{\Lambda_-}\vec x}$.
  Further, it can be seen that
  \begin{align*}
    \norm{\Pmap_{\Lambda_+}\vec x}
    = \norm{\vec u}\norm{\Pmap_{\Lambda_+}\vec x}
    \geq \ipCanonical{\vec u}{\Pmap_{\Lambda_+}\vec x}
    =-\ipCanonical{\vec u}{\Pmap_{\Lambda_-}\vec x}
    \geq \tau \norm[1]{\Pmap_{\Lambda_-} \vec x}
    \geq \tau \norm{\Pmap_{\Lambda_-} \vec x}
  \end{align*}
  using that $1 = \norm{\vec u}$ for the first equality, the Cauchy-Schwartz inequality for the first inequality, the lower bound on $\abs{u_i} \geq \tau$ for the second inequality, and the relationship $\norm[1]\argdot \geq \norm[2]\argdot$ for the final inequality.
  By similar reasoning swapping the roles of $\Lambda_+$ and $\Lambda_-$, we obtain:
  \begin{displaymath}
    \norm{\Pmap_{\Lambda_-} \vec x} \geq \tau \norm{\Pmap_{\Lambda_+}\vec x} \ .
  \end{displaymath}

  We note that one of $\norm{\Pmap_{\Lambda_+} \vec x}$ and $\norm{\Pmap_{\Lambda_-} \vec x}$ must be at least $\frac 1 {\sqrt 2} \norm{\vec x}$ by the Pytha\-gorean Theorem.
  In particular, both $\norm{\Pmap_{\Lambda_+} \vec x}$ and $\norm{\Pmap_{\Lambda_-} \vec x}$ are at least $\frac \tau {\sqrt 2} \norm{\vec x}$.
\end{proof}

The following fact will be useful when bounding individual coordinates of $\exp_{\vec u} (\vec x)$ for the random jump.
\begin{fact}\label{fact:sin-cos-bounds1}
  The trigonometric functions $\sin t$ and $\cos t$ can be lower bounded as follows:
  \begin{enumerate}
  \item When $t \in [0, 1]$, then $\sin t \geq \frac 2 3 t$.
  \item $\cos t \geq 1 - \frac 1 2 t^2$.
  \end{enumerate}
\end{fact}
\begin{subproof}
  We use the Taylor series of $\sin t$ and $\cos t$.
  For $\cos t$, this is a direct implication of the Taylor expansion.
  For $\sin t$, we note that $\sin t \geq t - \frac 1 3 t^3 \geq \frac 2 3  t$.
\end{subproof}

In the remaining Lemmas of this section, we provide conditions under which the random jump from step~\ref{alg:step:jump1} of \textsc{FindBasisElement} sets up the preconditions of \cref{prop:ui-spread-time-bound} with constant probability.
We first deal with the case where we are stagnated.
%We are stagnated if we do not have the necessary preconditions for progress under \cref{prop:ui-spread-time-bound}:
We say that the gradient iteration is stagnated for a choice of $\vec u$ if for the fixed point $\vec v$ of $\GG/{\sim}$ such that $v_i \neq 0$ if and only if $i$ is a ``large'' coordinate of $\vec u$, there is no $i$ among the large coordinates such that $\abs{u_i} > \abs{v_i}$ with a sufficient gap (as defined in the preconditions of \cref{prop:ui-spread-time-bound}).
We demonstrate that with at least constant probability, stagnation can be escaped by the random jump.
In the following Lemma, $C_{\ref{lem:lpert-analysis-expected}}$ is a strictly positive universal constant.
\begin{lem}\label{lem:lpert-analysis-expected}
  Let $\SS \subset [\dm]$ contain at least 2 elements.
  Let $\vec u \in \SS^{\dn - 1}$.
  Let $\vec v \in \porth^{\dn - 1}$ be the fixed point of $\GG / {\sim}$ such that $v_i \neq 0$ if and only if $i \in \SS$.
  Let $\eta := C_{\ref{lem:lpert-analysis-expected}} \frac{\sigma}{\sqrt \dn}(\frac {\beta \gamma}{\alpha \delta\abs \SS^\delta})^{\frac 1 \gamma} \tcompress^2$.
  Suppose that $\vec w$ is a random jump of $\vec u$ created according to the following process:
  Draw $\vec x$ uniformly at random from $\sigma \sphere^{\dn - 1} \cap \vec u^\perp$ and let $\vec w = \exp_{\vec u}(\vec x)$.

  If $\abs{u_i} \geq \tcompress$ for all $i \in \SS$,
  if $\norm{\Pmap_{\compl \SS} \vec u} \leq 4 \epsilon \delta \abs \SS^\delta / \beta$,
  if $\epsilon \leq \epsdiverge{\eta, \abs \SS}$,
  if $\sigma \leq \frac{1}{6\sqrt{2 \dn}}\tcompress^2 $,
  and if $\frac {\abs{u_i}}{v_i} < (1+\eta)^2$ for all $i \in \SS$,
  then with probability at least $\cprobi$, there exists $i \in \SS$ such that $\frac{\abs {w_i}}{v_i} \geq (1+\eta)^2$.
\end{lem}
\begin{proof}
  The main crux of the argument is captured in the following two claims.
  These claims use the construction of $\vec w$ and the choice of $\vec u$ from the lemma statement as unstated givens.
  \begin{claim}\label{claim:gi-precond-setup}
    Fix a constant $\Delta > 0$.
    Define $\kappa := \norm{\Pmap_{\compl \SS } \vec u}^2 + (1+\Delta)^2 - 1$. Suppose that $\kappa \leq \frac 1 2 (\frac{\beta \gamma}{\alpha \delta \abs \SS})^{\frac 1 {2\gamma}}$, and
    that $\frac{\abs{u_i}}{v_i} \leq 1 + \Delta$ for all $i \in \SS$ and that
    \begin{displaymath}
    {12\sqrt{2\dn}}\Big(\frac{\alpha \delta \abs \SS^\delta}{\beta \gamma}\Big)^{\frac 1 {\gamma}}\kappa / \tcompress^2
    \leq \sigma
    \leq \frac 1 {6\sqrt {2 \dn}} \tcompress^2 \ ,
    \end{displaymath}
    then with probability at least $\cprobi$, there exists $j \in \SS$ such that $\frac{\abs{w_j}}{v_j} \geq 1 + \Delta$.

    \vnote{Use case:  We will be plugging into the Lemma 6.15 preconditions to make $\abs{u_i}/\abs{v_i} \geq (1+\eta)^2$.  As such, we choose $\Delta$ large enough that $1+\Delta \geq (1 + \eta)^2$.  For this purpose, we may assume that $\eta \leq 1$.
    At the end, we will need to determine the appropriate $\epsilon$.}
  \end{claim}
  % \resetcasecount
  % \begin{case}
  %   Suppose that $\abs{u_i} \leq v_i + \Delta$ for all $i \in \SS$.
  % \end{case}
  \begin{subproof}
    We prove this claim in several parts.
    First, we demonstrate a lower bound on $\abs{u_i}$ in terms of $v_i$ for each $i \in \SS$.
    Then, we demonstrate that for some $j \in \SS$, $x_j$ is sufficiently large to make $\abs{w_j}/v_j \geq 1 + \Delta$.

    We now proceed with constructing a lower bound on the $\abs {u_i}$s.
    We fix a $j \in \SS$.
    \begin{align}
      u_j^2
      &= 1 - \norm{\Pmap_{\compl \SS}\vec u}^2 - \sum_{i \in \SS \setminus \{j\}} u_i^2
      = 1 - \norm{\Pmap_{\compl \SS}\vec u}^2 - \sum_{i \in \SS \setminus \{j\}} \frac{u_i^2}{v_i^2}v_i^2 \nonumber \\
      &\geq 1 - \norm{\Pmap_{\compl \SS}\vec u}^2 - (1+\Delta)^2\sum_{i \in \SS \setminus \{j\}} v_i^2 % \nonumber \\
      = 1 - \norm{\Pmap_{\compl \SS}\vec u}^2 - (1+\Delta)^2[1 - v_j^2] \nonumber \\
      % &= 1 - \norm{\Pmap_{\compl \SS} \vec u}^2 - \sum_{i \in \SS} \frac{u_i^2}{v_i^2}v_i^2 + \frac{u_j^2}{v_j^2}v_j^2  \nonumber \\
      &\geq v_j^2(1+\Delta)^2 - [(1+\Delta)^2-1 + \norm{\Pmap_{\compl \SS} \vec u}^2]
      = v_j^2(1+\Delta)^2 - \kappa . \label{eq:nearv-coord-lb}
    \end{align}

    We now demonstrate that with a random jump from $\vec u$, one of the coordinates $j \in \SS$ increases sufficiently from the lower bound in \cref{eq:nearv-coord-lb} with the desired probability.
    First, we note that there exists a 1-dimensional subspace $\mathcal \XX \subset \vec u^\perp \cap \spn(\myZv_{i_1}, \myZv_{i_2})$, where $i_1, i_2 \in \SS$.
    Let $\vec p$ be a unit vector in $\mathcal \XX$.
    We may apply \cref{fact:perturb-sphere-0sep} to obtain with probability at least $\cprobi$
    that $\abs{\ipCanonical{\vec p}{\vec x}} \geq \frac \sigma {2\sqrt{\dn - 1}} \geq \frac \sigma {2\sqrt \dn}$.
    We proceed under the assumption that this event occurs.

    Letting $\Lambda_+ := \{ i \in \{i_1, i_2\} \suchthat \sign(x_i) = \sign(u_i) \}$, we apply \cref{lem:pert-good-growth} to obtain $\norm{\Pmap_{\Lambda_+} \vec x} \geq \frac \tcompress {\sqrt 2} \norm{\vec x} \geq \frac 1 {2\sqrt {2\dn}} \sigma \tcompress$.
    In particular, there exists $j \in \SS$ such that $\sign(u_j)x_j \geq \frac 1 {2\sqrt {2\dn}} \sigma \tcompress$.
    It only remains to be seen that for this choice of $j$, $\abs{w_j}/\abs{v_j} \geq 1 + \Delta$.

    We note (using that $\norm{\vec x} = \sigma$):
    \begin{equation}\label{eq:lpert-1}
      w_j^2
      = [\exp_{\vec u}(\vec x)]_j^2
      \geq u_j^2 \cos^2(\sigma) + 2 u_j \cos( \sigma ) \frac{x_j}{\sigma}\sin(\sigma)
    \end{equation}
    Using \cref{fact:sin-cos-bounds1}, we see that
    \begin{displaymath}
      \cos^2(\sigma)
      \geq [1 - \frac 1 2 \sigma^2]^2
%      = [1-\frac 1 2 \sigma^2]^2
      \geq 1 - \sigma^2 \ .
    \end{displaymath}
    Further, since $\sigma \leq \tcompress \leq 1 < \frac \pi 3$, we see that $\cos(\sigma) \geq \frac 1 2$.
    Thus,
    \begin{displaymath}
      2 u_j \cos( \sigma ) \frac{x_j}{\sigma}\sin(\sigma)
      \geq u_j \frac{x_j}{\sigma}\sin(\sigma)
      \geq \frac 2 3 u_j x_j
      \geq \frac 1 {3\sqrt{2\dn}} \sigma \tcompress^2
    \end{displaymath}
    where we use \cref{fact:sin-cos-bounds1} in the second inequality and the lower bounds on $x_j$ and $u_j$ in the final inequality.
    Continuing from \cref{eq:lpert-1}, we obtain:
    \begin{align*}
      w_j^2
      &\geq (1-\sigma^2)u_j^2 + \frac 1 {3\sqrt{2\dn}} \sigma \tcompress^2
      \geq u_j^2[1-\sigma^2 + \frac 1 {3\sqrt {2\dn}}\sigma \tcompress^2] \\
      &\geq u_j^2[1+\frac 1 {6\sqrt{2\dn}} \sigma \tcompress^2]
      \geq \big[v_j^2(1+\Delta)^2 - \kappa \big] \cdot [1+\frac 1 {6\sqrt{2\dn}} \sigma \tcompress^2]
    \end{align*}
    by using the upper bound on $\sigma$ in the third inequality, and by applying \cref{eq:nearv-coord-lb} in the final inequality.
    Using the lower bound on $\sigma$, we see that
    \begin{align*}
    w_j^2
    &\geq [v_j^2(1+\Delta)^2 - \kappa]\Big[1 + 2\Big(\frac{\alpha \delta}{\beta \gamma}\abs\SS^\delta\Big)^{\frac 1 {\gamma}}\kappa \Big] \\
    &= v_j^2(1+\Delta)^2 + \kappa\left(2v_j^2\Big(\frac{\alpha \delta}{\beta \gamma}\abs\SS^\delta\Big)^{\frac 1 {\gamma}}(1+\Delta)^2 - 1 - 2\Big(\frac{\alpha \delta}{\beta \gamma}\abs\SS^\delta\Big)^{\frac 1 {\gamma}}\kappa\right) \\
    &\geq v_j^2(1+\Delta)^2 + \kappa\left(2 - 1 - 1\right)
    = v_j^2(1+\Delta)^2  ,
    \end{align*}
    where we use the lower bound on $v_j^2$ from \cref{lem:vi-spread} and the assumed upper bound on $\kappa$ in the final inequality.
    Rearranging terms completes the proof.
  \end{subproof}

  \begin{claim}\label{claim:gi-precond-setup2}
    Let $\lambda \in (0, 1]$ and set $\sigma = \frac {\lambda} {6\sqrt{2\dn}} \tcompress^2$.
    Suppose $\eta = \frac{C_1 \lambda}{\dn}\Big(\frac{\beta \gamma}{\alpha \delta \abs \SS^\delta}\Big)^{\frac 1 \gamma} \tcompress^4$ (where $C_1$ is a universal constant which happens to satisfy $C_1 \in (0, 1)$),
    and suppose $\epsilon \leq \epsdiverge{\eta, \SS}$.
    Then, with probability at least $\cprobi$, there exists $i \in \SS$ such that $\frac{\abs{w_i}}{v_i} \geq (1 + \eta)^2$.
  \end{claim}
  \begin{subproof}
  During the proof, we will wish to apply \cref{claim:gi-precond-setup} with $1+\Delta = (1+\eta)^2$.
  Since $\eta \leq 1$, we have that
  \begin{displaymath}
    (1+\Delta)^2
    = (1+\eta)^4
    = 1 + 4\eta + 6\eta^2 + 4\eta^3 + \eta^4
    \leq 1 + 15 \eta \ .
  \end{displaymath}
  As such, we may bound the expression $\kappa := \norm{\Pmap_{\compl \SS} \vec u}^2 + (1+\Delta)^2 - 1$ by
  \begin{displaymath}
    \kappa
    \leq 4 \delta \abs \SS^\delta \epsilon / \beta + 15 \eta
    < 16 \eta ,
  \end{displaymath}
  where we use a loose version of our bound $\epsilon \leq \epsdiverge{\eta, \SS}$ in the final inequality.
  We note that with the universal constant $C_1 = \frac 1 {2304} = \frac 1 {12\sqrt 2 \cdot 6\sqrt 2 \cdot 16}$, we obtain that
  \begin{displaymath}
    12\sqrt{2\dn}\Big(\frac{\alpha \delta \abs \SS^\delta}{\beta \gamma}\Big)^{\frac 1 \gamma} \frac{\kappa}{\tcompress^2}
    \leq 16 \cdot 12\sqrt{2\dn}\Big(\frac{\alpha \delta \abs \SS^\delta}{\beta \gamma}\Big)^{\frac 1 \gamma} \frac{\eta}{\tcompress^2}
    = \frac{\lambda}{6\sqrt{2\dn}} \tcompress^2
    = \sigma
  \end{displaymath}
  by using the upper bound on $\kappa$ in the inequality,
  by using the given choice of $\eta$ in the first equality,
  and by using the choice of $\sigma$ in the second equality.
  As such, we may apply \cref{claim:gi-precond-setup} to obtain that with probability at least $\cprobi$, there exists $i \in \SS$ such that $\abs{w_i} / v_i \geq (1+\eta)^2$.
   \end{subproof}
   To complete the proof, we apply \cref{claim:gi-precond-setup2}.
%    and $\eta$ as a function of $\sigma$.
   We note that
   \begin{displaymath}
     \eta = O\Big(\frac{\sigma}{\sqrt{\dn}}\Big(\frac{\beta \gamma}{\alpha \delta \abs \SS^\delta}\Big)^{\frac 1 \gamma}\tcompress^2\Big).
   \end{displaymath}
   In particular, it suffices that $\epsilon$ be upper bounded by $O\big(\epsdiverge{\frac{\sigma}{\sqrt{\dn}}(\frac{\beta \gamma}{\alpha \delta \abs \SS^\delta})^{\frac 1 \gamma}\tcompress^2, \abs \SS} \big)$.
\end{proof}

In the following two Lemmas, we consider the case where the gradient iteration is not stagnated, and we demonstrate with at least constant probability the random jump does not move us into stagnation.
In \cref{lem:jump-with-large-coord}, we deal with the case that the preconditions of \cref{prop:ui-spread-time-bound} are actually met, and we demonstrate that with some constant probability, applying the random jump does not cause the essential preconditions for gradient iteration progress to be undone.
Finally, in \cref{lem:jump-with-small-coord}, we demonstrate that if there is a coordinate of $\vec u$ which is not known to be small but has decreased beneath the threshold $\tcompress$, then applying the random jump leaves that coordinate small with constant probability.
In essence, these Lemmas demonstrate that the random jump does not undo unforeseen progress of the gradient iteration.
\begin{lem}\label{lem:jump-with-large-coord}
  Let $\SS \subset [\dm]$,
  let $\vec u \in \sphere^{\dn - 1}$,
  let $\eta > 0$ be a constant,
  and let $\vec v \in \porth^{\dn - 1}$ be the fixed point of $\GG/{\sim}$ such that $v_i \neq 0$ if and only if $i \in \SS$.
  %Suppose that $\abs{u_i} \geq \tcompress$ for all $i \in \SS$.
  Let $\vec w$ be constructed according to the following random process:
  Draw $\vec x$ uniformly at random from $\sigma \sphere^{\dn - 1} \cap \vec u^\perp$, and let $\vec w = \exp_{\vec u}(\vec x)$.

  If $\abs{u_i} \geq \tcompress$ for all $i \in \SS$,
  if there exists $j \in \SS$ such that $\abs{u_j} \geq (1+\eta)^2v_j$,
  and if $\sigma \leq \frac 2 {3\sqrt {2 \dn}} \tcompress$,
  then with probability at least $\frac 1 2 \cprobi$ we obtain $\abs{w_j} \geq \abs {u_j}$ and in particular $\abs{w_j} \geq (1+\eta)^2v_j$.
\end{lem}
\begin{proof}
  Since $\abs{u_j} > v_j$, it follows that $\vec v$ is not a canonical vector, and in particular that $\abs{\SS} \geq 2$.
  We set $A = \{ i, j \}$ by choosing $i \neq j$ such that $i \in \SS$.
  We apply \cref{fact:perturb-sphere-0sep} to see that with probability at least $\cprobi$, $\norm{\Pmap_A \vec x} \geq \frac 1 {2\sqrt \dn} \sigma$.
  Further, by the spherical symmetry of the distribution of $\vec x$, with probability at least $\frac 1 2 \cprobi$ both $\sign(u_j) x_j \geq 0$ and $\norm{\Pmap_A \vec x} \geq \frac 1 {2\sqrt \dn} \sigma$.
  We proceed under the assumption that this event occurs.

  Applying \cref{lem:pert-good-growth} with $\tau = \tcompress$, we see that $\abs{x_j} \geq \frac \tcompress {\sqrt 2} \norm{\Pmap_{A} \vec x} \geq \frac \tcompress {2\sqrt{2\dn}}\sigma$.
  As such, we obtain:
  \begin{align*}
    \abs{w_j}
    &= \abs{(\exp_{\vec u}(\vec x))_j}
    = \abs{u_j \cos(\sigma) + \frac 1 \sigma x_j \sin(\sigma)}
    \geq \abs{u_j (1-\frac 1 2 \sigma^2) + \frac 2 3 x_j} \\
    &\geq \abs{u_j}(1 - \frac 1 2 \sigma^2) + \frac \tcompress {3\sqrt{2\dn}} \sigma
    \geq \abs {u_j} - \frac 1 2 \sigma^2 + \frac \tcompress {3\sqrt{2\dn}} \sigma
    \geq \abs {u_j},
  \end{align*}
  where we use that $\norm{\vec x} = \sigma$ in the second equality,
  \cref{fact:sin-cos-bounds1} in the first inequality,
  that $\abs {u_j} \leq 1$ in the third inequality,
  and the bound on $\sigma$ in the final inequality.
\end{proof}

\begin{lem}\label{lem:jump-with-small-coord}
  Let $\vec u \in \sphere^{\dn - 1}$
  Suppose there exists $j \in [\dn]$ such that $\abs{u_j} \leq \tcompress$.
  Let $\vec w$ be a random jump of $\vec u$ constructed by the following process:
  Draw $\vec x$ uniformly at random from $\sigma \sphere^{\dn - 1} \cap \vec u^\perp$, and let $\vec w = \exp_{\vec u}(\vec x)$.

  If $\sigma \leq \frac 1 2 \tcompress$, then with probability at least $\frac 1 2$, $\abs{w_j} \leq \tcompress$.
\end{lem}
\begin{proof}
  We will assume without loss of generality that $u_j \geq 0$.
  Using the spherical symmetry of the sampling process, we see that with probability at least $\frac 1 2$, $x_j \leq 0$.
  We proceed under the assumption that this event occurs.
  We first upper bound $w_j$:
  \begin{displaymath}
    w_j
    = (\exp_{\vec u}(\vec x))_j
    = u_j \cos(\sigma) + \frac 1 \sigma x_j \sin(\sigma)
    \leq u_j \leq \tcompress .
  \end{displaymath}
  since both $\cos(\sigma) \leq 1$ and $x_j \leq 0$.
  Further, we may also lower bound $w_j$:
  \begin{align*}
    w_j
    &= u_j \cos(\sigma) + \frac 1 \sigma x_j \sin(\sigma)
    \geq u_j ( 1 - \frac 1 2 \sigma^2 ) - \abs{x_j}
    \geq - \frac 1 2 \tcompress \sigma^2  - \sigma \\
    &\geq - \frac 1 8 \tcompress^3 - \frac 1 2 \tcompress
    > - \tcompress,
  \end{align*}
  where we use \cref{fact:sin-cos-bounds1} to bound $\cos(\sigma)$ and $\sin(\sigma) \leq \sigma$ to bound $\sin(\sigma)$ in the first inequality,
  we use $0 \leq u_j \leq \tcompress$ in the second inequality,
  we use the given bound $\sigma \leq \frac 1 2 \tcompress^2$ in the third inequality,
  and we use $\tcompress \leq 1$ in the final inequality.
\end{proof}

\subsection{Gradient Iteration Proof of Robustness}

We now have all of the technical tools needed to prove that \textsc{RobustGI-Recovery} robustly recovers the hidden basis elements.
To do so, we first demonstrate that \textsc{FindBasisElement} can be used to approximate a single undiscovered basis element.
We then show that by repeated application of \textsc{FindBasisElement}, all hidden basis elements may be recovered.
In particular, we now prove this section's main theoretical results (\cref{thm:single-recovery-main,thm:full-recovery-main}).
We restate each theorem with more precise bounds before its proof.

For clarity, we denote strictly positive universal constant by $C_0, C_1, C_2, ... $.
\begin{thm}\label{thm:single-recovery-main-reprise}
  Suppose $\sigma \in \left(0, \frac 1 {6\sqrt {2\dn(1+2\delta)}} \tcompress^2\right]$ and
  \begin{displaymath}
    \epsilon \leq C_0 \epsdiverge{\frac{\sigma}{\sqrt \dn}\left(\frac {\beta \gamma}{\alpha \delta \dm^\delta}\right)^{\frac 1 \gamma}\tcompress^2, [\dm]}\dn^{-\delta} \ .
  \end{displaymath}
  Let $p_{\ref{thm:single-recovery-main-reprise}} \in (0,1)$.
  Suppose $N_1 \geq 2\Nsc$,
  $N_2 \geq 2\Nsc+C_0\Nspread{\frac{\sigma}{\sqrt \dn}(\frac {\beta \gamma}{\alpha \delta \dm^\delta})^{\frac 1 \gamma}\tcompress^2}$,
  and $I \geq C_1 \dm \lceil \log(\dm / p_{\ref{thm:single-recovery-main-reprise}}) \rceil$.
  Let $\pi$ be a permutation of $[\dm]$, let $s_1, \dotsc, s_k \in \{\pm 1\}$, and suppose that $\norm{s_i \gvec \mu_i - \myZv_{\pi(i)}} \leq 4\sqrt 2 \delta \epsilon / \beta$ for each $i \in [k]$.

  If we execute $\gvec \mu_{k+1} \leftarrow \textsc{FindBasisElement}(\{\gvec \mu_1, \dotsc, \gvec \mu_k\}, \sigma)$, then with probability at least $1 - p_{\ref{thm:single-recovery-main-reprise}}$, there will exist $s_{k+1} \in \{\pm 1\}$ and an index $j \in [\dm] \setminus [k]$ such that $\norm{s_{k+1} \gvec \mu_{k+1} - \myZv_{\pi(j)}} \leq 4\sqrt 2 \delta \epsilon / \beta$.
\end{thm}
\begin{proof}
  We define the set $A_1 := \{ \pi(\ell) \suchthat \ell \in [k] \}$.
  Notice that $\abs {A_1} = k$.
  By \cref{lem:main-loop-precondition}, we see that $\norm{(\Pmap_0 + \Pmap_{A_1}) \vec u} \leq 4(\dm - \abs {A_1})^\delta\delta \epsilon / \beta$ at the beginning of the first execution of the main loop of \textsc{FindBasisElement}.
  We now establish the following loop invariant.
  \begin{claim}\label{claim:main-loop-invariant}
    Suppose that at the start of the $i$\textsuperscript{th} iteration of the main loop of \textsc{FindBasisElement}, there exists $A_i \subset [\dm]$ such that $\norm{(\Pmap_0 + \Pmap_{A_i}) \vec u} \leq 4 (\dm - \abs {A_i})^\delta \delta \epsilon / \beta$.
    Then,
    \begin{enumerate}
    \item\label{main-loop-invariant:item:1}
      At the end of the $i$\textsuperscript{th} iteration of the main loop, there exists $A_{i+1} \subset [\dm]$ such that $A_{i+1} \supset A_i$ and $\norm{(\Pmap_0 + \Pmap_{A_{i+1}})\vec u} \leq 4(\dm - \abs{A_{i+1}})^\delta \epsilon / \beta$.
    \item\label{main-loop-invariant:item:2}
      If $\abs{A_i} \leq \dm - 2$, then with probability at least $\frac 1 2 \cprobi$, $A_{i+1}$ from part~\ref{main-loop-invariant:item:1} is a strict superset of $A_i$.
    \end{enumerate}
  \end{claim}
  \begin{subproof}
    We define $\SS_i := [\dm] \setminus A_i$.
    We proceed in our analysis at the start of the $i$\textsuperscript{th} iteration of the main loop of \textsc{FindBasisElement}.
    We view $\SS_i$ as being the set of large coordinates of $\vec u$.
    If for each $\ell \in \SS_i$, $\abs{u_\ell} \geq \tcompress$, then this is true in the informal sense that we have considered throughout our discussions.
    However, this is not guaranteed, and we must proceed in distinct cases.
    In all cases, we will make use of the following two facts (without explicitly saying we are doing so) when applying previous lemmas about the random jump and its effects:
    \begin{enumerate}
    \item At the start of the current iteration of the main loop, $\norm{\Pmap_{\compl \SS_i} \vec u} \leq 4 \delta \abs{\SS_i}^\delta \epsilon/\beta $.
    \item At the end of the execution of line~\ref{alg:step:jump1} of \textsc{FindBasisElement}, $\norm{\Pmap_{\compl \SS_i} \vec w} \leq \frac{\tcompress}{\sqrt{2(1+2\delta)}}$.
    In particular, this means $\norm{\Pmap_0 \vec w} \leq \frac 1 {\sqrt{2(1+2\delta)}}$ and $\norm{\Pmap_{\compl \SS_i} \vec w} \leq \tcompress$.
    \end{enumerate}
    To see the first fact, we use that $\norm{\Pmap_{\compl \SS_i} \vec u} = \norm{(\Pmap_0 + \Pmap_{A_i})\vec u} \leq 4 \delta (\dm - \abs{A_i})^\delta \delta \epsilon / \beta = 4\delta \abs{\SS_i}^\delta \epsilon/\beta$.
    To see the second fact, we apply \cref{lem:jump-small-coords} and recall also that $\tcompress < 1$ and $\delta > 0$.

    We let $\vec v \in \porth^{\dn - 1}$ be the fixed point of $\GG/{\sim}$
    such that $v_j \neq 0$ if and only if $j \in \SS_i$.
    We define
    $\eta := C_{\ref{lem:lpert-analysis-expected}} \frac{\sigma}{\sqrt \dn}(\frac {\beta \gamma}{\alpha \delta \dm^\delta})^{\frac 1 \gamma}\tcompress^2$.
    \resetcasecount
    \begin{case}
      $\abs{u_j} \geq \tcompress$ and $\abs{u_j} < (1+\eta)^2v_j$ for all $j \in \SS_i$.
    \end{case}
    \begin{caseblock}
      If $\abs{\SS_i} \geq 2$, then we apply
      \cref{lem:lpert-analysis-expected} to see that with probability at
      least $\cprobi$, at the end of line~\ref{alg:step:jump1} of
      \textsc{FindBasisElement} there exists $\ell \in \SS_i$ such that
      $\abs{w_\ell}/v_\ell \geq (1+\eta)^2$.
      If this happens, we apply \cref{prop:ui-spread-time-bound} to see
      that at the end of end of the current iteration of the main loop of \textsc{FindBasisElement}, there exists $j
      \in \SS_i$ such that $\norm{\Pmap_{\compl \SS_{i} \cup \{j\}} \vec u}
      \leq 4\epsilon\delta(\abs {\SS_i} - 1)^\delta / \beta$.
      In particular, when this occurs, we define $A_{i+1} := A_i \cup \{j\}$,
      and we see that $\norm{\Pmap_{\compl \SS_{i} \cup \{j\}} \vec u} =
      \norm{(\Pmap_0 + \Pmap_{A_{i+1}})\vec u} \leq 4\epsilon\delta(\dm - \abs
      {A_{i+1}})^\delta / \beta$ at the end of the $i$\textsuperscript{th}
      iteration of the main loop.

      If it occurs that there is no $\ell \in \SS_i$ such that
      $\abs{w_\ell}/v_\ell \geq (1+\eta)^2$, we define $A_{i+1} := A_i$ and
      apply \cref{prop:GI-Loop-Small-Coords-Error} to see that at
      the end of the $i$\textsuperscript{th} iteration of the main loop,
      $\norm{(\Pmap_0 + \Pmap_{A_{i+1}})\vec u} \leq 4\epsilon\delta(\dm -
      \abs {A_{i+1}})^\delta / \beta$.
    \end{caseblock}
    \begin{case}
      $\abs{u_j} \geq \tcompress$ for all $j \in \SS_i$ and there exists $\ell \in [\dm]$ such that $\abs{u_\ell} \geq (1+\eta)^2v_\ell$.
    \end{case}
    \begin{caseblock}
      We apply \cref{lem:jump-with-large-coord} to see that with
      probability at least $\frac 1 2 \cprobi$, at the end of
      line~\ref{alg:step:jump1} of \textsc{FindBasisElement} there exists
      $\ell \in \SS_i$ such that $\abs{w_\ell}/v_\ell \geq (1+\eta)^2$.
      When this occurs, we apply \cref{prop:ui-spread-time-bound} to see
      that at the end of the current iteration of the main loop of \textsc{FindBasisElement}, there exists $j
      \in \SS_i$ such that $\norm{\Pmap_{\compl \SS_{i} \cup \{j\}} \vec u}
      \leq 4\epsilon\delta(\abs {\SS_i} - 1)^\delta / \beta$.
      In particular, we define $A_{i+1} := A_i \cup \{ j
      \}$, and we see that $\norm{\Pmap_{\compl \SS_{i} \cup \{j\}} \vec u} =
      \norm{(\Pmap_0 + \Pmap_{A_{i+1}})\vec u} \leq 4\epsilon\delta(\dm - \abs
      {A_{i+1}})^\delta / \beta$ at the end of the $i$\textsuperscript{th}
      iteration of the main loop.

      If it happens that there is no $\ell \in \SS_i$ such that
      $\abs{w_\ell}/v_\ell \geq (1+\eta)^2$, we define $A_{i+1} := A_i$ and
      apply \cref{prop:GI-Loop-Small-Coords-Error} to see that at
      the end of the $i$\textsuperscript{th} iteration of the main loop,
      $\norm{(\Pmap_0 + \Pmap_{A_{i+1}})\vec u} \leq 4\epsilon\delta(\dm -
      \abs {A_{i+1}})^\delta / \beta$.
    \end{caseblock}
    \begin{case}
      There exists $\ell \in \SS_i$ such that $\abs{u_\ell} < \tcompress$.
    \end{case}
    \begin{caseblock}
      We apply \cref{lem:jump-with-small-coord} to see that with probability at least $\frac 1 2$, $\abs{w_\ell} \leq \tcompress$ at the end of the execution of line~\ref{alg:step:jump1} of \textsc{FindBasisElement}.
      If this occurs, we define $A_{i+1} := A_i \cup \{\ell\}$, and otherwise we define $A_{i+1} := A_i$.
      Then, applying \cref{prop:GI-Loop-Small-Coords-Error}, we see that $\norm{(\Pmap_0 + \Pmap_{A_{i+1}})\vec u} \leq 4(\dm - \abs{A_{i+1}})^\delta \epsilon / \beta$.
    \end{caseblock}
    \noindent
    Note that in all three cases, we have the following summary outcome:  If $\abs {\SS_i} \geq 2$, then with probability at least $\frac 1 2 \cprobi$, there exists $A_{i+1}$ a strict superset of $A_i$ such that
    $\norm{(\Pmap_0 + \Pmap_{A_{i+1}})\vec u}
    \leq 4\epsilon\delta(\dm - \abs {A_{i+1}})^\delta / \beta$.
    Further, it is guaranteed that there exists $A_{i+1} \supset A_i$ (where the superset is not necessarily strict) such that
    $\norm{(\Pmap_0 + \Pmap_{A_{i+1}})\vec u}
    \leq 4\epsilon\delta(\dm - \abs {A_{i+1}})^\delta / \beta$ at the end of the current iteration of the main loop of \textsc{FindBasisElement}.
    Noting that $\abs{\SS_i} \geq 2$ if and only if $\abs{A_i} \leq \dm - 2$ completes the proof of the claim.
  \end{subproof}
  To complete the proof, we will apply \cref{claim:main-loop-invariant} and study the state of $\vec u$ at the return statement of line~\ref{alg:step:exit} of \textsc{FindBasisElement}.

  To begin with, we note that if $A_{I+1}$ can contain at most $\dm - 1$ elements, since otherwise $\abs{A_{I+1}} = \dm$ would imply that $\norm{\vec u} = \norm{(\Pmap_0 + \Pmap_{A_{I+1}})\vec u} = 0$.
  Since all steps of \textsc{FindBasisElement} maintain that $\vec u \in \sphere^{\dn - 1}$, this would contradict that $\vec u$ is a unit vector.

  If $\abs{A_{I+1}} = \dm - 1$, then the loop invariant from \cref{claim:main-loop-invariant} implies that $A_{I+1} \supset A_1$, and hence the lone $j \in A_{I+1}$ satisfies $j \not \in \{\pi(1), \dotsc, \pi(\abs{A_1})\}$.
  By applying \cref{lem:hbe-error-bound}, we see that there exists $s_{k+1} \in \{\pm 1\}$ such that
  $\norm{s_{k+1}\vec u - \myZv_j}
  \leq \norm{(\Pmap_0 + \Pmap_{A_{I+1}})\vec u}^2 \sqrt 2
  \leq 4\sqrt 2 \delta \epsilon / \beta$ as desired.

  It only remains to be seen that $\abs{A_{I+1}} \geq \dm - 1$ with the claimed probability.
  To see this, it suffices to show that the size of $\abs{A_i}$ increases at least $\dm - k - 1$ times during the execution of the main loop.
  We will make use of the following claim:
  \begin{claim}\label{claim:main-probability-accumulation}
    Suppose that at the beginning of the $i_j$\textsuperscript{th} iteration of the main loop of \textsc{FindBasisElement} that $\abs{A_i} < \dm - 1$.
    Let $\eta \in (0, 1)$.
    If $N \geq \log(\frac 1 \eta) / \log((1 - \frac 1 2 \cprobi)^{-1})$, then after $N$ additional iterations of the main loop, with probability at least $1 - \eta$, $A_{i+N}$ is a strict superset of $A_i$.
  \end{claim}
  \begin{subproof}
%    After each additional iteration, .
    Using the probability bound $\frac 1 2 \cprobi$ from \cref{claim:main-loop-invariant} and that the random jumps in \textsc{FindBasisElement} are independent of each other, we see that $\Pr[A_{i+N} = A_i]$ is bounded by
    \begin{align*}
      \Pr[A_{i+N} = A_i]
      &\leq (1 - \frac 1 2 \cprobi)^N
      \leq (1-\frac 1 2 \cprobi)^{\log(\frac 1 \eta) / \log((1 - \frac 1 2 \cprobi)^{-1})} \\
      &= (1-\frac 1 2 \cprobi)^{\log_{(1 - \frac 1 2 \cprobi)}(\eta)}
      = \eta
    \end{align*}
    As such, $\Pr[A_{i+N} \text{ is a strict superset of } A_i ] \geq 1 - \eta$ since this is the complement event (by the loop invariant from \cref{claim:main-loop-invariant}).
  \end{subproof}
  We apply \cref{claim:main-probability-accumulation} with the choice of $\eta = p_{\ref{thm:single-recovery-main-reprise}} / (\dm - k - 1)$.
  By taking a union bound, we see that when $I \geq C_1(\dm-k-1) \lceil \log((\dm-k-1) / p_{\ref{thm:single-recovery-main-reprise}} ) \rceil$ with the choice of $C_1 = \frac 1 {\log((1-\frac 1 2 \cprobi)^{-1})}$, then with probability $1 - p_{\ref{thm:single-recovery-main-reprise}}$,  $\abs{A_{I+1}} \geq \abs{A_1} + {\dm - k - 1} = \dm - 1$ as desired.
  In particular, it suffices that $I \geq C_1 \dm \lceil \log(\dm / p_{\ref{thm:single-recovery-main-reprise}})\rceil$.
\end{proof}

\begin{thm}\label{thm:full-recovery-main-reprise}
  Suppose $\sigma \in (0,\, \frac 1 {6\sqrt {2\dn(1+2\delta)}} \tcompress^2]$,
  \begin{displaymath}
    \epsilon \leq C_0 \epsdiverge{\frac{\sigma}{\sqrt \dn}\left(\frac {\beta \gamma}{\alpha \delta \dm^\delta}\right)^{\frac 1 \gamma}\tcompress^2, [\dm]}\dn^{-\delta} \ ,
  \end{displaymath}
  $N_1 \geq 2\Nsc $,
  $N_2 \geq 2\Nsc + C_0\Nspread{\frac{\sigma}{\sqrt \dn}(\frac {\beta \gamma}{\alpha \delta \dm^\delta})^{\frac 1 \gamma}\tcompress^2}$,
  $p_{\ref{thm:full-recovery-main-reprise}} \in (0, 1)$,
  and $I \geq C_2 \dm \lceil \log(\dm / p_{\ref{thm:full-recovery-main-reprise}}) \rceil$.
  If we execute $\gvec \mu_1, \dotsc, \gvec \mu_{\hat \dm}\leftarrow \Call{RobustGI-Recovery}{\hat m}$ for some integer $\hat \dm \in [\dm, \dn]$,
  then $\gvec \mu_1, \dotsc, \gvec \mu_\dm$ forms a $4\sqrt{2}\delta \epsilon / \beta$-approximation to the hidden basis.
  More precisely, there exists a permutation $\pi$ of $[\dm]$ and signs $s_1, \dotsc, s_\dm \in \{+1, -1\}$ such that $\norm{ s_i \gvec \mu_i - \myZv_{\pi(i)} } \leq 4 \sqrt{2} \delta \epsilon / \beta$ for each $i \in [\dm]$.
\end{thm}
\begin{proof} %[\proofword\@of \cref{thm:robust-all-recovery-reprise}]
  We let $\gvec \mu_1, \dotsc, \gvec \mu_\dm$ denote the first $\dm$ approximate basis elements returned by \textsc{RobustGI-Recovery}.
  We proceed by induction on the following statement (with $k \in [\dm] \cup \{0\}$).  \newline \newline
  \noindent\textbf{Inductive Hypothesis:}
  With probability at least $1 - kp_{\ref{thm:full-recovery-main-reprise}} / \dm$,
  there exist sign values $s_1, \dotsc, s_k$ and a permutation $\pi_k$ of $[\dm]$ such that $\norm{s_i \gvec \mu_i - \myZv_{\pi_k(i)}}\leq 4 \sqrt{2} \delta \epsilon / \beta$ for each $i \leq k$.
  \newline\newline
  The base case $k = 0$ holds trivially.
  Suppose that the inductive hypothesis holds for some $k = n$ with $n < \dm$.
  In order to apply \cref{thm:single-recovery-main-reprise}, we set $p_{\ref{thm:single-recovery-main-reprise}} = \frac 1 \dm p_{\ref{thm:full-recovery-main-reprise}}$.
  In order to apply \cref{thm:single-recovery-main-reprise},
  we require that $I \geq C_1 \dm \lceil \log(\dm^2 / p_{\ref{thm:full-recovery-main-reprise}}) \rceil$ (where $C_1$ is as in \cref{thm:single-recovery-main-reprise}), for which it suffices that $I \geq 2C_1 \dm \lceil \log(\dm / p_{\ref{thm:full-recovery-main-reprise}}) \rceil$.
  In particular, it suffices that $C_2 = 2C_1$.

  We now consider the case $k = n$, and
  we operate conditionally on the case that there exist sign values $s_1, \dotsc, s_n$ and a permutation $\pi_n$ of $[\dm]$ such that $\norm{s_i \gvec \mu_i - \myZv_{\pi_n(i)}}\leq 4 \sqrt{2} \delta \epsilon / \beta$ for each $i \leq n$.
  By \cref{thm:single-recovery-main-reprise}, with probability at least $1 - \frac 1 \dm p_{\ref{thm:full-recovery-main-reprise}}$ there exists $j \in [\dm] \setminus \{\pi_n(i) \suchthat i \in [n]\}$ and a sign $s$ such that $\norm{s\gvec \mu_{n+1} - \myZv_j} \leq 4 \sqrt{2} \delta \epsilon / \beta$.
  Defining $s_{n+1} := s$ and $\pi_{n+1}$ to be a permutation of $[\dm]$ such that $\pi_{n+1}(n+1) = j$ and $\pi_{n+1}(i) = \pi_n(i)$ for $i \leq n$ gives the result for $\gvec \mu_{n+1}$.
  Further, we see that the probability that the inductive hypothesis holds for $k = n + 1$ is lower bounded by $(1 - np_{\ref{thm:full-recovery-main-reprise}} / \dm)(1 - p_{\ref{thm:full-recovery-main-reprise}} / \dm) \geq 1 - (n+1)p_{\ref{thm:full-recovery-main-reprise}} / \dm$ as desired.

  Applying induction on $k$ completes the proof.
\end{proof}

%%% Local Variables:
%%% mode: latex
%%% TeX-master: "main"
%%% End:

\section{{\Sef} growth and perturbation bounds}
In this appendix, we provide some useful bounds for $(\alpha, \beta, \gamma, \delta)$-robust {\sef}s that are used throughout the error analysis proofs.
In particular, \cref{lem:F-robust-implications} provides useful bounds on each $g_k$, $h_k$ and their derivatives.
\cref{lem:fg-upper-bounds} provides bounds on the magnitude of $\nabla F$ and its projections.
Then, the remaining lemmas provide bounds on estimation and location perturbation errors for $\nabla F$ and $\GG$.

\begin{lem}\label{lem:F-robust-implications}
    The following bounds hold for every $k \in [\dm]$:
    \begin{compactenum}
        \item \label{lem:F-rb-impl:item:monomials}
        For every $x \in [-1, 1]$, $\frac{\beta}{(\delta + 1)\delta} \abs x^{2+2\delta} \leq \abs{g_k(x)} \leq \frac{\alpha}{(\gamma + 1)\gamma} \abs x^{2+2\gamma}$.
        \item \label{lem:F-rb-impl:item:1}
        For every $x \in [-1, 1]$, $2\frac{\beta}{\delta} \abs x^{1+2\delta} \leq \abs{g'_k(x)} \leq 2\frac \alpha \gamma \abs{x}^{1+2\gamma}$.
        \item \label{lem:F-rb-impl:item:1-5}
        For every $x \in [-1, 1]$, $2(2+1/\delta)\beta \abs x^{2\delta} \leq \abs{g_k''(x)} \leq 2(2+1/\gamma)\alpha \abs x^{2\gamma}$.
        \item \label{lem:F-rb-impl:item:2}
        For every $x \in [-1, 1]$,
        $\frac \beta \delta \abs x^{2\delta} \leq \abs{h_k'(\sign(x)x^2)} \leq \frac \alpha \gamma \abs x^{2\gamma}$.
        \item \label{lem:F-rb-impl:item:defn-equiv}
        For every $x \in [-1, 0) \cup (0, 1]$, $\beta \abs x^{2\delta-2} \leq \abs{h_k''(\sign(x)x^2)} \leq \alpha \abs x^{2\gamma-2}$
    \end{compactenum}
\end{lem}
\begin{proof}
    Using the symmetries from \cref{assumpt:gsymmetries}, it suffices to consider $x \geq 0$.

    To see part~\ref{lem:F-rb-impl:item:defn-equiv}, we apply \cref{defn:BEF-robust} to $h_k''(x^2)$ to obtain $\beta x^{2(\delta - 1)} \leq \abs{h_k''(x^2)} \leq \alpha x^{2(\gamma - 1)}$ on $x > 0$.

    For part~\ref{lem:F-rb-impl:item:2}, we use that $h_k'(0) = 0$ by \cref{assumpt:deriv0} to obtain that $h_k'(x) = \int_0^{x} h_k''(t) dt$.
    Further, since $h_k$ is either strictly convex or strictly concave on $[0, 1]$, it follows that the sign of $h_k''$ unchanging on $(0, 1]$.
    Thus, $\abs{h_k'(x)} = \int_0^{x} \abs{h_k''(t)} dt$.
    The upper bound is obtained as
    \begin{displaymath}
    \abs{h_k'(x^2)}
    \leq \int_0^{x^2} \alpha t^{\gamma - 1} dt
    = \frac \alpha \gamma t^{\gamma}\Big|_{t=0}^{t=x^2}
    = \frac \alpha \gamma x^{2\gamma} \ .
    \end{displaymath}
    By similar reasoning (replacing $\leq$ with $\geq$, $\gamma$ with $\delta$, and $\alpha$ with $\beta$) we obtain that $\abs{h_k'(x^2)} \geq \frac \beta \delta x^{2\delta}$.

    To obtain parts~\ref{lem:F-rb-impl:item:1} and \ref{lem:F-rb-impl:item:1-5}, we use the formulas from \cref{lem:h-g-relations} to express the derivatives of $g_k$ as $g_k'(x) = 2h_k'(x^2)x$ and $g_k''(x) = \indicator{x \neq 0}[4h_i''(x^2)x^2 + 2h_i'(x^2)]$.
    By part~\ref{lem:F-rb-impl:item:2}, we obtain the desired bounds on $g_k'(x)$.
    We note that
    \begin{displaymath}
    \abs{g_k''(x)}
    \leq 4 \alpha x^{2\gamma} + 2 \frac \alpha \gamma x^{2\gamma}
    \leq 2\alpha (2 + 1 / \gamma) x^{2\gamma} \ .
    \end{displaymath}
    Note that $h_i''(x)$ and $h_i'(x)$ share the same sign on $(0, x]$ (see \cref{lem:h-props} and recall that $h_i$ is convex if $h_i'' \geq 0$ on its domain and concave if $h_i'' \leq 0$ on its domain).
    As such,
    \begin{displaymath}
    \abs{g_k''(x)}
    = \indicator{x \neq 0} \abs{4h_k''(x^2)x^2 + 2h_k'(x^2)}
    \geq 4 \beta x^{2\delta} + 2 \frac \beta \delta x^{2\delta}
    \geq 2 \beta (2+1/\delta) x^{2\delta} \ .
    \end{displaymath}

    To obtain part~\ref{lem:F-rb-impl:item:monomials}, we first find bounds for $h_k$.
    Since $h_k'$ is strictly monotonic and $h_k'(0) = 0$, it follows that $\abs{h_k(x)} = \int_0^x \abs{h_k'(t)} dt$.
    We obtain the upper bound as:
    \begin{align*}
    \abs{g_k(x)}
    &= \abs{h_k(x^2)}
    = \int_0^{x^2} \abs{ h_k'(t) } dt \\
    &\leq \frac \alpha \gamma \int_0^{x^2} t^{\gamma} dt
    = \frac{\alpha}{(\gamma + 1)\gamma} t^{\gamma+1} \bigr|_{t = 0}^{t = x^2}
    = \frac{ \alpha}{(\gamma + 1)\gamma} x^{2\gamma + 2} \ .
    \end{align*}
    The lower bound is obtained in a similar manner.
\end{proof}

\begin{lem}\label{lem:fg-upper-bounds}
    If $F$ is $(\alpha, \beta, \gamma, \delta)$-robust, then its gradient is bounded as follows for any $\vec u \in S^{\dn - 1}$.
    \begin{compactenum}
        \item\label{bound:gradf-upper2}
        Let $\SS \subset [\dn]$.
        Then, $\norm{\Pmap_\SS \nabla F(\vec u)} \leq \frac{2\alpha} \gamma \norm{\Pmap_{\SS \cap [\dm]} \vec u}^{1+2\gamma}$.
        \item\label{bound:gradf-lower}
        Let $S \subset [\dn]$.
        Then, $\norm{\Pmap_\SS \nabla F(\vec u)} \geq \frac {2\beta} \delta \norm{\Pmap_{\SS \cap [\dm]} \vec u}^{1+2\delta} / \abs {\SS \cap [\dm]}^{\delta}$.
    \end{compactenum}
\end{lem}
\begin{proof}
    We first prove part~\ref{bound:gradf-upper2}.  We let $A = \SS \cap [\dm]$.
    \begin{align*}
    \norm{\Pmap_{\SS} \nabla F(\vec u)}^2
    &= \sum_{i \in A} g_i'(u_i)^2
    \leq \sum_{i \in A} \left(2\frac{\alpha}{\gamma} \abs{u_i}^{1+2\gamma}\right)^2 \\
    &= 4\frac{\alpha^2}{\gamma^2} \norm{\Pmap_A \vec u}^{2+4\gamma}\sum_{i \in A} \biggr(\frac{u_i^2}{\norm{\Pmap_A \vec u}^2}\biggr)^{1+2\gamma} \ .
    \end{align*}
    In the above, the inequality uses \cref{lem:F-robust-implications} part~\ref{lem:F-rb-impl:item:1}.
    For each $i \in A$, $u_i^2/\norm{\Pmap_A\vec u}^2 \leq 1$ holds.
    Since $\gamma > 0$, it follows that $(u_i^{2} / \norm{\Pmap_A \vec u}^2)^{1+2\gamma} \leq u_i^{2} / \norm{\Pmap_A \vec u}^2$.
    Thus,
    \begin{displaymath}
    \norm{\Pmap_{\SS} \nabla F(\vec u)}^2
    \leq 4\frac{\alpha^2}{\gamma^2} \norm{\Pmap_A \vec u}^{2+4\gamma} \sum_{i \in A} \frac{u_i^2}{\norm{\Pmap_A \vec u}^2}
    = 4\frac{\alpha^2}{\gamma^2} \norm{\Pmap_A \vec u}^{2+4\gamma}  \ .
    \end{displaymath}

    We now prove part~\ref{bound:gradf-lower}.
    We let $A = \SS \cap [\dm]$, and we note:
    \begin{displaymath}
    \norm{\Pmap_{\SS} \nabla F(\vec u)}^2
    = \sum_{i \in A} g_i'(u_i)^2
    \geq \sum_{i \in A} \left(\frac{2\beta}{\delta} \abs{u_i}^{1+2\delta}\right)^2
    = \frac{4 \beta^2}{\delta^2} \abs{A}\sum_{i \in A} \frac{1}{\abs A}\bigr(u_i^2\bigr)^{1+2\delta} \ .
    \end{displaymath}
    In the above, the inequality uses \cref{lem:F-robust-implications}.
    But by Jensen's inequality, we see that
    \begin{displaymath}
    \sum_{i \in A} \frac{1}{\abs A}\bigr(u_i^2\bigr)^{1+2\delta}
    \geq \biggr(\sum_{i \in A} \frac{1}{\abs A}u_i^2\biggr)^{1+2\delta}
    = \biggr(\frac{\norm{\Pmap_A \vec u}^2}{\abs A}\biggr)^{1+2\delta}
    \end{displaymath}
    Thus, $\norm{\Pmap_{\SS} \nabla F(\vec u)}^2 \geq 4(\beta/\delta)^2 (\norm{\Pmap_A \vec u}^{2+4\delta})/\abs A^{\delta}$.
    Taking square roots gives the desired bound.
\end{proof}

\begin{lem}\label{lem:nablaF-pert-error}
    Suppose that $\vec u, \vec w \in \overline{B(0, 1))}$.
    Then, $\norm{\nabla F(\vec  u) - \nabla F(\vec w)} \leq 2(1+\frac 1 \gamma) \alpha \norm{\vec u - \vec w}$.
\end{lem}
\begin{proof}
    The proof is by the fundamental theorem of calculus and Minkowski's inequality for integrals:
    \begin{align*}
    \norm{\nabla F(\vec u) - \nabla F(\vec w)}
    &= \Norm{ \int_0^1 \HH F(t \vec u + (1-t) \vec w)(\vec u - \vec w) dt  } \\
    & \leq \int_0^1 \norm{\HH F(t \vec u + (1-t) \vec w)} \norm{\vec u - \vec w} dt \\
    & \leq 2\left(2 + \frac 1 \gamma\right)\norm{\vec u - \vec w} \ .
    \end{align*}
    In the last inequality, we note that for any $\vec p \in \overline{\vec u \vec w}$, $\HH F(\vec p)$ is a diagonal matrix, that as such $\norm {\HH F(\vec p)} \leq \max_{i \in \dn}\abs{[\HH F(\vec p)]_{ii}} = \max_{i \in [\dn]}\abs{g_{i}''(p_i)} \leq 2(2 + \frac 1 \gamma) \alpha$ by \cref{lem:F-robust-implications}.
\end{proof}

\begin{lem}\label{lem:G-est-error}
    Let $\SS \subset [\dm]$.
    If $\norm {\Pmap_{\compl \SS} \vec u} \leq \frac 1 {\sqrt{2(1+2\delta)}}$ and $\epsilon \leq \frac \beta {2\delta \abs \SS^\delta}$ , then $\norm{ \hat \GG(\vec u) - \GG(\vec u)} \leq 4 \delta \abs \SS^\delta \epsilon / \beta$.
\end{lem}
\begin{proof}
    \begin{align*}
    \norm{\hat \GG(\vec u) - \GG(\vec u)}
    &= \Norm{\frac{ \pertgF(\vec u)}{\norm{\pertgF(\vec u)}}
        - \frac{\nabla F(\vec u)}{\norm{\nabla F(\vec u)}}} \\
    &\leq \Norm{\frac{\pertgF(\vec u) - \nabla F(\vec u)}{\norm{\pertgF(\vec u)}}
        + \frac{(\norm{\nabla F(\vec u)}-\norm{\pertgF(\vec u)})\nabla F(\vec u)}{\norm{\nabla F(\vec u)}\norm{\pertgF(\vec u)}}} \\
    &\leq \frac{2\epsilon}{\norm{\pertgF(\vec u)}}
    \leq \frac{2\epsilon}{\norm{\nabla F(\vec u)} - \epsilon}
    \end{align*}
    We apply \cref{lem:P0small-gradf-lower-bound} to see that $\norm{\nabla F(\vec u)} \geq \frac{\beta}{\delta \abs{\SS}^{\delta}}$.
    Using the bound on $\epsilon$, we obtain that $\norm{\nabla F(\vec u)} - \epsilon \geq \frac{\beta}{2\delta \abs{\SS}^{\delta}}$.
    Thus, $\norm{\hat \GG(\vec u) - \GG(\vec u)} \leq 4 \delta \abs{\SS}^\delta \epsilon / \beta$ as desired.
\end{proof}

Note that perhaps the most interesting case of \cref{lem:G-est-error} is the case in which $\SS = [\dm]$.
In this case, the result simplifies to:
\emph{If $\norm{\Pmap_0 \vec u} \leq \frac{1}{\sqrt{2(1+2\delta)}}$ and $\epsilon \leq \frac{\beta}{2 \delta \dm^\delta}$, then $\norm{\hat \GG(\vec u) - \GG(\vec u)} \leq 4 \delta \dm^\delta \epsilon / \beta$}.

\begin{lem}\label{lem:G-lpert-error}
    Let $\vec u, \vec w \in S^{\dn - 1}$.
    Let $\SS \subset [\dm]$, and suppose that $\norm{\Pmap_{\compl \SS} \vec u} \leq \frac 1 {\sqrt {2(1+2\delta)} }$, and further suppose that $w_i \neq 0$ for some $i \in [\dm]$.
    Then, $\norm{\GG(\vec u) - \GG(\vec w)} \leq 2\frac \alpha
    \beta\delta(2+\frac 1 \gamma) \abs \SS^\delta \norm{\vec u - \vec w}$.
\end{lem}
\begin{proof}
    Since there exists $i \in [\dm]$ such that $w_i \neq 0$, it follows that  $\nabla F(\vec w) \neq \vec 0$.
    \begin{align*}
    \norm{\GG(\vec u) - \GG(\vec w)}
    &=\Norm{ \frac{\nabla F(\vec u)}{\norm{\nabla F(\vec u)}} - \frac{\nabla F(\vec w)}{\norm{\nabla F(\vec w)}} } \\
    &= \Norm{
        \frac{\norm{\nabla F(\vec w)}[\nabla F(\vec u) - \nabla F(\vec w)] + [\norm{\nabla F(\vec w)} - \norm{\nabla F(\vec u)}] \nabla F(\vec w) }{\norm{\nabla F(\vec u)} \norm{\nabla F(\vec w)}}
    } \\
    &\leq 2 \frac{\norm{\nabla F(\vec u) - \nabla F(\vec w)}}{\norm{\nabla F(\vec u)}} \ .
    \end{align*}
    By \cref{lem:P0small-gradf-lower-bound}, we have that $\norm{\nabla
        F(\vec u)} \geq \frac \beta \delta \abs \SS^{-\delta}$.
    Further, by \cref{lem:nablaF-pert-error}, we see that $\norm{\nabla
        F(\vec u) - \nabla F(\vec w)} \leq 2(2+\frac 1 \gamma) \alpha \norm{\vec u
        - \vec w}$.
    As such, we obtain that $\norm{\GG(\vec u) - \GG(\vec w)}
    \leq 2\frac \alpha \beta\delta(2+\frac 1 \gamma) \abs \SS^\delta
    \norm{\vec u - \vec w}$.
\end{proof}

In addition, the following is a useful Corollary to the \cref{lem:small-coord-decrease-step,lem:P0small-gradf-lower-bound}.
\begin{cor}\label{cor:tcompress}
    Let $\SS \subset [\dm]$ and let $A \subset [\dm]$ be non-empty.
    Suppose that $\vec u \in S^{\dn - 1}$ satisfies $\norm{\Pmap_{\compl A} \vec u} \leq \frac{1}{\sqrt{2(1+2\delta)}}$, and that $\epsilon \leq \frac \beta {2\delta}\dm^{-\delta}$.
    If for all $i \in \compl \SS \cap [\dm]$ that $\abs {u_i} \leq \tcompress$.
    Then, the following hold:
    \begin{compactenum}
        \item For all $i \in \compl \SS$, $\abs{\hat \GG_i(\vec u)} \leq \max( \frac 1 2 \abs{u_i}, 4 \delta\abs A^{\delta}\epsilon/\beta )$.
        \item $\norm{\Pmap_{\compl \SS} \hat \GG(\vec u)} \leq \max( \frac 1 2 \norm{\Pmap_{\compl \SS}\vec u}, 4 \delta\abs A^{\delta}\epsilon/\beta )$.
    \end{compactenum}
\end{cor}
\begin{proof}
    We note by \cref{lem:P0small-gradf-lower-bound}, $\norm{\nabla F(\vec u)} \geq \frac \beta \delta \abs A^{-\delta}$.
    Parts 1 and 2 follow by applying \cref{lem:small-coord-decrease-step} with the choice of $C = 1/2$ and with this lower bound for $\norm {\nabla F(\vec u)}$.
    For part 1, we apply \cref{lem:small-coord-decrease-step} to the set $\{i\}$, and for part 2 we apply \cref{lem:small-coord-decrease-step} to the set $\compl \SS$.
\end{proof}
Note that the most interesting cases of this corollary are when $A = \SS$ and when $A = [\dm]$.

\section{Miscellany}

In this section, we collect a number useful results and statements which are used in the proofs of the theorems of this paper.

% \begin{thm}[Inverse Function Theorem, based on Wikipedia article]
%   % Suppose that $M$ and $N$ are manifolds, that $f : M \rightarrow N$, and that $\Jacob F_p : T_pM \rightarrow T_{F(p)} M$ is a bijection.
%   Suppose that $f : \R^n \rightarrow \R^n$ is continuously differentiable on
%   an open set $\mathcal X \subset \R^n$.
%   If $\Jacob f_{\vec p}$ is non-singular (invertible) for some $\vec p \in \mathcal X$, the there exists a
%   neighborhood $U$ of $f(\vec p)$ on which the inverse function $f^{-1}$
%   exists and is continuously differentiable.
% \end{thm}

The following is a special case of Lemma 7.25 of~\mycitet{Rudin}{rudin1986real}.
\begin{thm}\label{thm:continuous-f-measure0}
  Let $U \subset \R^k$ be an open set.
  Suppose $f : V \rightarrow \R^k$ is differentiable on its entire domain.
  If $E \subset V$ has Lebesgue measure 0, then $f(E)$ has Lebesgue measure 0.
\end{thm}

\subsection{Rates of Convergence}
\label{sec:rates-convergence}

For reference, we recall here the definitions of various orders of convergence.

Let $\{\vec x_n\}_{n=0}^\infty$ be a sequence in a normed vector space that
converges to $L$.
If there exists $q \geq 1$ and $\mu \in (0, 1)$ such that
\begin{equation*}
  \lim_{n \rightarrow \infty} \frac{\norm{\vec x_{n+1} - L}}{\norm{\vec x_n - L}^q} = \mu,
\end{equation*}
then the sequence $\{\vec x_n\}_{n=0}^\infty$ is said to converge to $L$ with \text{order} $q$.
In the case where $q = 1$, convergence is said to be \textit{linear}.
If
\begin{equation*}
  \lim_{n \rightarrow \infty} \frac{\norm{\vec x_{n+1} - L}}{\norm{\vec x_n - L}} = 0,
\end{equation*}
then the rate of convergence is said to be \textit{superlinear}.

More generally, if there is a sequence $\{\epsilon_n\}_{n=0}^\infty$ in $\R$ such that $\epsilon_n$ converges to 0 with order $q$ (or superlinearly respectively) under the above definitions and if $\norm{\vec x_n - L} \leq \epsilon_n$ for all $n$, then we also say that $\vec x_n$ converges to $L$ with order at least $q$ (or superlinearly respectively).

\subsection{Error bounds on eigenvalues and eigenspaces}
\label{sec:error-bounds-eigenv}

We now recall some classic results about the perturbation of eigenvalues and eigenspaces.
The following inequality is a known version of Weyl's inequality for matrix eigenvalues.
\begin{thm}[Weyl's inequality]\label{thm:Weyls-inequality}
  Let $A$, $\tilde A$, and $H$ be symmetric (or more generally Hermitian) $n \times n$ matrices such that $\tilde A = A + H$.
  Let the eigenvalues of $A$, $\tilde A$, and $H$ be given by $\lambda_1, \dotsc, \lambda_n$, $\tilde \lambda_1, \dotsc, \tilde \lambda_n$, and $\rho_1, \dotsc, \rho_n$ respectively.
  Assume that the eigenvalues are indexed in decreasing order, i.e., $\lambda_1 \geq \cdots \geq \lambda_n$.
  Then, for each $i \in [n]$, $\lambda_i + \rho_i \leq \tilde \lambda_i \leq \lambda_i \rho_n$.
\end{thm}

The next Theorem (namely the  $\sin \Theta$ theorem of~\mycitet{Davis and Kahan}{davis1970rotation}) allows us to bound the error in eigenvector subspaces of a matrix under a perturbation.
This theorem requires a bit more explanation.
In particular, we will still assume that we have a Hermitian matrix $A$ which is the matrix we are interested in, and that $\tilde A = A + H$ is a perturbed version of $A$ (with $\tilde A$ and $H$ also both Hermitian).
Suppose that $A = \sum_{i=1}^n \lambda_i \vec v_i \vec v_i^T$ and $\tilde A \sum_{i=1}^n \tilde \lambda_i \tilde {\vec v_i} \tilde {\vec v_i}^T$ give eigendecompositions with the ordering of the eigenvalues $\lambda_i$ not yet determined.
We may split the indices at a point $k$ and define the matrices $A_0 = \sum_{i=1}^k \lambda_i \vec v_i \vec v_i^T$, $A_1 = \sum_{i=k+1}^n \lambda_i \vec v_i \vec v_i^T$, $\tilde A_0 = \sum_{i=1}^k \tilde \lambda_i \tilde {\vec v_i}\tilde{ \vec v_i}^T$, $\tilde A_0 = \sum_{i=k+1}^m \tilde \lambda_i \tilde {\vec v_i}\tilde{ \vec v_i}^T$.
\begin{thm}[Davis-Kahan $\sin \Theta$ theorem]\label{thm:sin-theta-thm}
  Suppose that there exists an interval $[\alpha, \beta]$ and a $\Delta > 0$ such that the eigenvalues of $A_0$ lie within $[\alpha, \beta]$ and the eigenvalues of $\tilde A_1$ all lie outside the interval $(\alpha - \Delta, \beta + \Delta)$ [or alternatively, the eigenvalues of $\tilde A_1$ lie within $[\alpha, \beta]$ and the eigenvalues of $A_0$ all lie outside the interval $(\alpha - \Delta, \beta + \Delta)$].
  Then, $\Delta \norm{\sin \Theta_0} \leq \norm H$.
  % \vnote{Technically, the statement is given in terms of a residual $R$ rather than $H$, but see discussion at the top of page 7 of \mycitet{Davis and Kahan}{davis1970rotation}, and note that $R = H E_0$ in their notation with the norm $\norm{\cdot}$ being unitary invariant.  $E_0$ is just an eigenspace corresponding to the eigenvalues of $A_0$.  As such, $\norm{R} \leq \norm{H}$.}
\end{thm}
The definition of $\sin \Theta_0$ is somewhat involved (see \citep{davis1970rotation} for details).
 In our setting it suffices to note that $\norm{\sin \Theta_0}$ bounds certain projection operators.
In particular, if $\Pi_0 = \sum_{i=1}^k \vec v_i \vec v_i^T$ and $\tilde \Pi_0 = \sum_{i=1}^k \tilde{\vec v_i} \tilde{\vec v_i}^T$, then $\norm{(\Id-\tilde \Pi_0) \Pi_0} \leq \norm{\sin \Theta_0} \leq \frac 1 \Delta \norm H$.

\subsection{Concentration and anti-concentration of the $\chi^2$ distribution}
The following bounds are a direct implication of~\citep[Lemma 1]{laurent2000adaptive}.
\begin{lem}\label{lem:chi2-concentration}
  Let $Z$ be distributed according to the $\chi^2$ distribution with $D$ degrees of freedom.
  Then, for all $x > 0$, the following hold:
  \begin{compactenum}
  \item $\Pr[Z - D \geq 2\sqrt{Dx} + 2x] \leq \exp(-x)$.
  \item $\Pr[D - Z \geq 2\sqrt{Dx}] \leq \exp(- x)$.
  \end{compactenum}
\end{lem}

We have the following Corollary, which is useful in our error analysis.
\begin{cor}\label{cor:chi2-concentration}
    There exists universal constant $C > 0$ such that the following holds:
    Let $Z$ be distributed according to the $\chi^2$ distribution with $D$ degrees of freedom.
    Then,
    \begin{compactenum}
      \item
       $\Pr[Z \geq 2D] \leq \exp(- \frac{2-\sqrt 3}{2} D)$.
      \item $\Pr[Z \leq \frac 1 2 D] \leq \exp(-\frac 1 {16} D)$.
    \end{compactenum}
     with probability at least $\exp(-C D)$,
\end{cor}
\begin{proof}
    We first prove part 1.
    We note that
    \begin{displaymath}
        \Pr[Z \geq 2D] = \Pr[Z - D \geq D]
    \end{displaymath}
    In order to apply \cref{lem:chi2-concentration}, we need to choose $x$ such that $2\sqrt{Dx} + 2x = D$.
    That is, $2\sqrt{Dx} + 2x - D = 0$.
    But by the quadratic formula, it follows that
    \begin{displaymath}
        \sqrt x
        = \frac{-2\sqrt D + \sqrt{4 D + 8 D}}{4}
        = \frac{\sqrt 3 - 1}{2}\sqrt D \ .
    \end{displaymath}
    Squaring both sides yields $x = \frac{2-\sqrt 3}{2}D$.

    In order to prove part 2, we note that $\Pr[Z \leq \frac 1 2 D] = \Pr[D - Z \geq \frac 1 2 D]$.
    We apply \cref{lem:chi2-concentration} with the choice of $x$ such that $2\sqrt{D x} = \frac 1 2 D$, which is to say $x = \frac{1}{16} D$.
\end{proof}

%%% Local Variables:
%%% mode: latex
%%% TeX-master: "main"
%%% End:

% \input{function-bounds-3robust}
%\input{function-bounds-old}
\end{vlong}

\section*{Acknowledgments}
This material is based upon  work supported by the National Science Foundation under Grants No.\@ 1350870, 1422830, 15507576, 1657939
and 1117707.

\bibliography{biblio,proceedings}

\end{document}